\documentclass[runningheads]{llncs}
\usepackage[T1]{fontenc}
\usepackage{graphicx}
\usepackage{booktabs}
\usepackage[misc]{ifsym}
\newcommand{\corr}{(\Letter)}

\AtEndEnvironment{proof}{\phantom{}\qed}

\usepackage[leqno]{amsmath}
\usepackage{amssymb}
\usepackage{bbm}
\usepackage{bm}
\usepackage{tikz}
\usepackage[font=scriptsize]{subcaption}
\usepackage[title]{appendix}
\usepackage{multirow}

\usepackage{hyperref}
\hypersetup{colorlinks=true}

\usepackage[export]{adjustbox}

\usepackage{array}
\newcolumntype{P}[1]{>{\centering\arraybackslash}p{#1}}
\usepackage{booktabs}

\usetikzlibrary{positioning}
\usetikzlibrary{shapes}
\usetikzlibrary{fit}
\usetikzlibrary{chains}
\usetikzlibrary{arrows}

\tikzstyle{latent} = [circle,fill=white,draw=black,inner sep=1pt,
minimum size=20pt, font=\fontsize{10}{10}\selectfont, node distance=1]
\tikzstyle{obs} = [latent,fill=gray!25]
\tikzstyle{const} = [rectangle, inner sep=0pt, node distance=1]
\tikzstyle{factor} = [rectangle, fill=black,minimum size=5pt, inner
sep=0pt, node distance=0.4]
\tikzstyle{det} = [latent, diamond]

\tikzstyle{plate} = [draw, rectangle, rounded corners, fit=#1]
\tikzstyle{wrap} = [inner sep=0pt, fit=#1]
\tikzstyle{gate} = [draw, rectangle, dashed, fit=#1]

\tikzstyle{caption} = [font=\footnotesize, node distance=0] %
\tikzstyle{plate caption} = [caption, node distance=0, inner sep=0pt,
below left=5pt and 0pt of #1.south east] %
\tikzstyle{factor caption} = [caption] %
\tikzstyle{every label} += [caption] %

\tikzset{
    position/.style args={#1:#2 from #3}{
        at=(#3.#1), anchor=#1+180, shift=(#1:#2)
    }
}

\newcommand{\xmark}{
\tikz[scale=0.23] {
    \draw[line width=0.7,line cap=round] (0,0) to [bend left=6] (1,1);
    \draw[line width=0.7,line cap=round] (0.2,0.95) to [bend right=3] (0.8,0.05);
}}

\def\ci{\perp\!\!\!\perp} 
\def\leqM{\leq_{\scriptscriptstyle \model}}

\newcommand{\vect}[1] {\boldsymbol{#1}}
\newcommand{\mat}[1]{\mathtt{#1}}
\newcommand{\data} {\mathcal{D}}
\newcommand{\model} {\mathcal{M}}
\newcommand{\x} {\vect{x}}
\newcommand{\X} {\vect{X}}
\newcommand{\Dx} {\mathcal{X}}
\newcommand{\y} {y}
\newcommand{\Y} {Y}
\newcommand{\Dy} {\mathcal{Y}}
\newcommand{\params} {\vect{\theta}}
\newcommand{\Params} {\vect{\Theta}}
\newcommand{\Dparams} {\Omega_{\params}}
\newcommand{\epist}[2][e] {\mathcal{U}^{#1}_{#2}}

\newcommand{\KL}[2]{D_{\scriptscriptstyle \textsc{KL}}\!\left(#1 \,\|\, #2\right)}
\newcommand{\E}[1]{\mathbb{E}\!\left[#1\right]}

\newcommand{\cvar}[2]{\mathbb{V}\!\left(#1\,|\,#2\right)}
\newcommand{\mi}[2] {I\left(#1 \, ; \, #2\right)}
\newcommand{\cmi}[3] {I\left(#1 \, ; \, #2 \, | \, #3\right)}

\newcommand{\cent}[2] {H\left(#1 \, | \, #2\right)}

\begin{document}

\title{On the Calibration of Epistemic Uncertainty: Principles, Paradoxes and Conflictual Loss}

\titlerunning{On the Calibration of Epistemic Uncertainty}


\author{Mohammed Fellaji\inst{1}\orcidID{0009-0007-8604-4424} \corr \and \\ Frédéric Pennerath\inst{1}\orcidID{0000-0002-2202-5896}\and \\ Brieuc Conan-Guez\inst{2}\orcidID{0000-0003-4583-9777} \and \\ Miguel Couceiro\inst{2}\orcidID{0000-0003-2316-7623}}
\author{Mohammed Fellaji\inst{1} \corr \and Frédéric Pennerath\inst{1} \and \\ Brieuc Conan-Guez\inst{2} \and Miguel Couceiro\inst{2}}
\authorrunning{M. Fellaji et al.}
\institute{CentraleSupélec, Université Paris Saclay, CNRS, LORIA, France \email{\{mohammed.fellaji,frederic.pennerath\}@centralesupelec.fr} \and 
Université de Lorraine, CNRS, LORIA, France \\ \email{brieuc.conan-guez@univ-lorraine.fr} \\ \email{miguel.couceiro@loria.fr}}


\maketitle

\begin{abstract}
The calibration of predictive distributions has been widely studied in deep learning, but the same cannot be said about the more specific epistemic uncertainty as produced by Deep Ensembles, Bayesian Deep Networks, or Evidential Deep Networks. Although measurable, this form of uncertainty is difficult to calibrate on an objective basis as it depends on the prior for which a variety of choices exist.
Nevertheless, epistemic uncertainty must in all cases satisfy two formal requirements: firstly, it must decrease when the training dataset gets larger and, secondly, it must increase when the model expressiveness grows. Despite these expectations, our experimental study shows that on several reference datasets and models, measures of epistemic uncertainty violate these requirements, sometimes presenting trends completely opposite to those expected. These paradoxes between expectation and reality raise the question of the true utility of epistemic uncertainty as estimated by these models.
A formal argument suggests that this disagreement is due to a poor approximation of the posterior distribution rather than to a flaw in the measure itself.
Based on this observation, we propose a regularization function for deep ensembles, called \emph{conflictual loss} in line with the above requirements. We emphasize its strengths by showing experimentally that it fulfills both requirements of epistemic uncertainty, without sacrificing either the performance nor the calibration of the deep ensembles.

\end{abstract}

\section{Introduction\label{sec:introduction}}

All neural networks, from small discriminative classifiers to large generative models, can be seen as probabilistic models that estimate some distribution.
This distribution captures the uncertainty of the predicted variable, induced both by \emph{latent factors} that are inherent in the process that has generated the data, and by the \emph{model bias}, which reflects the lack of expressiveness of the model to represent the true distribution.
The uncertainty related to latent factors is sometimes referred to as \emph{aleatoric or data uncertainty} in contrast to the completely different \emph{epistemic or model uncertainty} that is meant to measure \emph{estimator variance} / \emph{overfitting}, {\it i.e.,} the uncertainty about the output distribution itself, due to the limited size of the training dataset.
While every probabilistic model de facto takes into account aleatoric uncertainty, epistemic uncertainty becomes measurable
only by models whose output distribution is a random variable. This includes \emph{Bayesian Neural Networks} (BNN) that apply (approximate) Bayesian inference on network weights \cite{mackay_practical_1992,hinton_keeping_1993,neal_bayesian_1996,gal_dropout_2016}, \emph{Deep Ensembles} (DE) that sample the prior distribution \cite{lakshminarayanan_simple_2017}, \emph{Prior Networks} \cite{malinin_predictive_2018}, \emph{Evidential Deep Learning} (EDL) \cite{sensoy_evidential_2018} and derived methods that directly learn parameters of a second-order distribution.
All of these models produce not only the \emph{posterior predictive distribution} for a given input, but also a measure of epistemic uncertainty that quantifies the part of uncertainty that can be further reduced by observing more data in the vicinity of the input.
Since this measure is usually computed as the \emph{mutual information} between the model output and the parameters (conditioned on the input and the training dataset), we will stick to this choice in the sequel, even if the choice of a better epistemic uncertainty metric remains an arguable subject, as discussed in \cite{wimmer_quantifying_2023}.
Regardless of the choice of metric, epistemic uncertainty appears as a relevant criterion for deciding to label a new example in the context of active learning~\cite{depeweg_decomposition_2018,gal_deep_2017}, to tackle the exploration-exploitation dilemma in reinforcement learning~\cite{silva_uncertainty-aware_2020}, or to detect OOD examples \cite{lee_training_2018}, although some approaches advocate an even finer decomposition to either detect OODs, distinguishing between epistemic and \emph{distributional uncertainty} \cite{malinin_predictive_2018},
or to account for \emph{procedural variability} \cite{huang_efficient_2023}, \emph{i.e.}, uncertainty coming from the randomness of the optimization procedure.
Accurately quantifying epistemic uncertainty is thus crucial from both theoretical and application perspectives.

Whereas there is a large body of work on the calibration of predictive uncertainty as produced by deep neural networks, including Bayesian ones \cite{kuleshov_accurate_2018,ovadia_can_2019}, to our knowledge, there is no work dealing specifically with the calibration of epistemic uncertainty.
In this paper, we address the question of how to evaluate the quality of epistemic uncertainty produced by deep networks.
One difficulty that may explain the lack of research in this field is that the amount of epistemic uncertainty depends on the prior distribution over parameters for which certain freedom of choice exists~\cite{fortuin_priors_2022}, whether it is an \emph{informative} or an \emph{objective prior} like \emph{Jeffreys prior}~\cite{jeffreys_invariant_1946}.
Consequently, the definition of a quantitative score to measure the quality of epistemic uncertainty appears as a questionable objective, and thus we do not consider it.
Instead, we adopt a qualitative standpoint by stating two properties that every measure of epistemic uncertainty should ideally fulfill:
the first property that we call hereafter \emph{data-related principle}, states that the amount of epistemic uncertainty decreases as the model observes more data.
The second property, referred to as \emph{model-related principle}, states that epistemic uncertainty must increase with model complexity, {\it i.e.}, the number of weights, as a consequence of the \emph{curse of dimensionality}.

While we can check these requirements are de facto true for a simple probabilistic model like Bayesian linear regression,
it is not obvious that this still holds for Bayesian deep networks or their alternatives, since the parameters of these more complex non-convex models converge somewhat randomly to one of the many local optima.
With this in mind, we conducted an experimental study of epistemic uncertainty as produced by \emph{Deep Ensembles} \cite{lakshminarayanan_simple_2017}, \emph{MC-Dropout} \cite{gal_dropout_2016} and \emph{Evidential Deep Learning} \cite{sensoy_evidential_2018}.
The results are surprising: we observe that in all data regimes and for all tested methods, the average measures of mutual information computed on the test set, completely contradict the \emph{model-related principle}: the larger the model, the smaller the epistemic uncertainty when precisely the opposite is expected.
The \emph{data-related principle} seems globally but not perfectly respected, with some blatant counter-examples.
The same paradoxes are consistently observed, even when calibration techniques are used like \emph{label smoothing} \cite{szegedy_rethinking_2015} and \emph{confidence penalty} \cite{pereyra_regularizing_2017}.
This disagreement between expectation and reality thus raises the question of the true utility of epistemic uncertainty as estimated by these models.

A necessary condition to solve these inconsistencies is to ensure that epistemic uncertainty is maximal in the absence of data, a property that can only result from an appropriate choice of prior, or equivalently, of the regularizer in the loss function.
Based on this observation, we designed an elementary regularization function for ensembles of deep classifiers, called \emph{conflictual loss} for reasons that will become obvious later on. We emphasize the strengths of the resulting \emph{Conflictual Deep Ensembles} by showing experimentally that it restores both properties of epistemic uncertainty, without sacrificing either the performance or the calibration of the deep ensembles. To summarize, our contributions are the following:
\begin{itemize}
\item A method for assessing the quality of epistemic uncertainty of a model based on two principles.
\item The empirical demonstration, using this method, that common models and calibration techniques do not satisfy (and even sometimes contradict) these quality criteria.
\item A theoretical argument that suggests that these inconsistencies are due to the poor posterior approximation and not to the metric itself.
\item A new regularizer for deep ensembles, called the \emph{conflictual loss function}, designed to ensure the data-related principle of epistemic uncertainty.
\item Experimental results showing that this technique restores both quality criteria of epistemic uncertainty without degrading the other performance scores (accuracy, calibration, OOD detection).
\end{itemize}
The rest of the paper is structured as follows:
Section~\ref{sec:related-works} presents some previous works in the field of uncertainty, calibration, and prior.
Section~\ref{sec:theory} formalizes the two fundamental properties of epistemic uncertainty and gives some theoretical insights about them.
Section~\ref{sec:conflictual-loss} describes the \emph{conflictual loss} for deep ensembles.
Section~\ref{sec:experiments} presents experimental results and section~\ref{sec:conclusion} concludes.

\section{Related Work\label{sec:related-works}}

The calibration of a model reflects how its predictive distributions are consistent with its errors on a test dataset.
In~\cite{nixon_measuring_2020}, the authors discussed existing calibration metrics such as the \emph{Expected Calibration Error} (ECE)~\cite{naeini_calibration_2015} and introduced new measures like \emph{Static Calibration Error} (SCE) to better take multiclass problems into account. Post-hoc calibration techniques are also popular: histogram binning~\cite{zadrozny_obtaining_2001}, isotonic regression~\cite{zadrozny_transforming_2002} and temperature scaling~\cite{guo_calibration_2017} to name the most common.
The latter performs generally well on in-domain data but falls short when the data undergoes a distributional shift or are
out-of-distribution (OOD)~\cite{ovadia_can_2019}. 
Few works have studied calibration of posterior predictive~\cite{ovadia_can_2019,yao_quality_2019}.

When it comes to priors and regularization in deep learning, there is a considerable literature that can be classified into two main categories: parameter-based (regularizers L1, L2, etc.) and output-based (such as label smoothing~\cite{szegedy_rethinking_2015}, confidence penalty~\cite{pereyra_regularizing_2017}). The latter techniques are introduced primarily to avoid peaky outputs, which are a sign of overfitting~\cite{meister_generalized_2020}.
Priors have given rise to numerous works in general, with some specific to Bayesian deep learning (see survey \cite{fortuin_priors_2022} on this subject). To the best of our knowledge, there is no work focusing on calibrating epistemic uncertainty based on priors and without using an additional model to compute uncertainties.

To avoid the multiple evaluations of BNNs at inference time, some authors also propose to estimate the epistemic uncertainty more directly: the model predicts its uncertainty about the prediction. More precisely, the model produces a second-order distribution, that is a distribution over class distributions. Such ``all in one'' approaches (Evidential Deep Learning~\cite{sensoy_evidential_2018}, Prior Networks~\cite{malinin_predictive_2018}, Information Aware Dirichlet~\cite{tsiligkaridis_information_2021}, to name a few) require some specific training schemes. Indeed, \cite{bengs2022pitfalls} shows that training in a classical way such models by minimizing a second-order loss, does not entail well-calibrated uncertainty estimates.


The task of OOD detection is important in many applications. OOD samples are fundamentally different from the samples used during training~\cite{ren_likelihood_2019,sensoy_evidential_2018}. In theory, these examples should yield a high epistemic uncertainty. The task of OOD detection is challenging as shown in~\cite{ovadia_can_2019}: most of the methods tested on different benchmarks resulted in high-confidence predictions on OOD samples.
As shown in~\cite{ovadia_can_2019,mucsányi2024benchmarking}, DE and MC-Dropout are competitive benchmarks across different tasks including OOD detection.
Additionally, some approaches to detect OOD data involve training the model with both in-domain and OOD samples~\cite{malinin_etal_2017_incorporating,lee2018training,malinin_predictive_2018}. However, they have some limitations such as relying on the choice of the OOD dataset and that epistemic uncertainty is shifted into aleatoric uncertainty as discussed in~\cite{kirsch_pitfalls_2021}.

\section{Principles of Epistemic Uncertainty\label{sec:theory}}

In this section, we introduce two principles that characterize an \emph{idealized} measure of epistemic uncertainty.
For each of them, we give a formal definition, we justify why it is a desirable feature and we give a first analysis on its practical validity.

In what follows, we consider a family of probabilistic models that estimate the distribution of some measurable output $\Y \in \Dy$ given some input vector $\x \in \Dx$, thanks to a parametric function $f_{\params}$ parameterized by a vector $\params \in \Dparams$ of parameters, {\it i.e.,} $p(\y \,|\, \x, \params) = f_{\params}(\x)$. As Bayesian inference requires the definition of a prior $p(\params)$, the notation $\model$ for a \emph{formal model} refers hereafter to the pair $\model = \left(f_{\params},p(\params)\right)$. Given such a model $\model$ conditioned on some training sample $\data \in (\Dx \times \Dy)^*$, we consider a metric function $\epist{\data,\model}: \Dx \rightarrow \mathbb{R}^+$ that maps to an input $\x$, the measure of epistemic uncertainty conveyed by joint distribution $p(\y,\params\,|\,\x,\data)$.
We assume that this metric grows with epistemic uncertainty and is non-negative.
This is the case of common metrics like \emph{mutual information}
\[
  \epist[i]{\data,\model}(\x) = \cmi{\Y}{\Theta}{\x,\data} = \cent{\Y}{\x,\data} - \cent{\Y}{\Params,\x,\data}, 
\]
where $\cent{\cdot}{\cdot}$ denotes conditional entropy.
In case of regression ({\it i.e.,} $\Y \in \mathbb{R}$), another option is \emph{difference of variances}
\[
 \epist[v]{\data,\model}(\x) = \cvar{\Y}{\x,\data} - \cvar{\Y}{\Params,\x,\data} \, ,
\]
where $\cvar{\Y}{\x,\data}$ refers to the variance of the output when it follows the posterior predictive $p(\Y\,|\,\x, \data)$ and $\cvar{\Y}{\Params,\x,\data}$ refers to the average of variances of individual models $p(\Y\,|\,\params, \x)$, \emph{i.e.},
\[
\cvar{\Y}{\Params,\x,\data} = \int \cvar{\Y}{\params,\x} \, p(\params \,|\, \data) \, d\params \, .
\]    


\subsection{Data-related Principle of Epistemic Uncertainty}

The first principle simply states that epistemic uncertainty reduces as more training samples become available.
\begin{definition}[First principle]
  An epistemic uncertainty metric $\epist{\data,\model}$ is (ideally) a non-increasing function of training samples $\data$, {\it i.e.,}
  \[
    \forall \model, \forall \x, \forall \data_1, \forall \data_2, \quad \data_1 \subseteq \data_2 \Rightarrow \epist{\data_1,\model}(\x) \geq  \epist{\data_2,\model}(\x) \, .
  \]
\end{definition}
The reason why this property is desirable is illustrated by the next \emph{thought experiment} in the context of \emph{active learning}:
suppose that a model $\model$ has been trained on samples $\data_1$ so far. Now comes a new unlabelled sample $\x$.
Since epistemic uncertainty is the ideal criterion for measuring the information that could be gained by labeling a new sample, the measure $\epist{\data_1,\model}(\x)$ is compared to some decision threshold $\sigma$ in order to decide whether the sample is worth being labeled by an expert. Assuming that this is not the case, sample $\x$ is discarded.
Later on, the train set has been enriched with more samples $\data_2$. 
If it is possible that $\epist{\data_1 \cup \data_2,\model}(\x) > \epist{\data_1,\model}(\x)$, then it is also possible that $\epist{\data_1 \cup \data_2,\model}(\x) \geq \sigma$ so that this time the system would have asked for the labeling of $\x$.
This behavior would go against what we expect, {\it i.e.,} a model trained on more data has necessarily learned more information. The first principle bans such a scenario.

Next, we analyze the extent to which mutual information satisfies this principle.
Although examples can be found such that the observation of a specific sample increases mutual information rather than reducing it, we can ask whether the first principle is satisfied when averaging over all possible observations, {\it i.e.,} in expectation.
At first glance, we would be tempted to answer in the negative, as there exist random variables $X$, $Y$ and $Z$ such that $\mi{X}{Y} < \cmi{X}{Y}{Z}$. For making the answer positive, we need the assumption that samples are iid.
\begin{theorem}
  The \emph{mutual information metric} satisfies the first principle in expectation with respect to new random iid samples $\data_2$, {\it i.e.,}
  \[
    \forall \model, \forall \x, \forall \data_1, \quad \epist[i]{\data_1,\model}(\x) \geq  \E{\epist[i]{\data_1\cup \data_2,\model}(\x)} \, .
  \]
  \label{th:first-principle-in-expectation}
\end{theorem}
\begin{proof}
For the sake of clarity and without loss of generality, we assume $\data_2$ is made of one single sample $(\x', Y')$. Also for conciseness, we denote $\kappa$ the triplet  $(\data_1,\x,\x')$ as these terms have no incidence on the proof below.
Then considering a ``test'' input $\x$ for which we want to estimate the epistemic uncertainty conveyed by its output $\Y$, we consider the difference
\begin{eqnarray*}
\Delta I & = & \epist[i]{\data_1,\model}(\x) - \E{\epist[i]{\data_1 \cup \data_2,\model}(\x)} \\
         & = & \cmi{\Y}{\Params}{\x, \data_1} - \cmi{\Y}{\Params}{\x, \data_1, \Y', \x'} \\
         & = & \cmi{\Y}{\Params}{\data_1, \x, \x'} - \cmi{\Y}{\Params}{\Y', \data_1, \x, \x'} \text{ as } \Y,\Params \ci \X' \,|\, \X \\
         & = & \cent{\Y}{\kappa} - \cent{\Y}{\Params, \kappa} - \cent{\Y}{\Y', \kappa} + \cent{\Y}{\Params, \Y', \kappa} \\
         & = & \cent{\Y}{\kappa} - \cent{\Y}{\Y',\kappa}  - \left(\cent{\Y}{\Params,\kappa} - \cent{\Y}{\Params, \Y', \kappa}\right) \\
         & = & \cmi{\Y}{\Y'}{\kappa} - \cmi{\Y}{\Y'}{\Params, \kappa} \, .
\end{eqnarray*}
But $\Y$ and $\Y'$ are iid samples, {\it i.e.,} they are independent given $\Params$ and $\kappa$, and thus $\cmi{\Y}{\Y'}{\Params, \kappa} = 0$. Hence, $\Delta I = \cmi{\Y}{\Y'}{\kappa} \geq 0$, and the proof is now complete.

  
\end{proof}

\subsection{Model-related Principle of Epistemic Uncertainty}

The second principle 
essentially expresses \emph{overfitting}: given two models trained with the same set of samples, if one has more expressive power than the other, then it should have a larger measure of epistemic uncertainty since the choice of model candidates is wider, {\it i.e.,} its \emph{posterior distribution} is more spread out.

While this principle 
seems intuitive, the formalization of the underlying notion of expressive power requires the use of complex theories of statistical learning  ({\it e.g.,} VC dimension), which we avoid since it is unnecessary.
Indeed, we can only consider models that are, by construction, ordered in increasing order of complexity, as defined below.
\begin{definition}[Submodel]\label{def:submodel}
  We say model $\model_a = \left(f^a_{\params_a}, p_a(\params_a)\right)$ is a \emph{submodel} of 
  model $\model_b= \left(f^b_{\params_b}, p_b(\params_b)\right)$,  denote by $\model_a \leqM \model_b$, if $\params_a$ is a subset of parameters $\params_b$ so that $\params_b = (\params_a, \params_{b'})$ and there exists a constant vector $\params^0_{b'} \in \Omega_{\params_{b'}}$ such that 
  \[
    \forall \params_a \in \Omega_{\params_a}, \quad  f^a_{\params_a} = f^b_{(\params_a, \params^0_{b'})}
    \text{ and } \, p_a(\params_a) = p_b\!\left(\params_a \,|\, \Params_{b'} = \params^0_{b'} \right) \, .
  \]
\end{definition}

Moreover, when priors are chosen in such a way that individual parameters are independent, 
freezing 
$\params_{b'}$ 
has no 
impact on the prior of $\params_a$, so that the condition on priors simplifies to $p_a(\params_a) = p_b(\params_a)$.
Since the submodel relation is reflexive, transitive, and antisymmetric, it defines a partial ordering on the set of parameterized models, that can be used to state the second principle.
\begin{definition}[Second principle]
  An epistemic uncertainty metric $\epist{\data,\model}$ \\ 
  should be a non-decreasing function over the set of parameterized models, {\it i.e.,}
  \[
    \forall \data, \forall \x, \forall \model_1, \forall \model_2, \quad \model_1 \leqM \model_2 \Rightarrow \epist{\data,\model_1}(\x) \leq \epist{\data,\model_2}(\x) 
  \]
\end{definition}

To see why this principle is desirable, consider the following example in the field of \emph{explainability}.
Given two black-box models $\model_1$ and $\model_2$ with comparable performance, let's assume that $\epist{\data,\model_1}$ is larger on average than $ \epist{\data,\model_2}$ when estimated on a test set. A larger value of epistemic uncertainty for a given input $\x$ means more diverse and thus inconsistent stochastic functions $p(\y\,|\,\params,\x)$ as $\params$ follows the posterior distribution.
Therefore, model $\model_2$ provides on average more similar and consistent functions than $\model_1$.
Explaining the output $p(\y \,|\, \x, \data)$ of a model amounts to summarize in an understandable format these stochastic functions $p(\y\,|\,\params,\x)$ taken as a whole.
Model $\model_2$ is thus preferred since the explanation of its output is shorter on average.
But if the second principle is broken, it is possible that $\model_1$ is formally a submodel of $\model_2$. If so, $\model_1$ is by construction a restriction of $\model_2$, providing systematically shorter explanations, a contradiction. In summary, without the second principle, measures of epistemic uncertainty could make inconsistent \emph{Occam's razor principle} and all related concepts (sparsity, minimum description length, etc).

Focusing again on the mutual information metric, we see that the second principle is also verified in expectation, as it is always true that $\cmi{\Y}{\Params_1}{\data, \x} \leq \cmi{\Y}{\Params_1, \Params_2}{\data, \x}$.
Indeed, the left term of this inequality can be understood as a weighted average of mutual information $\cmi{\Y}{\Params_1}{\params^0_2, \data, \x}$ of every submodel $\params^0_2$.
However, the weight of submodel $\params^0_2$ is given by the posterior $p(\params_2 \, |\, \data)$.
The interpretation of this inequality in expectation is therefore difficult and of little interest in the context of \emph{model selection} since, in practice, we want to compare a model with a given submodel, not with the whole distribution of submodels.

To illustrate, consider a \emph{Bayesian linear regression model} with known homoskedastic variance and isotropic normal prior, {\it i.e.},
\[
  \Y = \sum_{i \in \mathcal{I}} \Theta_i \, \psi_i(\x) + \varepsilon \text{ with } \varepsilon \sim \mathcal{N}(0,\sigma^2) \text{ and } \Theta_i \sim \mathcal{N}(0, \sigma^2_0) \, .
\]
where the regressor functions $\psi_i$ are chosen in a large collection indexed by $\mathcal{I}$.
After making some \emph{variable selection}, we can force to zero some coefficients $\theta_i$, only keeping regressors of the index in $\mathcal{I}' \subseteq \mathcal{I}$. As priors on coefficients are independent, the definition of the resulting submodel $\model_{\mathcal{I}'}$ is obtained just by replacing $\mathcal{I}$ by $\mathcal{I}'$ in the above definition.
Now given such a submodel $\model_{\mathcal{I}'}$, we would like to ensure the measure of epistemic uncertainty is smaller for $\model_{\mathcal{I}'}$ than for $\model_{\mathcal{I}}$.
For the sake of simplicity, suppose that regressor functions are decorrelated, {\it i.e.,} $\E{\psi(\vect{X}) \, \psi(\vect{X})^T} = \mat{Id}$. It can be shown that the epistemic uncertainty of submodel $\model_{\mathcal{I}'}$ as estimated by the \emph{difference of variances}, is
\[
  \epist[v]{\data,\model_{\mathcal{I}'}}(\x) = \cvar{\Y}{\x,\data} - \cvar{\Y}{\Params,\x,\data} = \frac{|\mathcal{I}'|}{\sigma_0^{-2} + |\data|\, \sigma^{-2}} \, .
\]
As this value increases with the number $|\mathcal{I}'|$ of parameters and decreases with the number $|\data|$ of examples, the first and second principles are always satisfied.

This result naturally raises the question of whether it can be generalized to models like deep neural networks.
As there are no simple analytical answers for such complex models, we present an experimental study in section~\ref{sec:experiments} to assess the extent to which the two principles are satisfied by various classification models. This includes the method of \emph{Conflictual Deep Ensembles}, which we introduce in the next section.

\section{Conflictual Deep Ensembles\label{sec:conflictual-loss}}

As we shall see in Sect.~\ref{sec:experiments}, several classifiers present abnormally low levels of epistemic uncertainty in the low-data regime. This ``hole'' is contrary to the first principle which states that epistemic uncertainty should be maximal in the absence of training data. Why does this happen?
By rewriting mutual information $\mi{\Y}{\Params \,|\, \x, \data}$ as
\[
 \mi{\Y}{\Params \,|\, \x, \data}= \int p(\params\,|\, \data) \, \KL{p(\y\,|\,\params, \x)}{p(\y \,|\, \x, \data)}\, d\params \, ,
\]
we can interpret it as a weighted average of divergence between predictions $p(\y\,|\,\params, \x)$ of individual models and prediction $p(\y \,|\, \x, \data)$ of the averaged model.
Therefore, the hole reflects the absence of diversity between output distributions $p(\y\,|\,\params, \x)$ in the low-data regime.
But in this regime, this lack of variability is mostly a consequence of the choice of the prior or, equivalently, of the regularization term in the loss function. This explains why the hole is particularly visible in experiments using \emph{label smoothing}, since this regularization technique drives output distributions closer to the same uniform distribution.

This observation also suggests that designing a prior that favors diversity or, in other words, \emph{discordance} between output distributions, could fill the hole of epistemic uncertainty.
Such an objective can be achieved simply by constructing a so-called \emph{conflictual deep ensemble}, where each classifier in the ensemble slightly favors a class of its own.
In the absence of data, these slight tendencies are enough to create discordance in the output distributions and therefore a high level of mutual information.

In practice, a \emph{conflictual deep ensemble} of order $k$ is implemented as an ensemble of $k \times C$ deep classifiers such that every class $c \in \{1, \dots, C\}$ is mapped to $k$ models $\{ \params^c_i \}_{1\leq i \leq k}$.
Denoting by $P(\y \,|\, \params, \x)$ the probability of class $\y$ as predicted for input $\x$ by model $\params$, we define the \emph{conflictual loss} for class $c$ as
\begin{equation}
    L_c(\params) = - \sum_{(\x,\y) \in \data} \left( \log P(\y \,|\, \params, \x) + \lambda \, \log P(c \,|\, \params, \x) \right), \label{eq:conflictual-loss}
\end{equation}
where the first term is the log-likelihood and the second term is the bias that slightly favors class $c$.
This bias term can be interpreted as if for each observed example $(\x,\y)$ in the train set, we add $\lambda$ faked examples of selected class $c$.
In practice $\lambda$ has been empirically fixed to $0.05$, meaning there is one faked example for $20$ real examples.
We then train every model of the ensemble independently, using for model $\params^c_i$ the loss function $L_{c}$.

The Conflictual Loss both resembles and differs from Label Smoothing: like LS, the conflictual term $\log P(c \,|\, \params, \x)$ is not an independent regularization term, but factorized into the sum of the log-likelihoods.
However, unlike LS which encourages classifiers to be more concordant by promoting the uniform distribution, the Conflictual Loss encourages classifiers to be contradictory.

Existing works~\cite{deMathelin2023,deMathelin2023a} have sought to address the diversity of the models in the ensemble. Their focus was on an ``anti-regularization'' of the model's weights resulting in weights with high magnitudes without sacrificing the model's performance. Although, by construction, Conflictual loss aims at creating diversity in the ensemble outputs, we expect each model in the ensemble to adjust its weights accordingly.

\section{Empirical Analysis\label{sec:experiments}}

We conducted several experiments to assess to what extent both principles of epistemic uncertainty discussed in Sect.~\ref{sec:theory} are verified. We compared \emph{Conflictual Deep Ensemble} to \emph{MC-Dropout} \cite{gal_dropout_2016}, \emph{Deep Ensemble} \cite{lakshminarayanan_simple_2017} and \emph{EDL} \cite{sensoy_evidential_2018}. In addition, we tested \emph{Label Smoothing} (LS) \cite{szegedy_rethinking_2015} in combination with MC-Dropout.
Moreover, we evaluated \emph{Confidence Penalty} regularization combined with MC-Dropout and reported the results in Appendix~\ref{app:confidence-penalty} due to the space limit.~\footnote{Code available at: \url{https://github.com/fellajimed/Conflictual-Loss}}

For the varying number of samples used to train the models, we considered, after a $20\%$ validation-train split identical for all models, fractions of the entire training set that grow exponentially from $0.005$ to $1$.
To evaluate the \emph{data-related principle}, we made sure that by increasing the training set, new examples are added to the previous training set rather than randomly selecting a new independent subset from the entire training set.
Additionally, for a fixed ratio, we emphasize that the models are trained on the same samples.

We used Multilayer Perceptron (MLP) models with two hidden layers for a straightforward control of its size: 
since layers are dense, they are invariant under the permutation of neurons. As a consequence the submodel relation defined in Sect.~\ref{sec:theory} is simplified: given two networks $\model_1$ and $\model_2$ of $L$ dense layers whose sizes are respectively $(n^{[1]}_1,\dots, n^{[L]}_1)$ and $(n^{[1]}_2,\dots, n^{[L]}_2)$,
\[
  \model_1 \leqM \model_2 \quad \Longleftrightarrow \quad \forall i, \, n^{[i]}_1 \leq n^{[i]}_2 
\]
We thus consider for every method a chain of submodels whose sizes of the hidden layers grow exponentially, starting from $(128,64)$ neurons up to $(2048,1024)$.
Epistemic uncertainty is estimated using $20$ forward passes at inference time for MC-Dropout and $10$ MLPs in the case of Deep Ensemble. We set the order $k=1$ for the Conflictual DE so that it also contains $C = 10$ MLPs. Each hidden layer is followed by a Dropout layer ($p=0.3$) and a ReLU activation function. 

\begin{figure}[htbp]
  \begin{subfigure}[t]{\dimexpr0.185\textwidth+20pt\relax}
    \makebox[20pt]{\raisebox{25pt}{\rotatebox[origin=c]{90}{\scriptsize MNIST}}}%
    \includegraphics[width=\dimexpr\linewidth-20pt\relax]{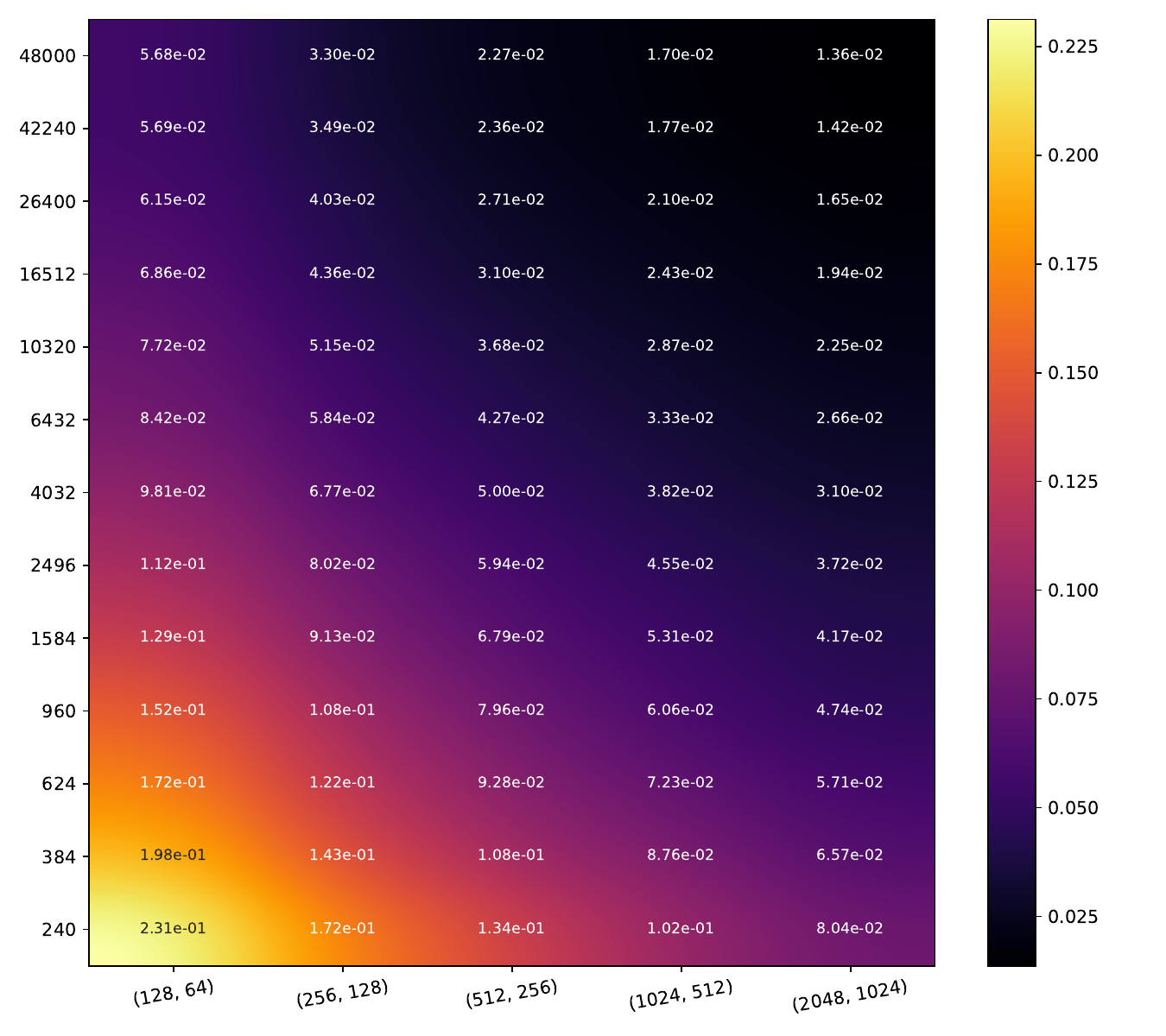}
    \makebox[20pt]{\raisebox{25pt}{\rotatebox[origin=c]{90}{\scriptsize SVHN}}}%
    \includegraphics[width=\dimexpr\linewidth-20pt\relax]{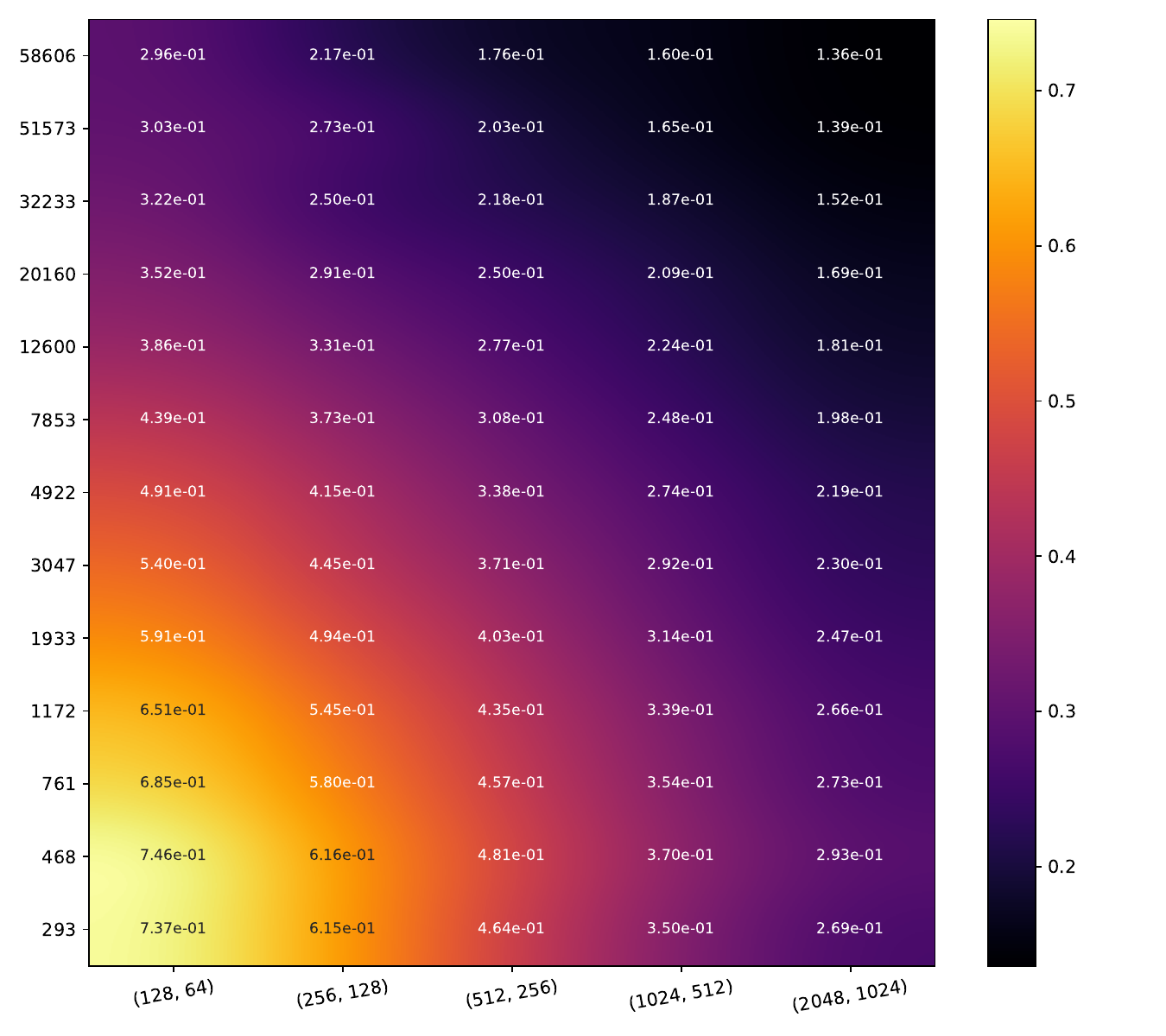}
    \makebox[20pt]{\raisebox{25pt}{\rotatebox[origin=c]{90}{\scriptsize CIFAR10}}}%
    \includegraphics[width=\dimexpr\linewidth-20pt\relax]{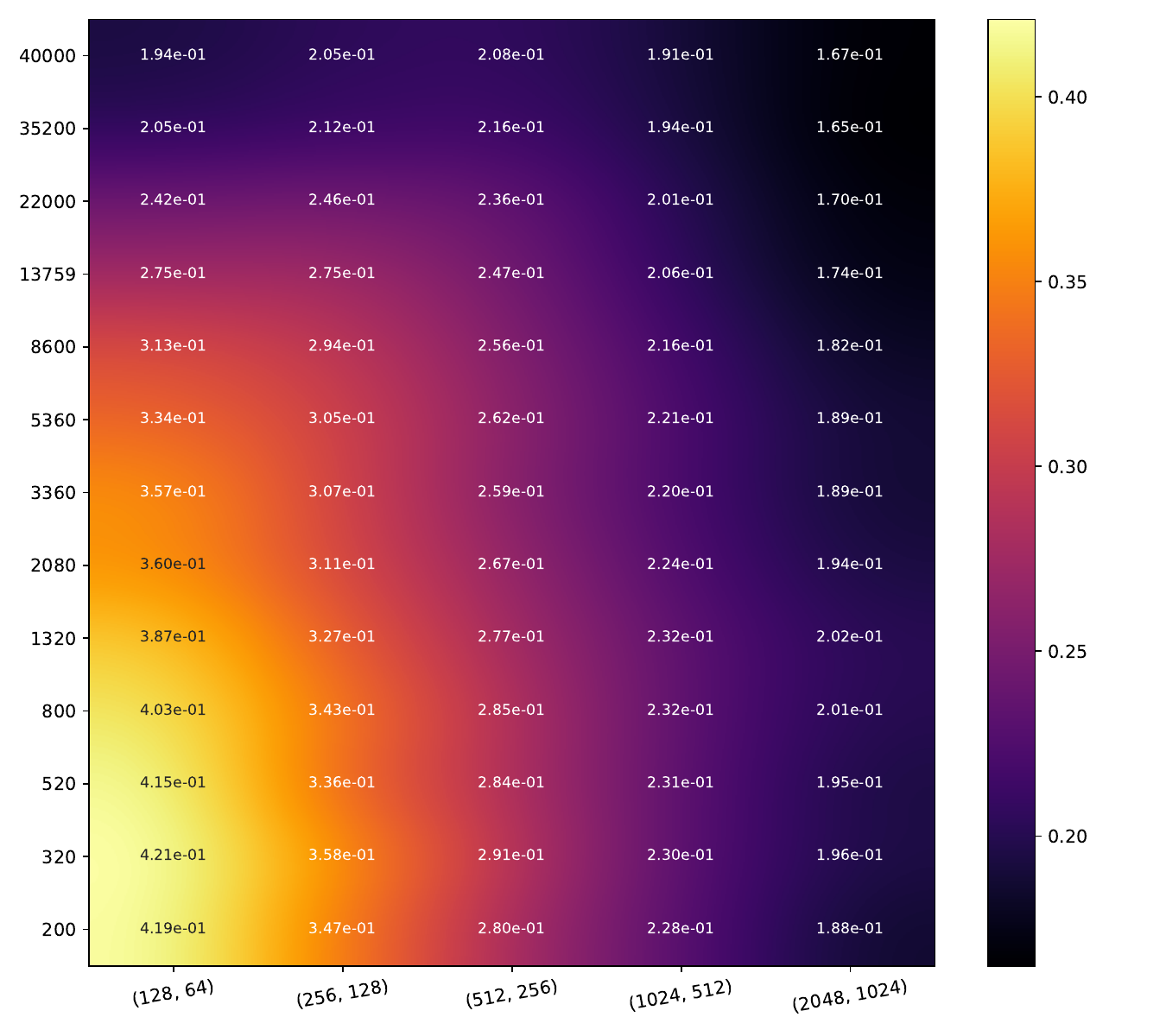}
    \caption*{\qquad MC-Dropout}
  \end{subfigure}\hfill
  \begin{subfigure}[t]{0.185\textwidth}
    \includegraphics[width=\textwidth]{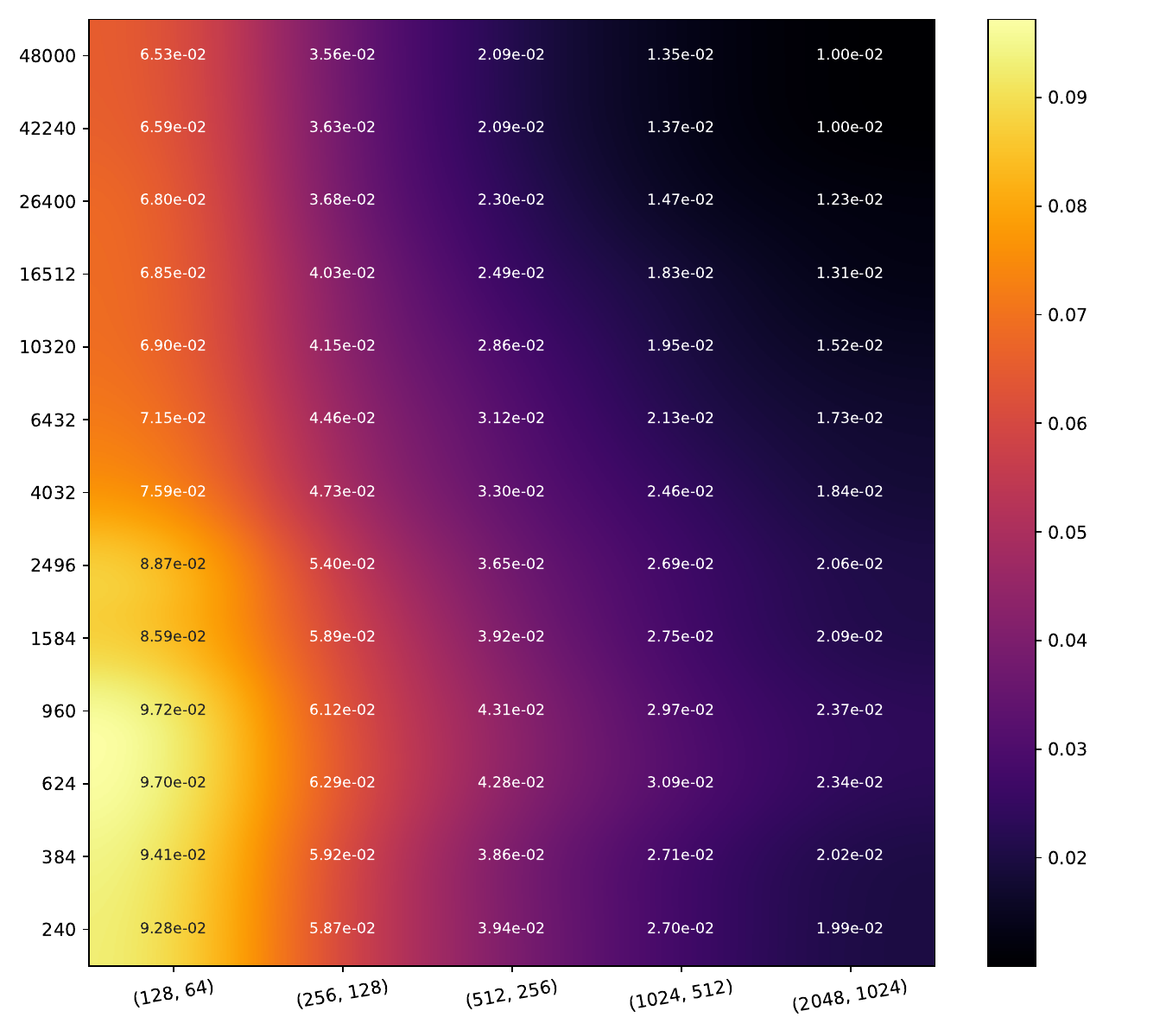}
    \includegraphics[width=\textwidth]{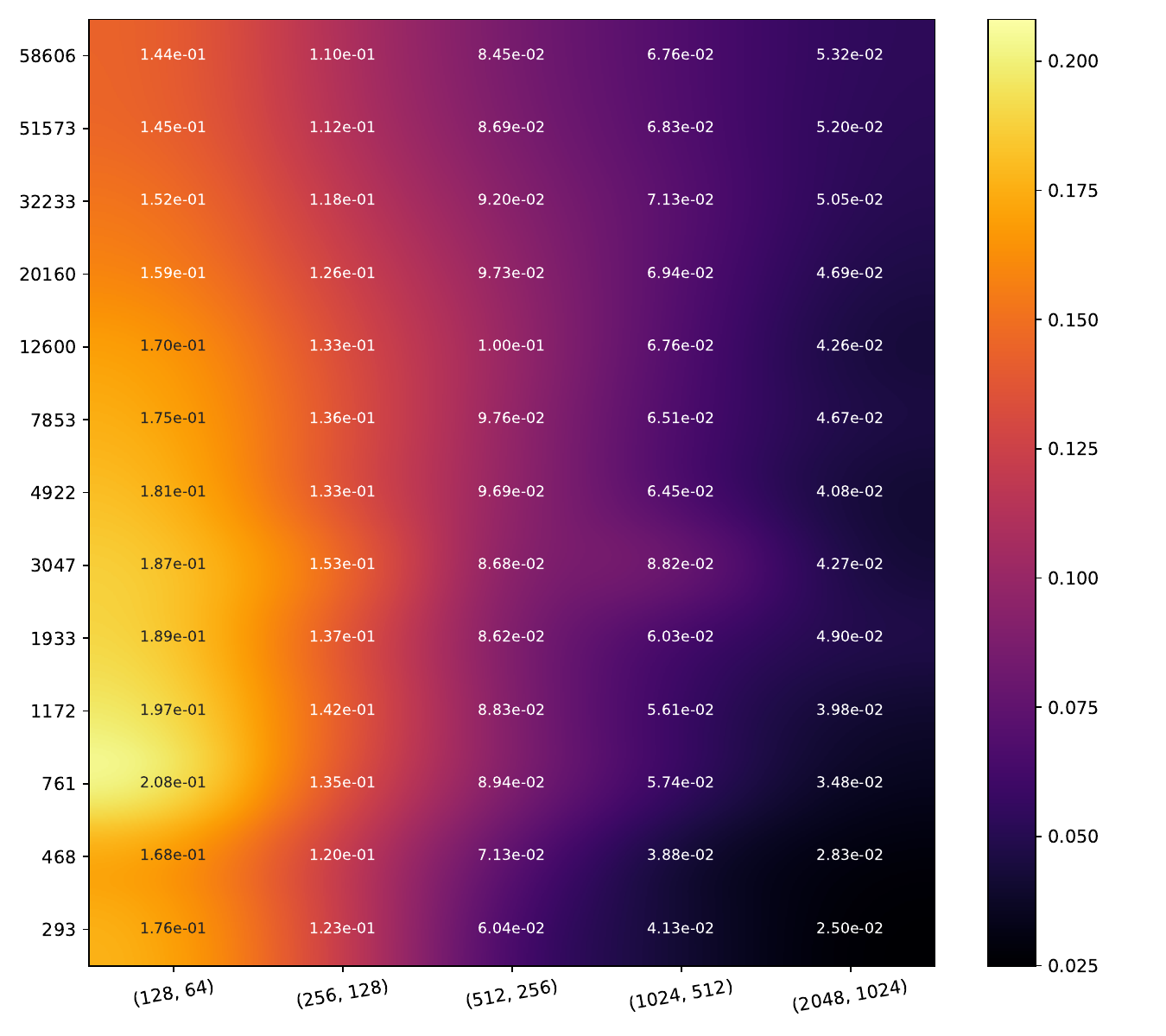}
    \includegraphics[width=\textwidth]{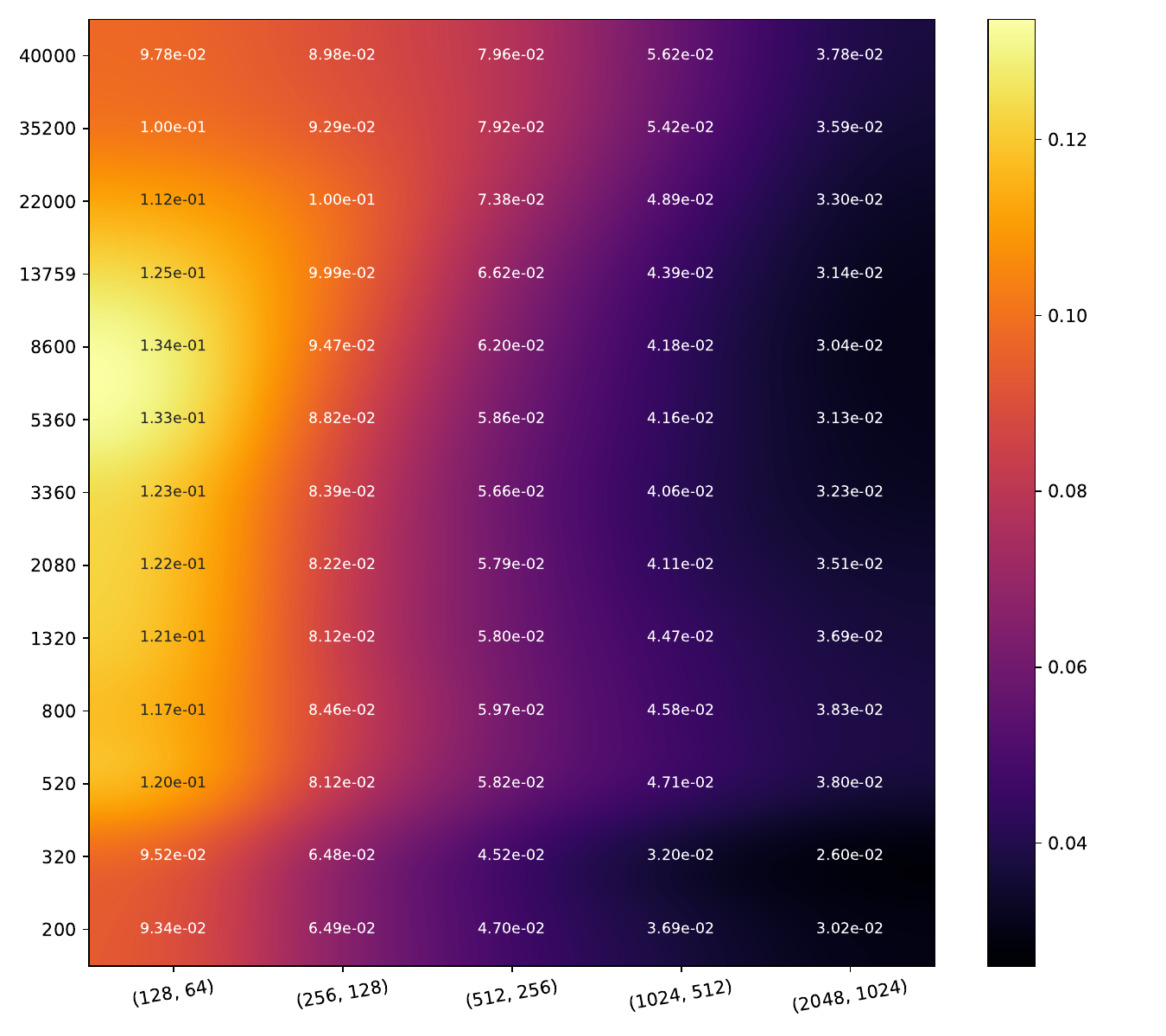}
    \caption*{MC-Dropout LS}
  \end{subfigure}\hfill
  \begin{subfigure}[t]{0.185\textwidth}
    \includegraphics[width=\textwidth]{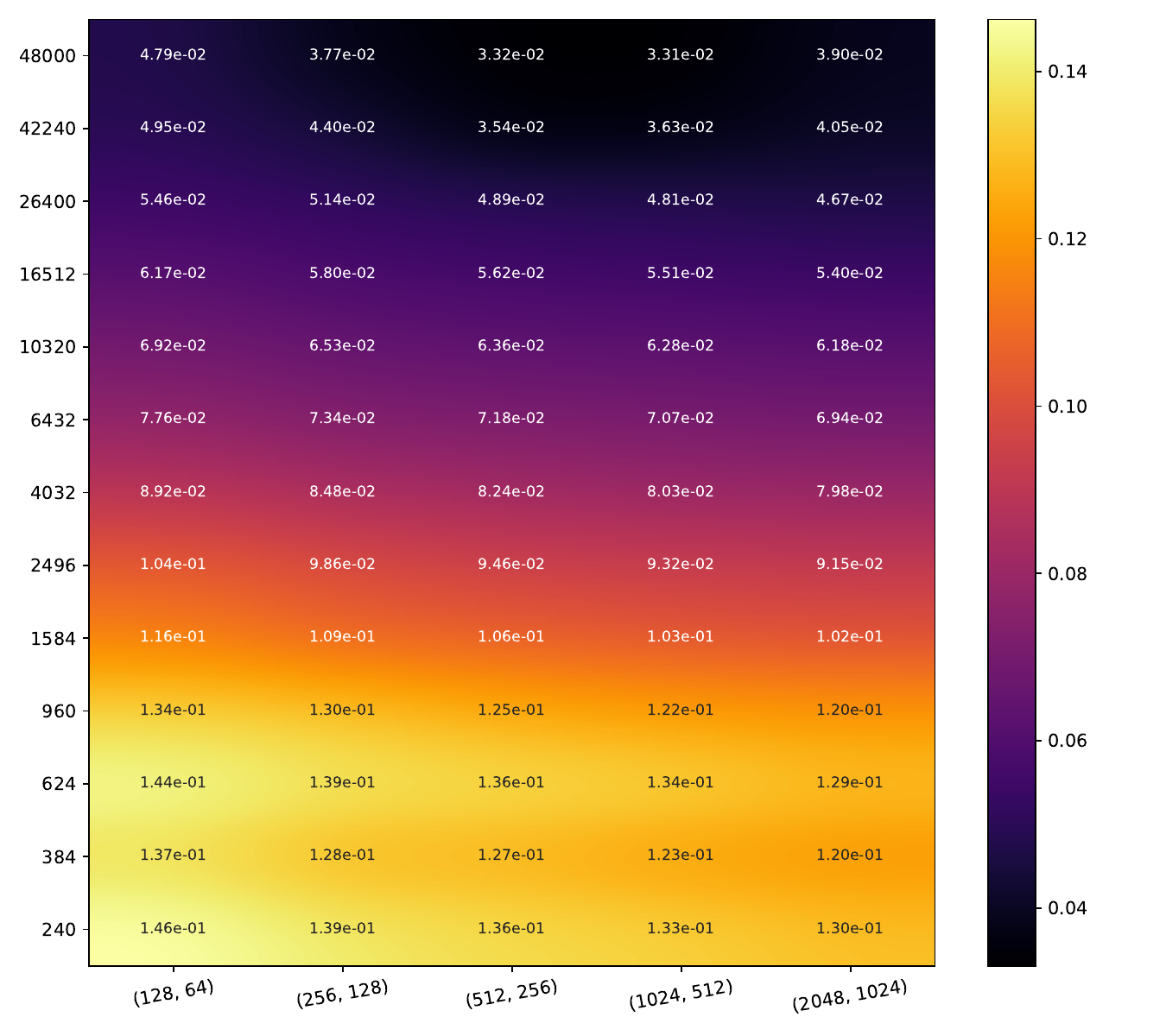}
    \includegraphics[width=\textwidth]{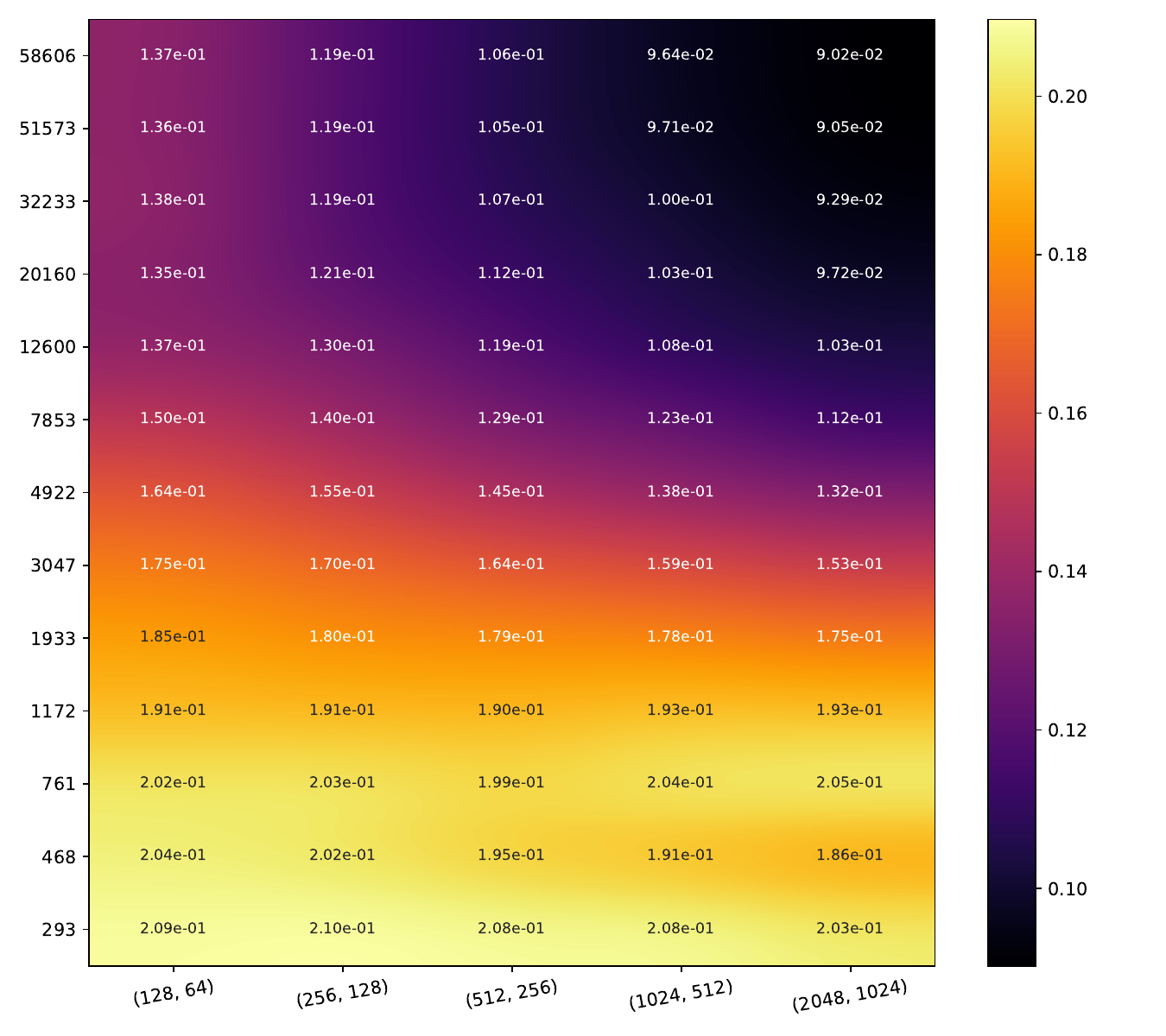}
    \includegraphics[width=\textwidth]{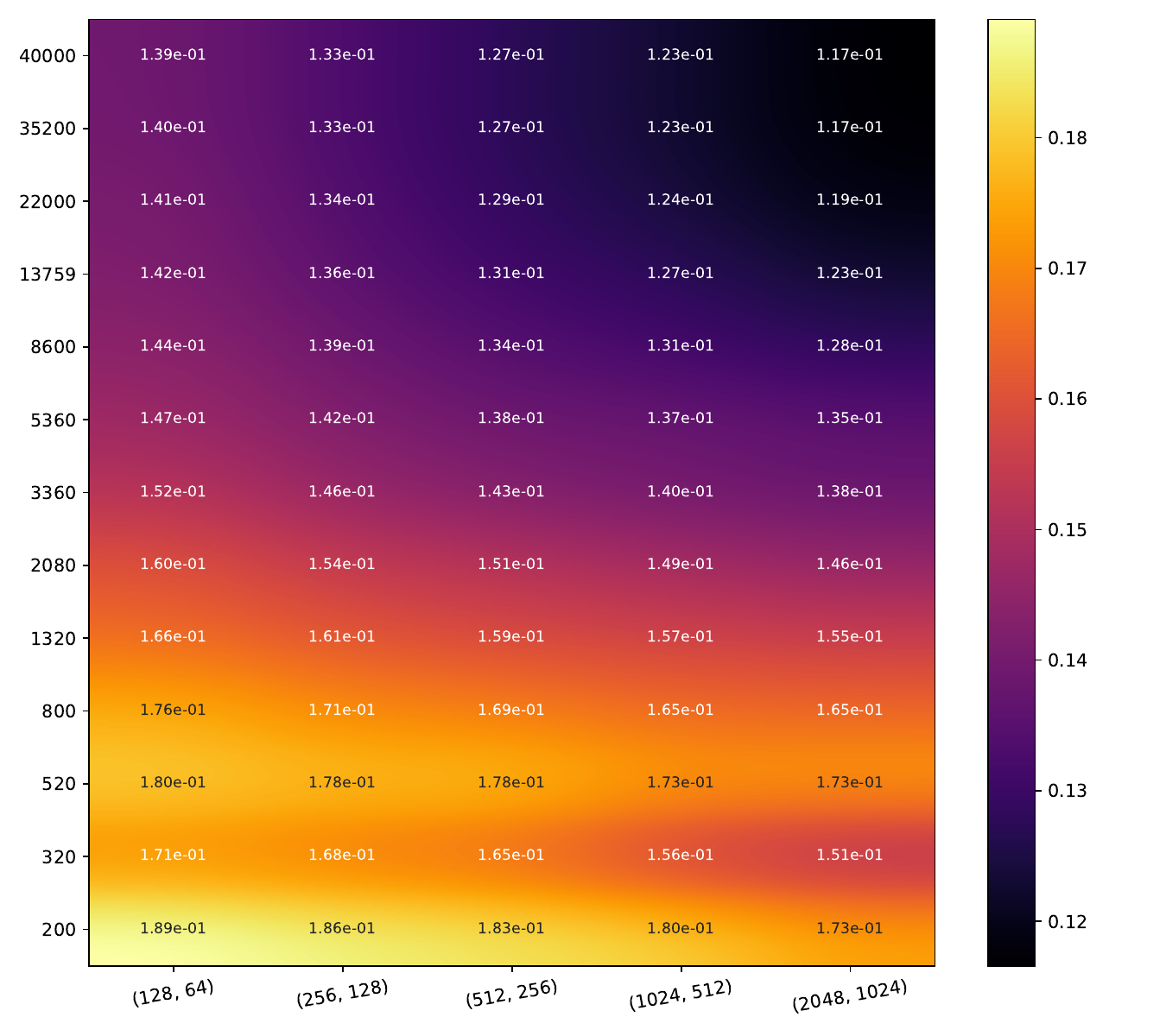}
    \caption*{EDL}
  \end{subfigure}\hfill
  \begin{subfigure}[t]{0.185\textwidth}
    \includegraphics[width=\textwidth]{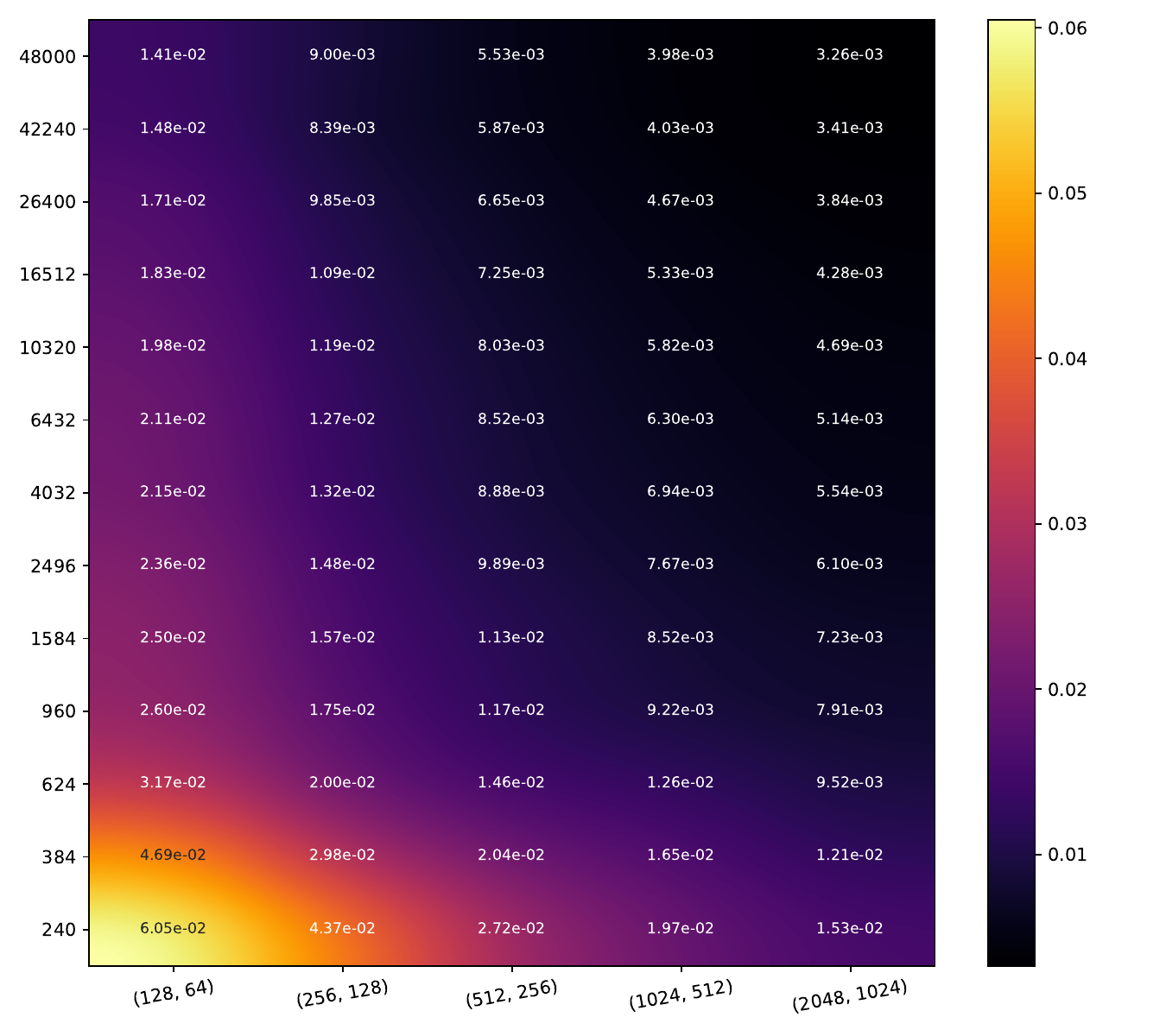}
    \includegraphics[width=\textwidth]{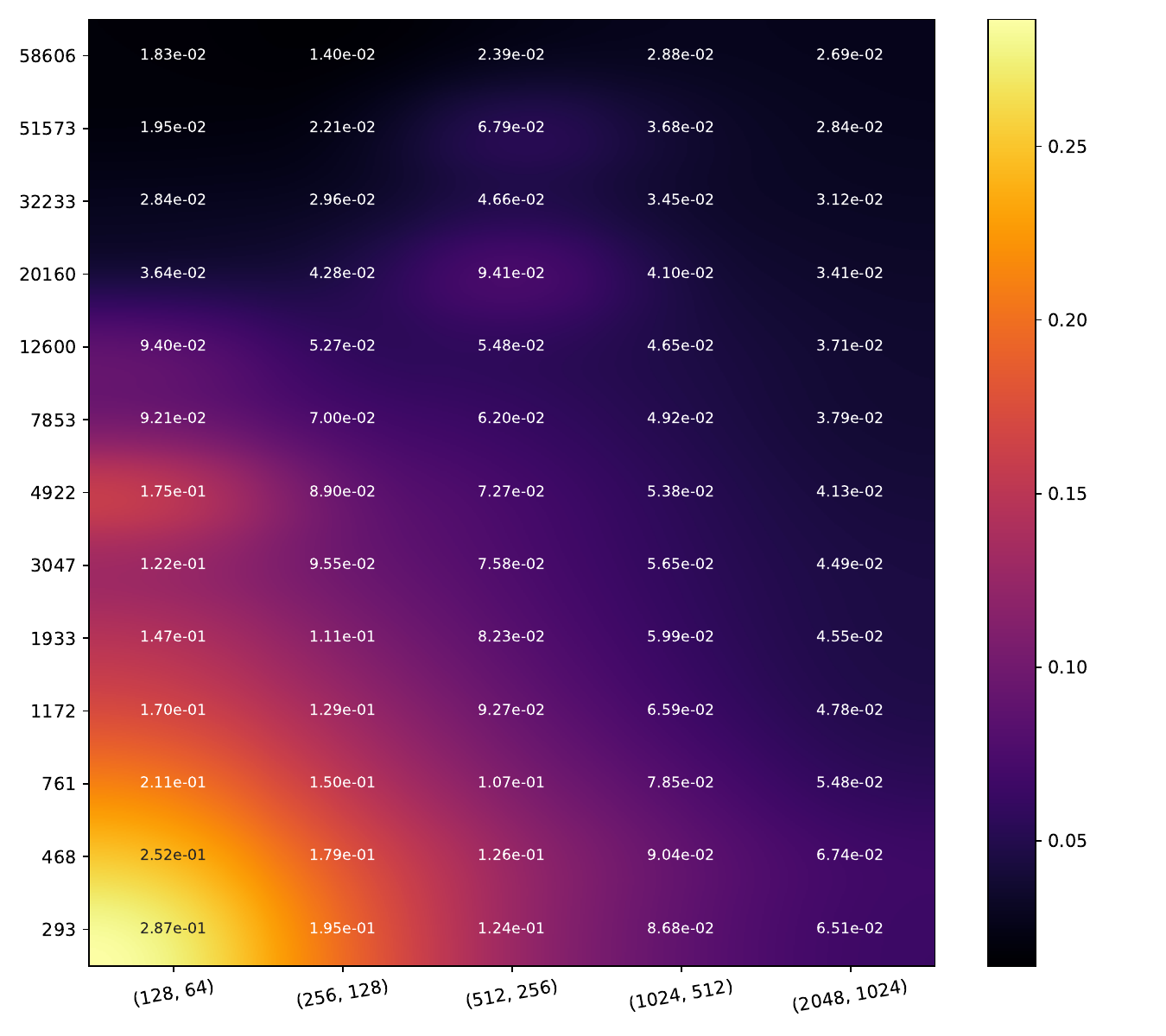}
    \includegraphics[width=\textwidth]{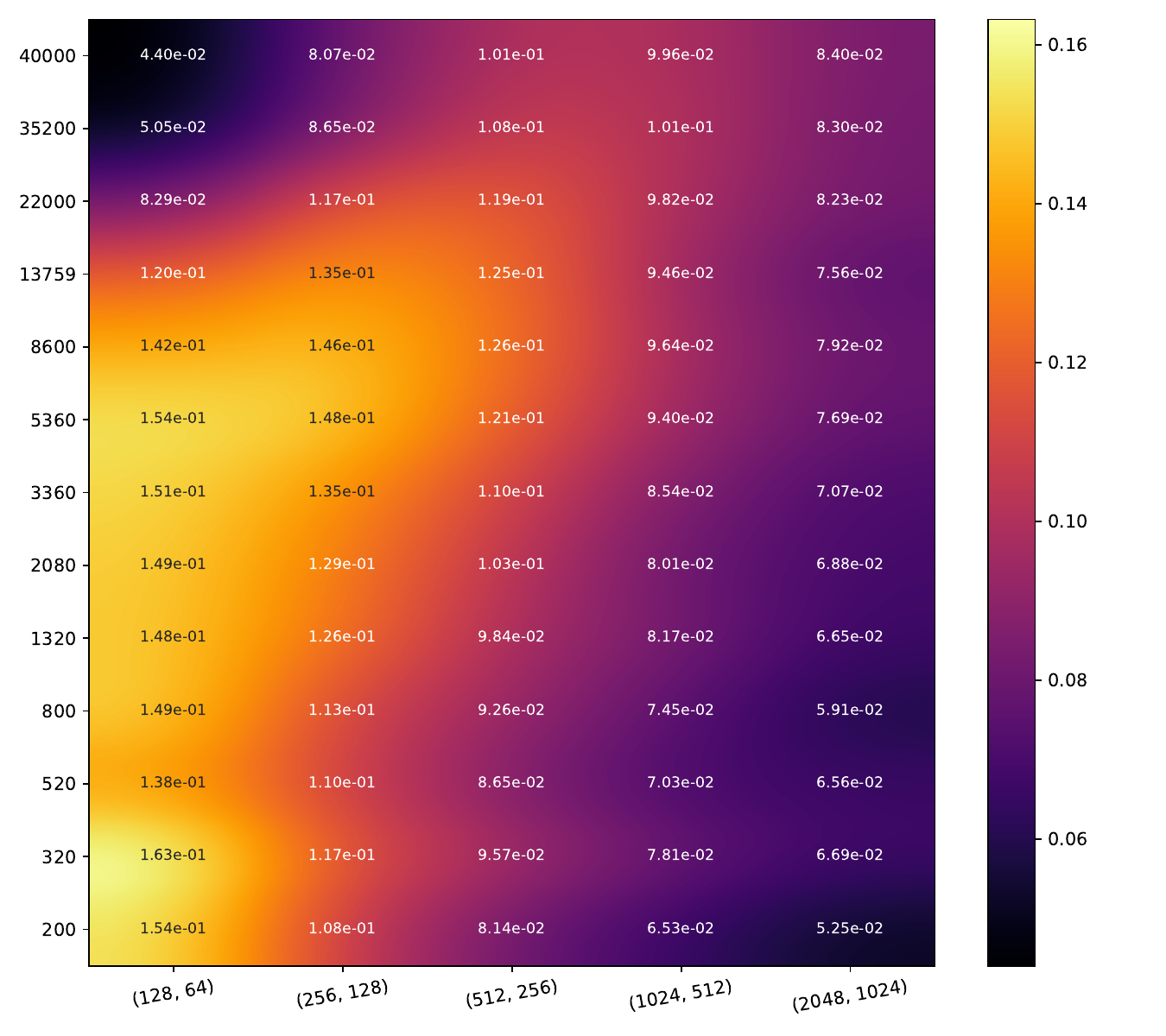}
    \caption*{DE}
  \end{subfigure}\hfill
  \begin{subfigure}[t]{0.185\textwidth}
    \includegraphics[width=\textwidth]{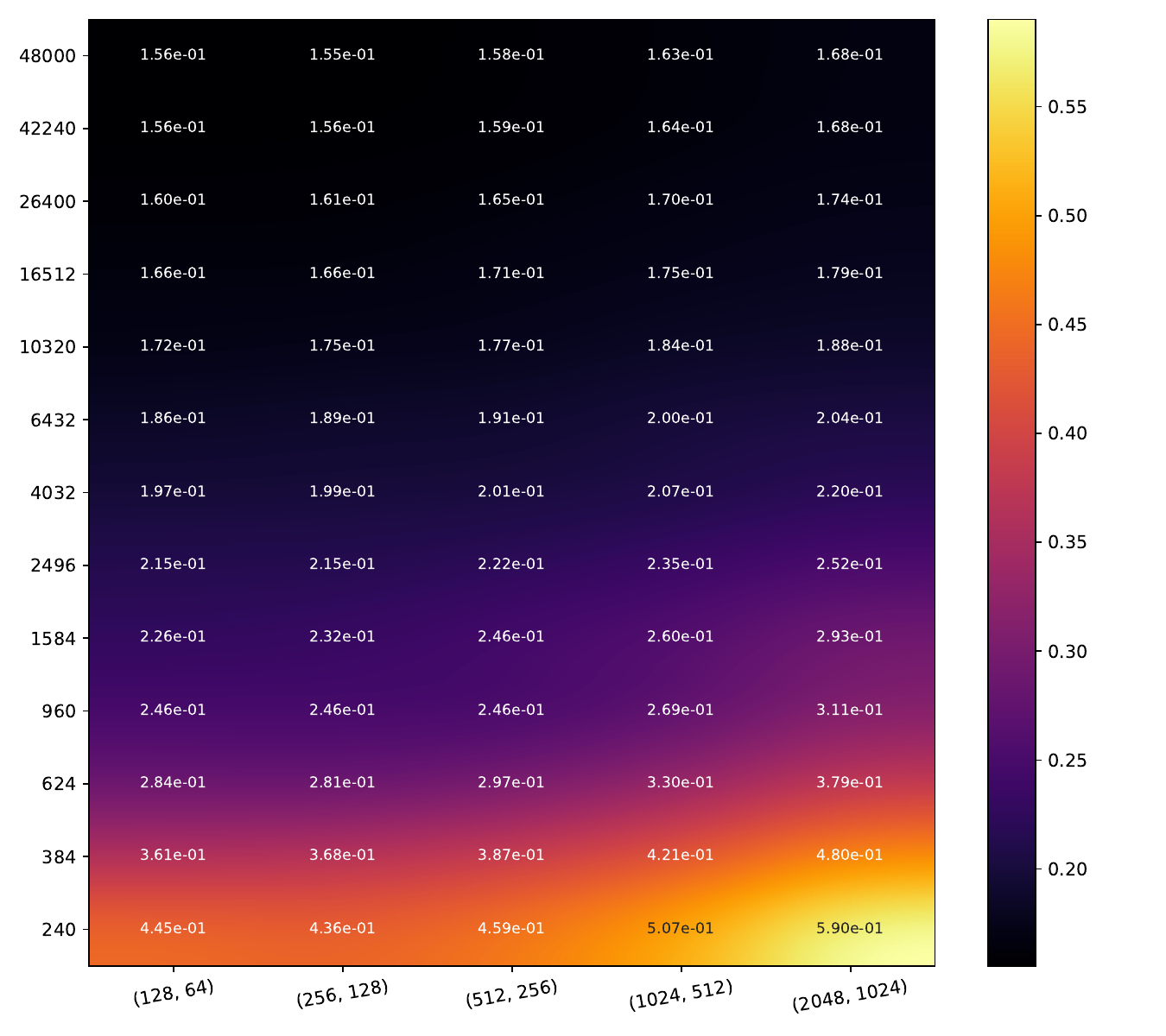}
    \includegraphics[width=\textwidth]{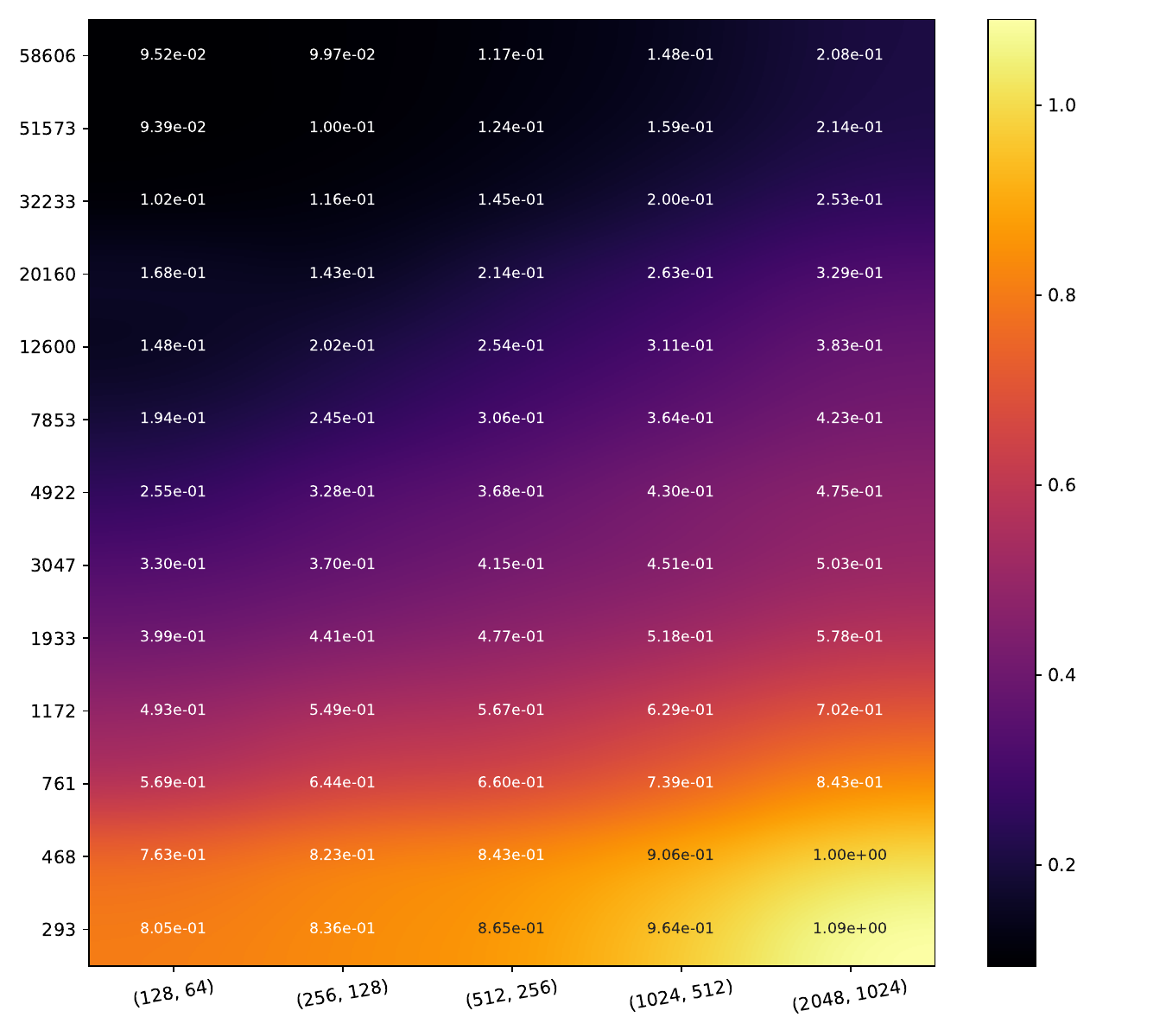}
    \includegraphics[width=\textwidth]{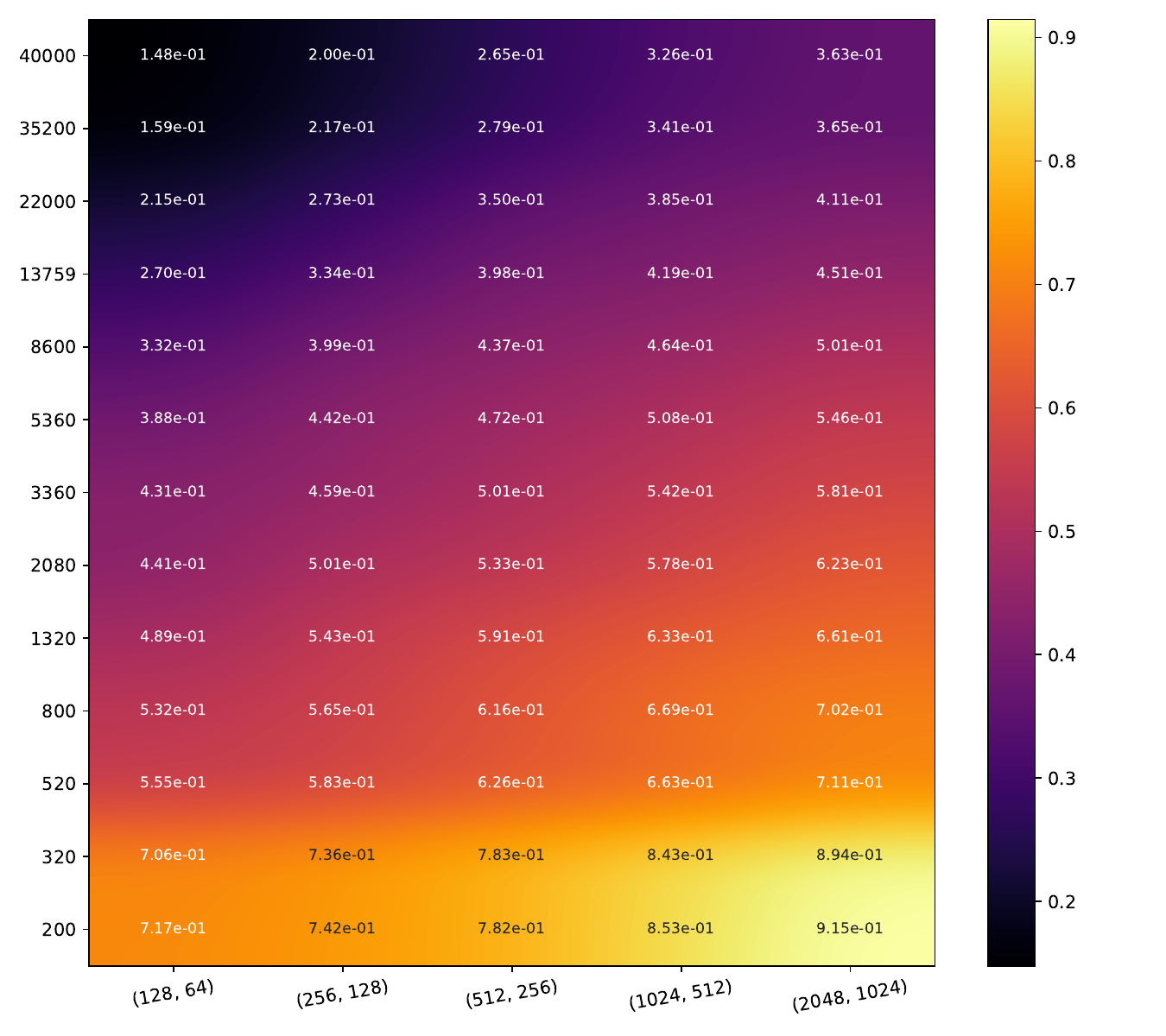}
    \caption*{Conflictual DE}
  \end{subfigure}\hfill
  \caption{Heatmaps of epistemic uncertainty (mutual information) on MNIST, SVHN, and CIFAR10 datasets; for MC-Dropout, label smoothing combined with MC-Dropout (MC-Dropout LS), EDL, Deep Ensembles (DE), and Conflictual DE. For each heatmap, the x-axis gives the sizes of the hidden layers and the y-axis gives the number of training samples. Both have logarithmic scales. Color scales are different. Epistemic uncertainty should decrease along the y-axis (data-related principle) and increase along the x-axis (model-related principle).}\label{fig:epistemic-reg}
\end{figure}

These models are tested on three datasets: MNIST, SVHN and CIFAR10. For the former, we apply the MLPs to the raw images, whereas for the latter, image embeddings are computed using a pre-trained ResNet34 model and used as inputs for the MLPs.
In all the heatmaps, we represent the sizes of the hidden layers on the x-axis, while the y-axis corresponds to the length of the training set.
Empirically, we set the value of $\lambda$ in Eq.~\ref{eq:conflictual-loss} to $0.05$.
The hyperparameters of label smoothing, confidence penalty, and EDL are taken from their respective papers~\cite{szegedy_rethinking_2015,malinin_predictive_2018,pereyra_regularizing_2017} and are set to $0.1$ (for the first two) and $0.01$ for EDL.
We refer the reader to Appendix~\ref{app:implementation-details} for more details.
\begin{table}[htbp]
  \centering
  \begin{tabular}{p{0.1\textwidth} c P{0.155\textwidth} P{0.19\textwidth} P{0.1\textwidth} P{0.1\textwidth} P{0.175\textwidth}}
    \toprule
     & & MC-Dropout & MC-Dropout LS & EDL & DE & Conflictual DE \\
    \midrule
    \multirow{3}{*}{E.U.} %
        & MNIST   & $0.067$ & $0.040$ & $0.087$ & $0.014$ & $0.245$ \\
        & SVHN    & $0.348$ & $0.099$ & $0.152$ & $0.079$ & $0.427$ \\
        & CIFAR10 & $0.258$ & $0.068$ & $0.149$ & $0.102$ & $0.507$ \\
    \midrule
    \multirow{3}{*}{Acc. $\uparrow$} %
        & MNIST   & $93.30\%$ & $\mathbf{93.90\%}$ & $91.46\%$ & $93.52\%$ & $93.86\%$ \\
        & SVHN    & $68.10\%$ & $68.58\%$ & $63.46\%$ & $69.00\%$ & $\mathbf{69.32\%}$ \\
        & CIFAR10 & $60.71\%$ & $60.56\%$ & $60.93\%$ & $61.62\%$ & $\mathbf{61.84\%}$ \\
    \midrule
    \multirow{3}{*}{\shortstack[l]{Brier \\ score} $\downarrow$} %
        & MNIST   & $0.101$ & $0.216$ & $0.164$ & $0.102$ & $\mathbf{0.094}$ \\
        & SVHN    & $0.457$ & $0.498$ & $0.491$ & $0.470$ & $\mathbf{0.419}$ \\
        & CIFAR10 & $0.576$ & $0.572$ & $0.540$ & $0.589$ & $\mathbf{0.528}$ \\
    \midrule
    \multirow{3}{*}{SCE $\downarrow$} %
        & MNIST   & $\mathbf{0.0056}$ & $0.0661$ & $0.0335$ & $0.0076$ & $0.0120$ \\
        & SVHN    & $0.0289$ & $0.0485$ & $0.0276$ & $0.0345$ & $\mathbf{0.0158}$ \\
        & CIFAR10 & $0.0377$ & $0.0435$ & $0.0235$ & $0.0444$ & $\mathbf{0.0194}$ \\
    \midrule
    \multirow{3}{*}{OOD $\uparrow$} %
        & MNIST   & $82.72\%$ & $58.70\%$ & $81.85\%$ & $\mathbf{88.33\%}$ & $86.18\%$ \\
        & SVHN    & $78.59\%$ & $\mathbf{85.09\%}$ & $60.39\%$ & $77.20\%$ & $78.22\%$ \\
        & CIFAR10 & $68.81\%$ & $63.22\%$ & $70.29\%$ & $69.09\%$ & $\mathbf{74.20\%}$ \\
    \midrule
    \multirow{3}{*}{Mis. $\uparrow$} %
        & MNIST   & $92.54\%$ & $73.78\%$ & $87.83\%$ & $\mathbf{93.44\%}$ & $93.16\%$ \\
        & SVHN    & $80.11\%$ & $64.24\%$ & $78.00\%$ & $\mathbf{82.63\%}$ & $78.21\%$ \\    
        & CIFAR10 & $73.51\%$ & $58.13\%$ & $72.36\%$ & $\mathbf{74.79\%}$ & $74.09\%$ \\    
    \bottomrule
  \end{tabular}
  \caption{Average value per heatmap for different metrics: E.U. (Epistemic uncertainty, Fig~\ref{fig:epistemic-reg}), Acc. (accuracy, Fig~\ref{fig:accuracy}), Brier score (Fig~\ref{fig:brier}), SCE (Fig~\ref{fig:sce}), OOD detection (Fig~\ref{fig:auc-ood}) and Mis. (misclassification, Fig~\ref{fig:auc-mis}).
  \emph{Warning:} if the values of accuracy seem low, it is because they are average values computed over the entire accuracy heatmap, including very small datasets and models.}
  \label{tab:mean-heatmaps}
\end{table}
\begin{table}[htbp]
  \centering
  \begin{tabular}{l c P{0.155\textwidth} P{0.19\textwidth} P{0.083\textwidth} P{0.084\textwidth} P{0.175\textwidth}}
    \toprule
     & & MC-Dropout & MC-Dropout LS & EDL & DE & Conflictual DE \\
    \midrule
    \multirow{3}{*}{1\textsuperscript{st} principle} & MNIST & $\mathbf{100\%}$ & $78.33\%$ & $91.67\%$ & $\mathbf{100\%}$ & $\mathbf{100\%}$ \\
                                                     & SVHN & $91.67\%$ & $55\%$ & $88.33\%$ & $88.33\%$ & $\mathbf{96.67\%}$ \\
                                                     & CIFAR10 & $76.67\%$ & $38.33\%$ & $91.67\%$ & $41.67\%$ & $\mathbf{98.33\%}$ \\    
    \midrule
    \multirow{3}{*}{2\textsuperscript{nd} principle} & MNIST & $0\%$ & $0\%$ & $5.77\%$ & $0\%$ & $\mathbf{96.15\%}$ \\
                                                     & SVHN & $0\%$ & $0\%$ & $15.38\%$ & $17.30\%$ & $\mathbf{98.08\%}$ \\
                                                     & CIFAR10 & $11.54\%$ & $0\%$ & $5.77\%$ & $15.38\%$ & $\mathbf{100\%}$ \\
    \bottomrule
  \end{tabular}
  \caption{Aggregation of the heatmaps in Fig.~\ref{fig:epistemic-reg} showing to which degree the two principles are satisfied. The percentages for the first (resp. second) principle represent the frequencies with which jumping to the next larger dataset (resp. to the next smaller model) results in a decrease of mutual information.}
  \label{tab:principles-compliance}
\end{table}

Figure~\ref{fig:epistemic-reg} represents by heatmaps the average of mutual information estimated on the test set as a function of data and model sizes, while table~\ref{tab:principles-compliance} summarizes these heatmaps, indicating the frequencies with which methods comply with the two principles.
As aforementioned, the mutual information validates the two principles in expectation. Therefore, we report the average of the test set as a Monte Carlo estimate of this expectation.
We see that all methods but \emph{Conflictual DE} are in complete contradiction with the \emph{model-related principle}: the larger the model, the smaller the epistemic uncertainty when precisely the opposite is expected.
The \emph{data-related principle} seems better respected, but not perfectly, especially on CIFAR10. In line with comments of Sect.~\ref{sec:conflictual-loss}, \emph{MC-Dropout LS} is the method that most often breaks the first principle, even if none of these methods achieves a perfect score.
Although the exact cause of these phenomena remains unknown, the partial violation of the first principle allows us to circumscribe the origin of the problem. According to property~\ref{th:first-principle-in-expectation}, we know mutual information necessarily satisfies the data-related property in expectation when estimated on the exact posterior distribution. Since the violation of this principle is not related to the metric, the only possible cause is a poor quality of the posterior approximation, due to \emph{procedural variability}, \emph{i.e.}, convergence randomness during stochastic gradient descent.

In comparison, \emph{Conflictual DE} obtains almost perfect scores for both principles, proving the conflictual loss's regularizing effect.
While the method was specifically designed to satisfy the first principle, compliance with the second principle was not expected at such level. It seems that the expressive power of networks serves as an amplifier for the discord sown between classifiers by the conflictual loss.

\begin{figure}[ht]
  \centering
  \begin{subfigure}[t]{\dimexpr0.185\textwidth+20pt\relax}
    \makebox[20pt]{\raisebox{25pt}{\rotatebox[origin=c]{90}{\scriptsize MNIST}}}%
    \includegraphics[width=\dimexpr\linewidth-20pt\relax]{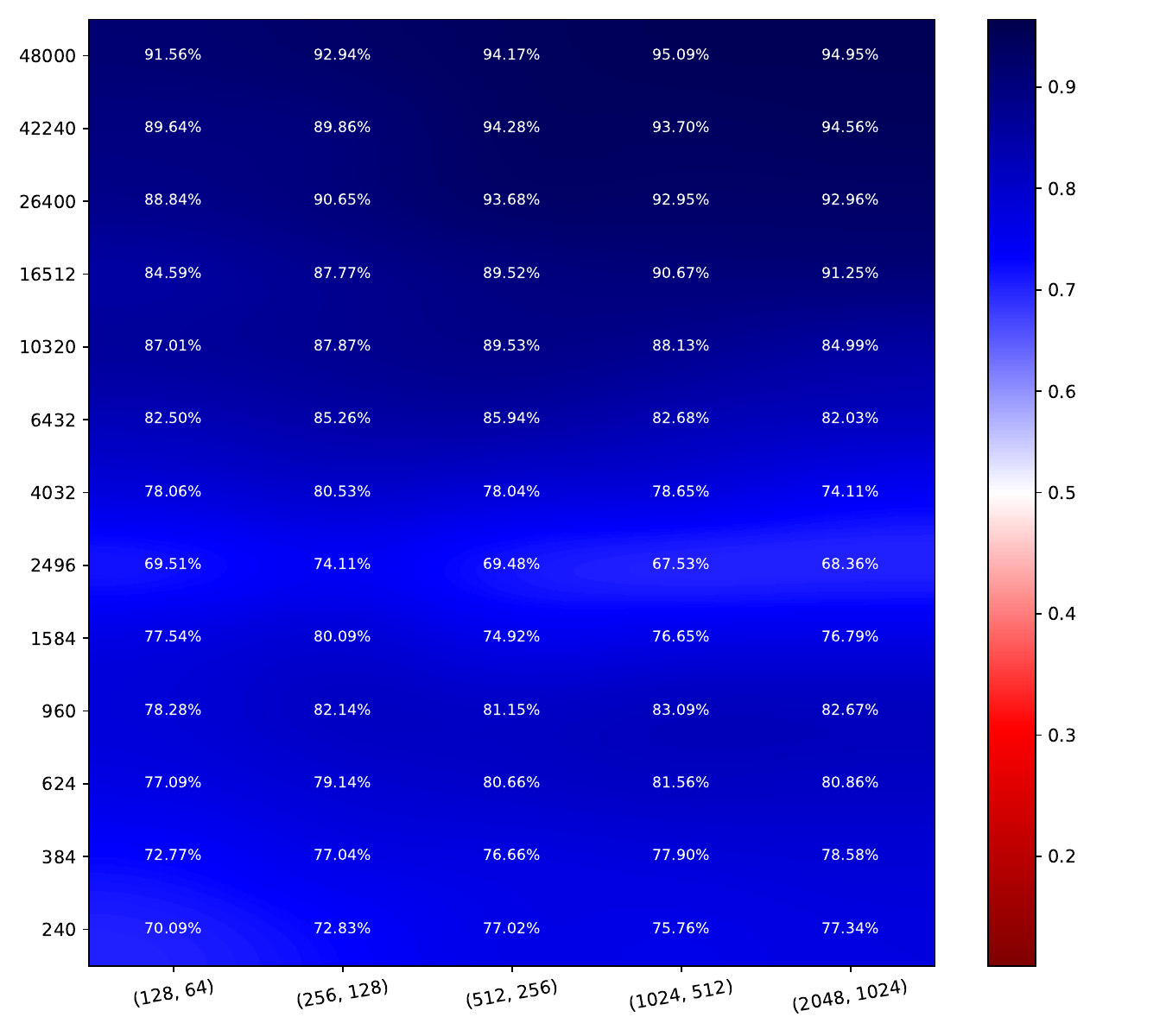}
    \makebox[20pt]{\raisebox{25pt}{\rotatebox[origin=c]{90}{\scriptsize SVHN}}}%
    \includegraphics[width=\dimexpr\linewidth-20pt\relax]{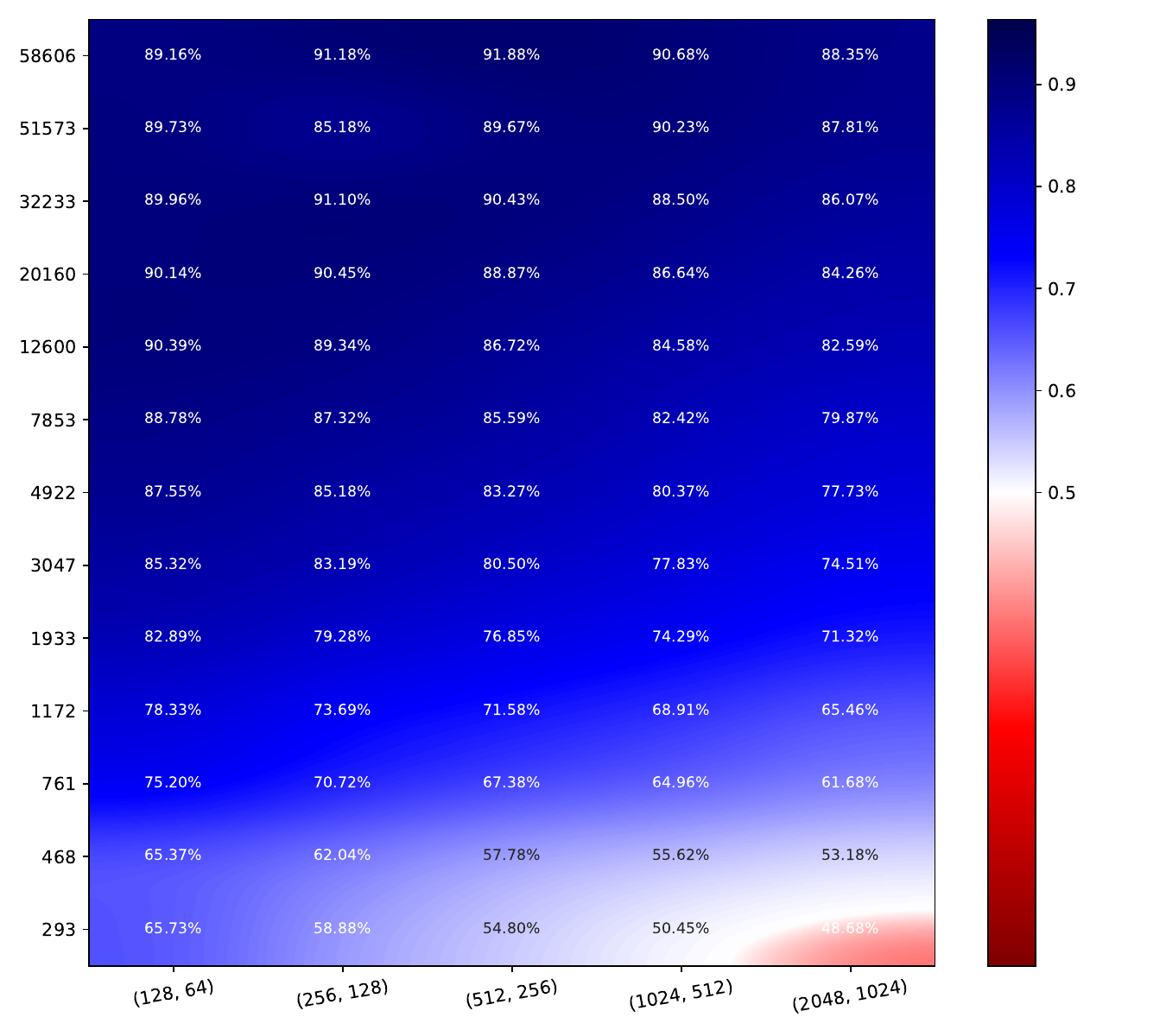}    \makebox[20pt]{\raisebox{25pt}{\rotatebox[origin=c]{90}{\scriptsize CIFAR10}}}%
    \includegraphics[width=\dimexpr\linewidth-20pt\relax]{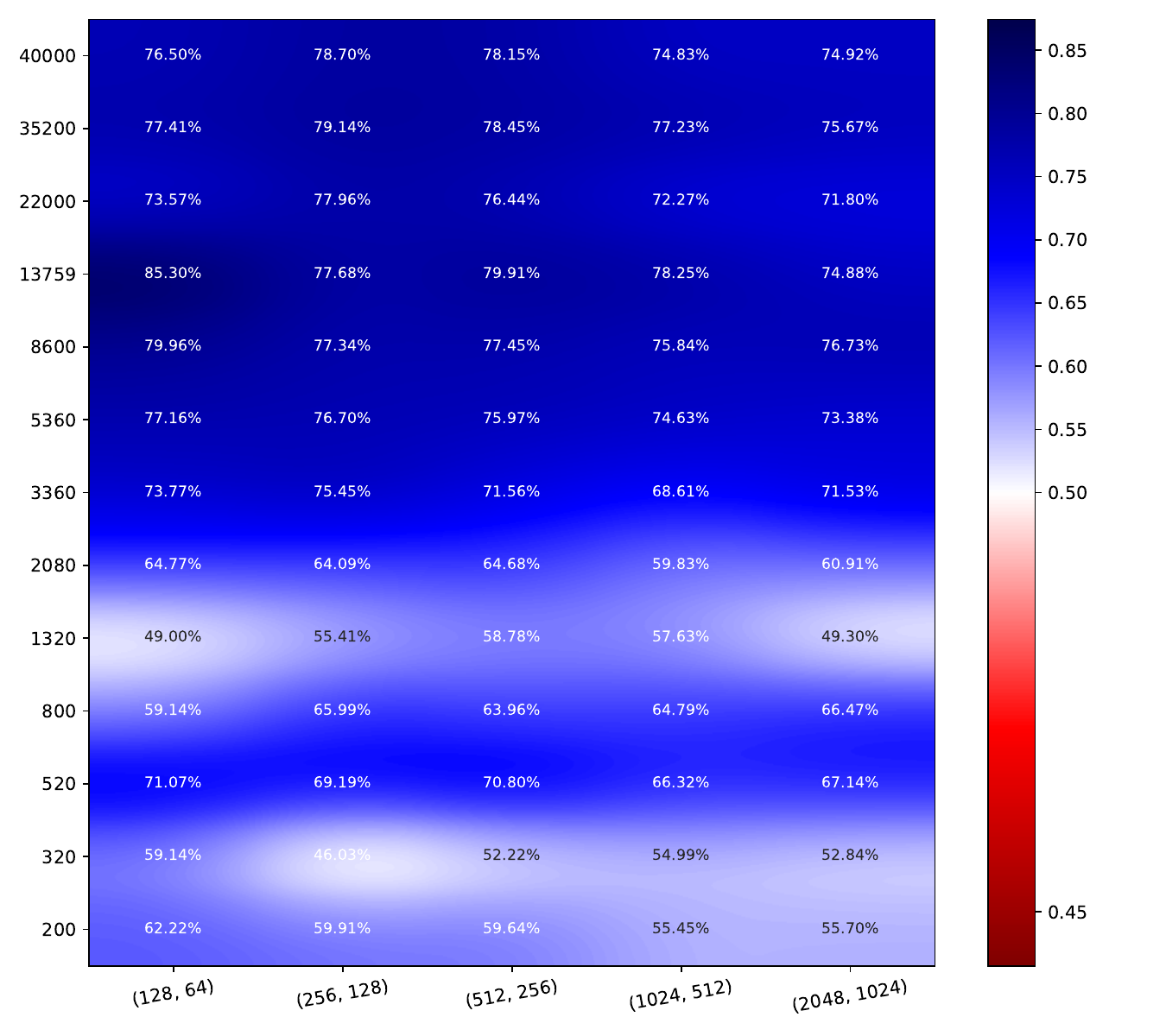}
    \caption*{\qquad MC-Dropout}
  \end{subfigure}\hfill
  \begin{subfigure}[t]{0.185\textwidth}
    \includegraphics[width=\textwidth]{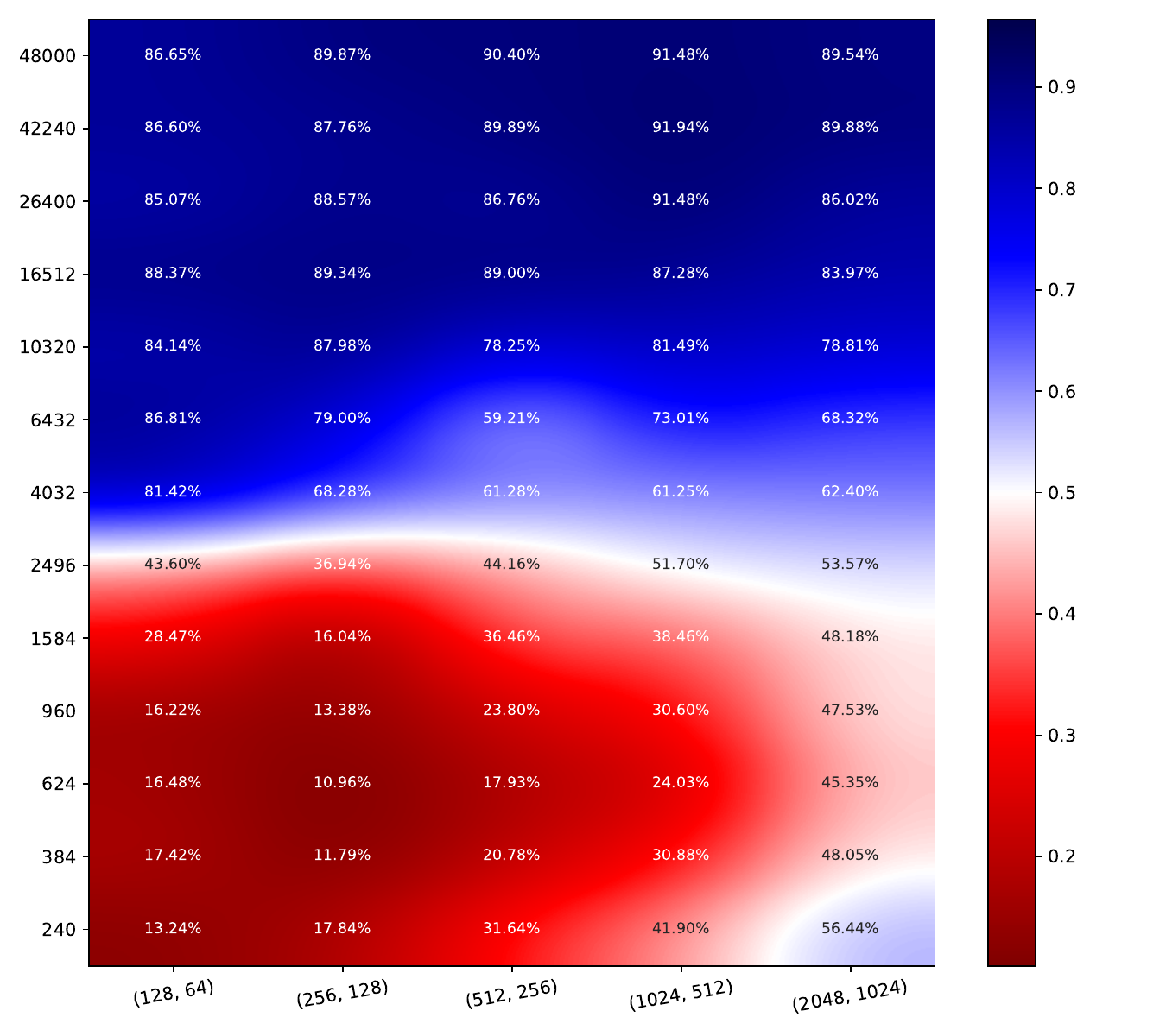}
    \includegraphics[width=\textwidth]{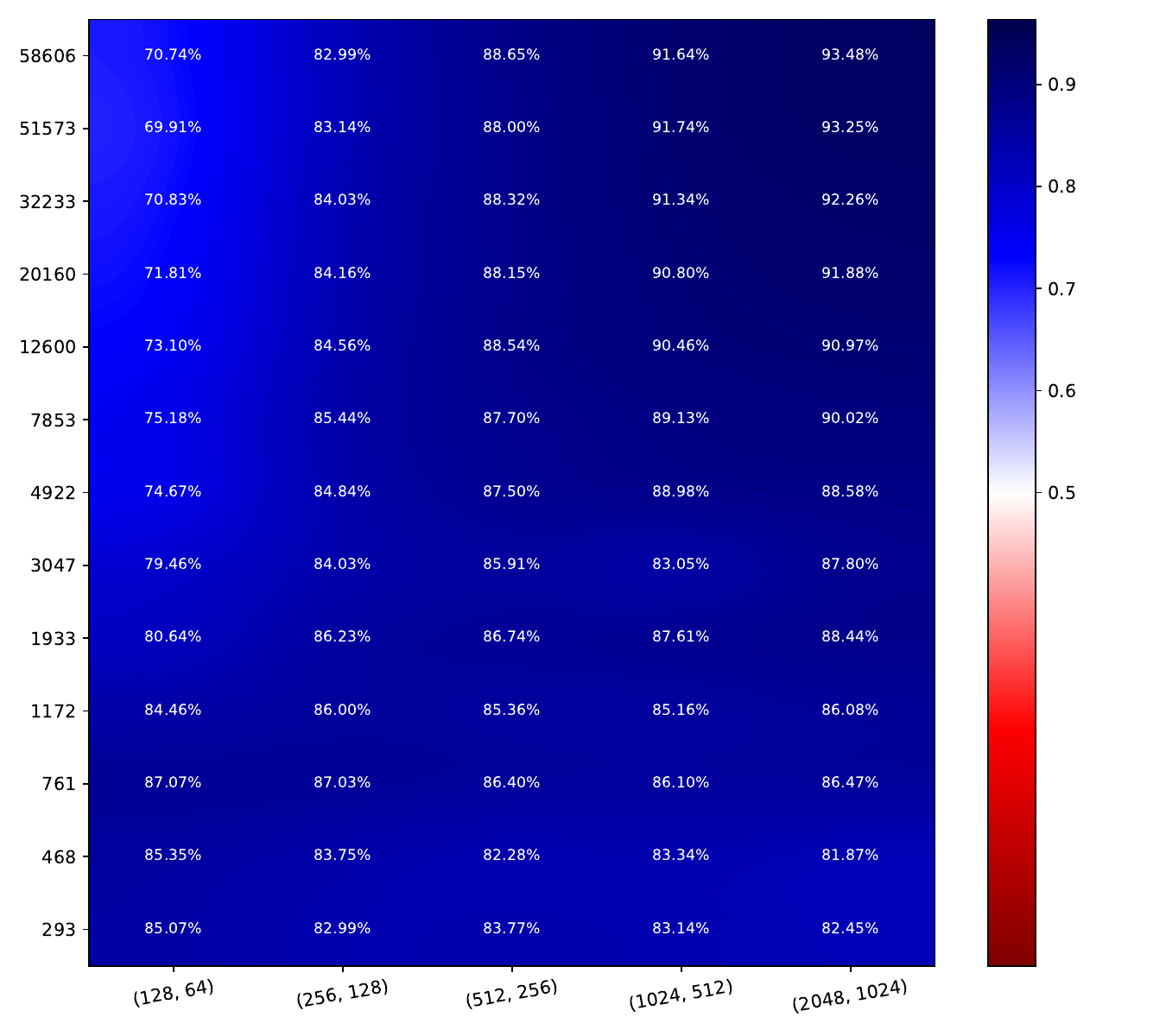}
    \includegraphics[width=\textwidth]{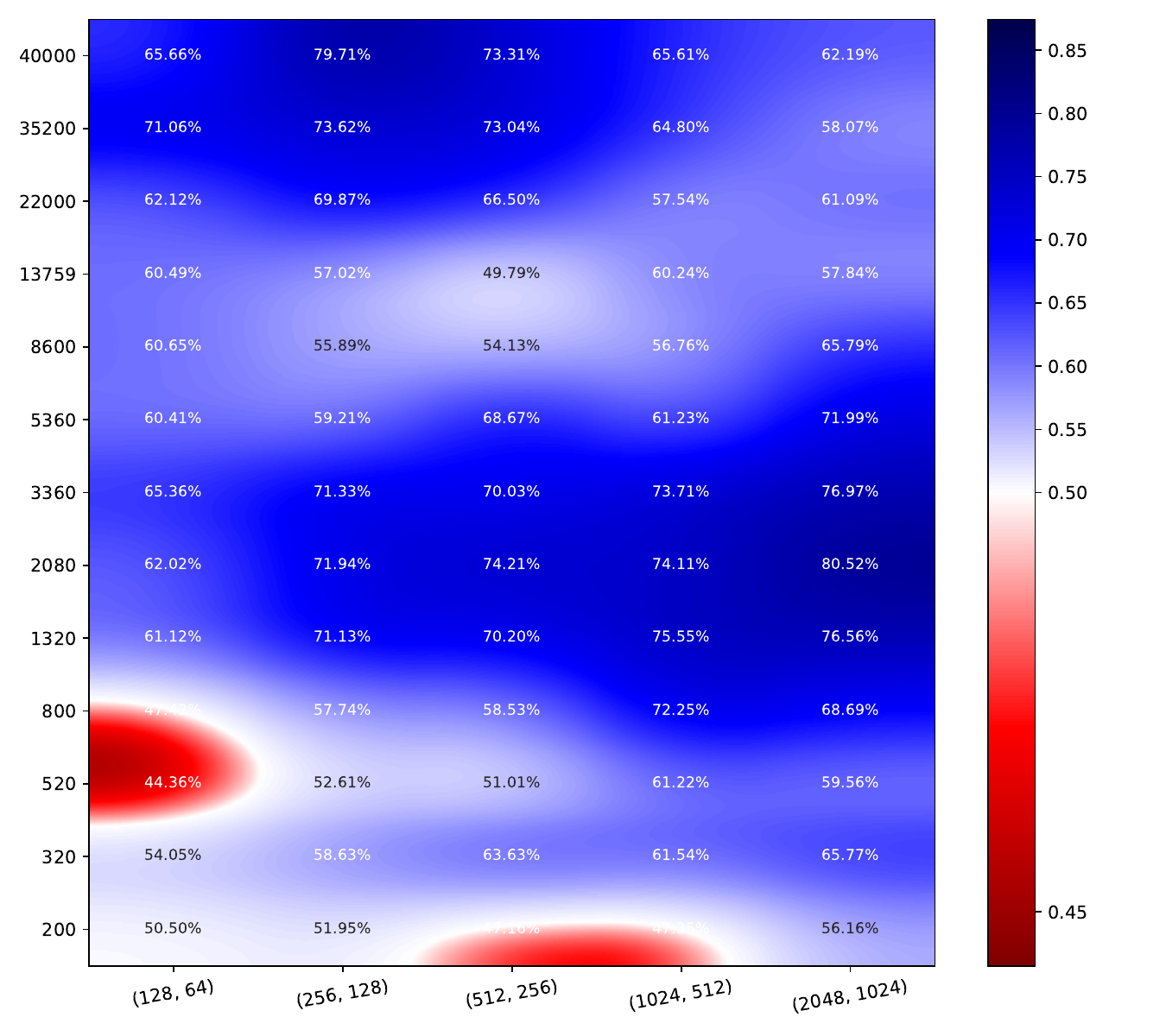}
    \caption*{MC-Dropout LS}
  \end{subfigure}\hfill
  \begin{subfigure}[t]{0.185\textwidth}
    \includegraphics[width=\textwidth]{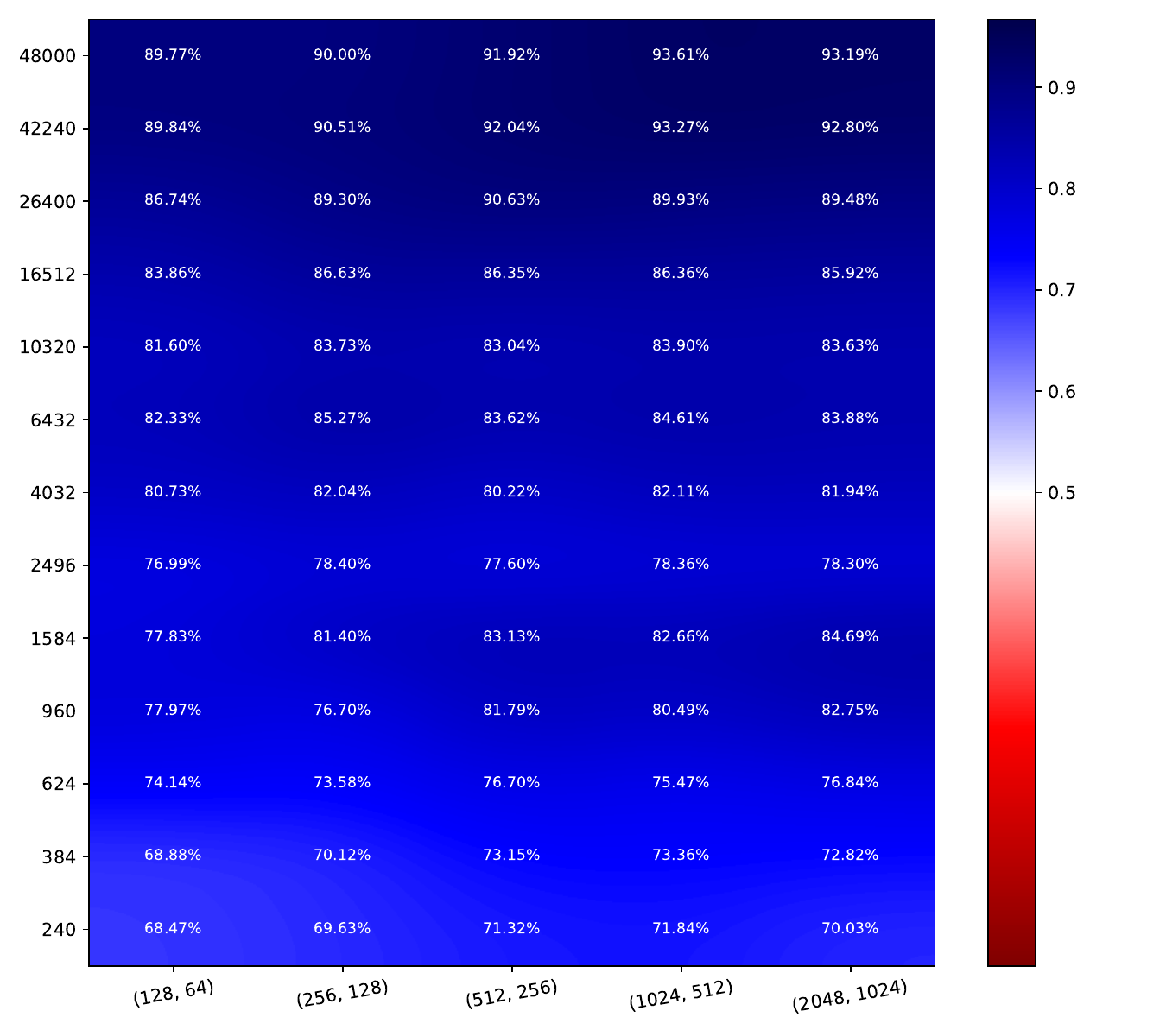}
    \includegraphics[width=\textwidth]{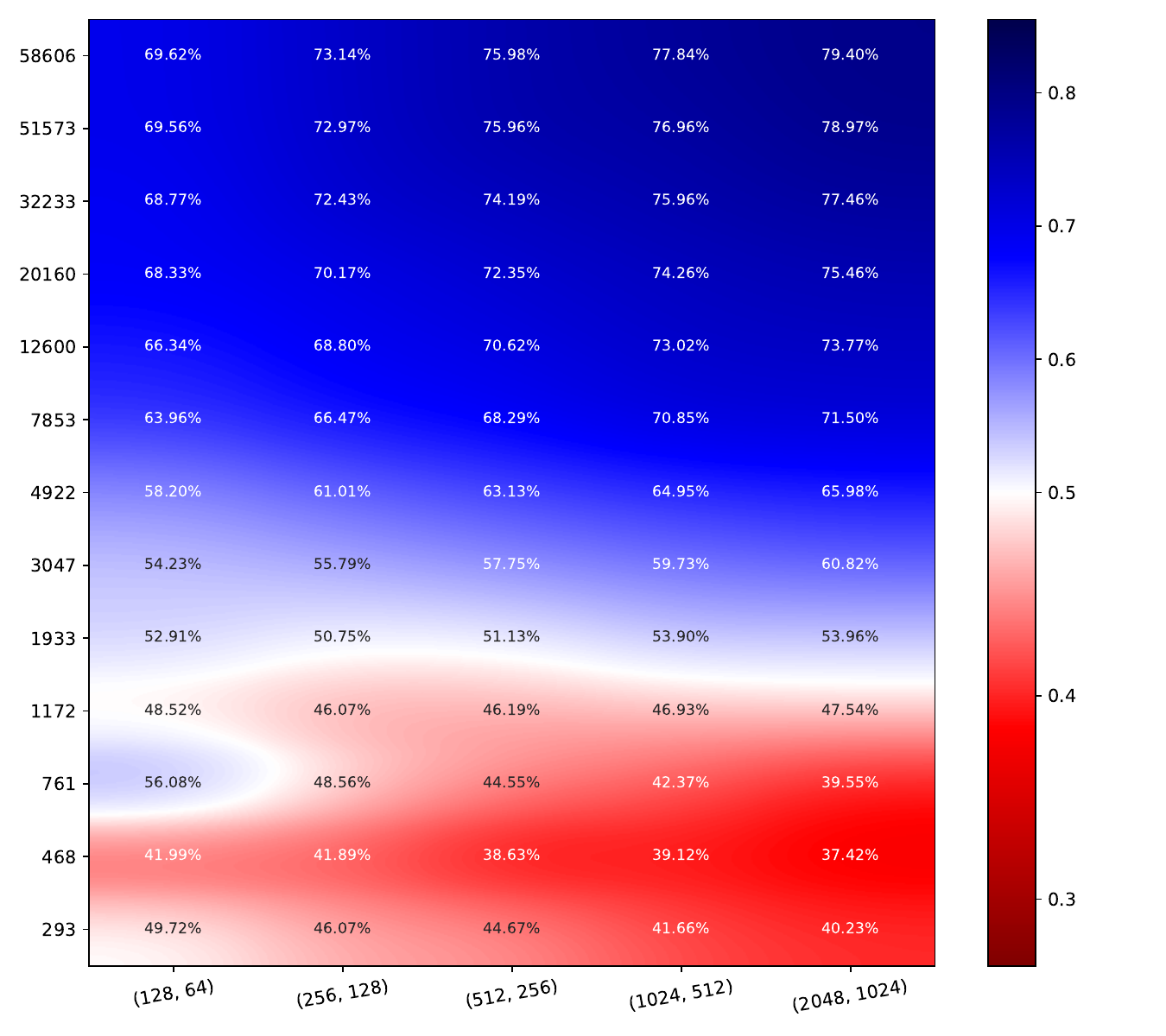}
    \includegraphics[width=\textwidth]{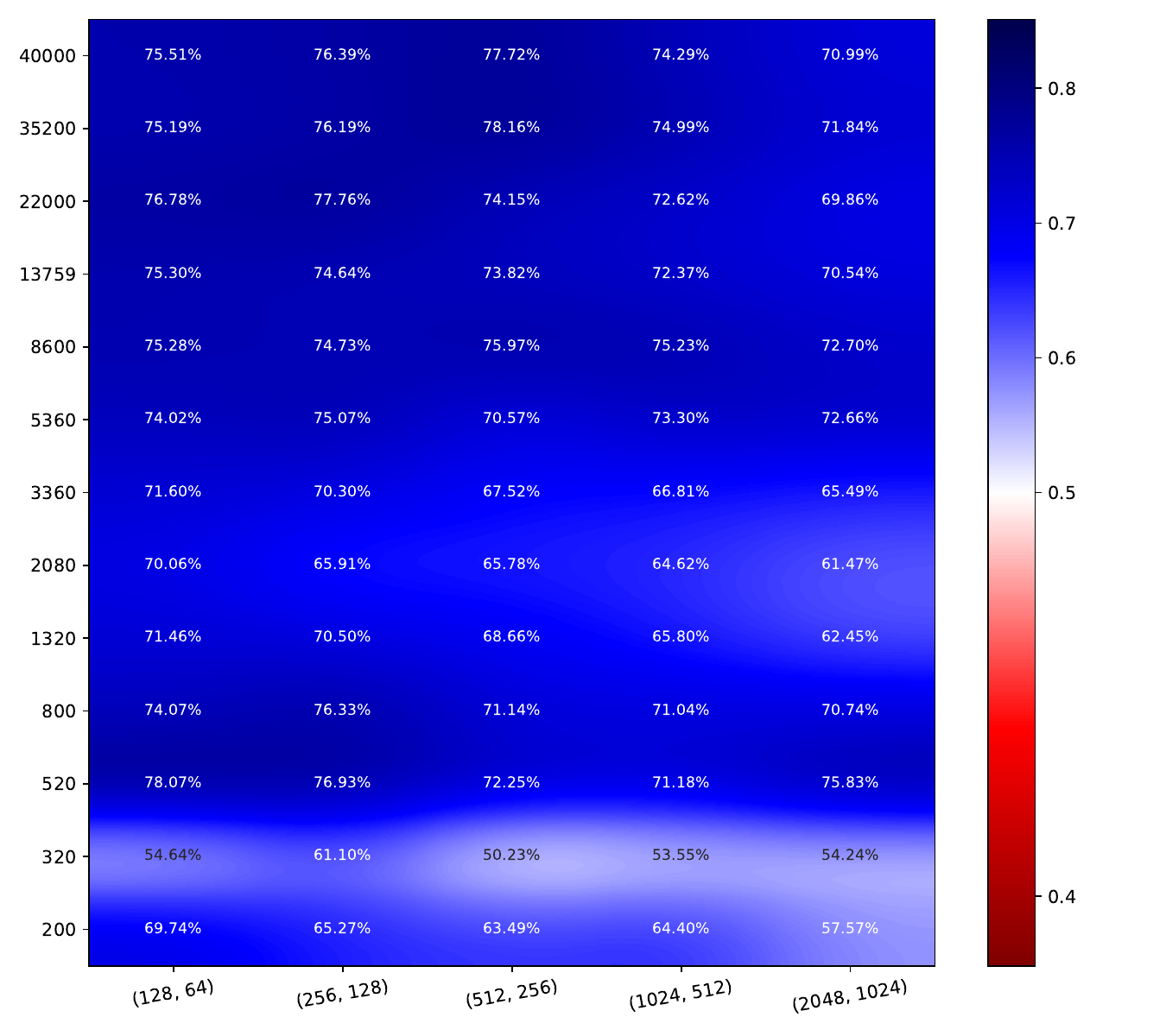}
    \caption*{EDL}
  \end{subfigure}\hfill
  \begin{subfigure}[t]{0.185\textwidth}
    \includegraphics[width=\textwidth]{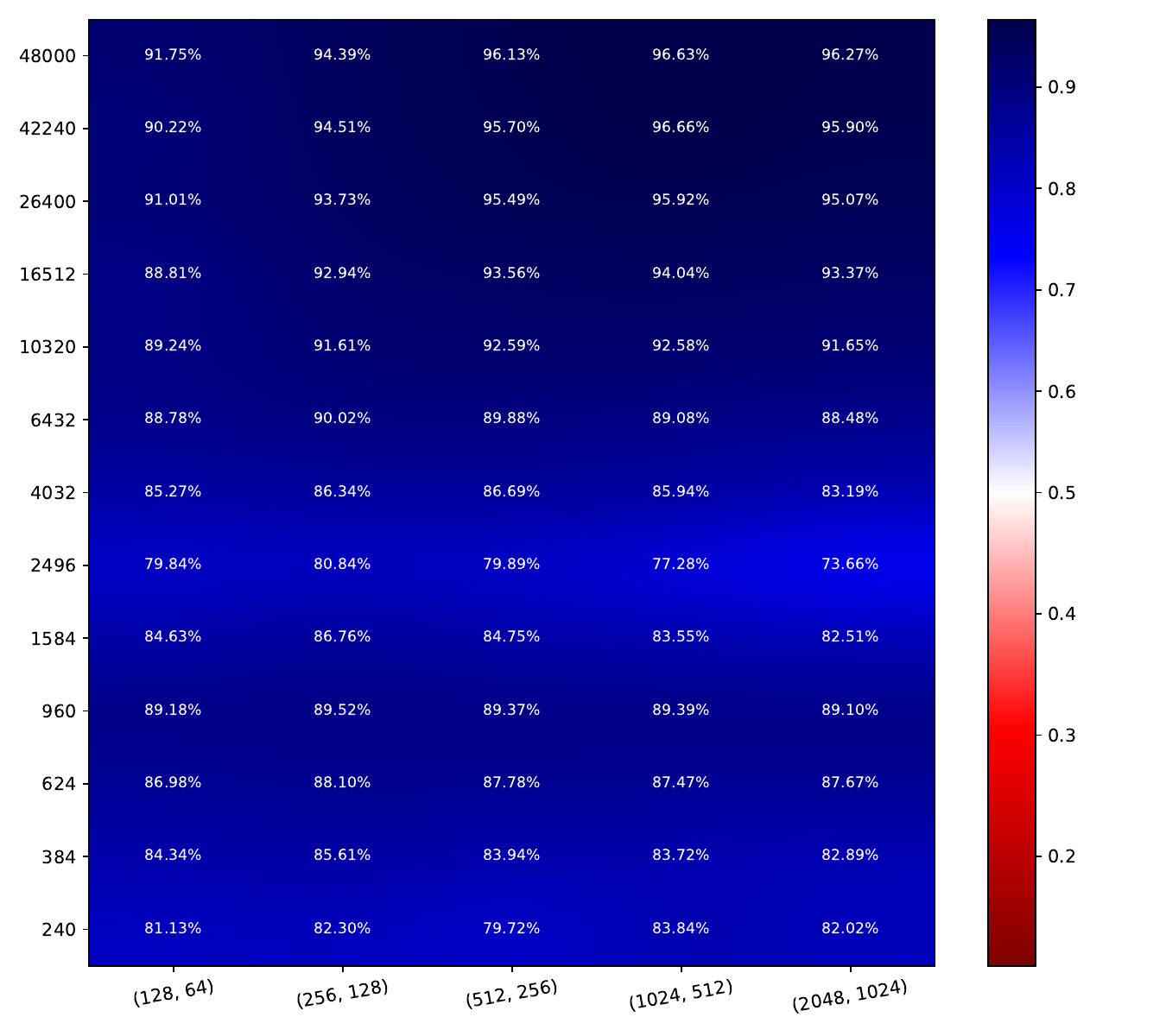}
    \includegraphics[width=\textwidth]{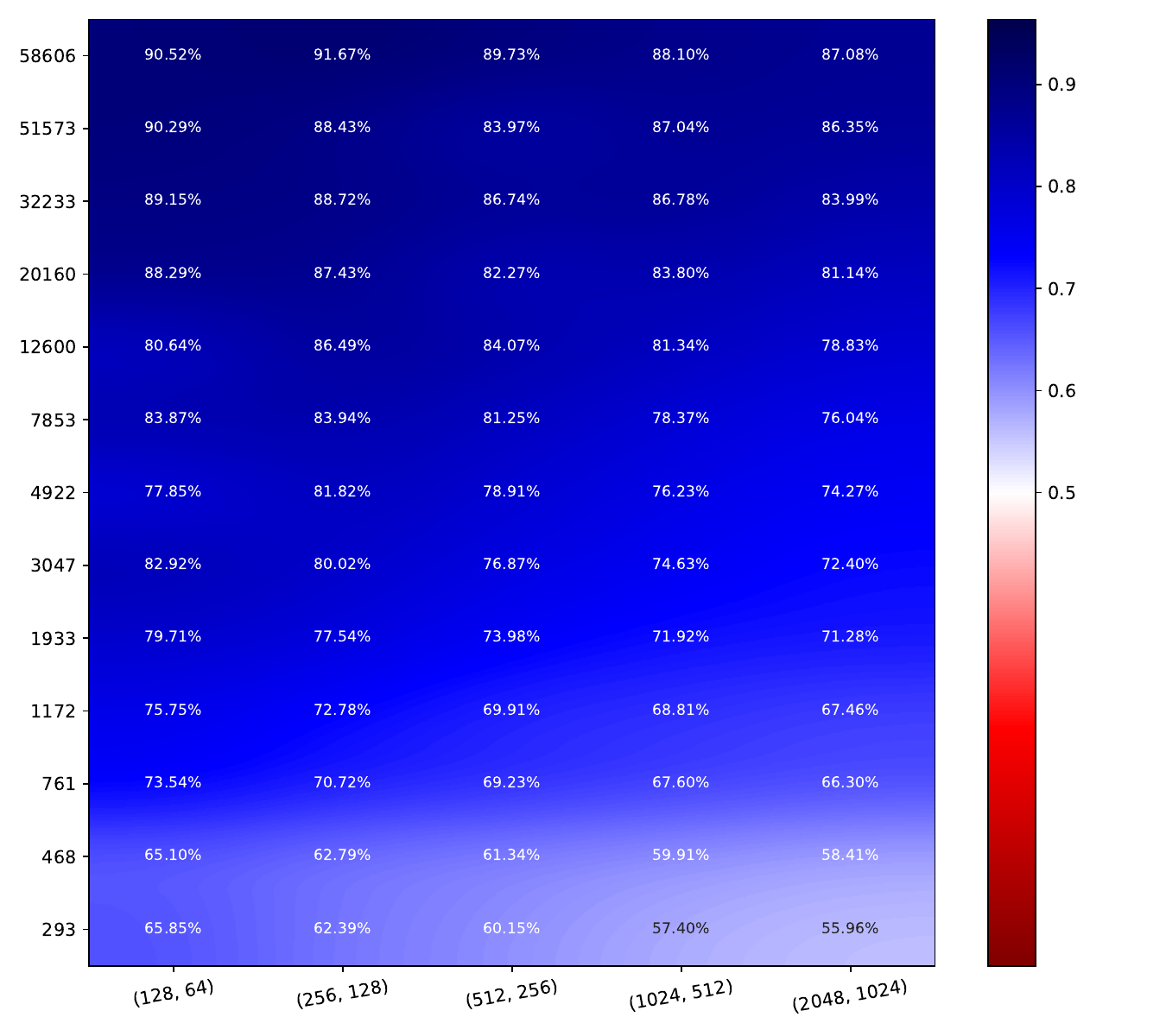}
    \includegraphics[width=\textwidth]{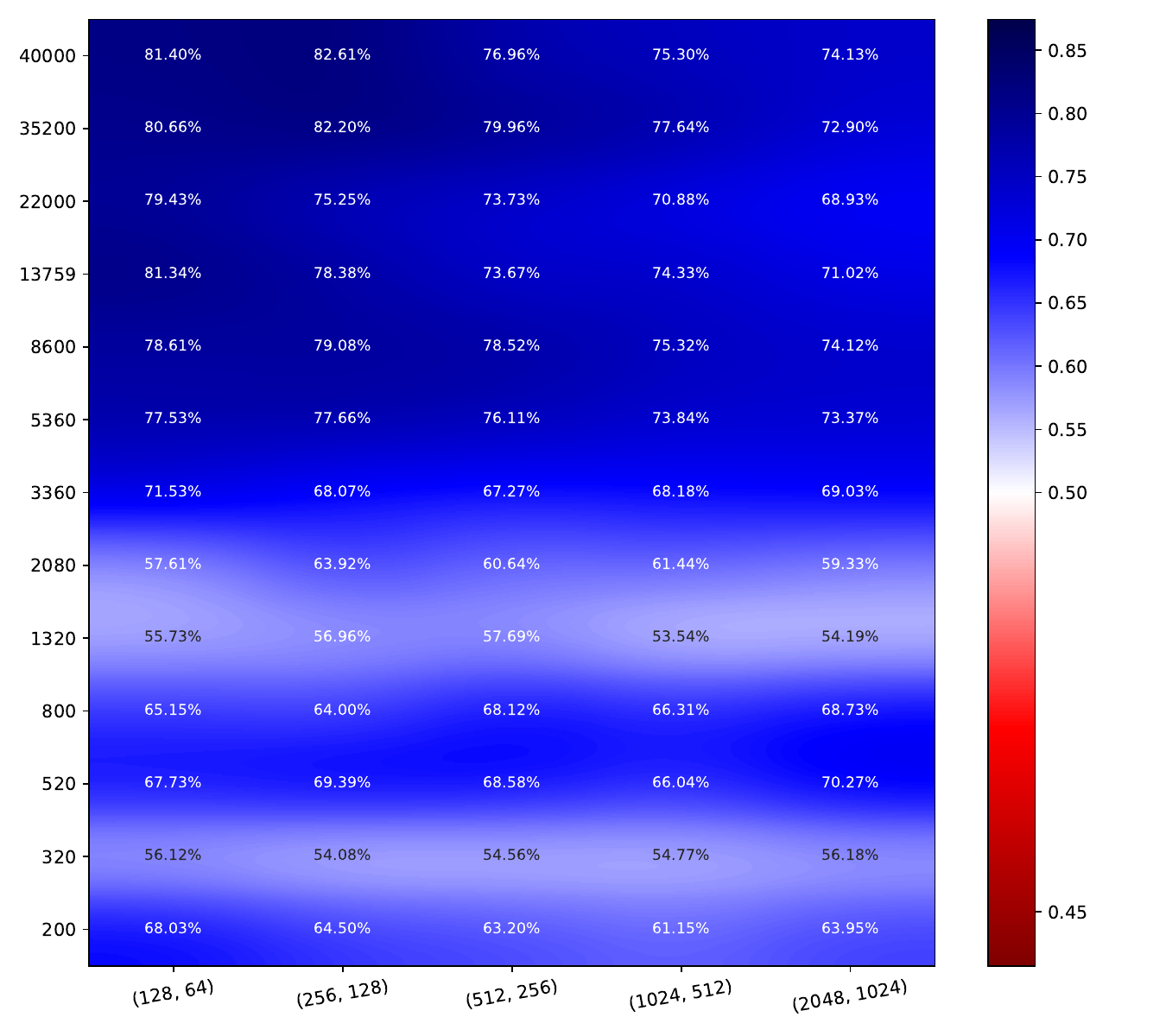}
    \caption*{DE}
  \end{subfigure}\hfill
  \begin{subfigure}[t]{0.185\textwidth}
    \includegraphics[width=\textwidth]{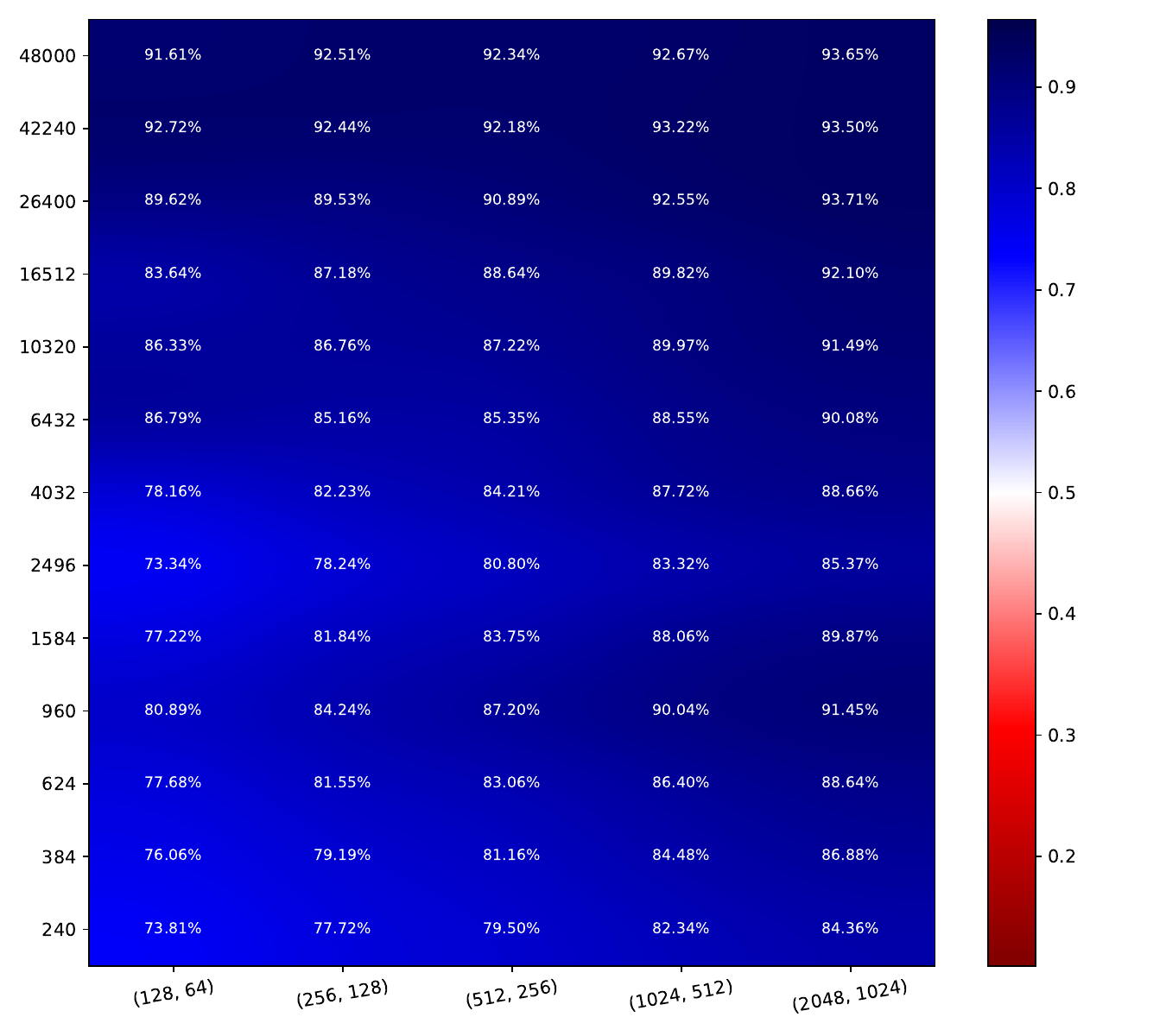}
    \includegraphics[width=\textwidth]{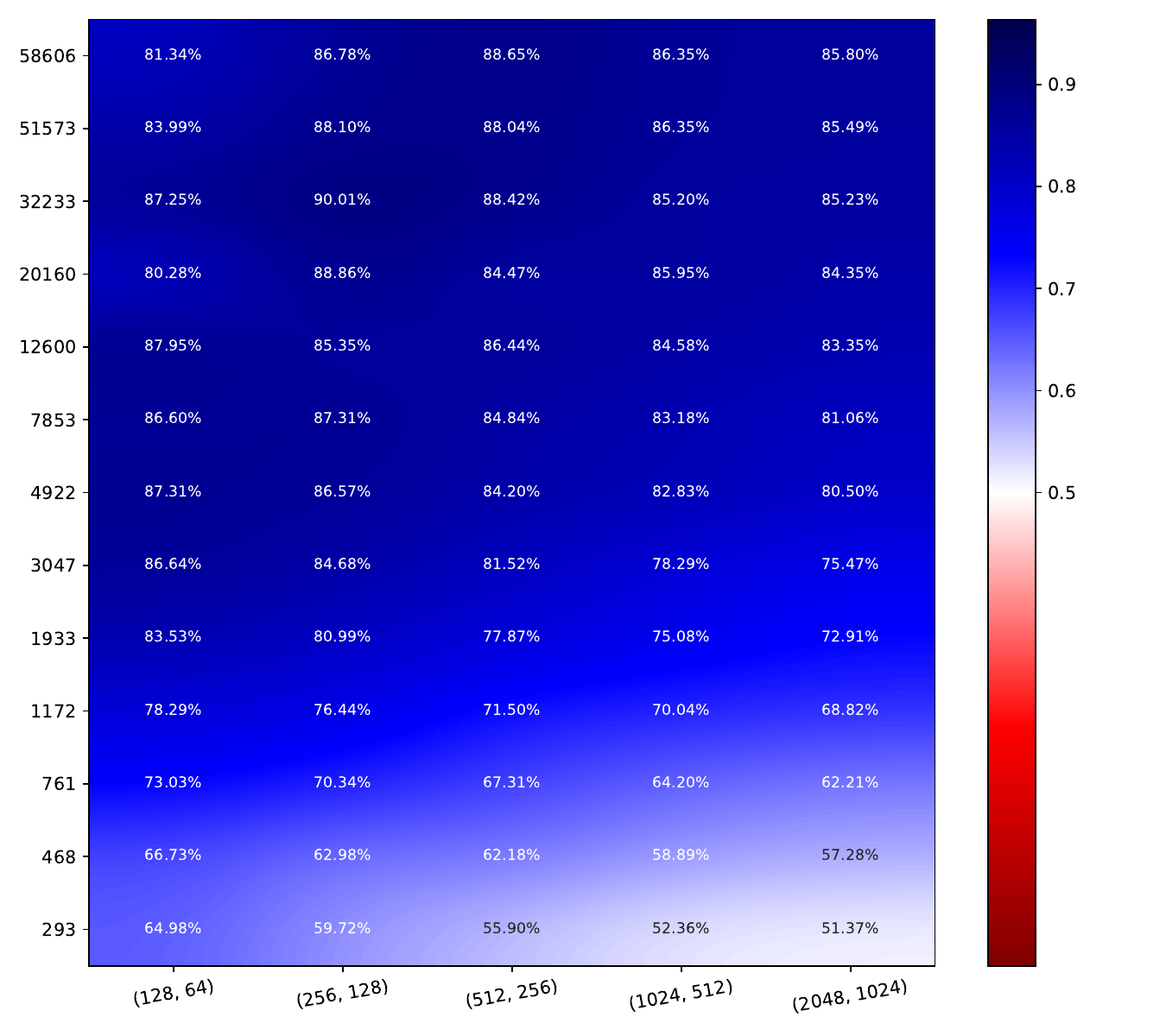}
    \includegraphics[width=\textwidth]{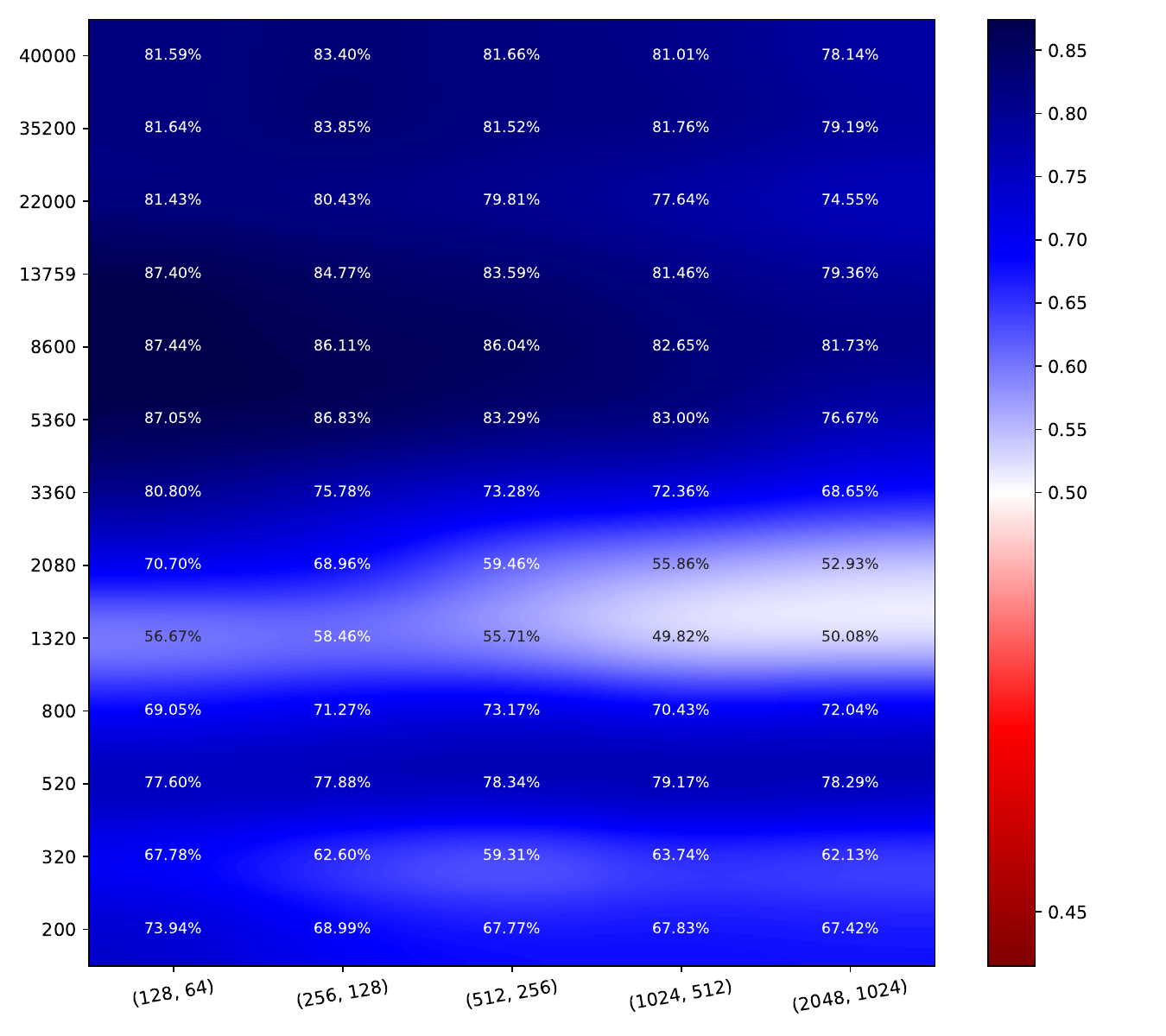}
    \caption*{Conflictual DE}
  \end{subfigure}\hfill
  \caption{Heatmaps of AUROC for OOD detection based on epistemic uncertainty. Same representation as Fig.~\ref{fig:epistemic-reg} but with the same color scale per dataset.}
  \label{fig:auc-ood}
\end{figure}

As shown in Appendix~\ref{app:additional-results} and summarized in Tab.~\ref{tab:mean-heatmaps}, the calibrated epistemic uncertainty resulting from Conflictual DE does not come at the expense of the performance of the model. In fact, Conflictual DE yields comparable or even superior performance compared to DE, as it performs the best overall on CIFAR10 in terms of accuracy (Fig.~\ref{fig:accuracy}) and has comparable results to the best method on MNIST. 
When taking into account the accuracy of the output probabilities with the Brier score, Conflictual DE appears to be the best model on both datasets.

We also checked the quality of epistemic uncertainty to discriminate between OOD and in-distribution (ID) examples, since epistemic uncertainty is expected to be higher for OOD samples than for the ID samples.
We use FashionMNIST, SVHN, and CIFAR10 as OOD datasets for the models trained on MNIST, CIFAR10, and SVHN respectively.
As shown in Fig.~\ref{fig:auc-ood}, Conflictual DE yields high AUROC overall on CIFAR10 and in the low-data regime on MNIST.
MC-Dropout LS performs the worst on MNIST and CIFAR10 at the task of OOD detection as it sometimes results in lower epistemic uncertainty in OOD samples than ID samples.
To some extent, we notice the same results in the task of misclassification detection which consists in distinguishing between the correctly classified samples and the misclassified samples based on epistemic uncertainty (see Appendix~\ref{app:additional-results}: Fig.~\ref{fig:auc-mis}).

Finally, we take a look at the calibration of the models with different methods. As shown in Fig.~\ref{fig:sce}, the SCE with Conflictual DE is the most consistent and the lowest, yielding calibrated models even with small training sets.

\begin{figure}[htbp]
  \centering
  \begin{subfigure}[t]{\dimexpr0.185\textwidth+20pt\relax}
    \makebox[20pt]{\raisebox{25pt}{\rotatebox[origin=c]{90}{\scriptsize MNIST}}}%
    \includegraphics[width=\dimexpr\linewidth-20pt\relax]{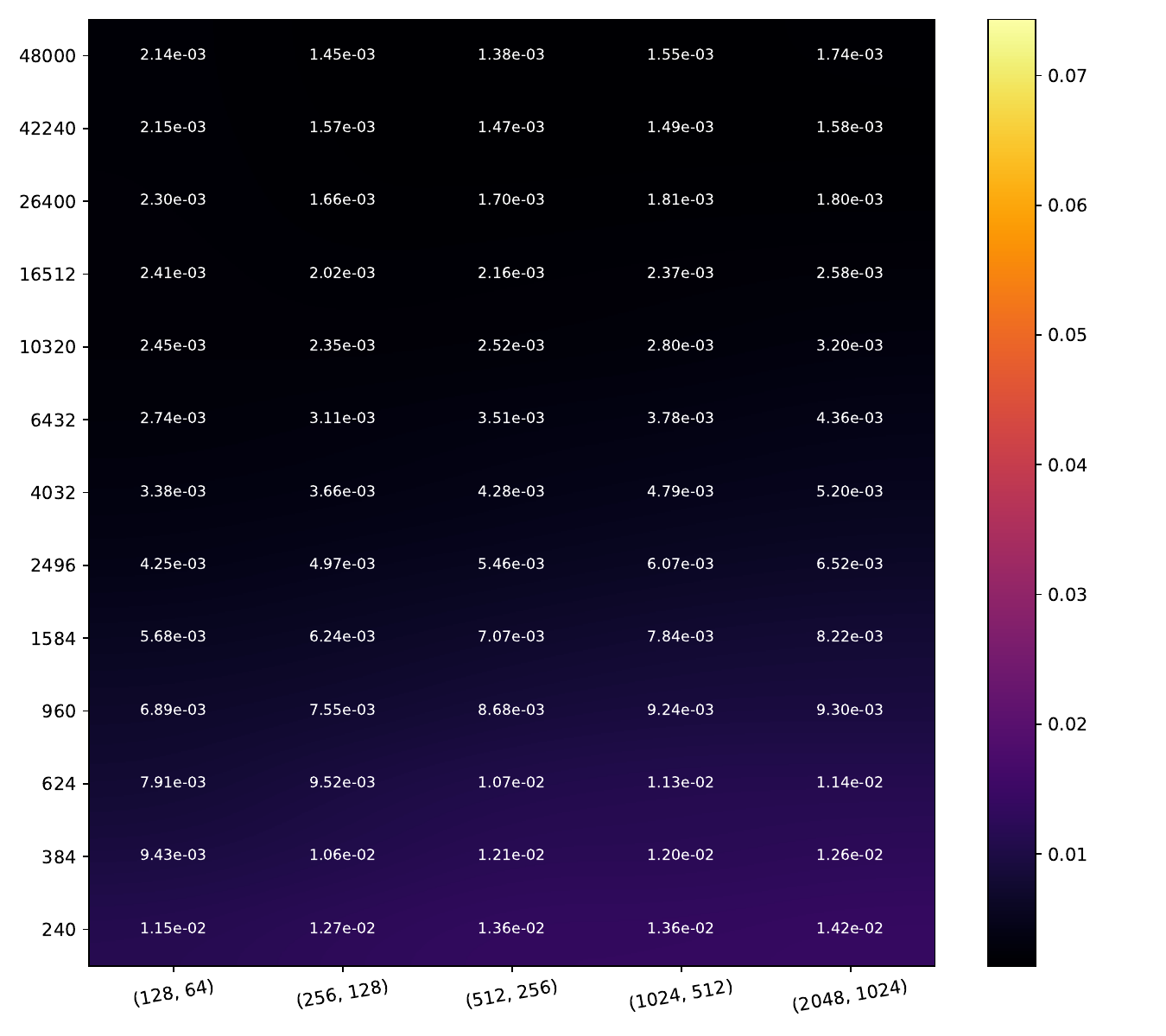}
    \makebox[20pt]{\raisebox{25pt}{\rotatebox[origin=c]{90}{\scriptsize SVHN}}}%
    \includegraphics[width=\dimexpr\linewidth-20pt\relax]{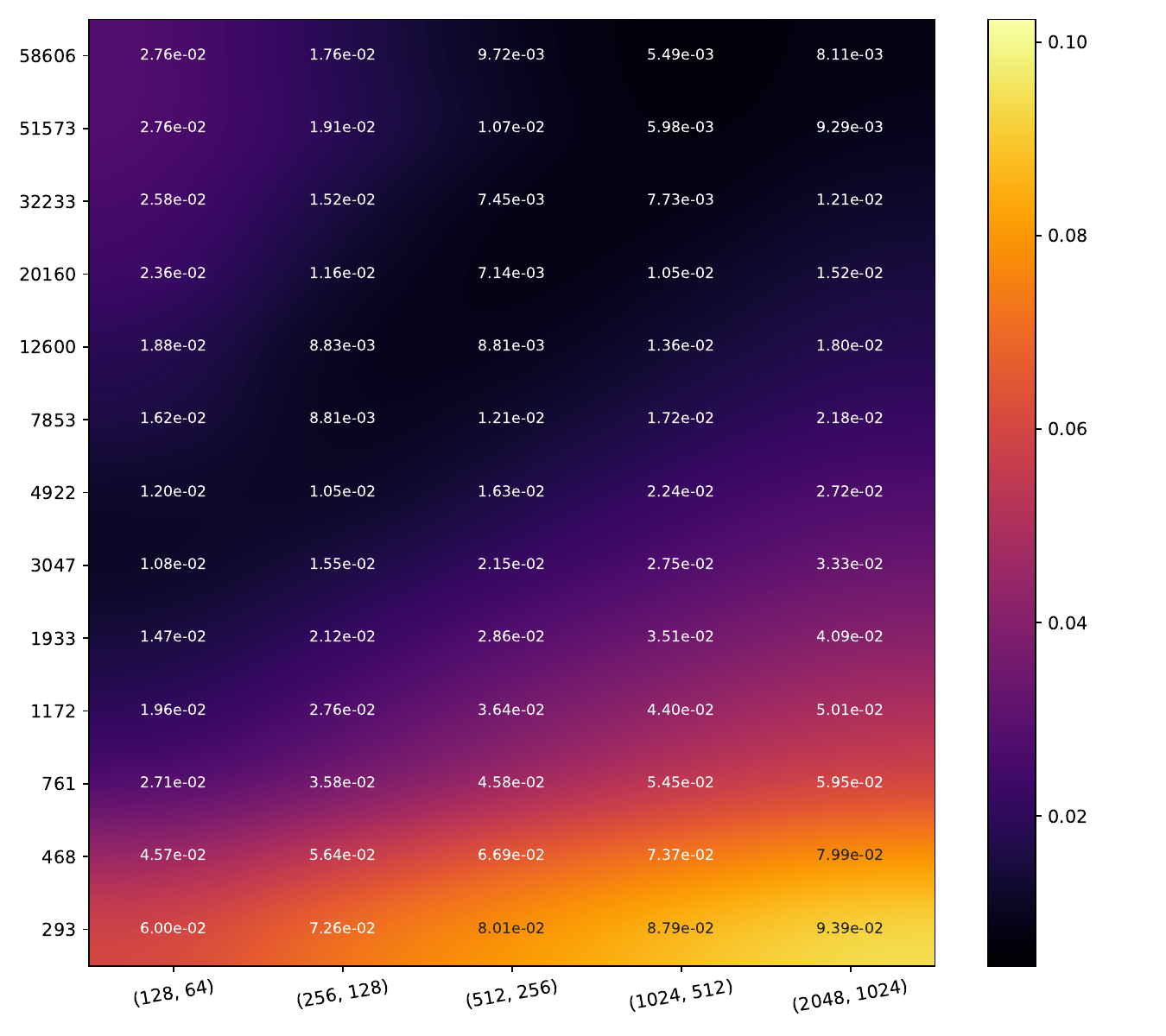}
    \makebox[20pt]{\raisebox{25pt}{\rotatebox[origin=c]{90}{\scriptsize CIFAR10}}}%
    \includegraphics[width=\dimexpr\linewidth-20pt\relax]{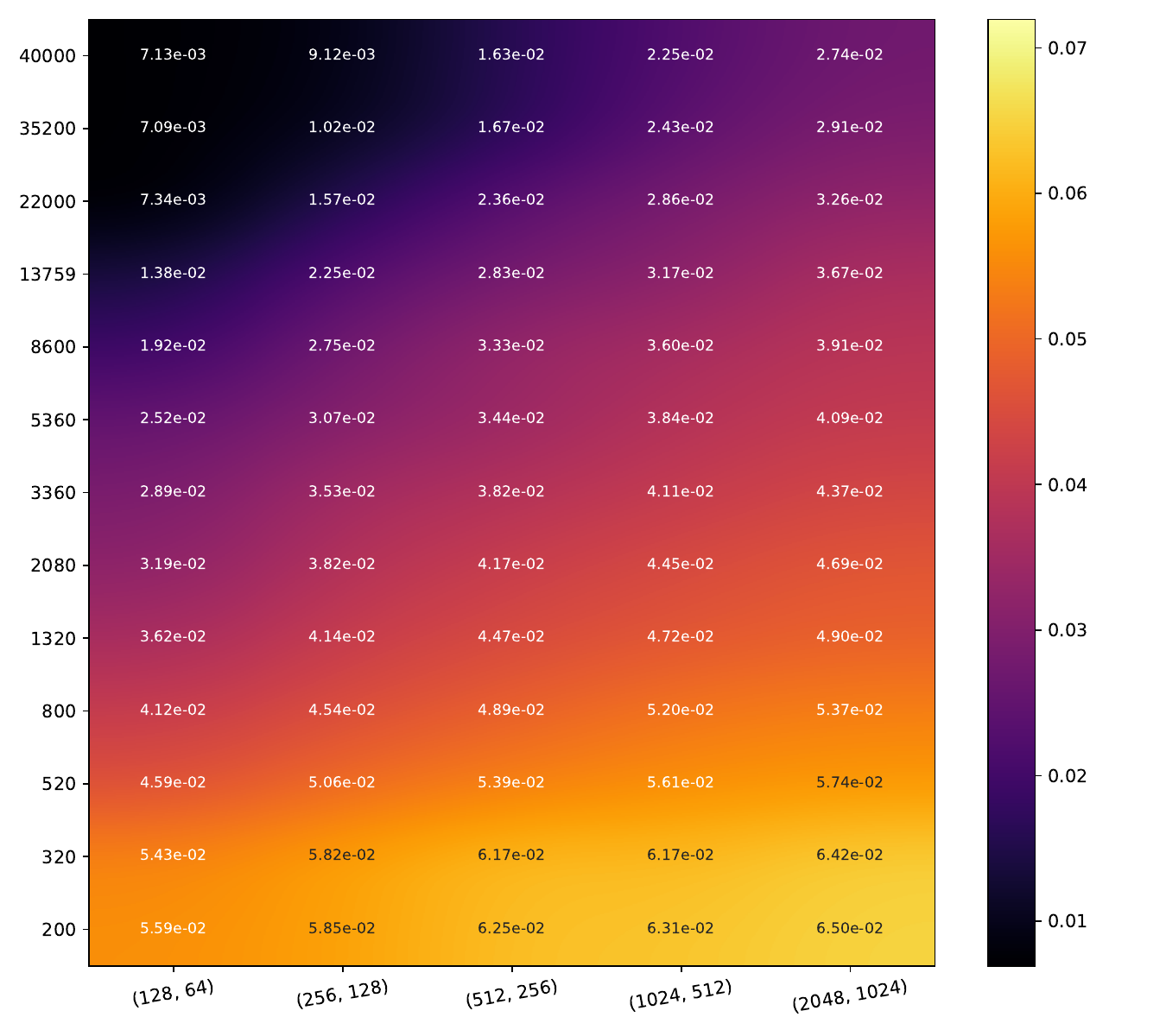}
    \caption*{\qquad MC-Dropout}
  \end{subfigure}\hfill
  \begin{subfigure}[t]{0.185\textwidth}
    \includegraphics[width=\textwidth]{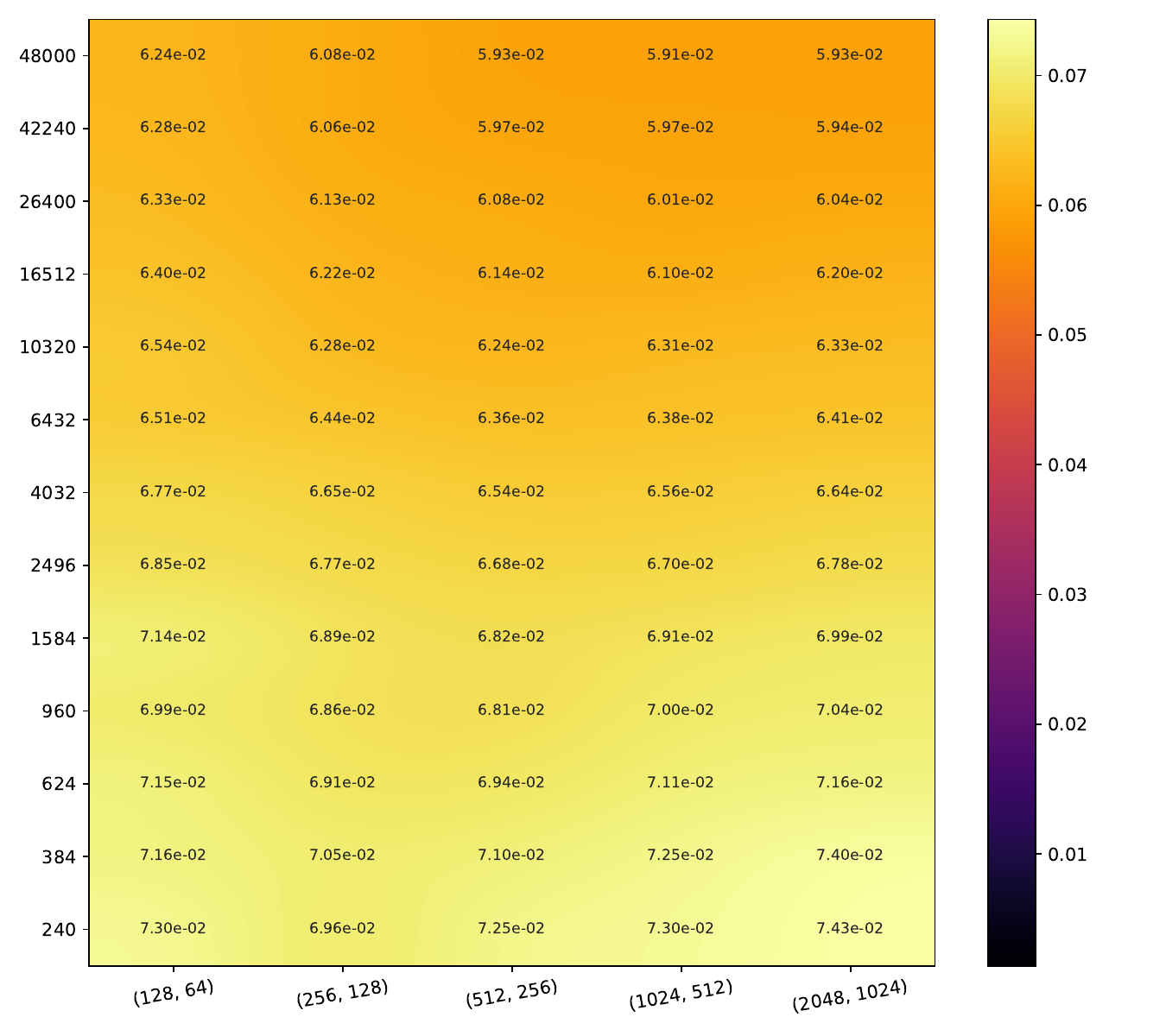}
    \includegraphics[width=\textwidth]{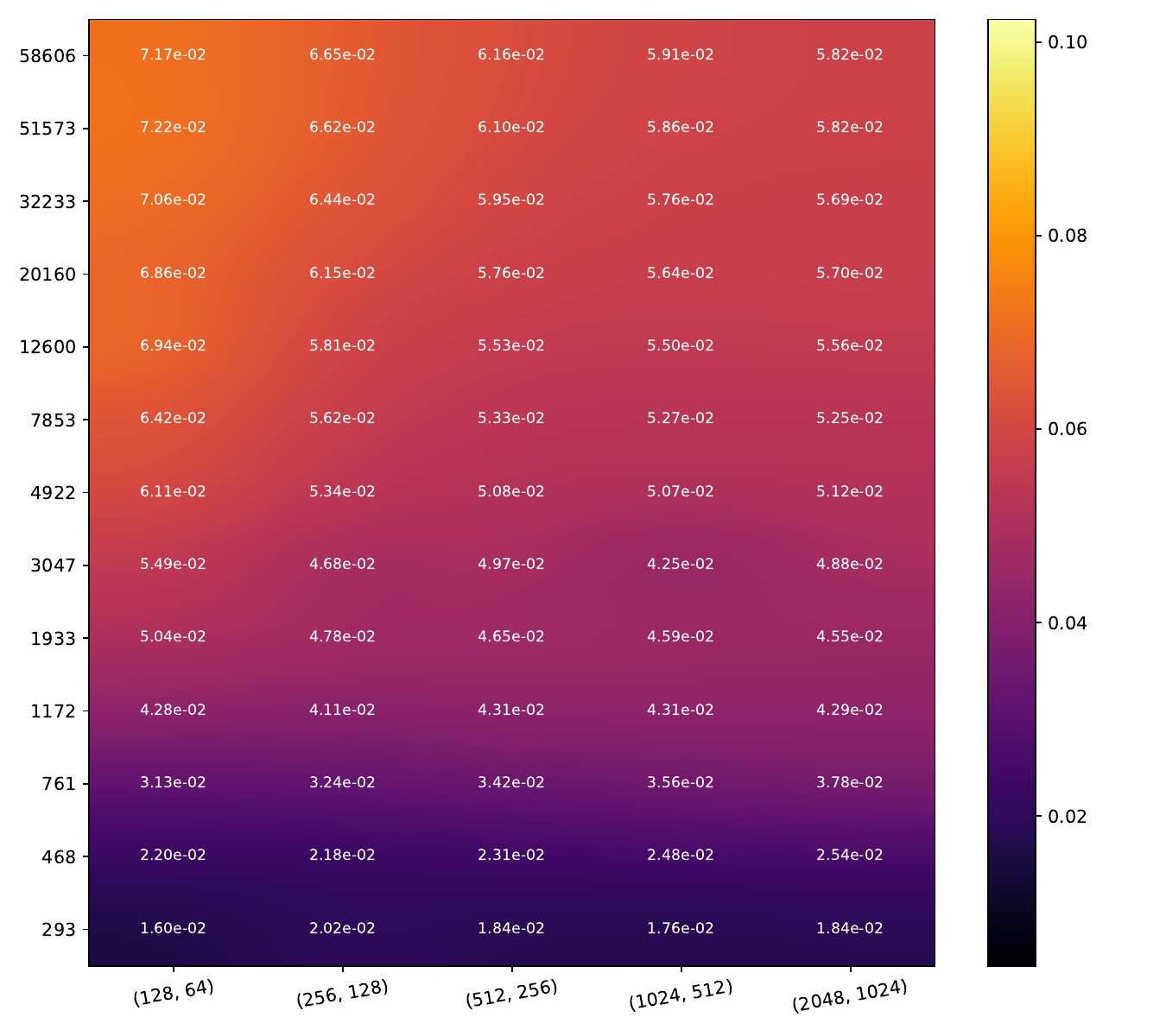}
    \includegraphics[width=\textwidth]{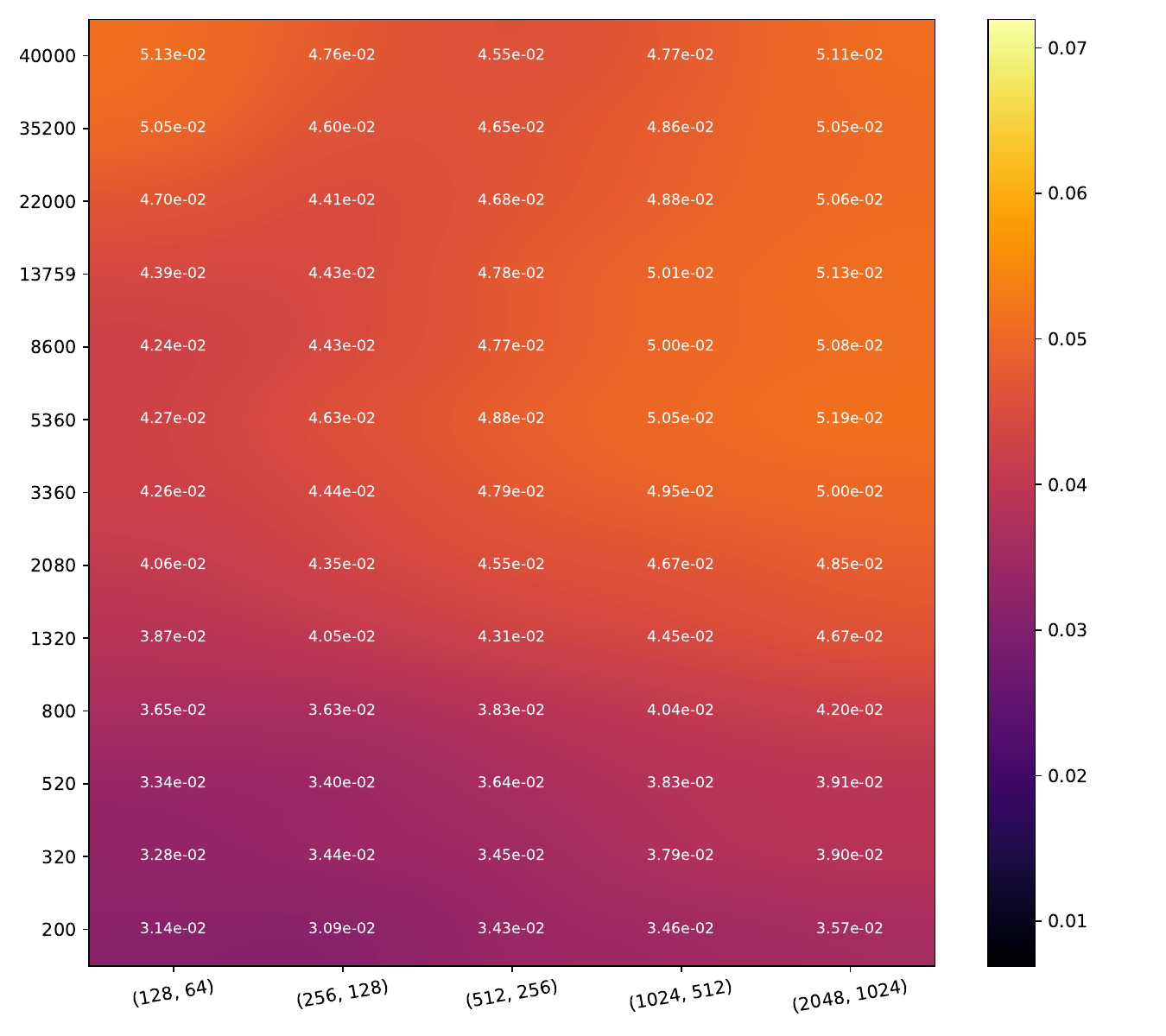}
    \caption*{MC-Dropout LS}
  \end{subfigure}\hfill
  \begin{subfigure}[t]{0.185\textwidth}
    \includegraphics[width=\textwidth]{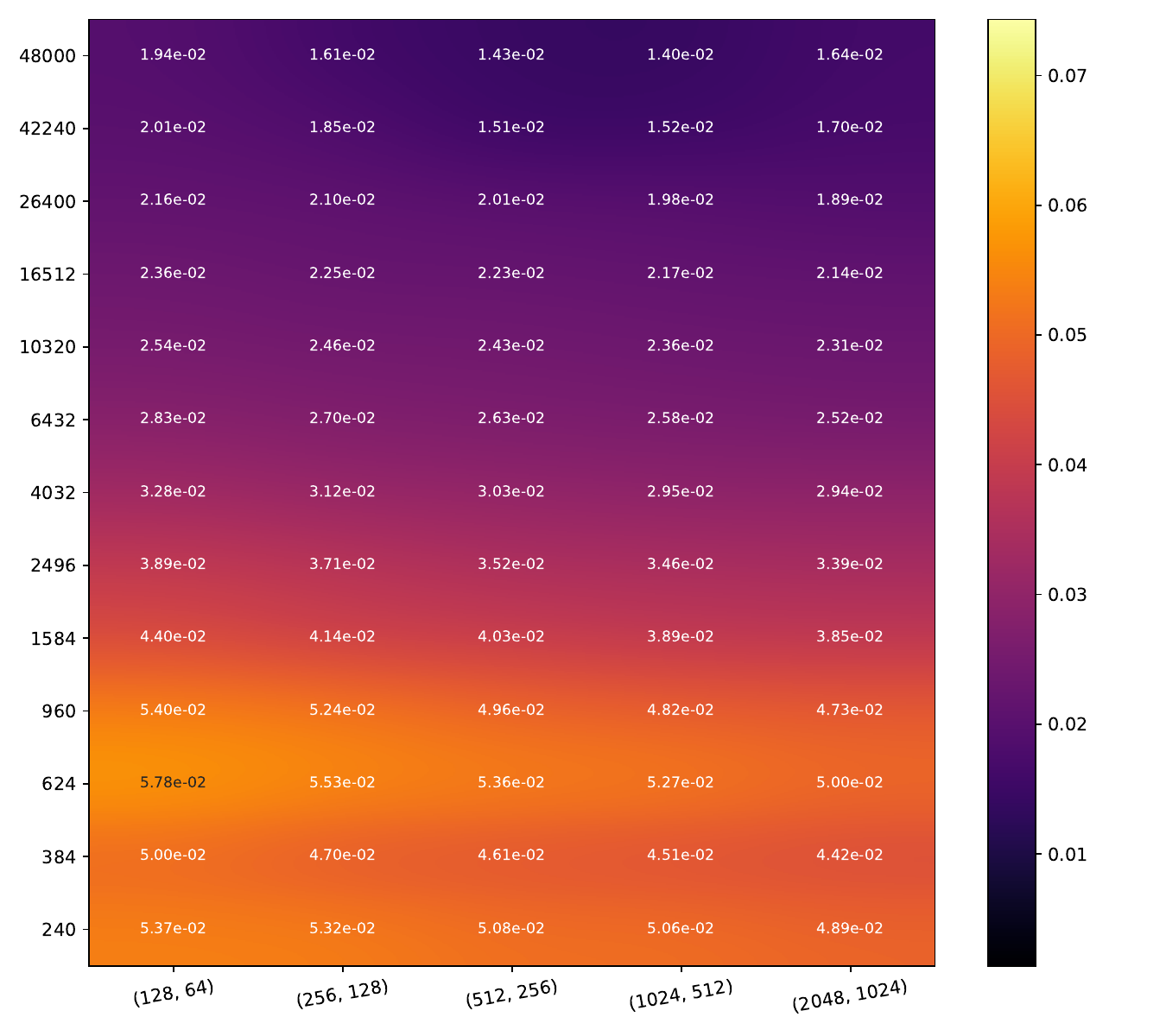}
    \includegraphics[width=\textwidth]{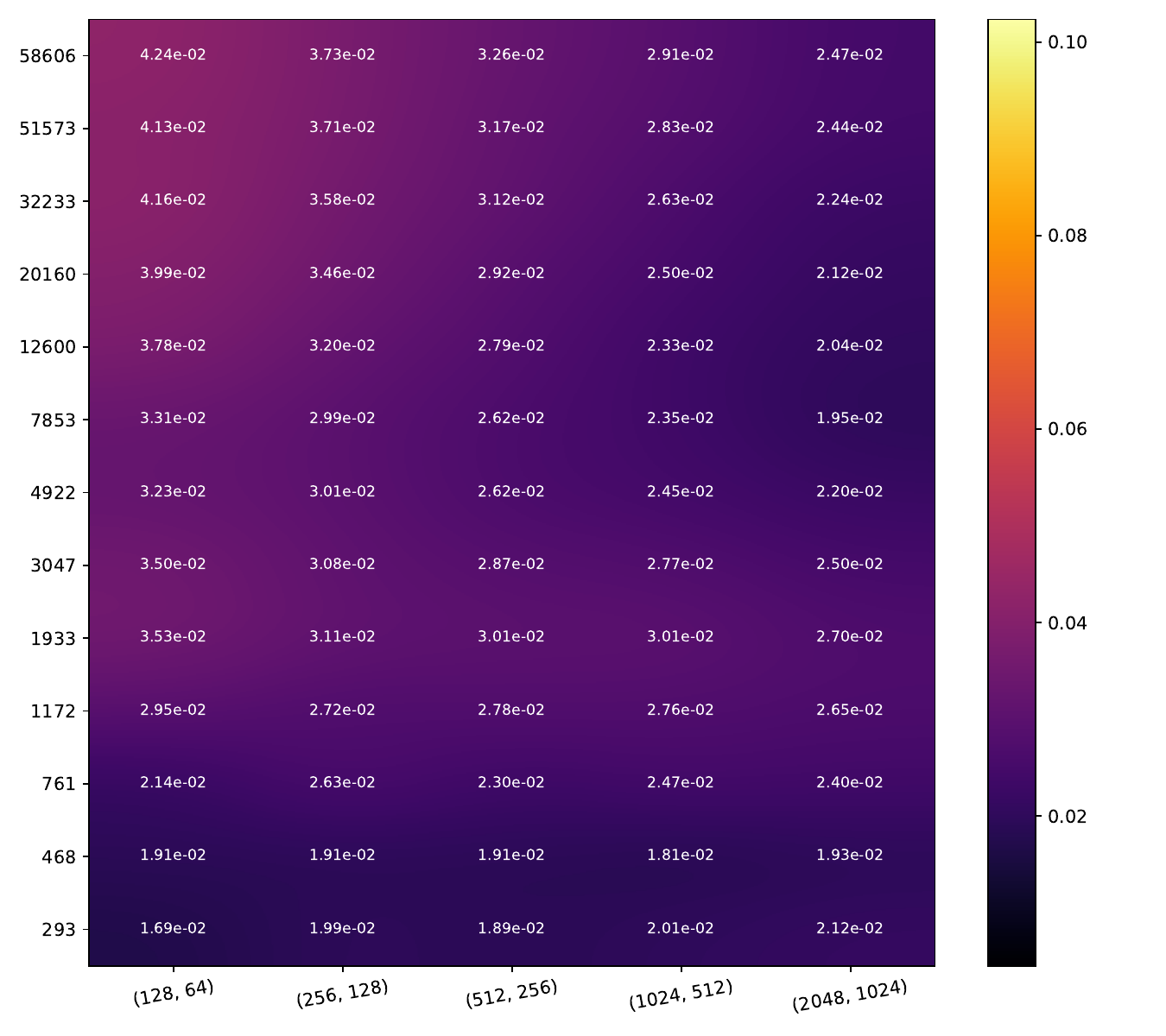}
    \includegraphics[width=\textwidth]{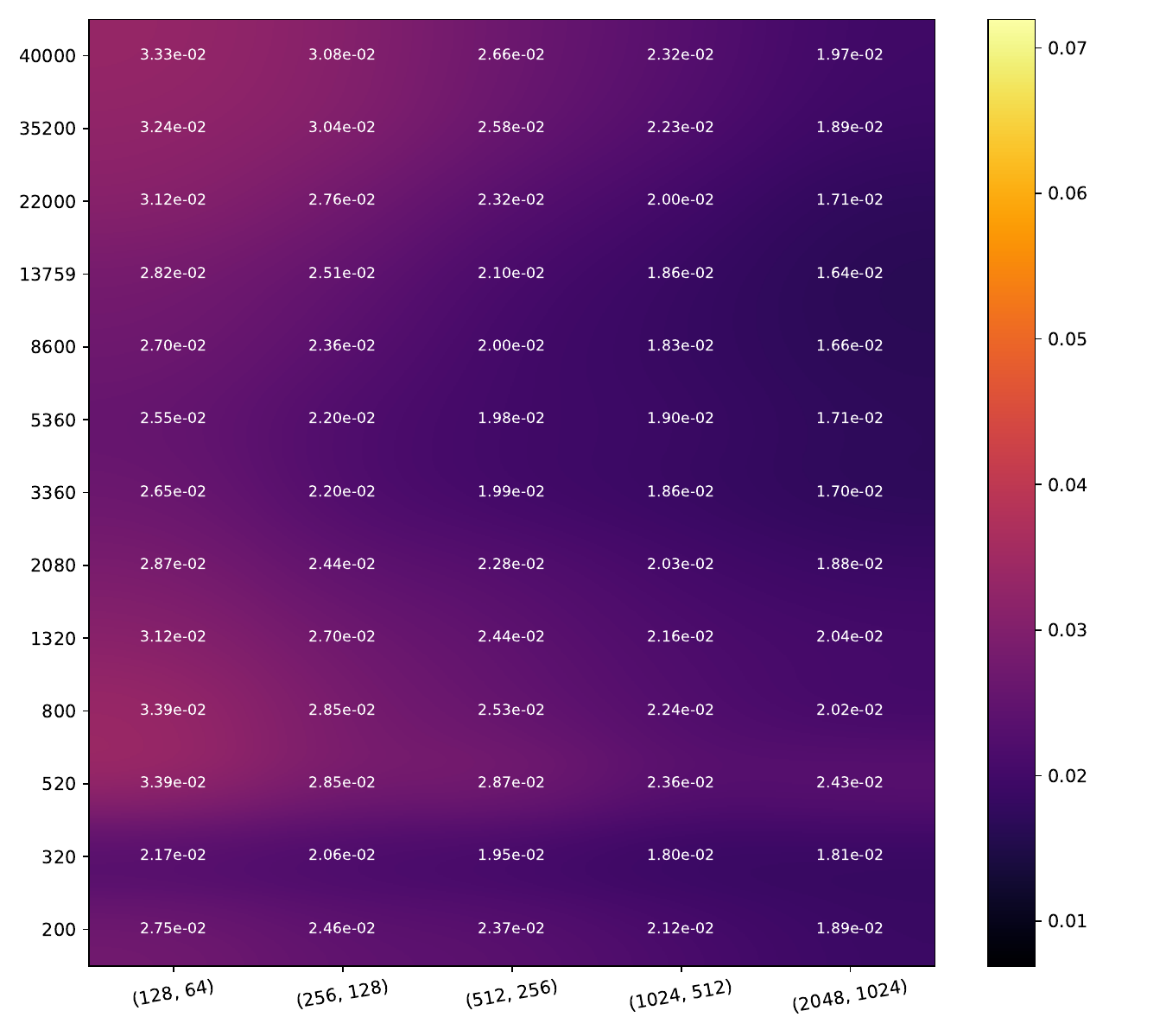}
    \caption*{EDL}
  \end{subfigure}\hfill
  \begin{subfigure}[t]{0.185\textwidth}
    \includegraphics[width=\textwidth]{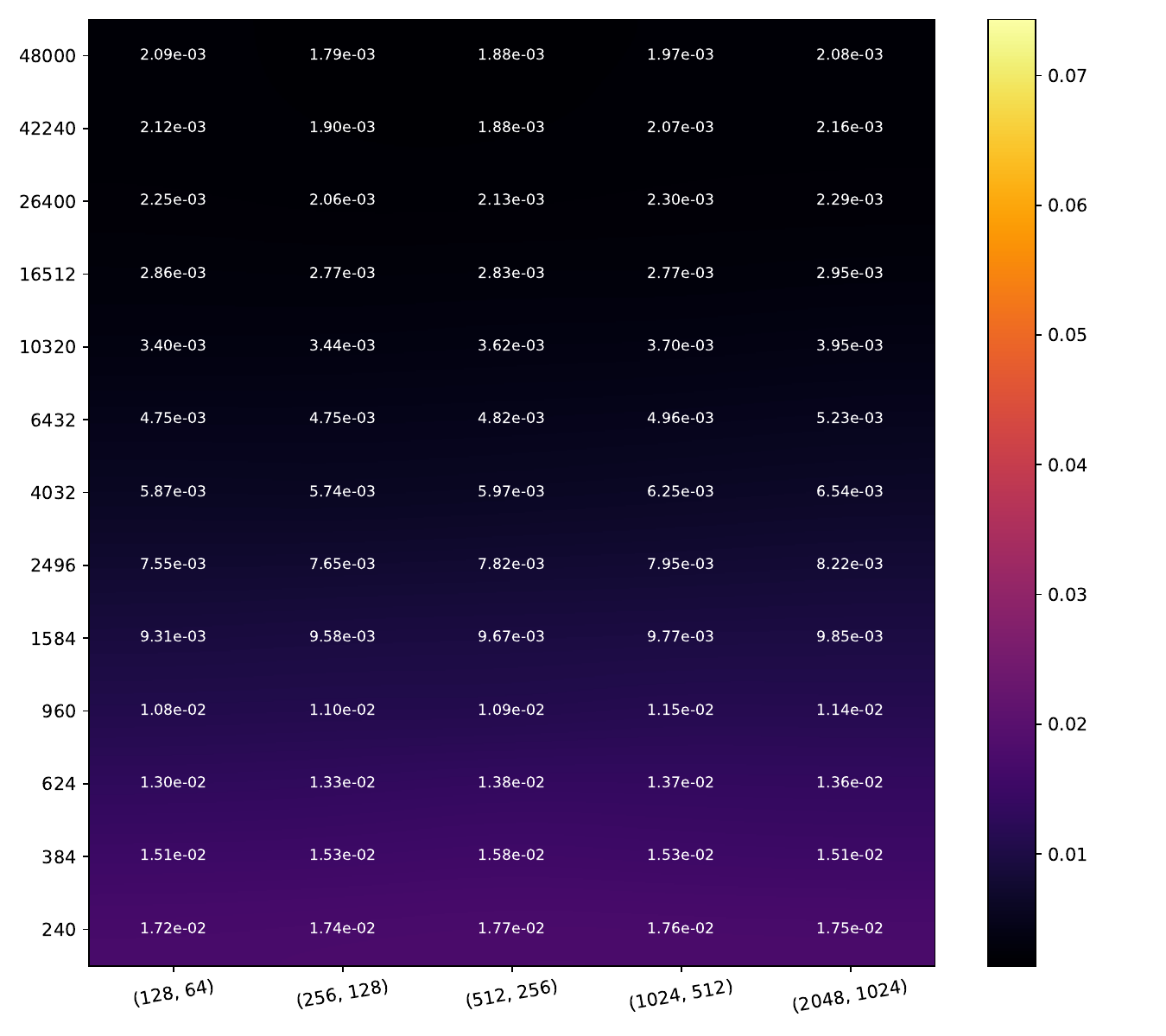}
    \includegraphics[width=\textwidth]{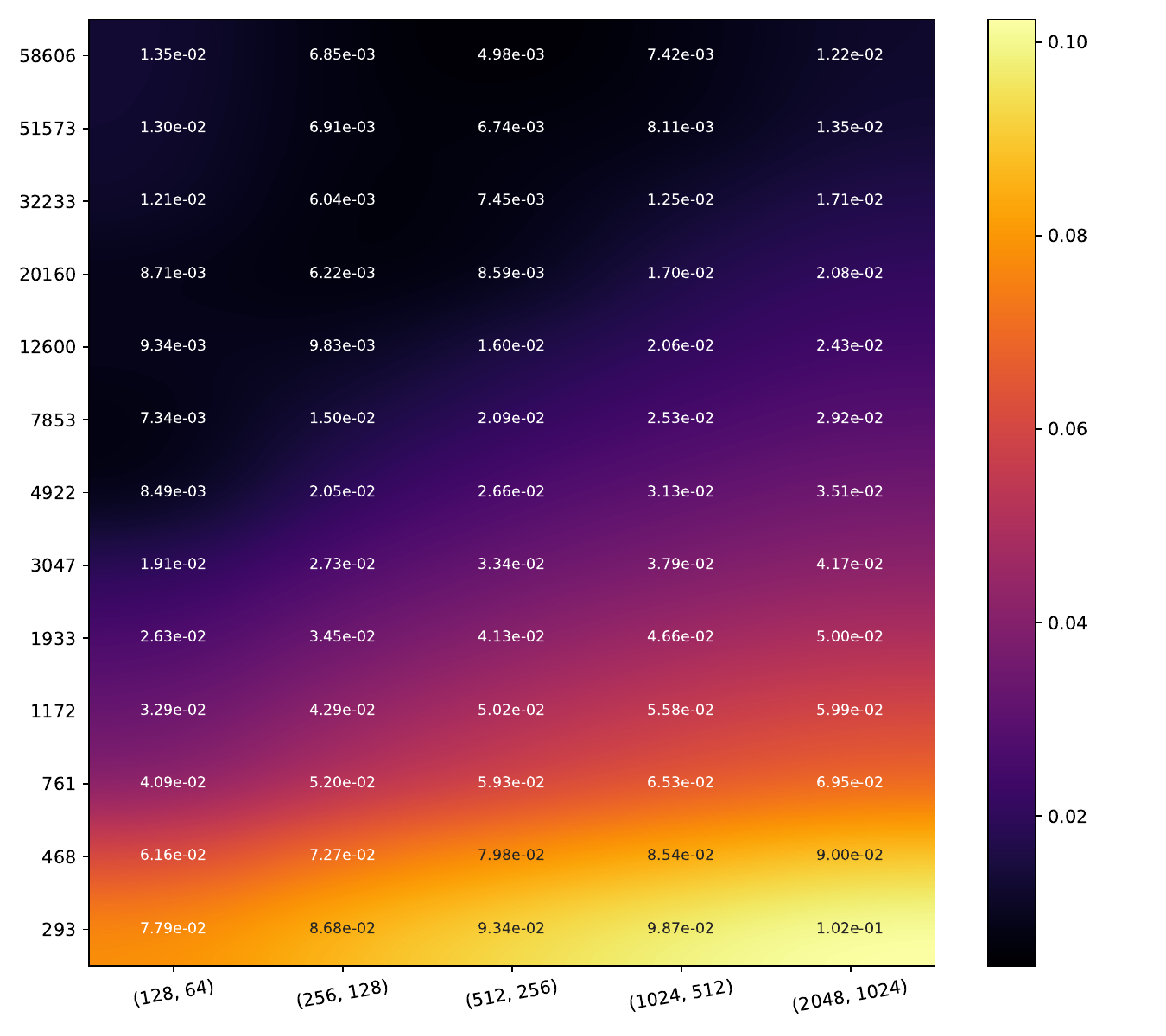}
    \includegraphics[width=\textwidth]{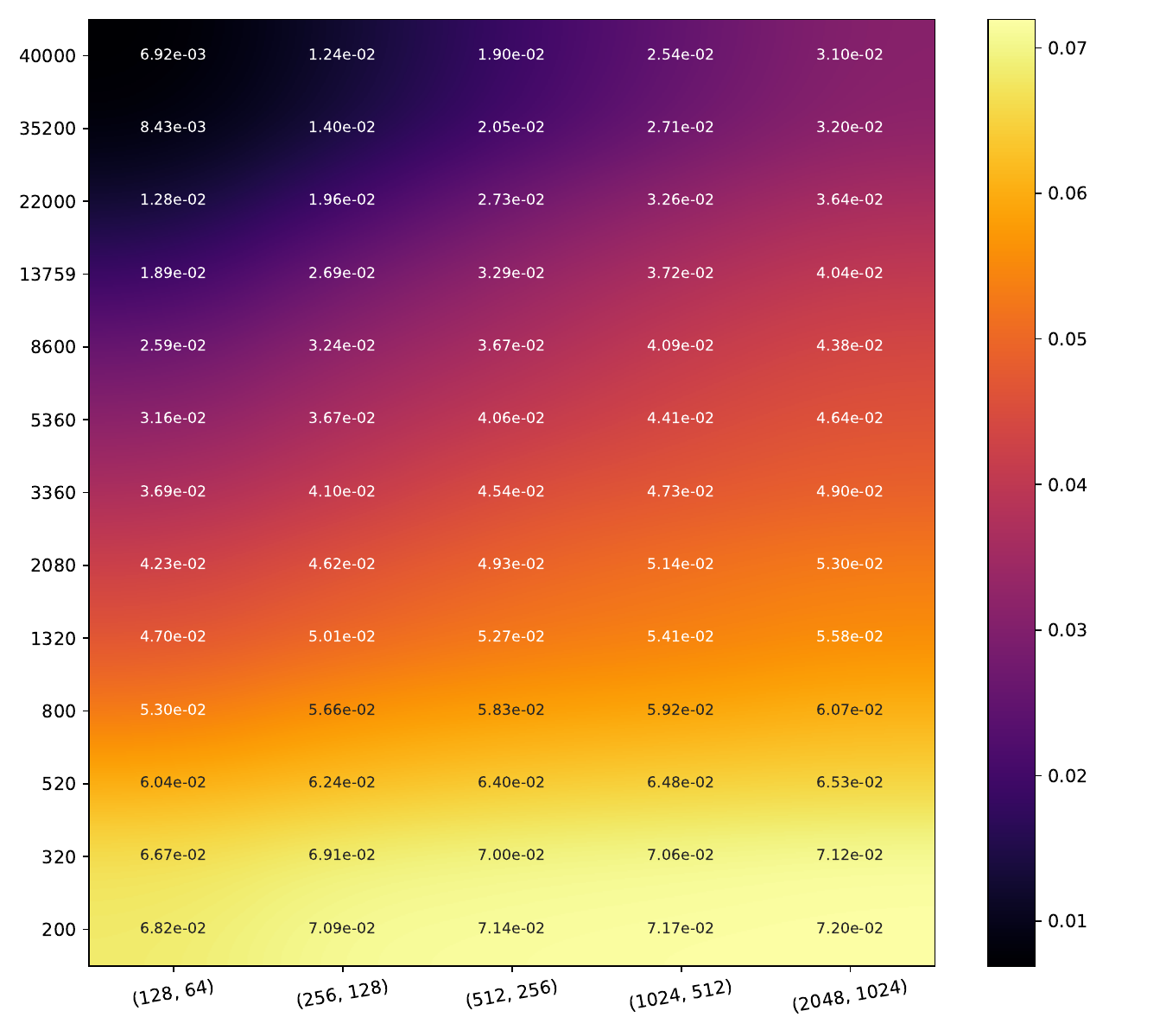}
    \caption*{DE}
  \end{subfigure}\hfill
  \begin{subfigure}[t]{0.185\textwidth}
    \includegraphics[width=\textwidth]{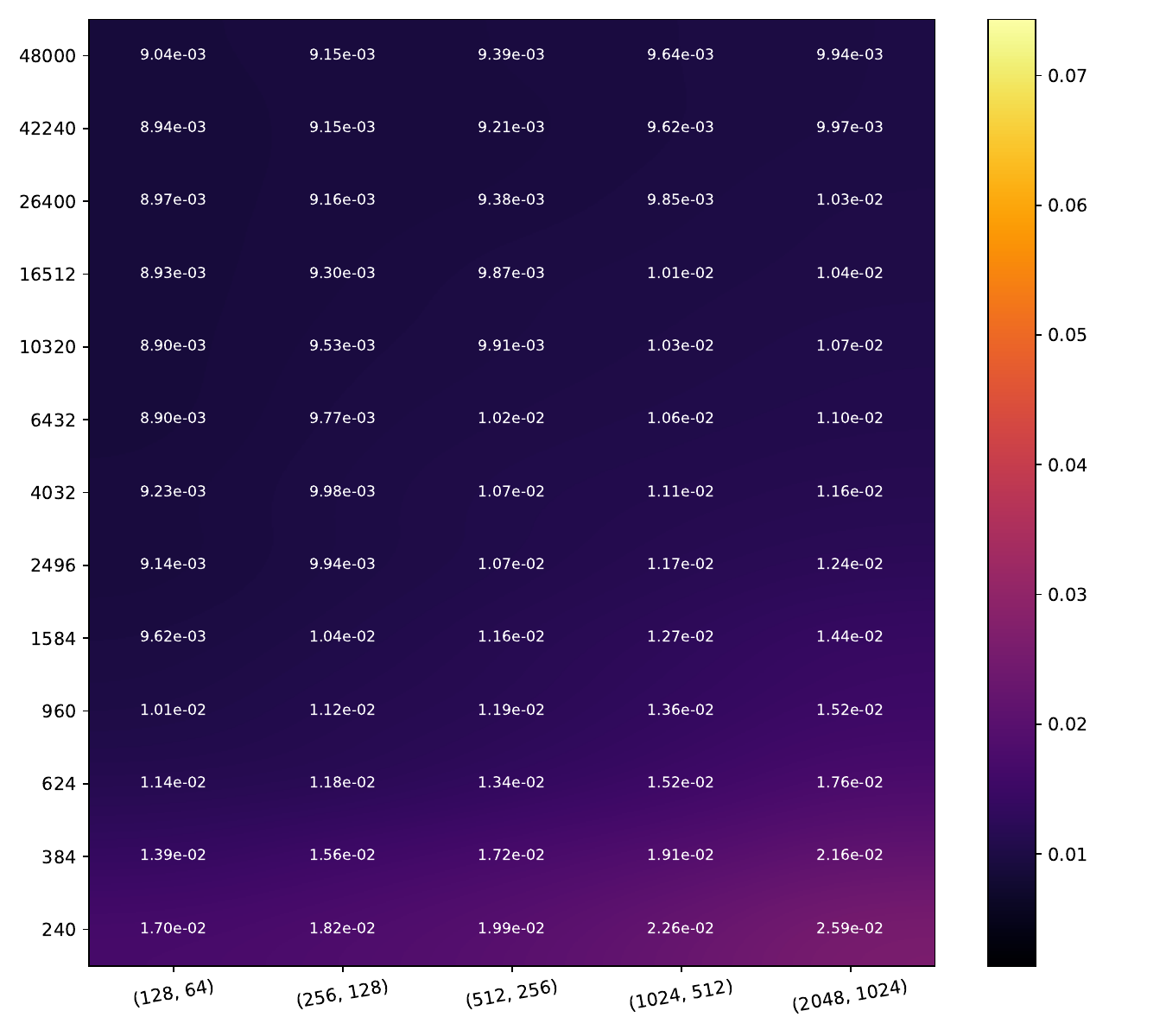}
    \includegraphics[width=\textwidth]{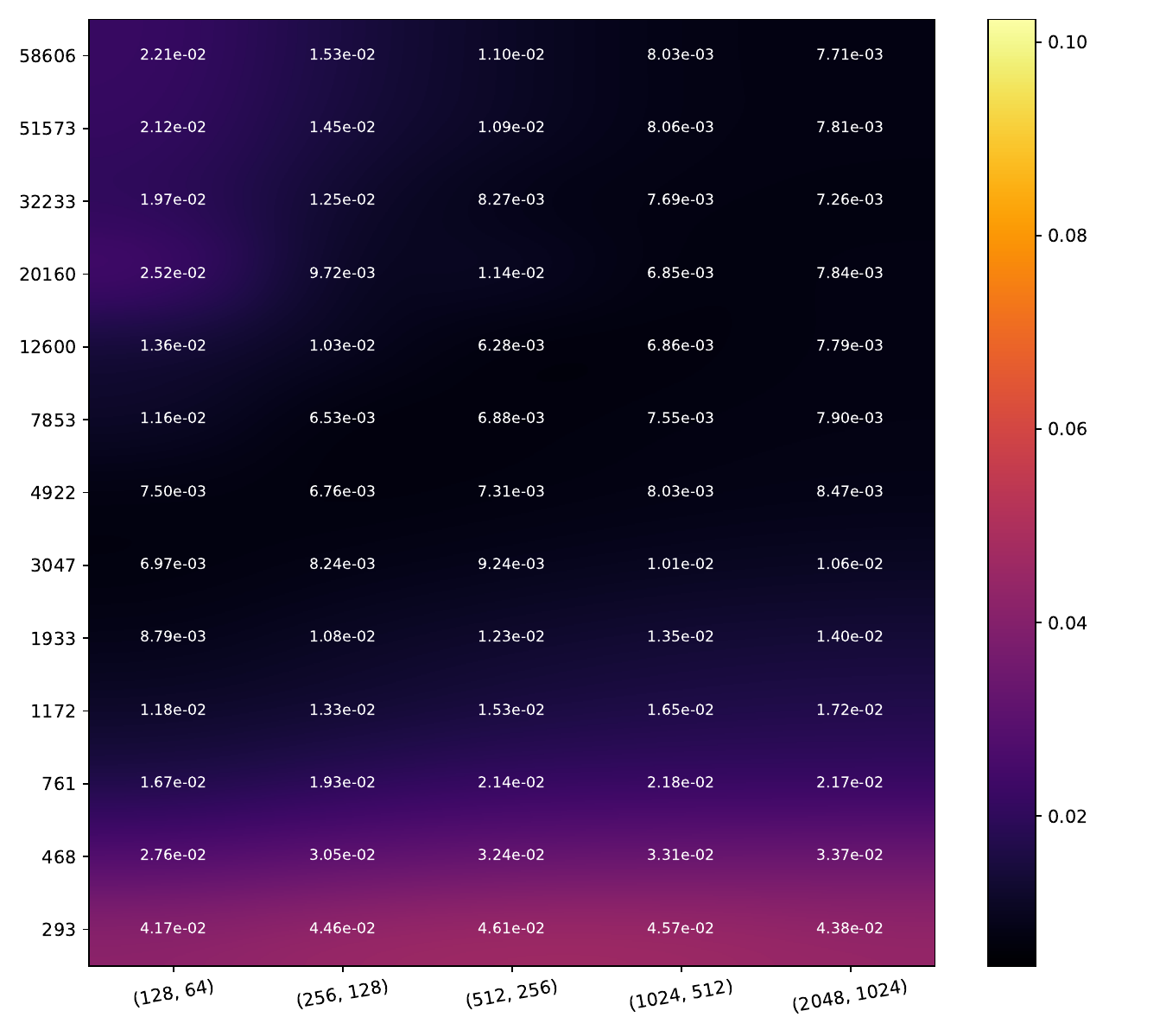}
    \includegraphics[width=\textwidth]{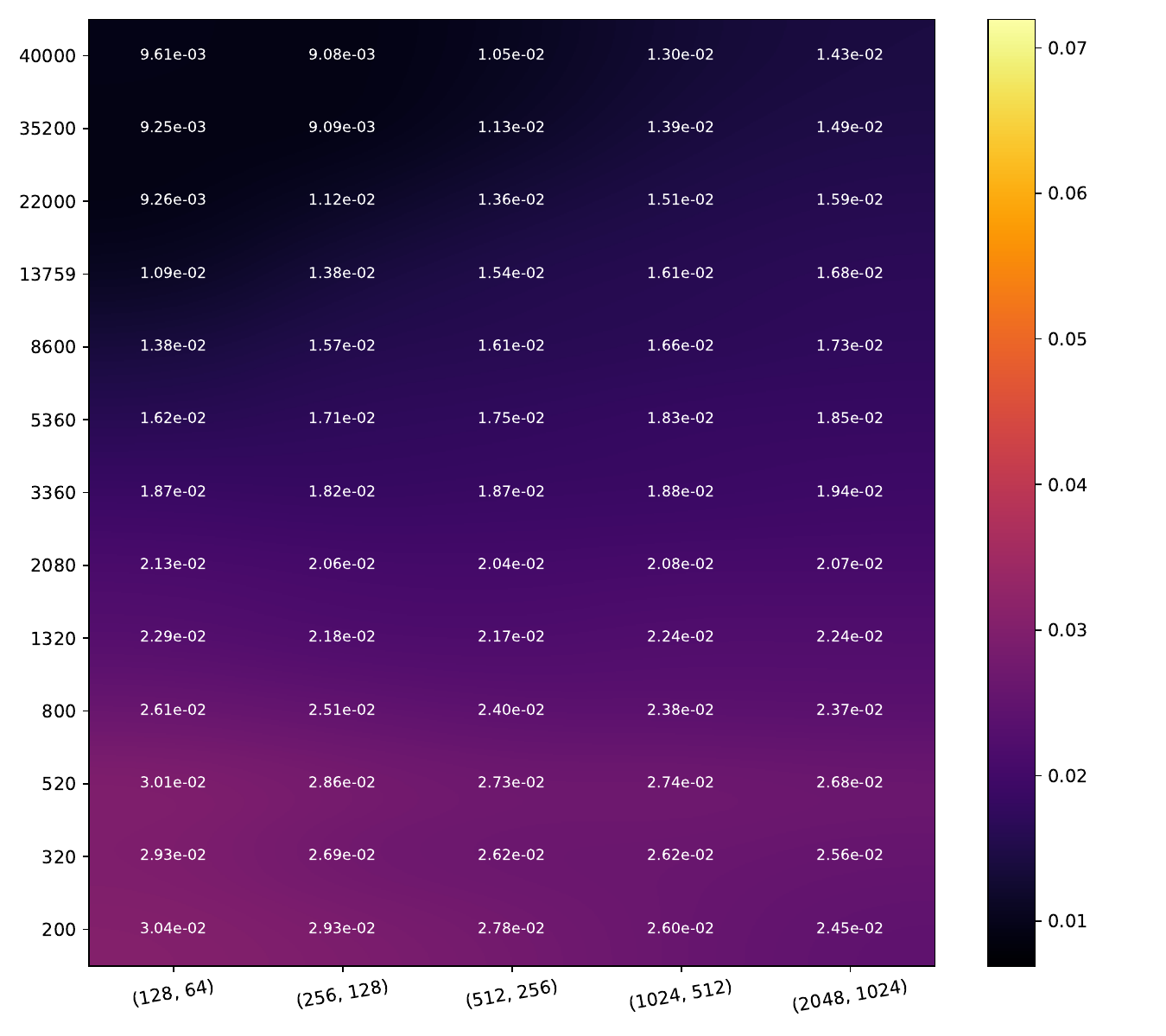}
    \caption*{Conflictual DE}
  \end{subfigure}\hfill
  \caption{Heatmaps of the SCE. Same representation as Fig.~\ref{fig:auc-ood}. Lower is better.}
  \label{fig:sce}
\end{figure}

As the number of hyperparameters to be set should be kept to a minimum, we wondered whether the hyperparameter $\lambda$ introduced by Conflictual DE could replace advantageously that of weight-decay. Therefore we carried out experiments without weight-decay to study its effect on performance (Appendix~\ref{app:no-wd}).
While some methods see their performance deteriorate, the results with Conflictual DE are more or less the same. This robustness suggests that Conflictual DE can be used without weight-decay, resulting in a zero balance for the number of hyperparameters.
To conclude this section, we provide a qualitative and comparative summary of methods on Tab.~\ref{table:summary}.
\begin{table}[htbp]
  \centering
  \begin{tabular}{p{0.19\textwidth} P{0.17\textwidth} P{0.21\textwidth} P{0.09\textwidth} P{0.09\textwidth} c}
    \toprule
     & MC-Dropout & MC-Dropout LS & EDL & DE & Conflictual DE \\
    \midrule
    1\textsuperscript{st} principle & $\checkmark$ & $\xmark$ & $\checkmark$ & $\xmark$ & $\checkmark$ \\
    2\textsuperscript{nd} principle & $\xmark$ & $\xmark$ & $\xmark$ & $\xmark$ & $\checkmark$ \\
    \midrule
    Accuracy & $++$ & $++$ & $+$ & $++$ & $++$ \\
    Brier score & mixed & mixed & $+$ & mixed & $++$ \\
    \midrule
    Calibration & mixed & $-$ & $+$ & mixed & $++$ \\
    \midrule
    OOD & $+$ & mixed & mixed & $++$ & $++$ \\
    Misclassification & $+$ & $-$ & $+$ & $++$ & $++$ \\
    \bottomrule
  \end{tabular}
  \caption{A comparative summary of the performance of the tested methods. Heatmaps of accuracy, Brier score, and misclassification can be found in Fig.~\ref{fig:accuracy}, Fig.~\ref{fig:brier} and Fig.~\ref{fig:auc-mis} respectively (see Appendix).}
  \label{table:summary}
\end{table}

\section{Conclusion\label{sec:conclusion}}

We have shown in this paper that, contrary to expectations, epistemic uncertainty as produced by state-of-the-art models, does not decrease steadily as the training data increases or as the model complexity decreases.
We then introduced \emph{conflictual deep ensembles} and showed that they 
restore not only the first but both principles of epistemic uncertainty, without compromising performance.

Still, this work raises several questions and many perspectives of research:
The exact reasons why epistemic uncertainty paradoxically collapses when network complexity grows, have yet to be found.
While conflictual deep ensembles have been specially designed to satisfy the first principle, it is surprising how well this technique solves the second principle as well. Why is this so?
Although the conflictual loss naturally suits deep ensembles, nothing is preventing it from being applied to BNNs, provided that the samples of the prior can be efficiently partitioned into class-specific subsets. The validity of such an approach remains to be demonstrated.
The question remains whether the observed phenomena generalize to models more complex than MLPs or to other problems like regression.
We hope that these and other perspectives will entail further studies on this subject.

\begin{credits}
\subsubsection{\discintname}
The authors have no competing interests to declare that are relevant to the content of this article.
\end{credits}
%
%
%
\bibliographystyle{splncs04}
\bibliography{biblio}

\clearpage
\newpage
\begin{subappendices}
\section{Implementation details\label{app:implementation-details}}

\paragraph{Datasets.} As detailed in the paper, the presented models were trained on MNIST and CIFAR10. A $20\%$ validation-train split is first applied and then subsets were taken from the train sets (a total size of $48000$ for MNIST, $58606$ for SVHN, and $40000$ for CIFAR10) for training. We made sure that the subsets used were balanced. CIFAR10 is encoded using a pre-trained ResNet34 model which is equivalent to the training of a ResNet34 model where the feature blocks are fixed and only the classification part is learned.

\paragraph{Data transformations.} We apply a standard normalization (mean of $0$ and standard deviation of $1$) to the datasets. The same transformation is applied then to the test samples (whether there are ID or OOD samples). Only the training samples are used to train the models and no data augmentation is applied.

\paragraph{Training.} The models were trained for $500$ epochs on MNIST, $600$ epochs on SVHN and $700$ on CIFAR10. We used the SGD optimizer with weight decay, parameterized with (learning rate, momentum):  ($0.01$, $0.95$) for MNIST, ($0.02$, $0.95$) for SVHN, and ($0.04$, $0.9$) for CIFAR10. Each ensemble was trained on a single GPU and the best model (based on the validation loss) was tracked during training and used for early stopping and the learning rate scheduler.

\paragraph{Duration.} The experiments took a total of $311$, $186$, and $70$ GPU hours on MNIST, SVHN, and CIFAR10 respectively. The training was done on a cluster with several GPUs. See Tab.~\ref{tab-app:duration-mlp} for more details. The difference is mainly due to how CIFAR10 experiments are implemented: we first compute the embeddings of dimension $512$ (once) using a pre-trained ResNet34 which are stored on disk.

\begin{table}[ht]
  \centering
  \begin{tabular}{p{0.12\textwidth} P{0.17\textwidth} P{0.21\textwidth} P{0.09\textwidth} P{0.09\textwidth} c}
    \toprule
     & MC-Dropout & MC-Dropout LS & EDL & DE & Conflictual DE \\
    \midrule
    MNIST   & $21.75$ & $20.19$ & $41.35$ & $111$   & $116.73$ \\
    SVHN    & $12.36$ & $13.37$ & $14$    & $68.33$ & $76.40$  \\
    CIFAR10 & $6.83$  & $6.64$  & $7.01$  & $23.85$ & $26$     \\
    \bottomrule
  \end{tabular}
  \caption{Training duration for each method and each dataset in GPU hours.}
  \label{tab-app:duration-mlp}
\end{table}
\paragraph{\textbf{NOTE:}} the MNIST experiments took longer than the experiments of SVHN and CIFAR10 due to 2 main reasons:
\begin{itemize}
    \item We don't take into account the time needed to compute and save the embeddings for CIFAR10 and which is done only once.
    \item Most importantly, the format of the files for the CIFAR10' embeddings was optimized and thus it is faster. We further apply the same optimization to the MNIST dataset (by using an \emph{identity} "embedding") and MNIST experiments should run faster with the new changes. This format was applied to SVHN which explains the performance gains. Refer to the code for more details.
\end{itemize}
The execution time for MNIST, after changes, should be less than the training time of SVHN. 

\paragraph{Implementation.} All experiments were implemented with PyTorch. Code available at: \url{https://github.com/fellajimed/Conflictual-Loss}

\section{Additional results\label{app:additional-results}}
We report additional metrics for the experiments of Sec~\ref{sec:experiments} and we use the same representation: for each heatmap, the hidden layers on the x-axis and the number of samples used for training on the y-axis. They both have logarithmic scales. Results are on the same datasets and methods.

\begin{figure}[ht]
  \centering
  \begin{subfigure}[t]{\dimexpr0.185\textwidth+20pt\relax}
    \makebox[20pt]{\raisebox{25pt}{\rotatebox[origin=c]{90}{\scriptsize MNIST}}}%
    \includegraphics[width=\dimexpr\linewidth-20pt\relax]{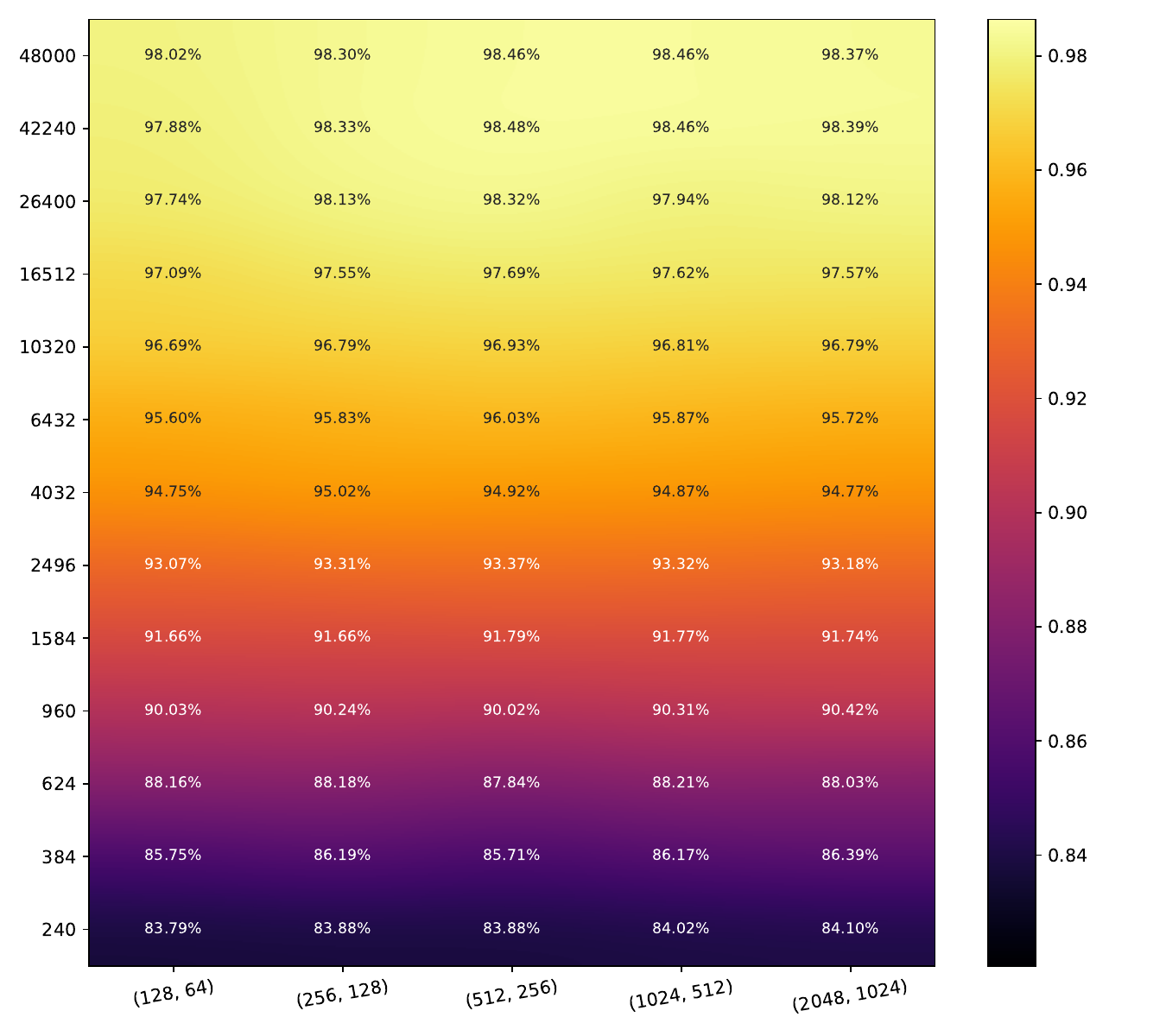}
    \makebox[20pt]{\raisebox{25pt}{\rotatebox[origin=c]{90}{\scriptsize SVHN}}}%
    \includegraphics[width=\dimexpr\linewidth-20pt\relax]{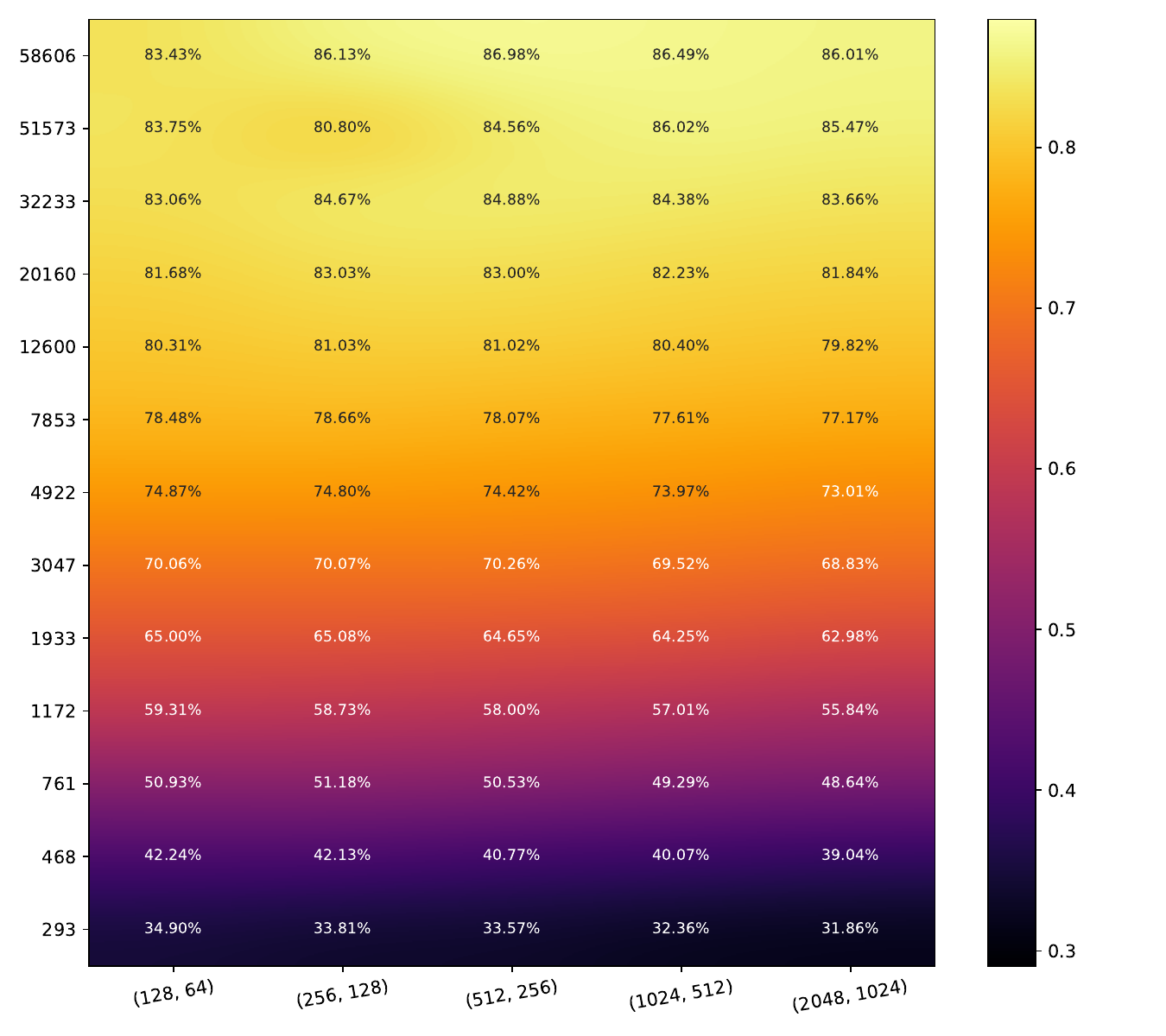}
    \makebox[20pt]{\raisebox{25pt}{\rotatebox[origin=c]{90}{\scriptsize CIFAR10}}}%
    \includegraphics[width=\dimexpr\linewidth-20pt\relax]{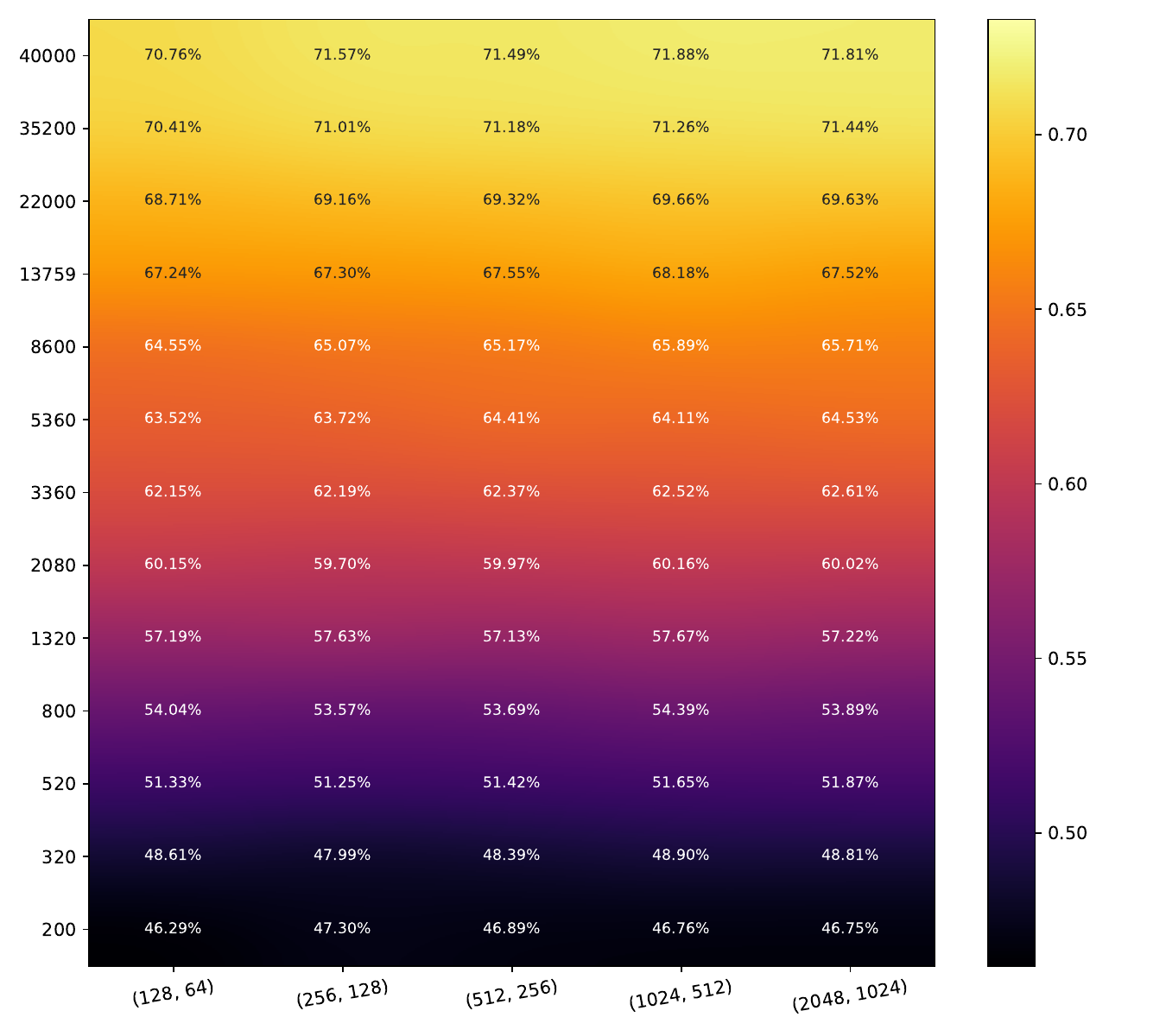}
    \caption*{\qquad MC-Dropout}
  \end{subfigure}\hfill
  \begin{subfigure}[t]{0.185\textwidth}
    \includegraphics[width=\textwidth]{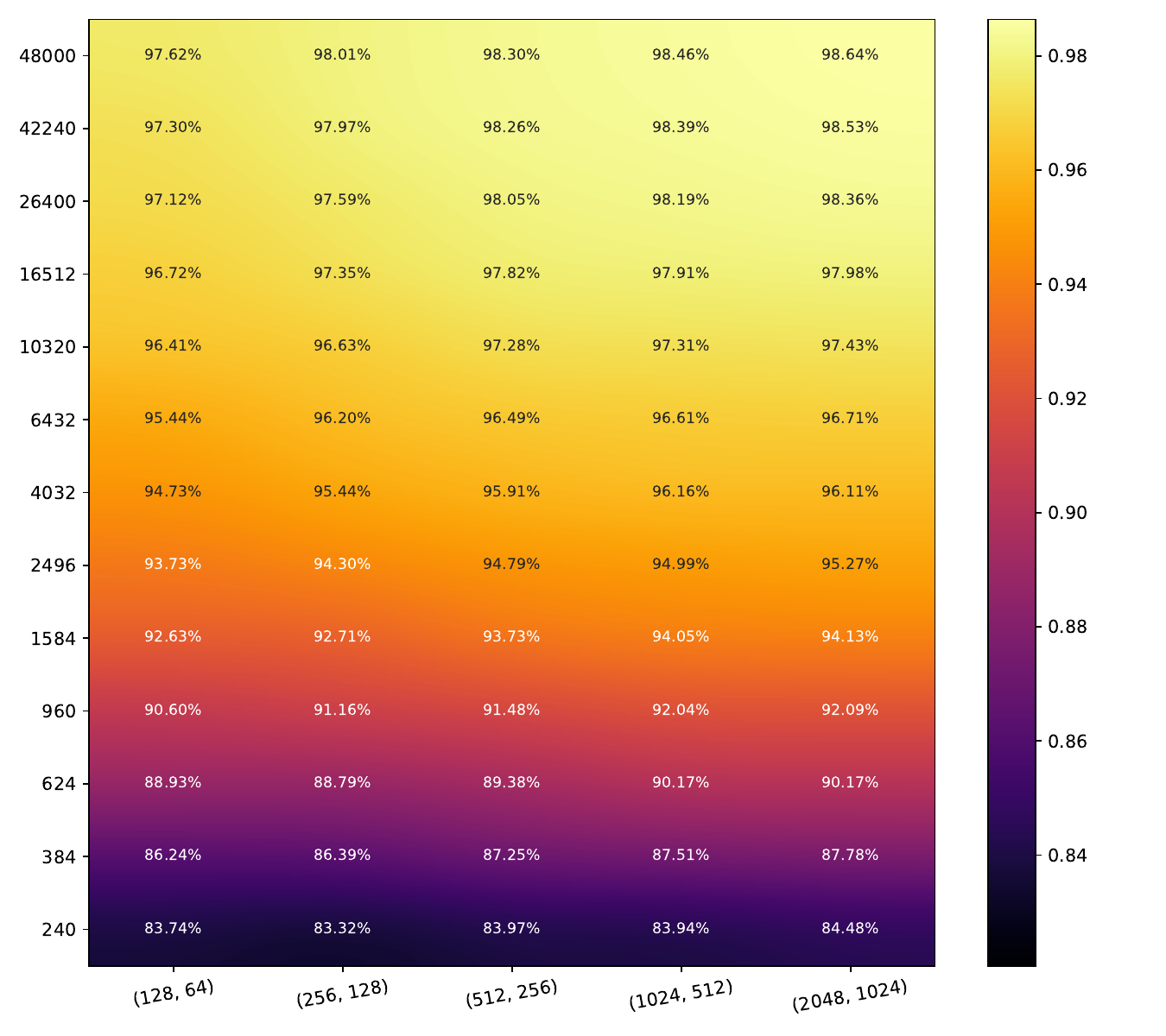}
    \includegraphics[width=\textwidth]{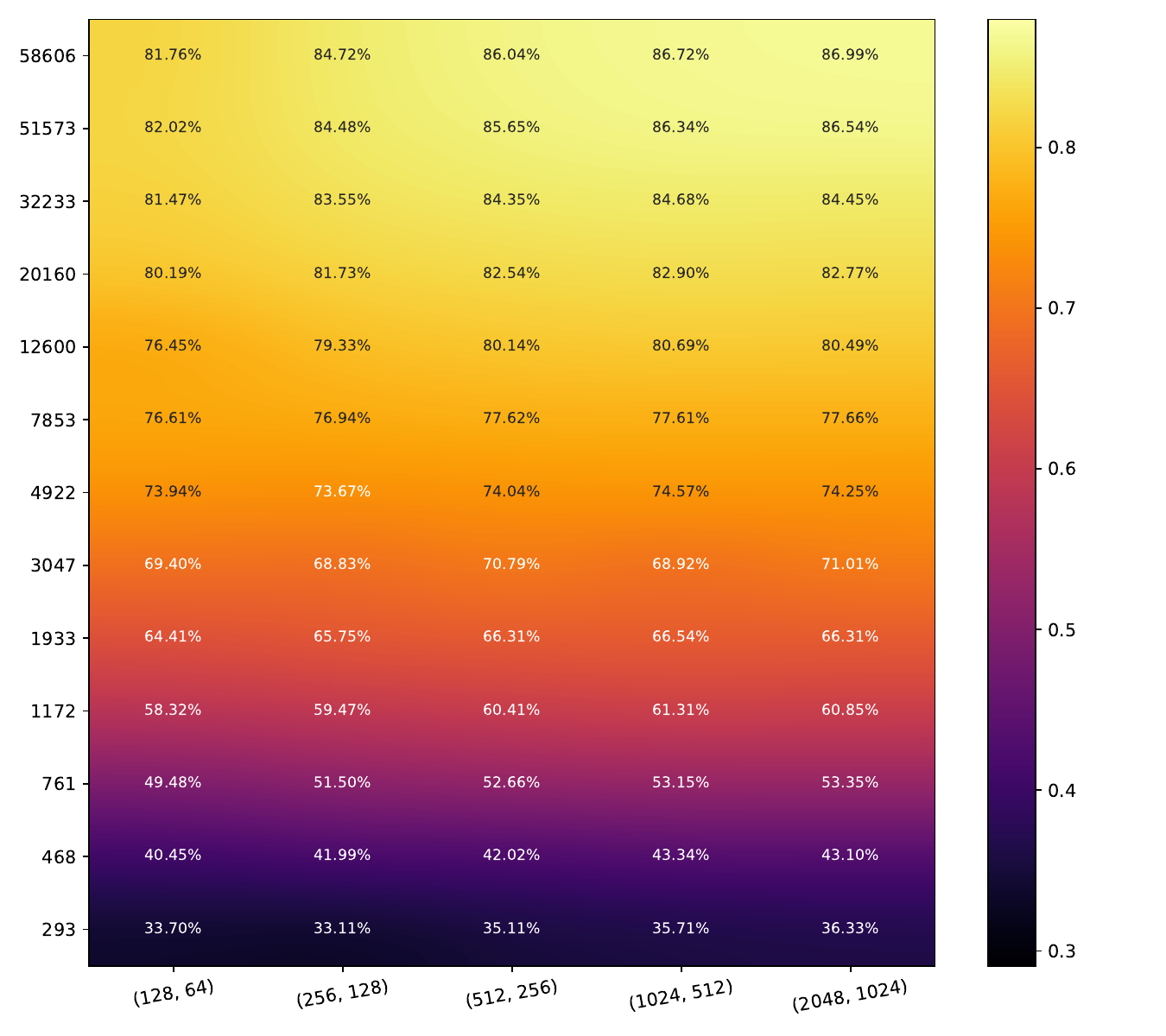}
    \includegraphics[width=\textwidth]{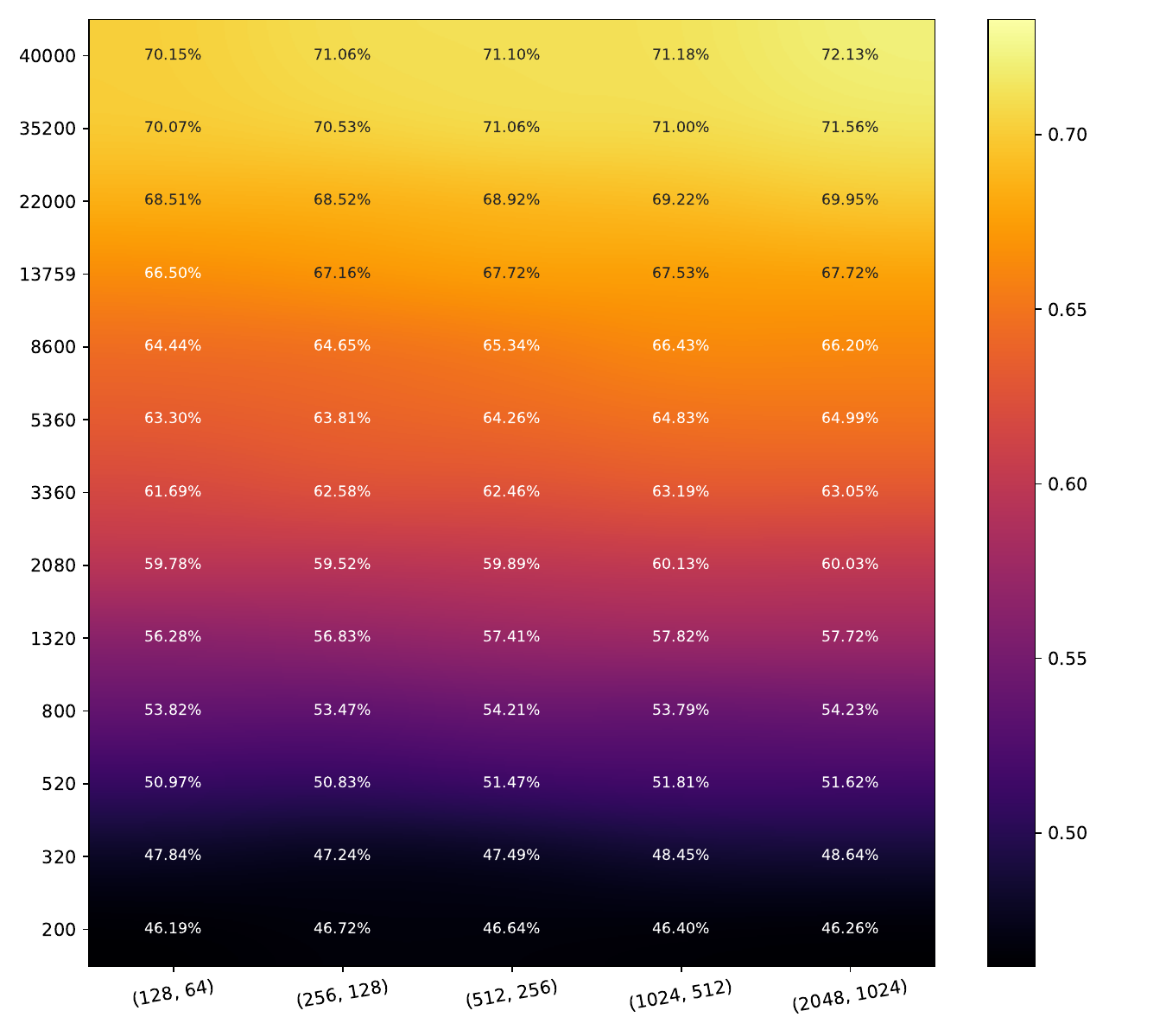}
    \caption*{MC-Dropout LS}
  \end{subfigure}\hfill
  \begin{subfigure}[t]{0.185\textwidth}
    \includegraphics[width=\textwidth]{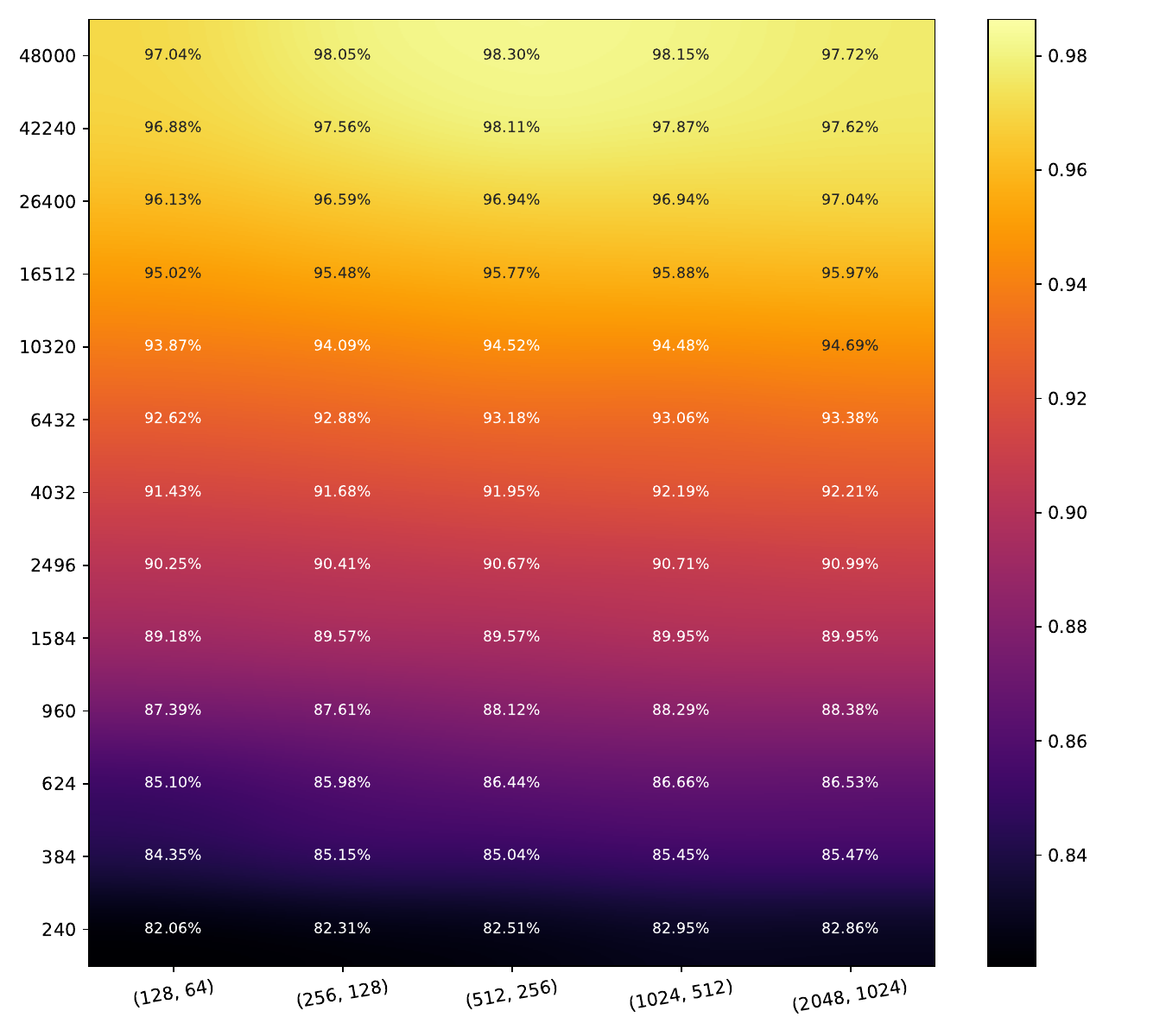}
    \includegraphics[width=\textwidth]{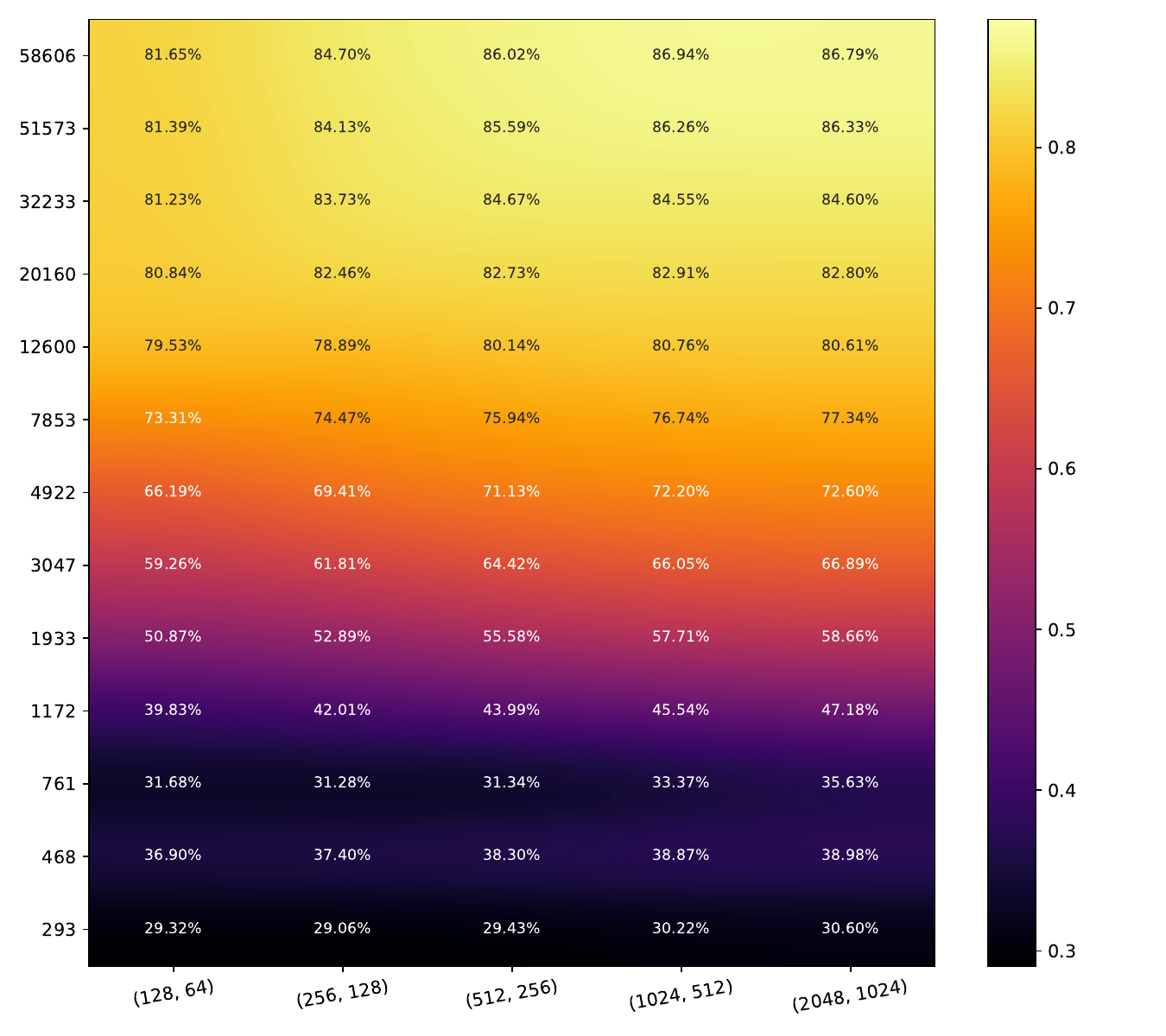}
    \includegraphics[width=\textwidth]{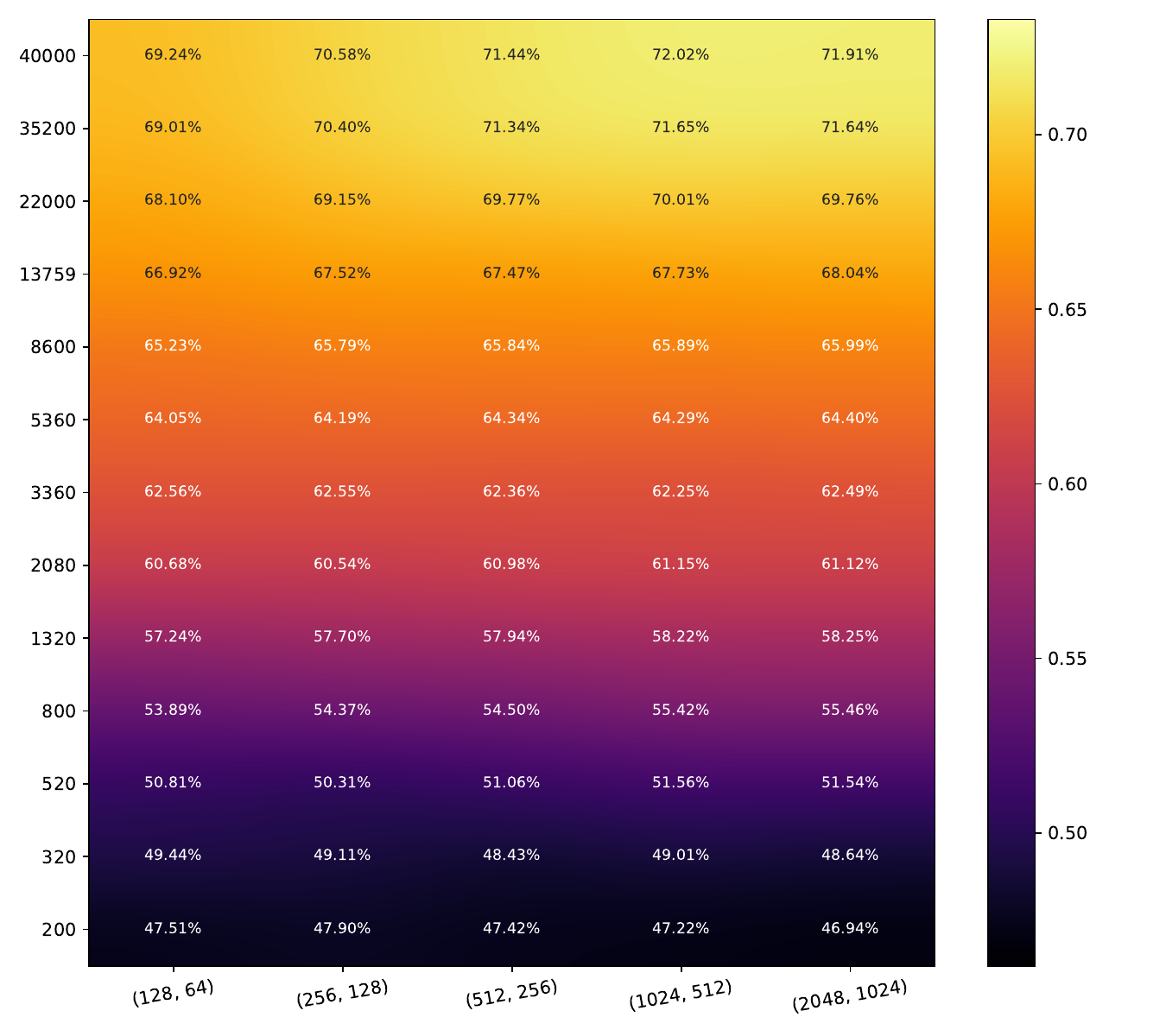}
    \caption*{EDL}
  \end{subfigure}\hfill
  \begin{subfigure}[t]{0.185\textwidth}
    \includegraphics[width=\textwidth]{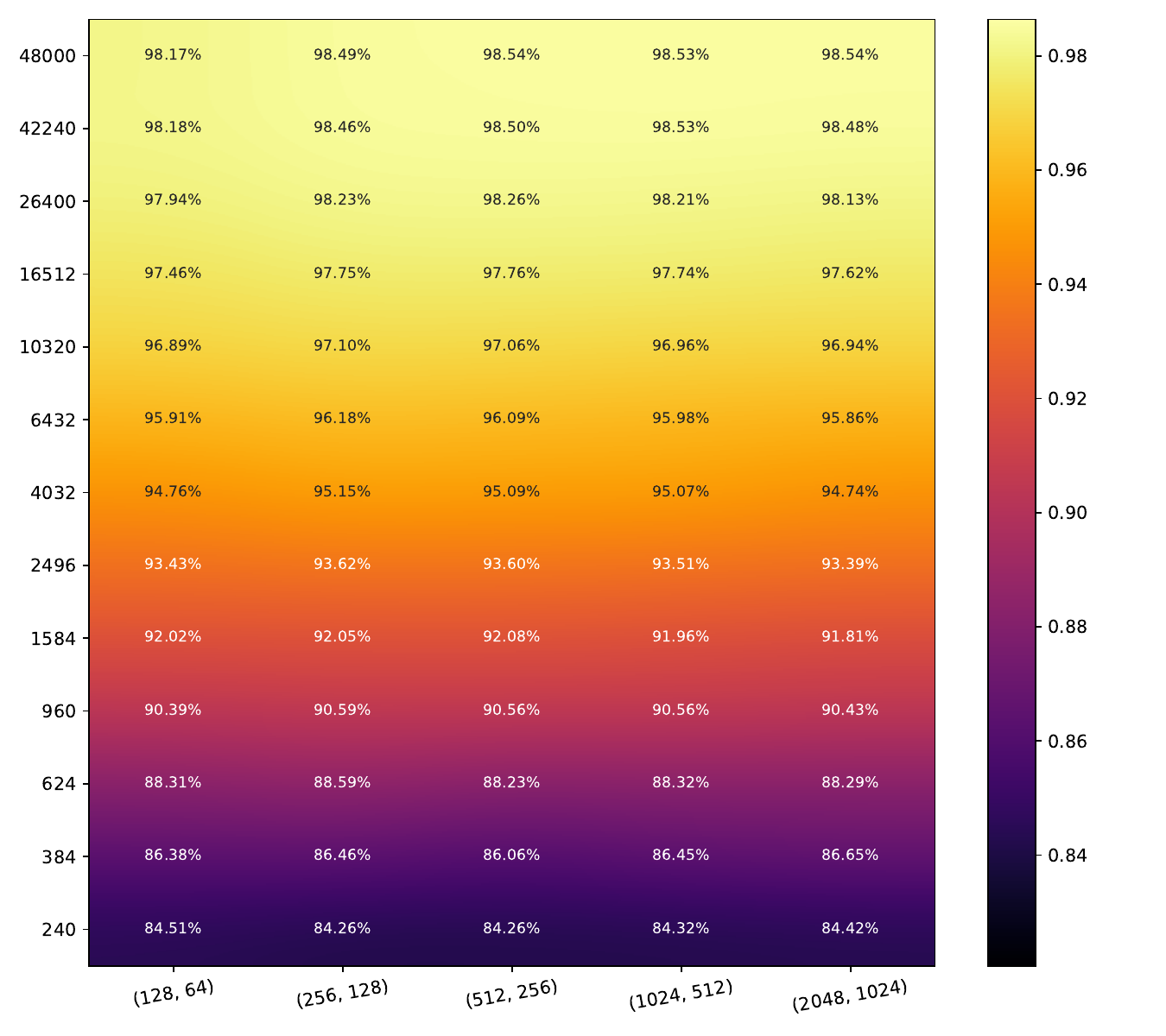}
    \includegraphics[width=\textwidth]{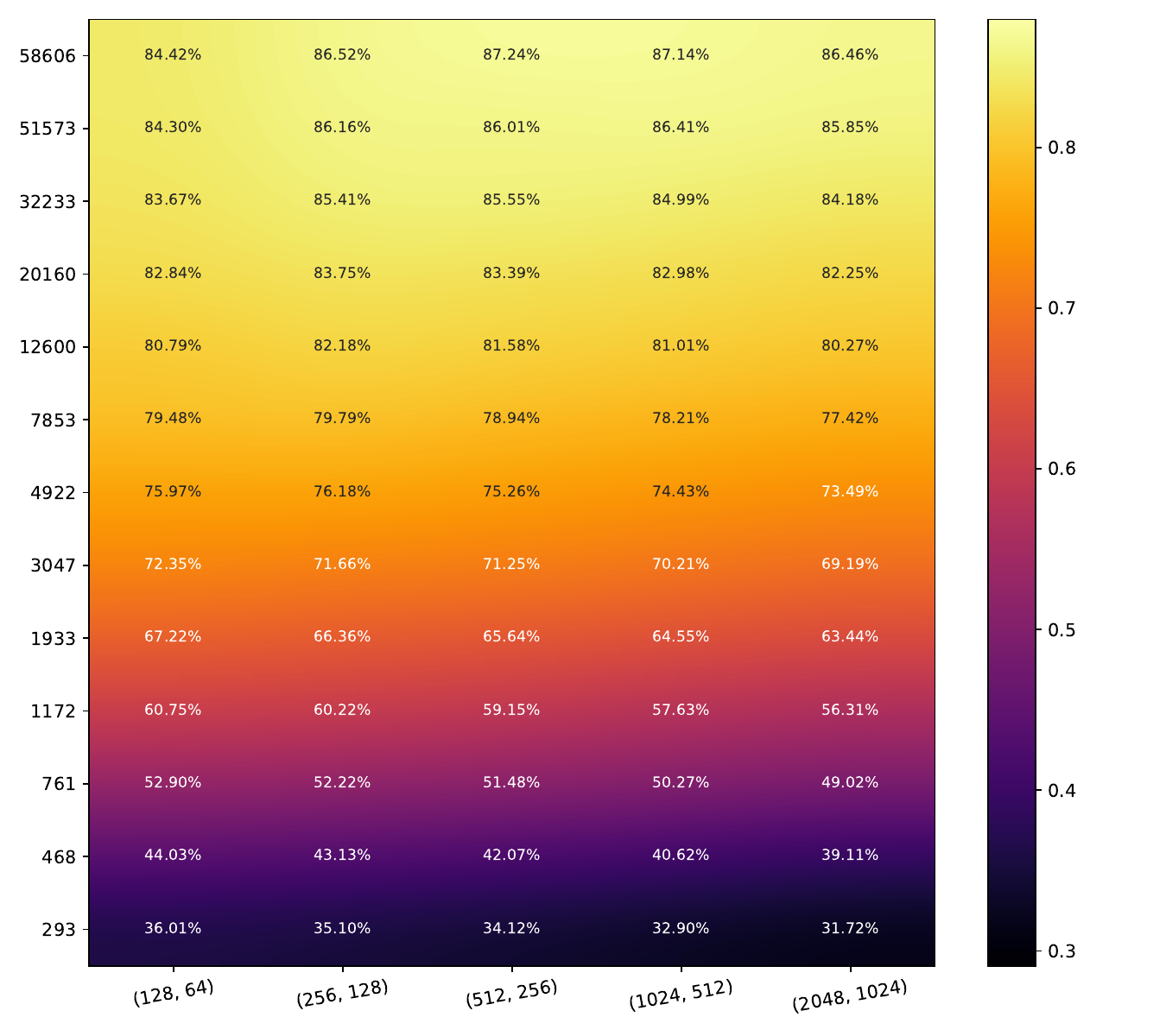}
    \includegraphics[width=\textwidth]{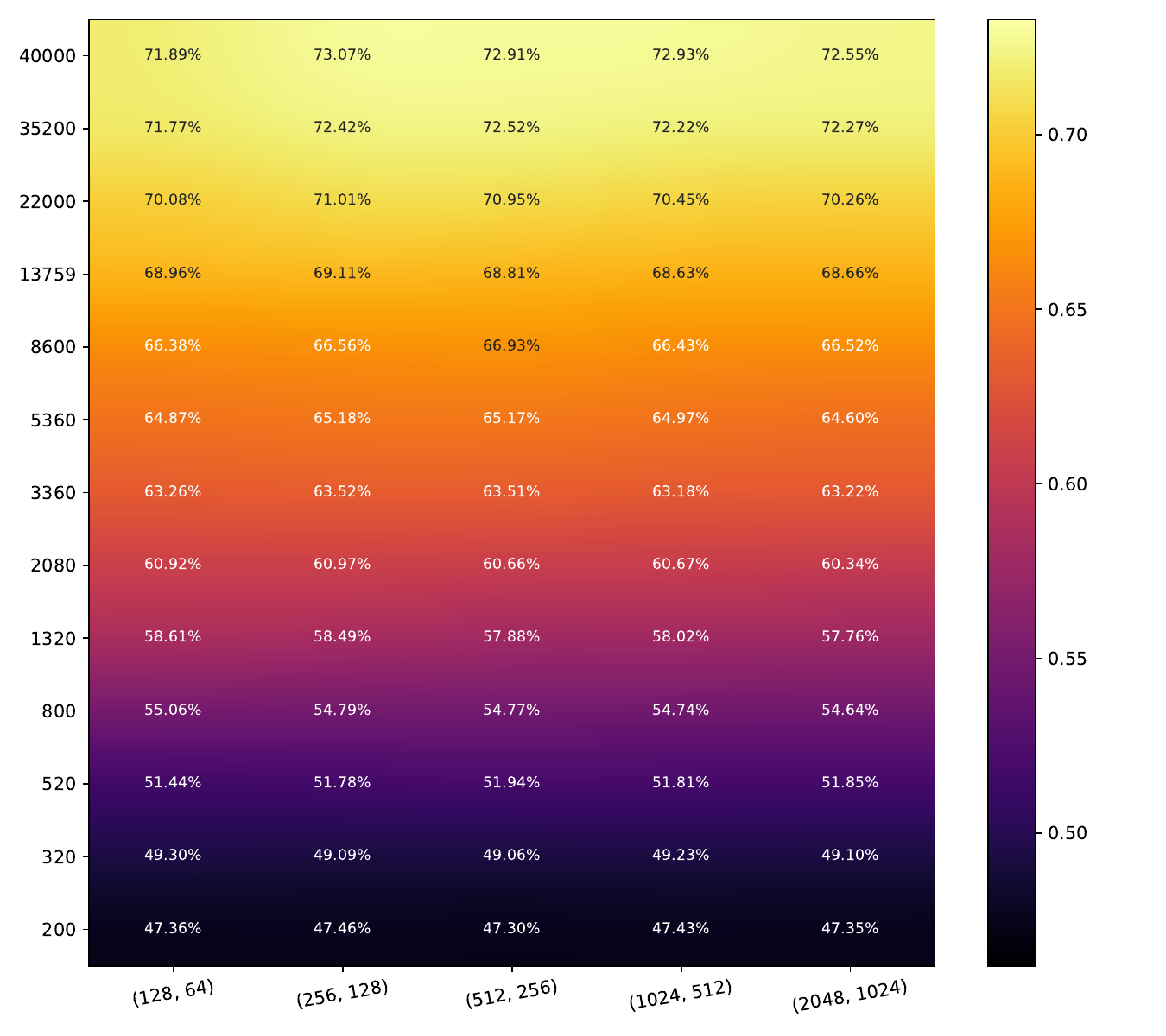}
    \caption*{DE}
  \end{subfigure}\hfill
  \begin{subfigure}[t]{0.185\textwidth}
    \includegraphics[width=\textwidth]{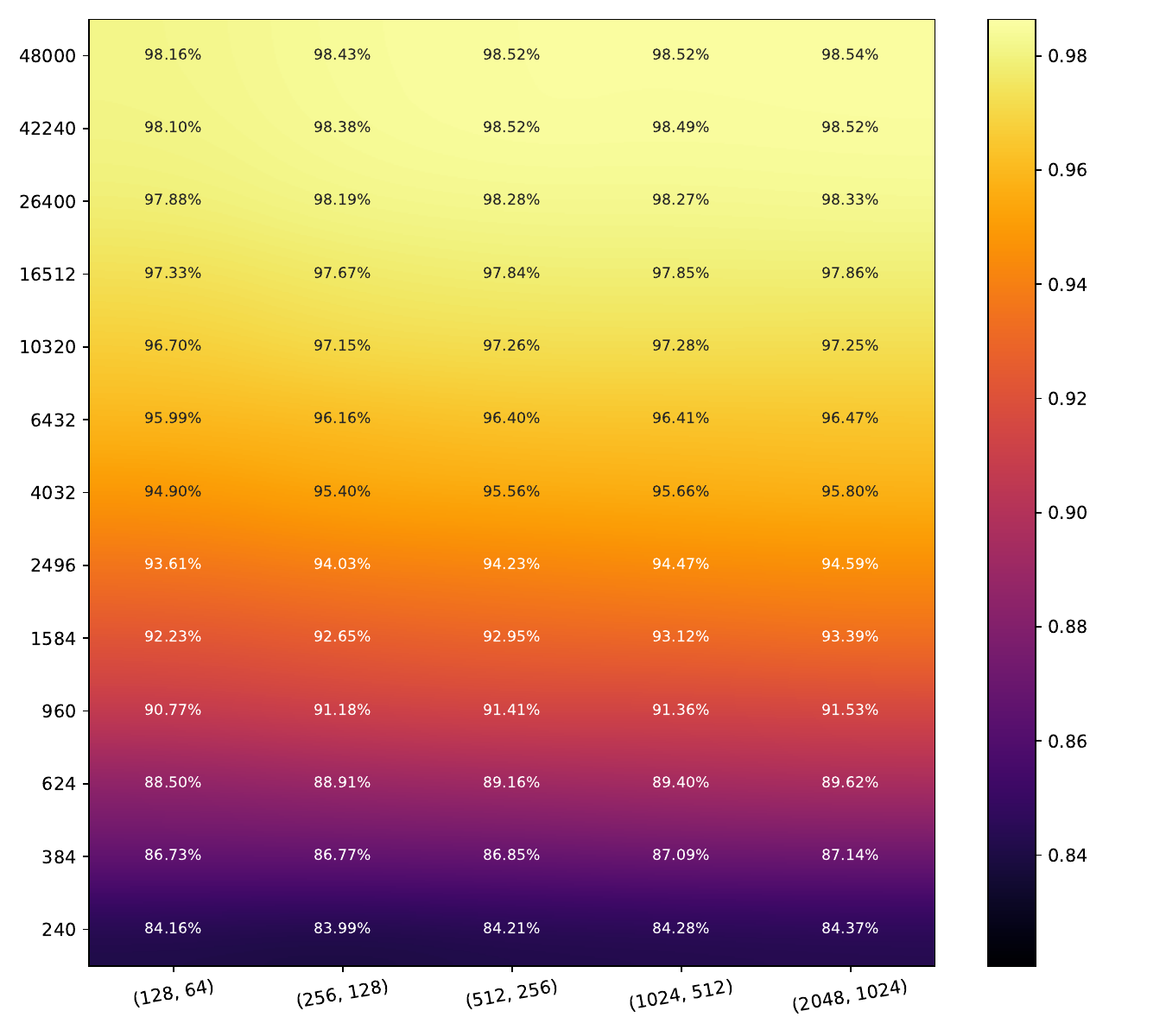}
    \includegraphics[width=\textwidth]{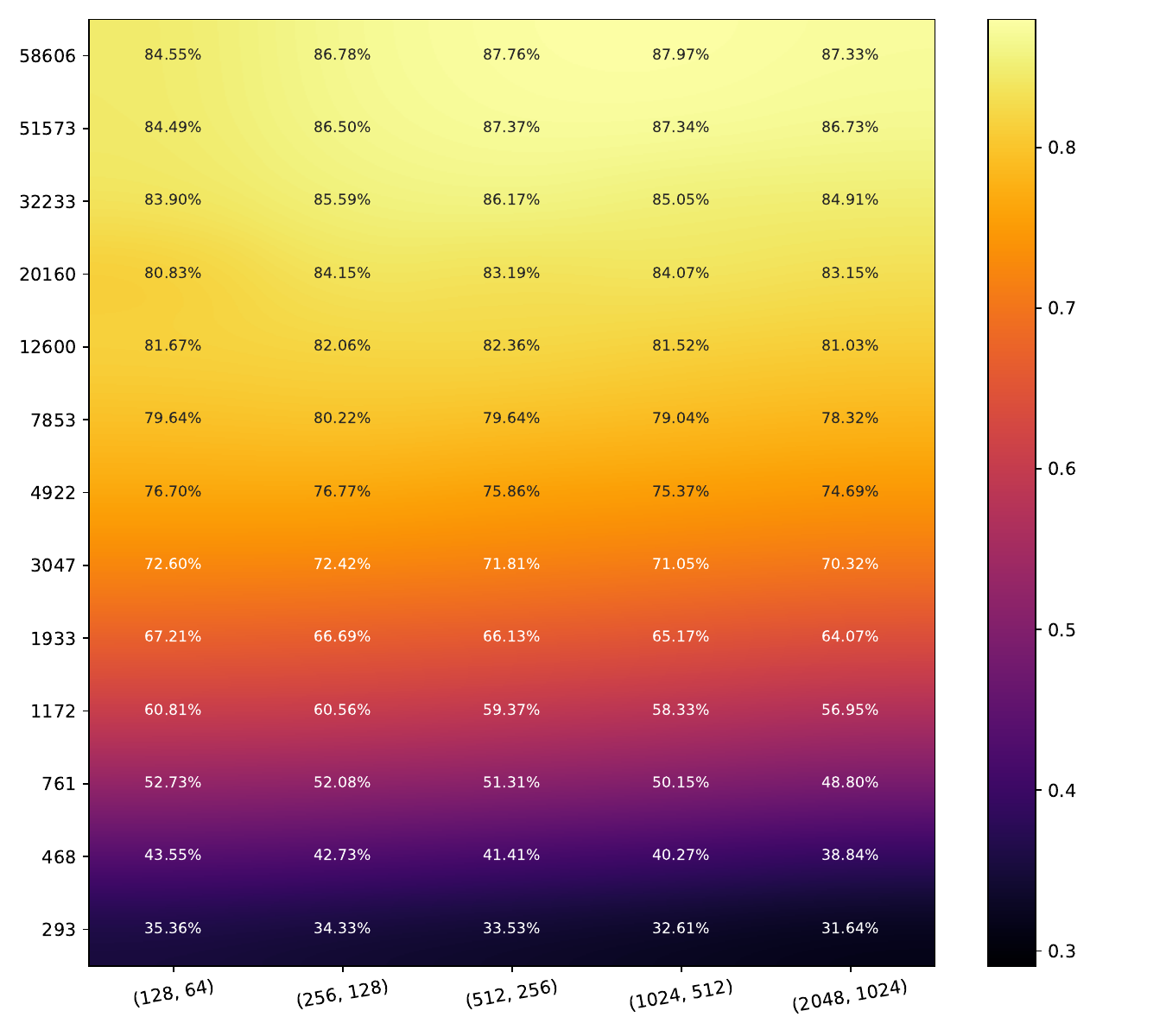}
    \includegraphics[width=\textwidth]{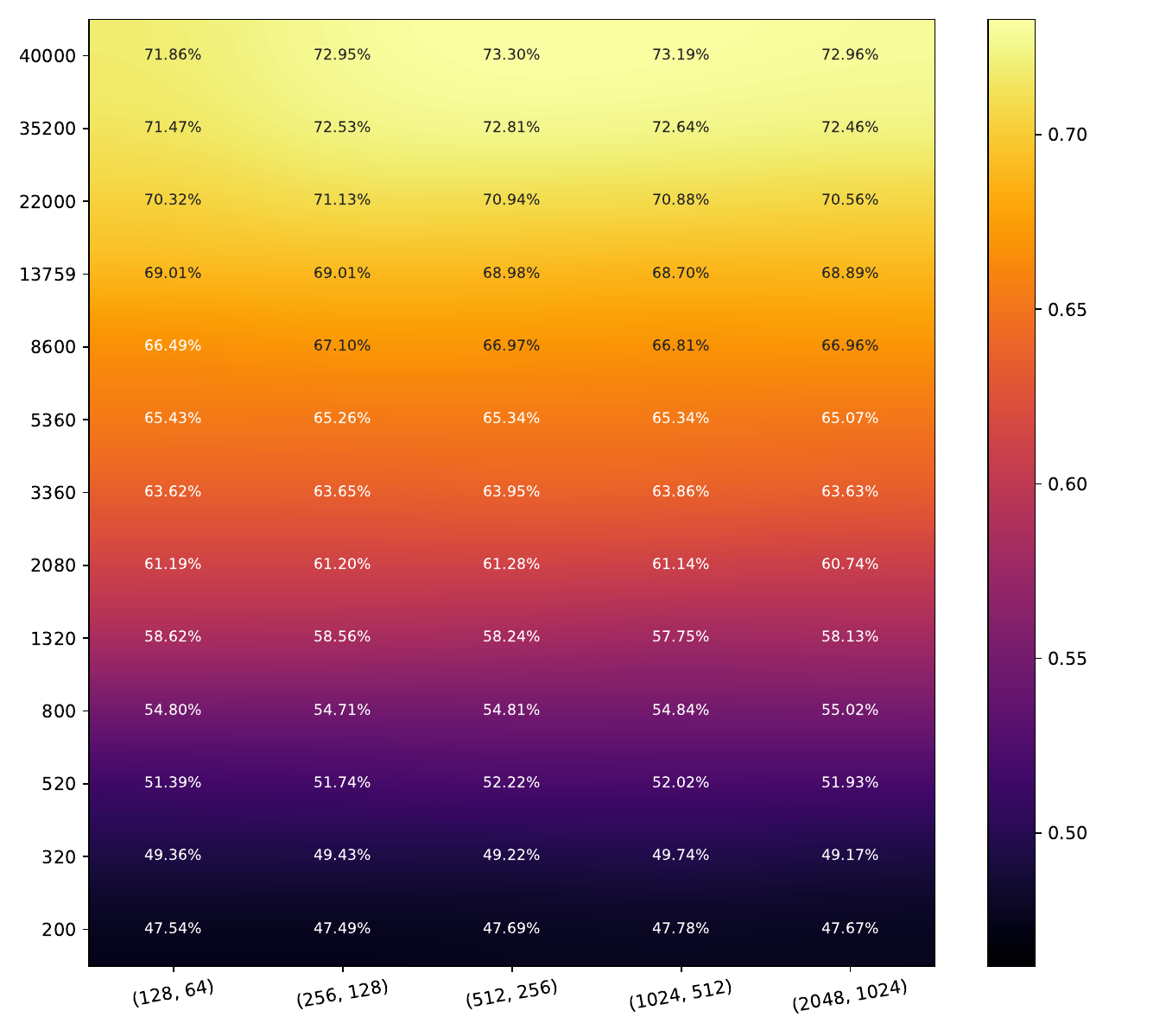}
    \caption*{Conflictual DE}
  \end{subfigure}\hfill
  \caption{Heatmaps of the accuracy. Color scales are the same per dataset.}
  \label{fig:accuracy}
\end{figure}

\begin{figure}[ht]
  \centering
  \begin{subfigure}[t]{\dimexpr0.185\textwidth+20pt\relax}
    \makebox[20pt]{\raisebox{25pt}{\rotatebox[origin=c]{90}{\scriptsize MNIST}}}%
    \includegraphics[width=\dimexpr\linewidth-20pt\relax]{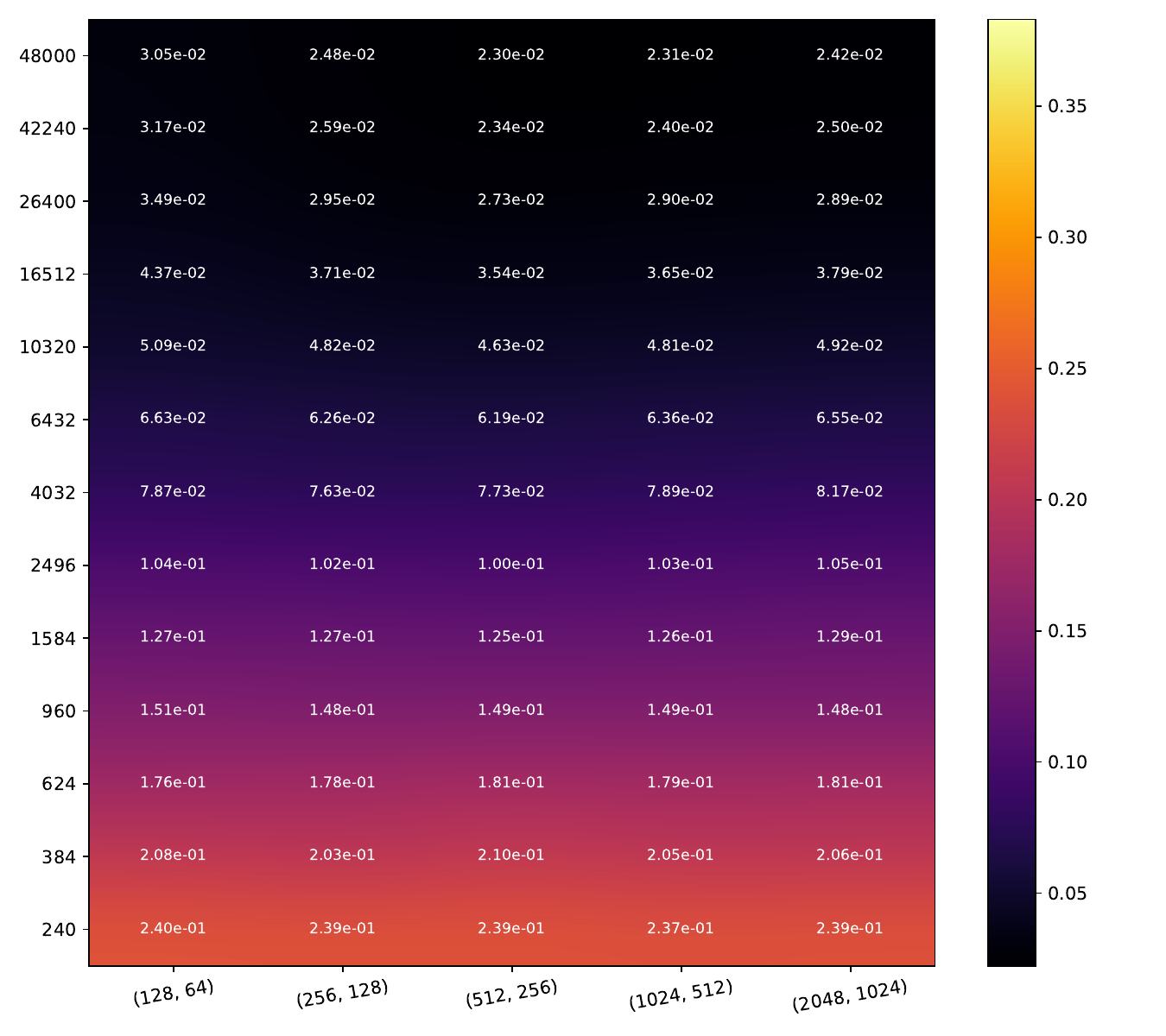}
    \makebox[20pt]{\raisebox{25pt}{\rotatebox[origin=c]{90}{\scriptsize SVHN}}}%
    \includegraphics[width=\dimexpr\linewidth-20pt\relax]{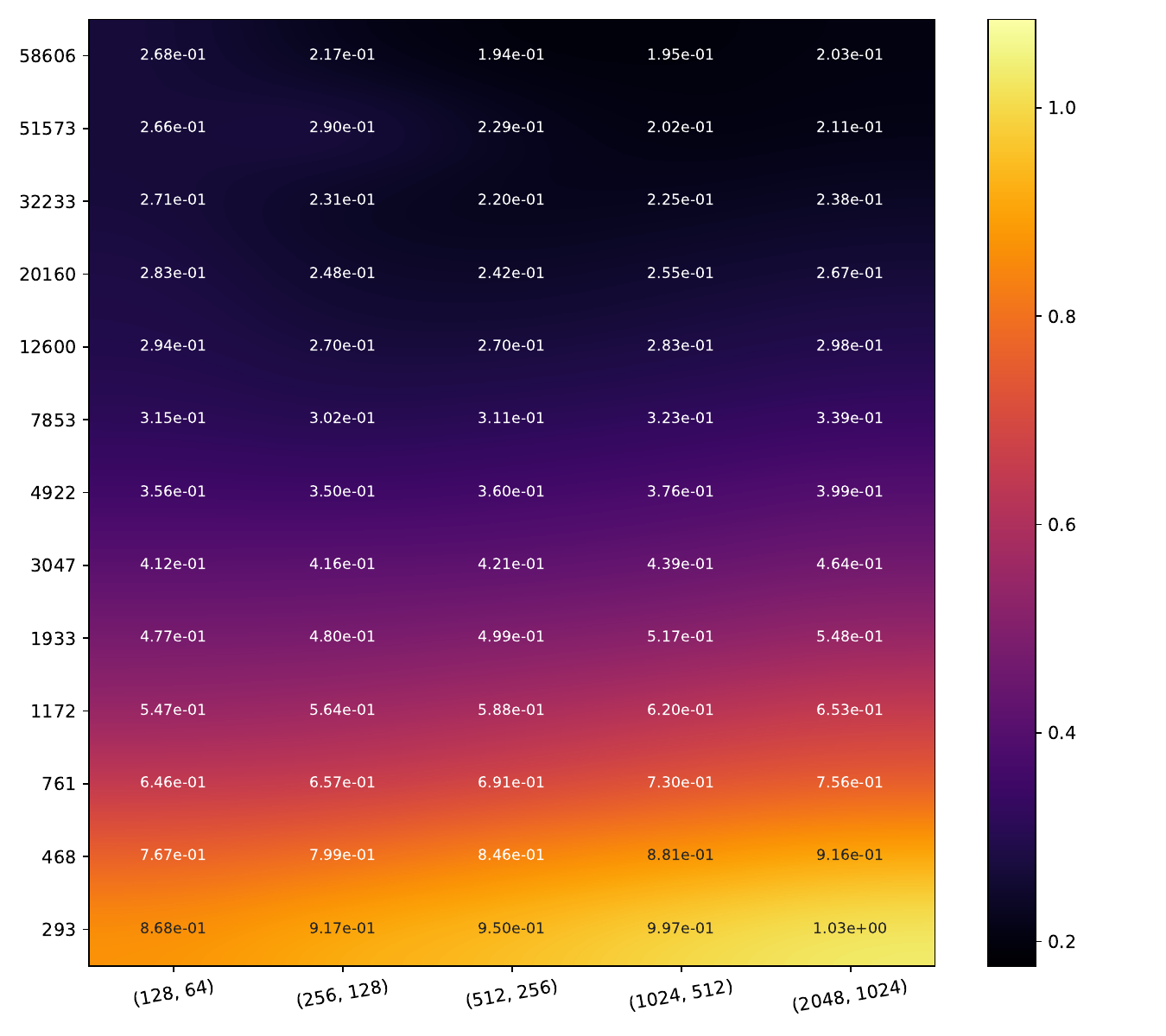}
    \makebox[20pt]{\raisebox{25pt}{\rotatebox[origin=c]{90}{\scriptsize CIFAR10}}}%
    \includegraphics[width=\dimexpr\linewidth-20pt\relax]{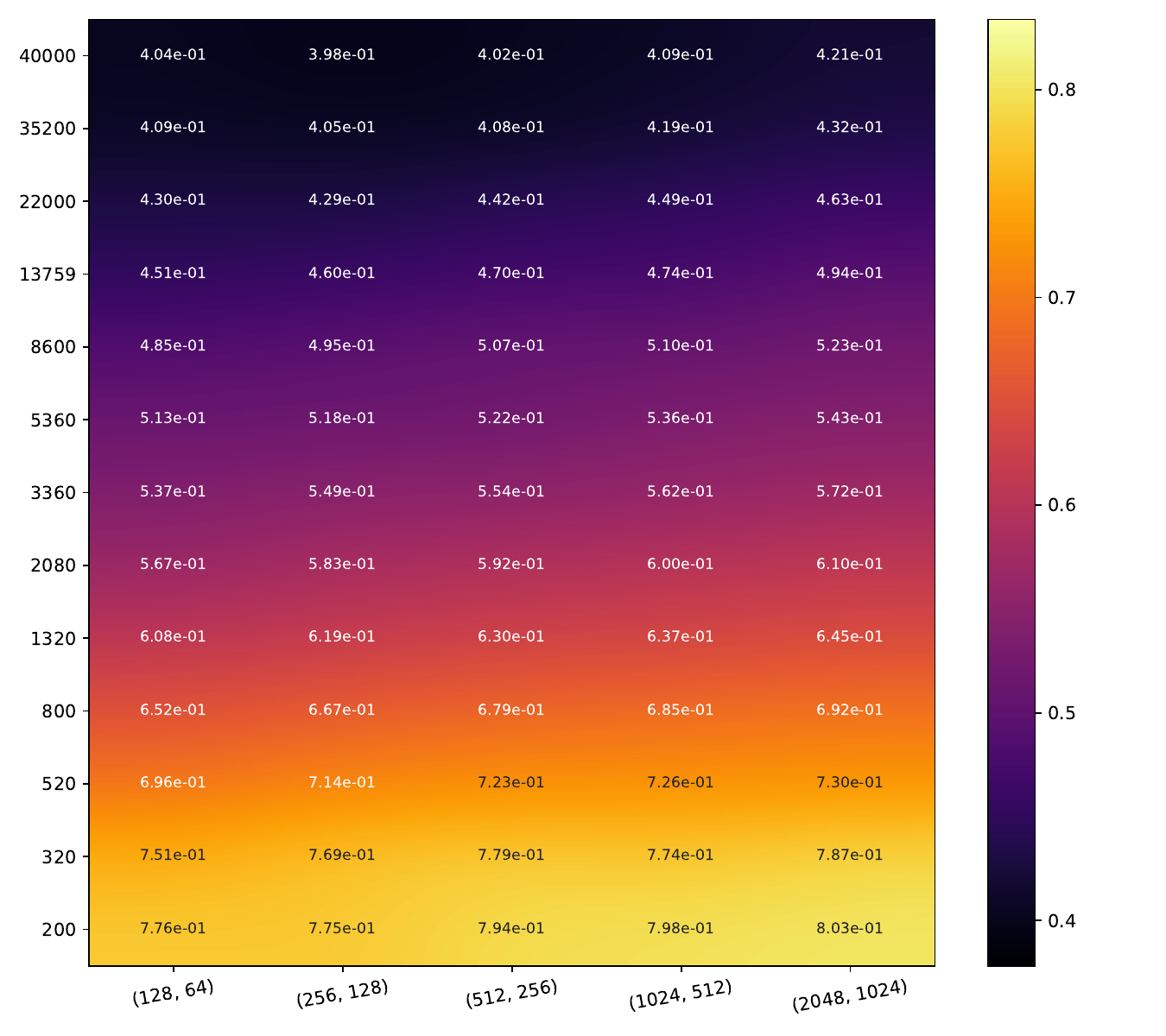}
    \caption*{\qquad MC-Dropout}
  \end{subfigure}\hfill
  \begin{subfigure}[t]{0.185\textwidth}
    \includegraphics[width=\textwidth]{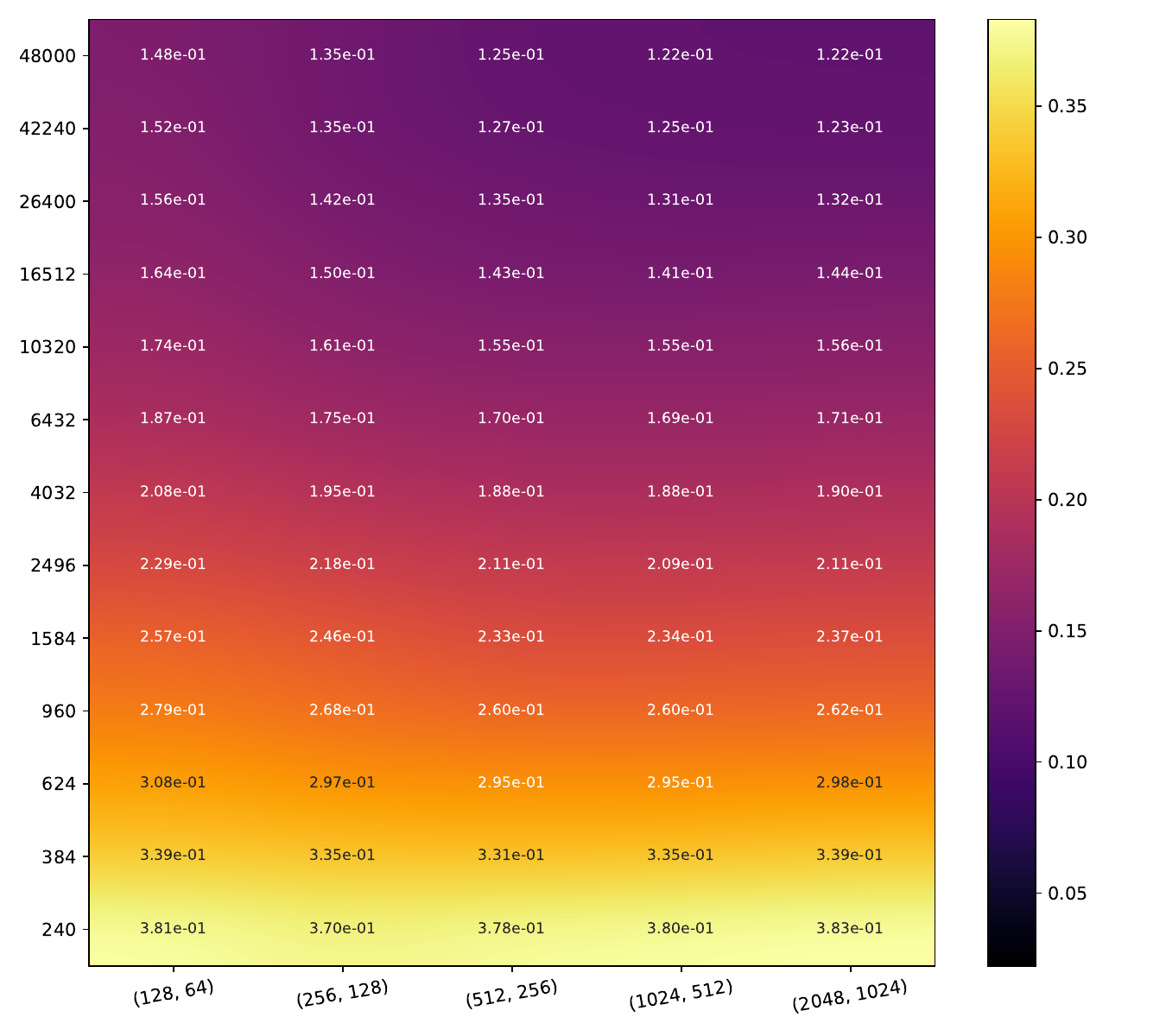}
    \includegraphics[width=\textwidth]{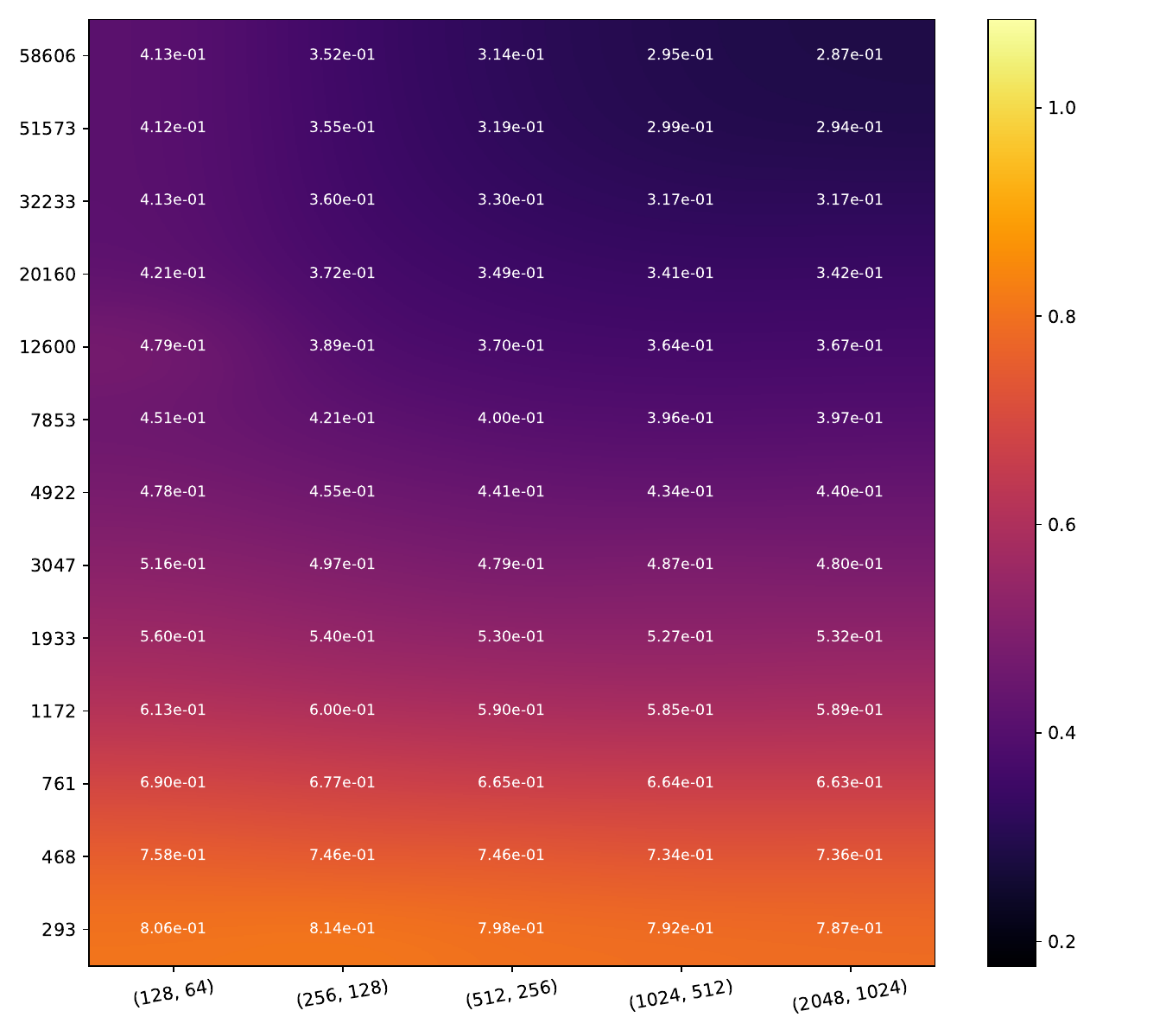}
    \includegraphics[width=\textwidth]{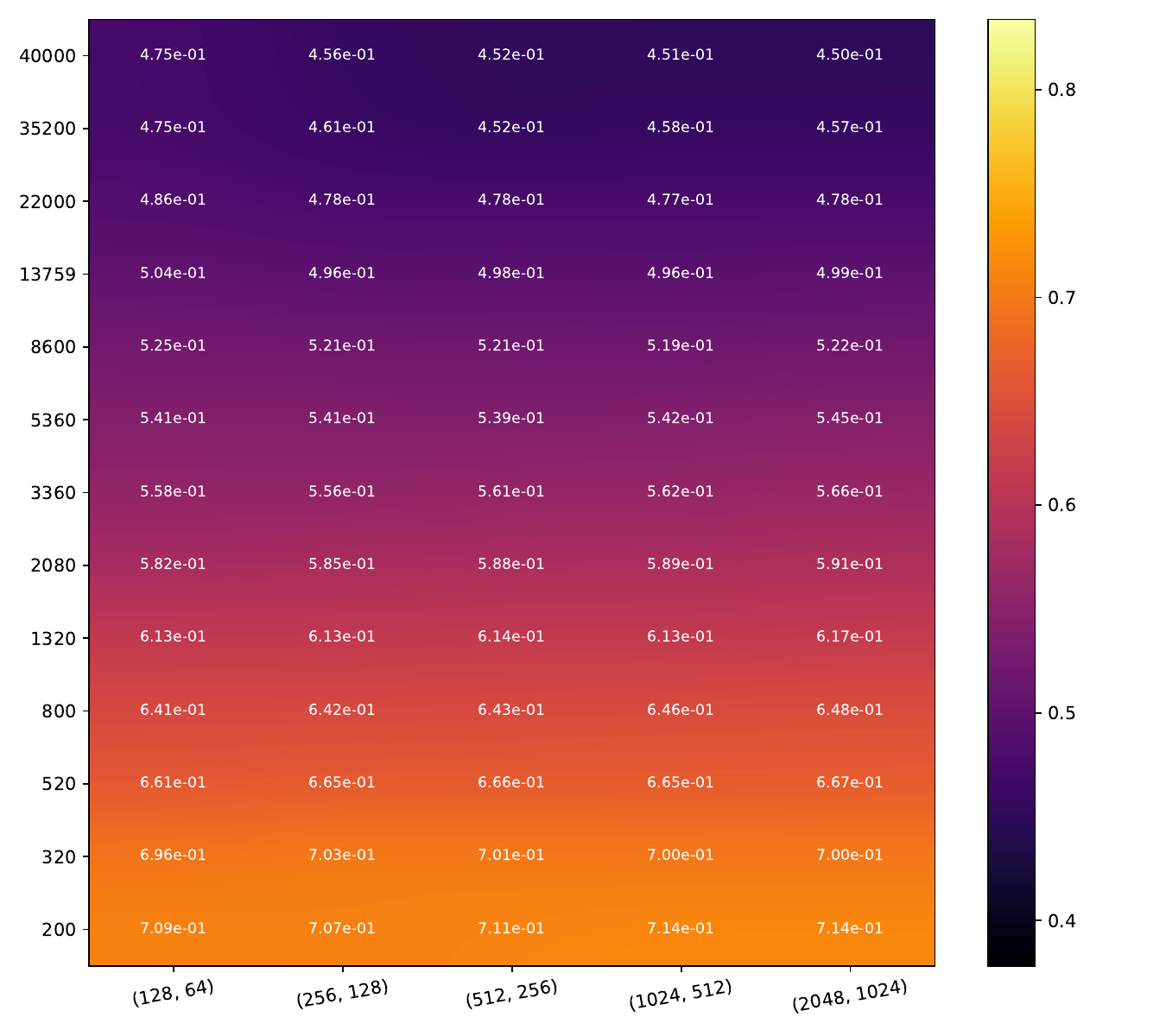}
    \caption*{MC-Dropout LS}
  \end{subfigure}\hfill
  \begin{subfigure}[t]{0.185\textwidth}
    \includegraphics[width=\textwidth]{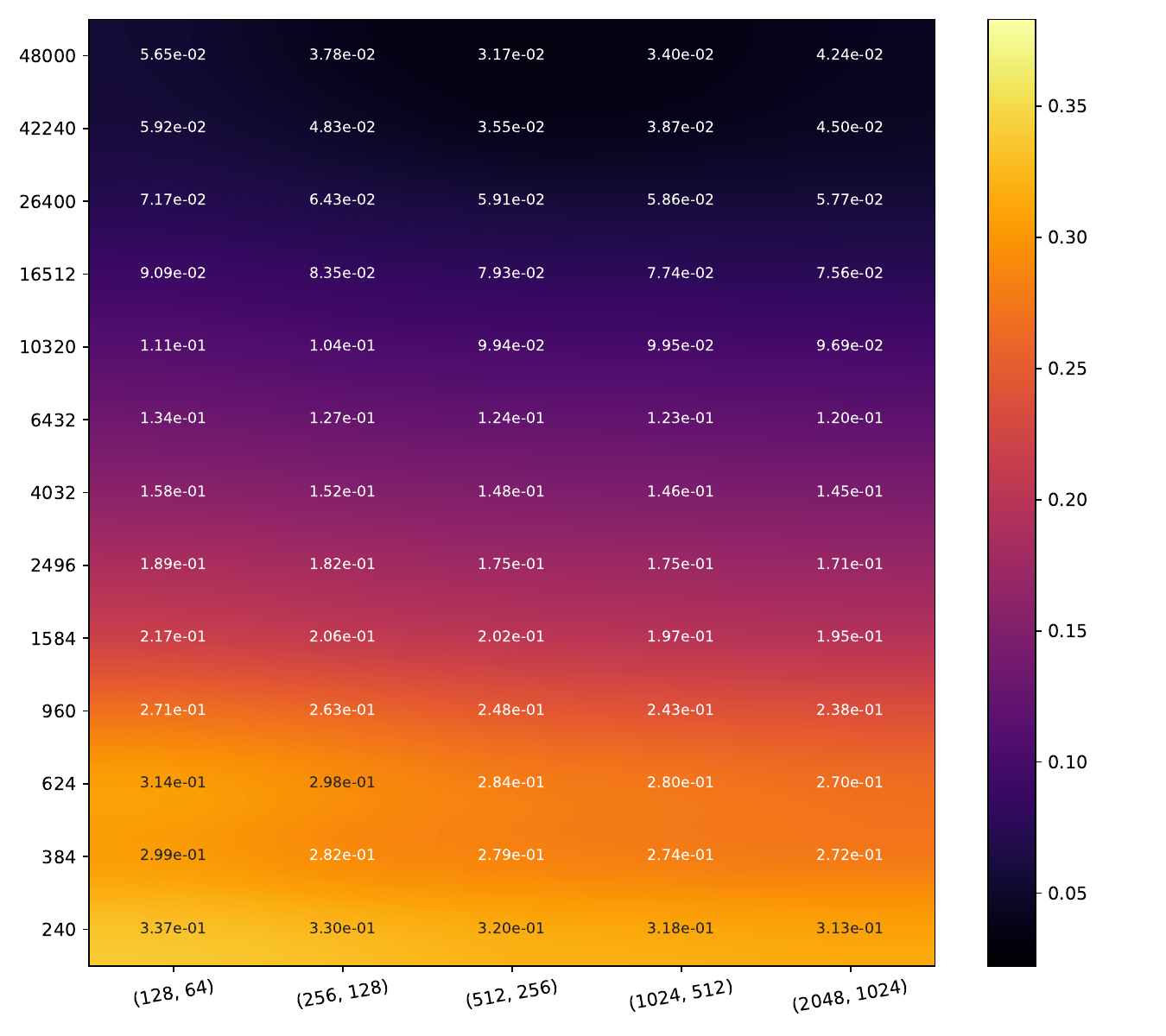}
    \includegraphics[width=\textwidth]{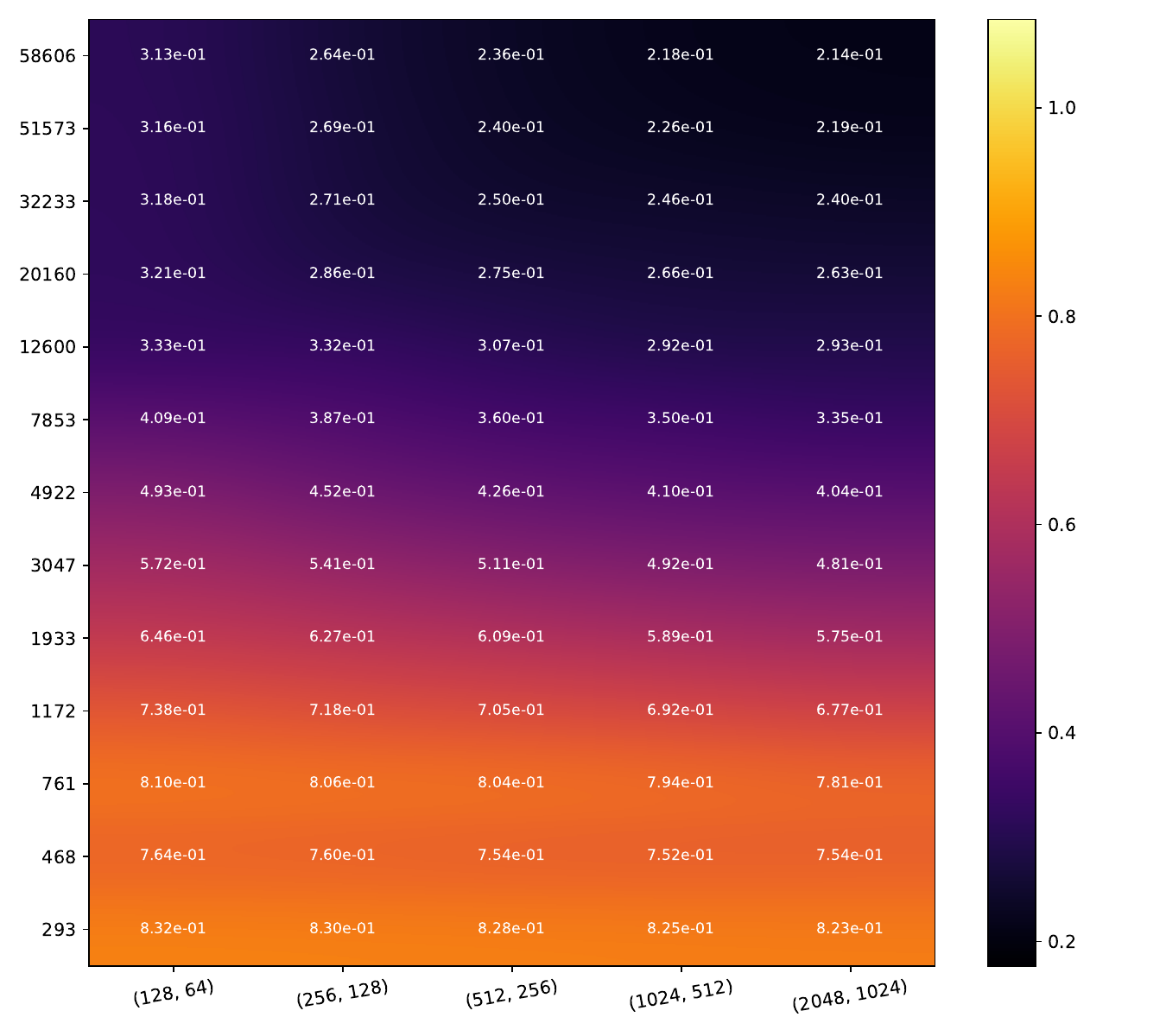}
    \includegraphics[width=\textwidth]{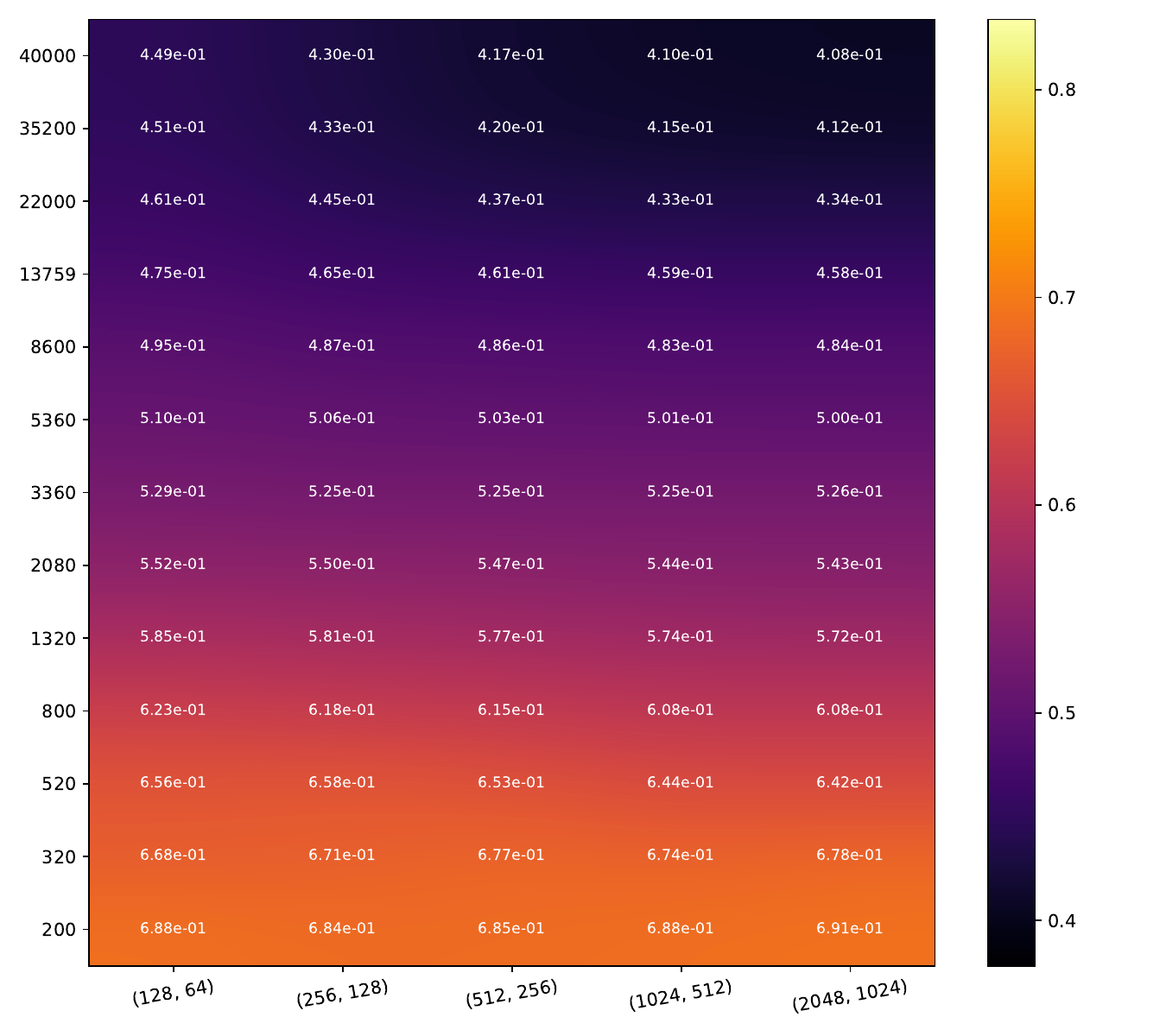}
    \caption*{EDL}
  \end{subfigure}\hfill
  \begin{subfigure}[t]{0.185\textwidth}
    \includegraphics[width=\textwidth]{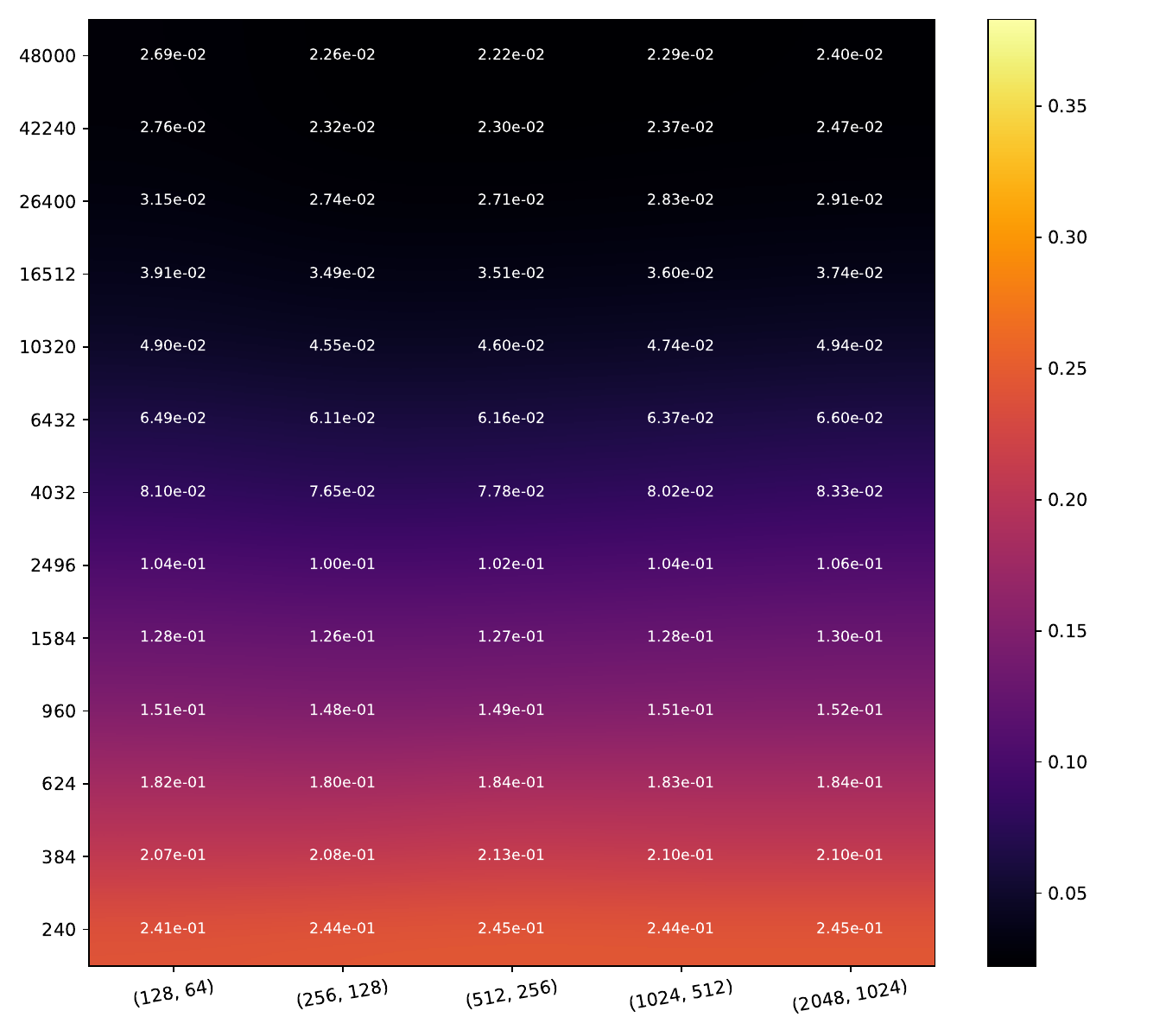}
    \includegraphics[width=\textwidth]{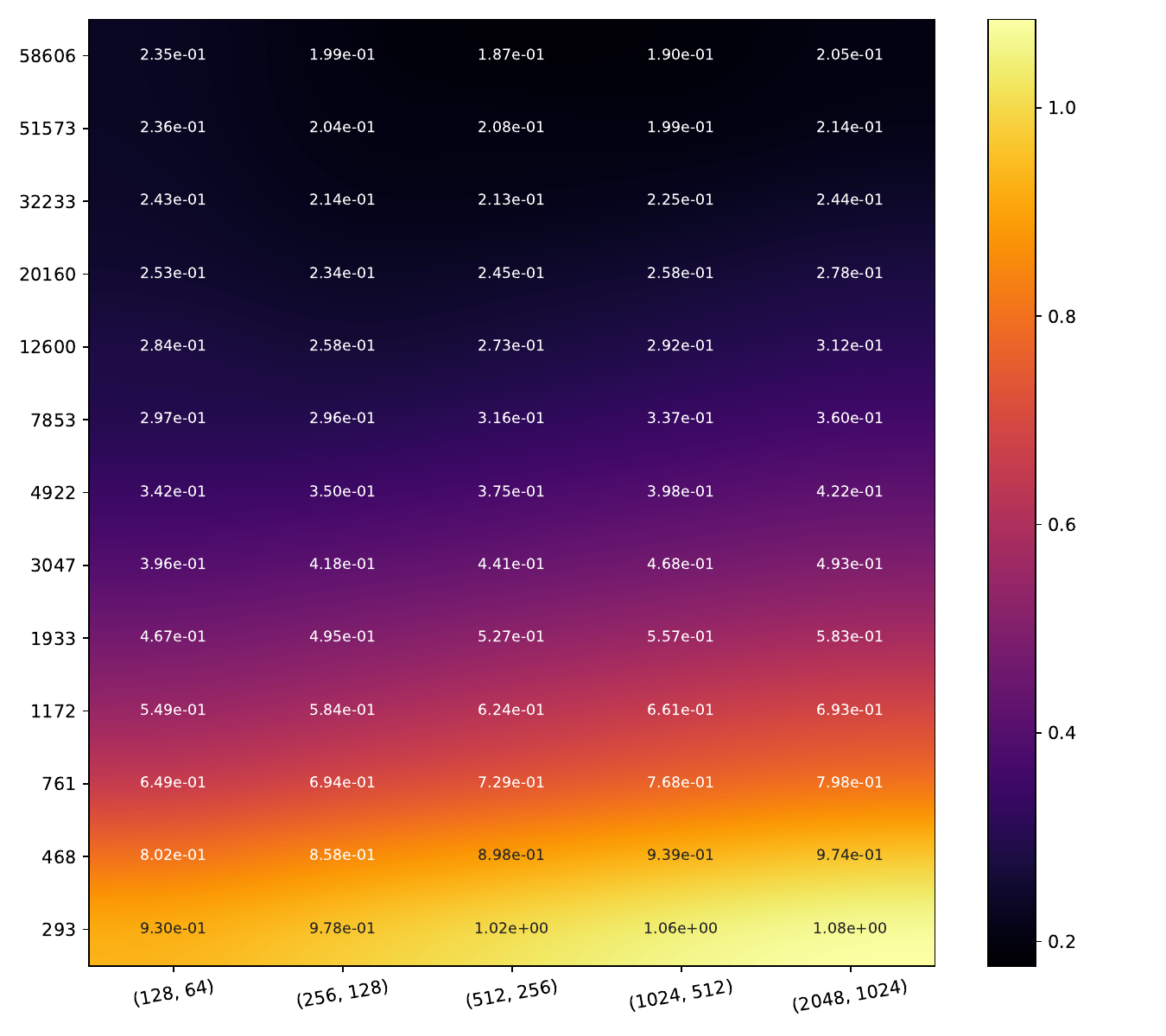}
    \includegraphics[width=\textwidth]{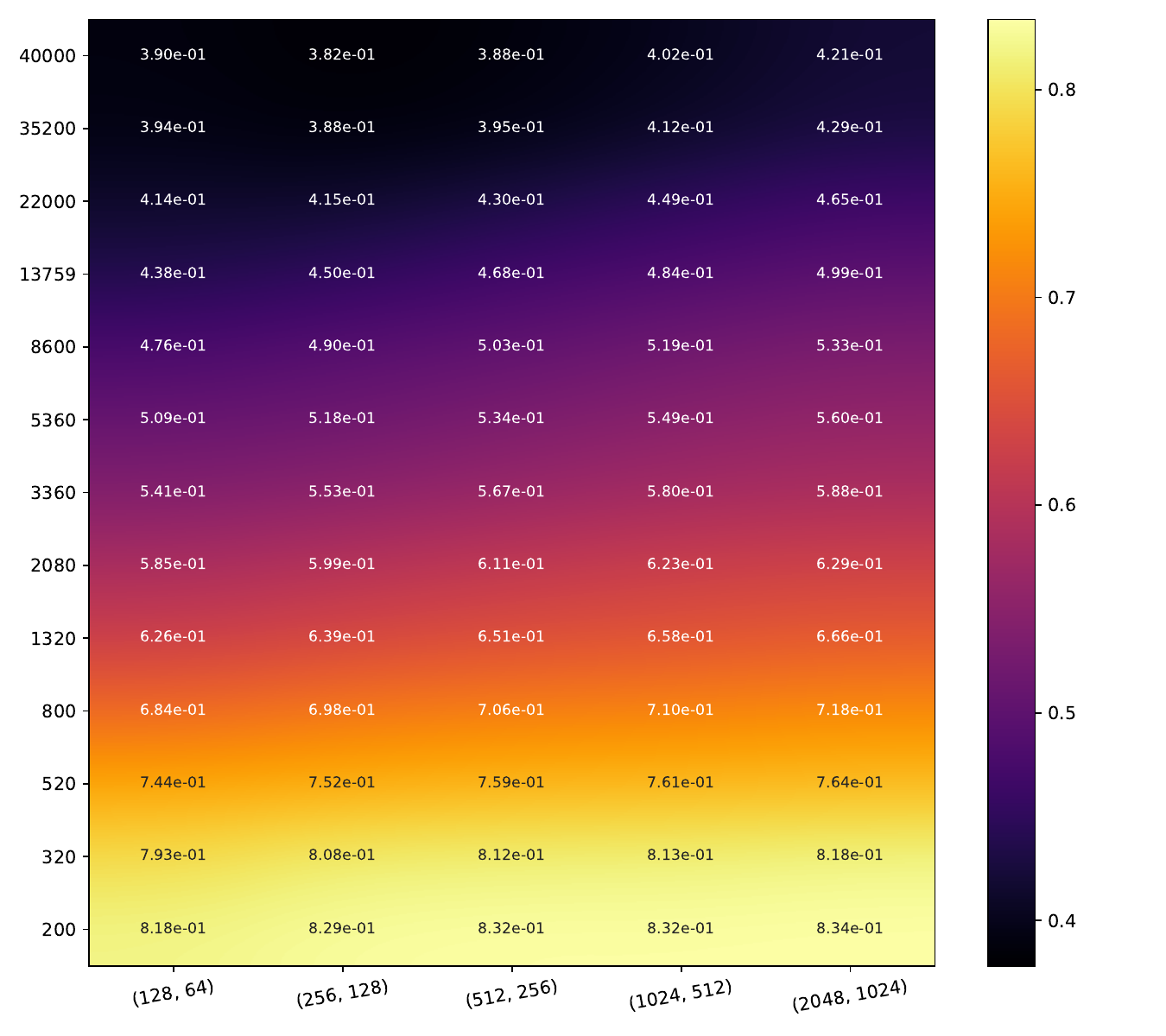}
    \caption*{DE}
  \end{subfigure}\hfill
  \begin{subfigure}[t]{0.185\textwidth}
    \includegraphics[width=\textwidth]{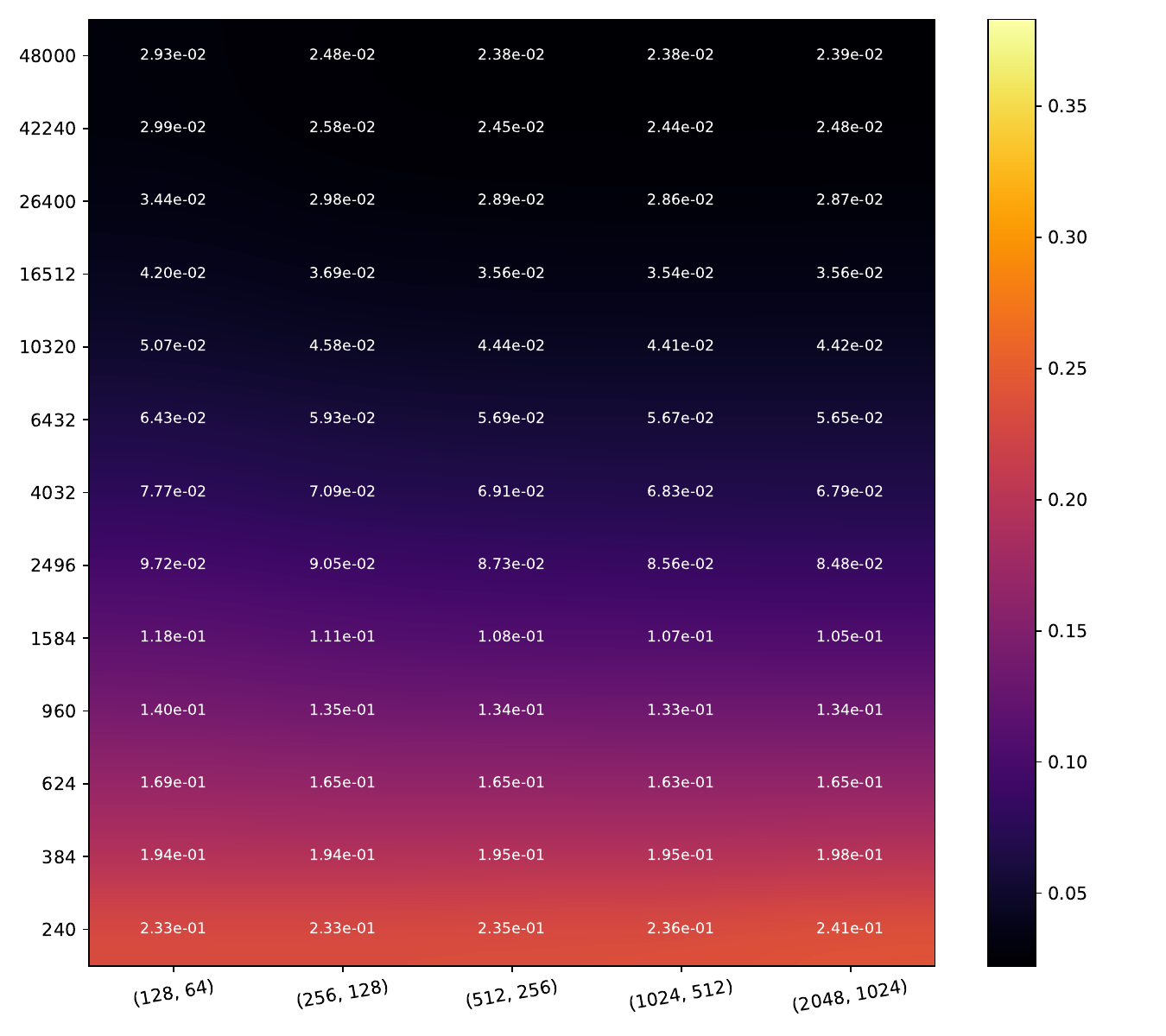}
    \includegraphics[width=\textwidth]{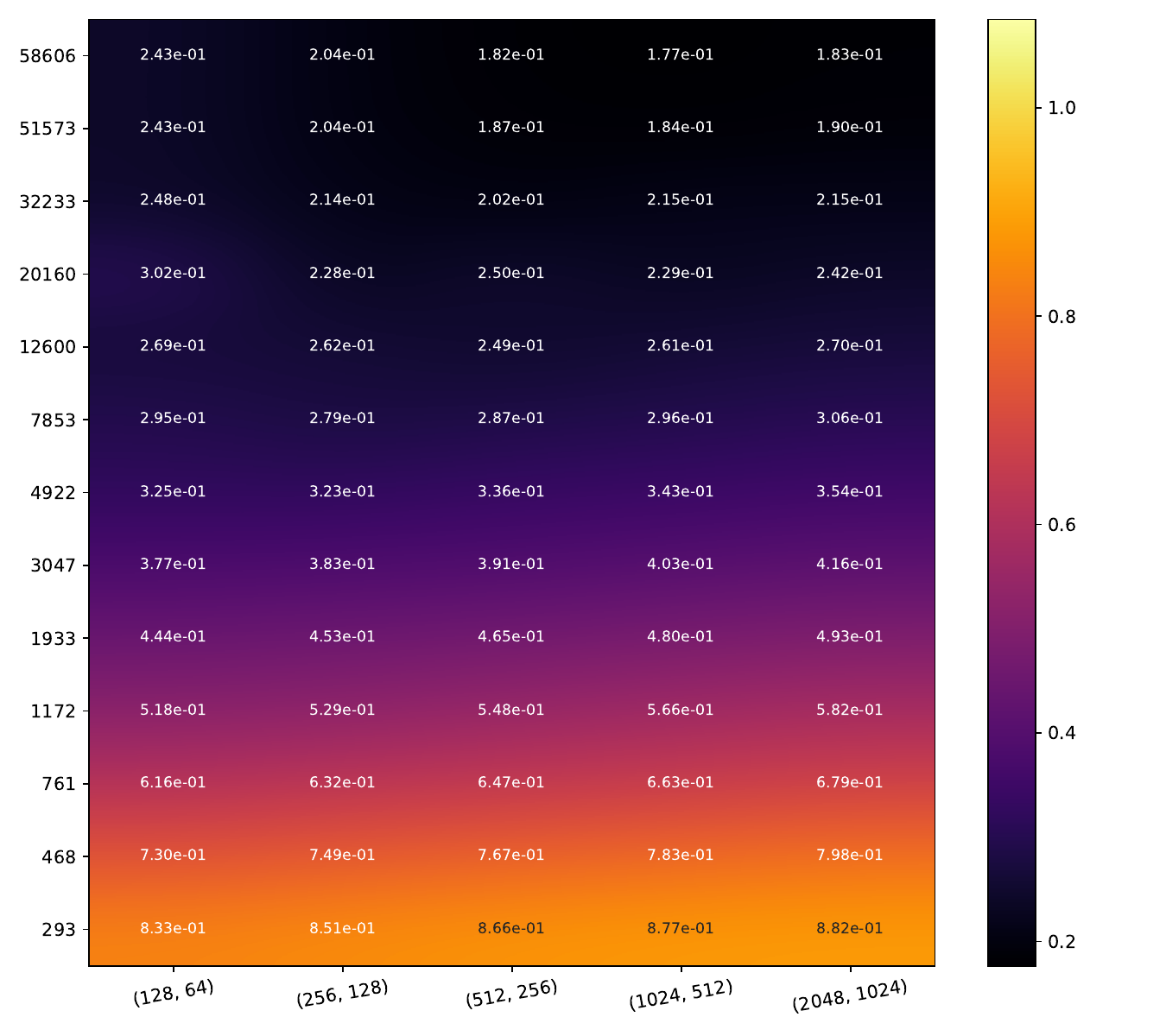}
    \includegraphics[width=\textwidth]{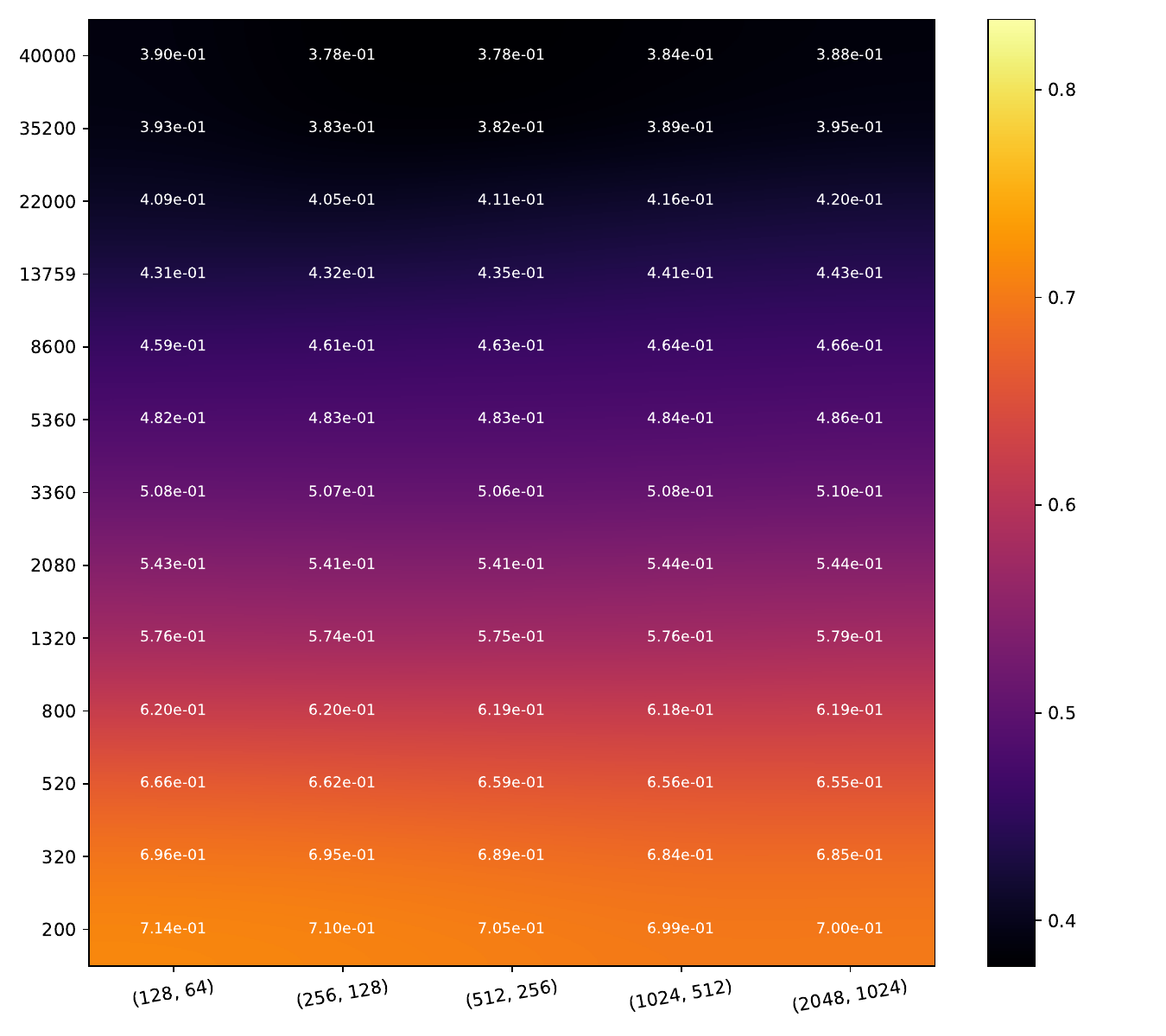}
    \caption*{Conflictual DE}
  \end{subfigure}\hfill
  \caption{Heatmaps of the Brier score. Color scales are the same per dataset.}
  \label{fig:brier}
\end{figure}

\begin{figure}[ht]
  \begin{subfigure}[t]{\dimexpr0.185\textwidth+20pt\relax}
    \makebox[20pt]{\raisebox{25pt}{\rotatebox[origin=c]{90}{\scriptsize MNIST}}}%
    \includegraphics[width=\dimexpr\linewidth-20pt\relax]{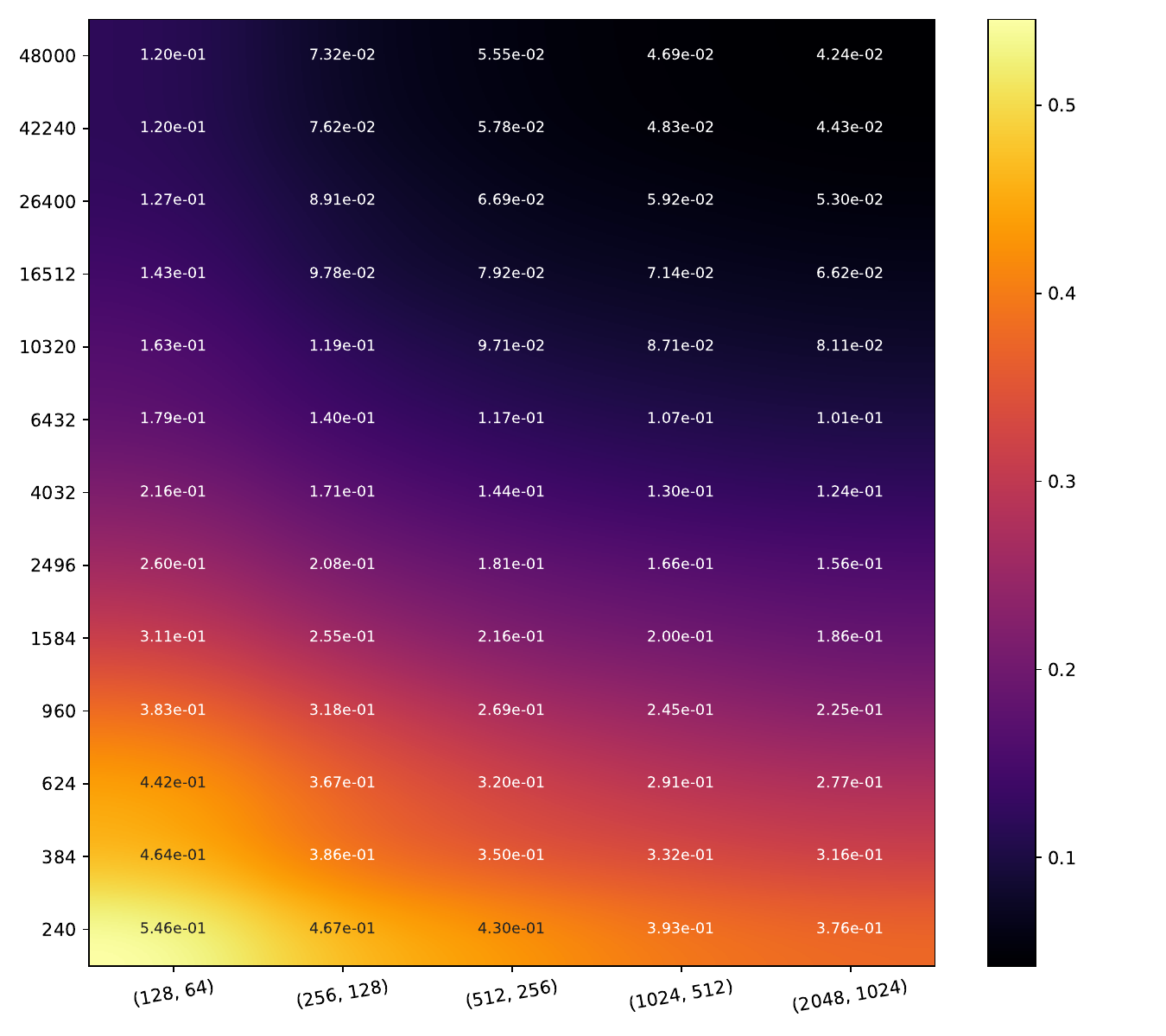}
    \makebox[20pt]{\raisebox{25pt}{\rotatebox[origin=c]{90}{\scriptsize SVHN}}}%
    \includegraphics[width=\dimexpr\linewidth-20pt\relax]{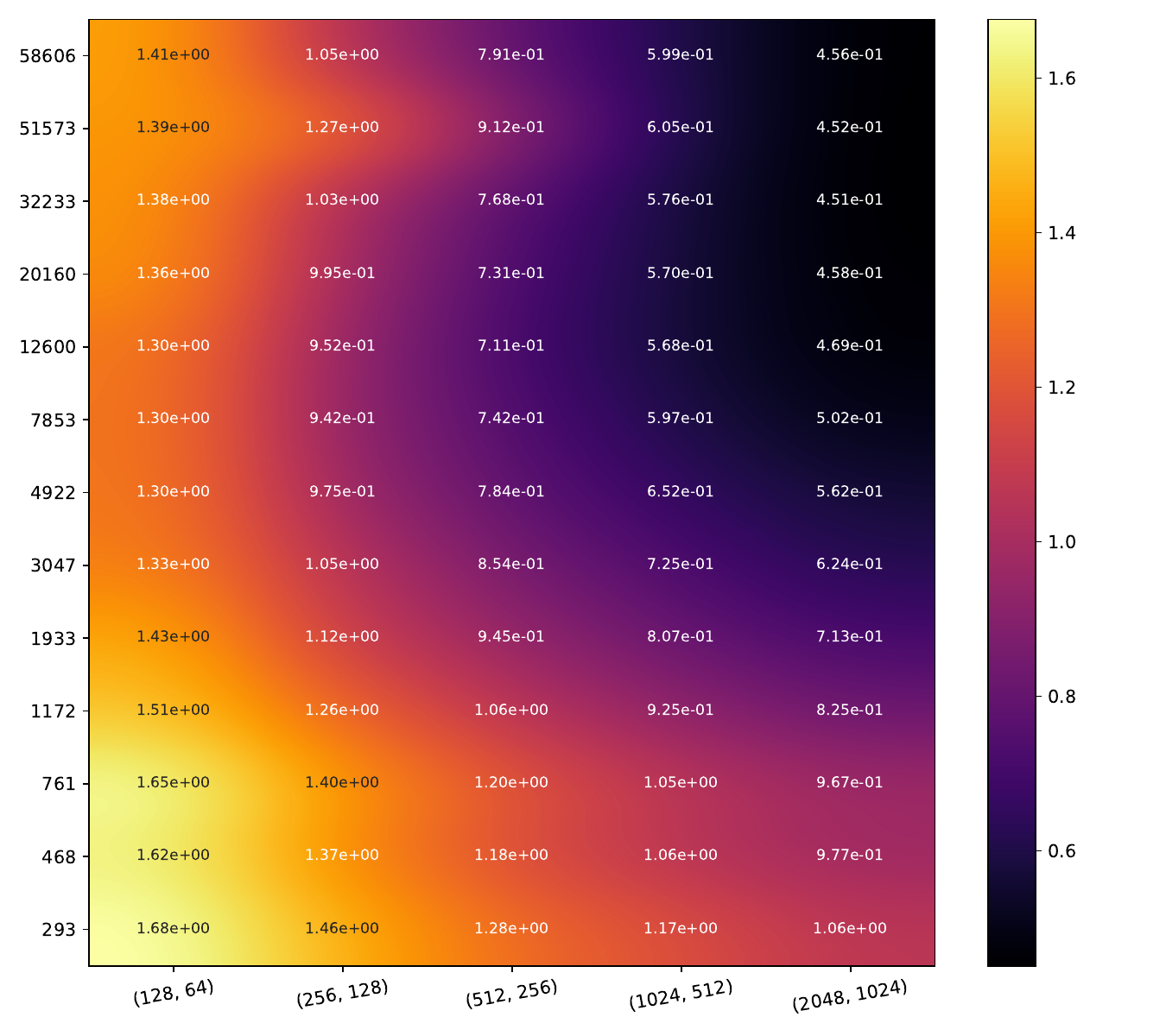}
    \makebox[20pt]{\raisebox{25pt}{\rotatebox[origin=c]{90}{\scriptsize CIFAR10}}}%
    \includegraphics[width=\dimexpr\linewidth-20pt\relax]{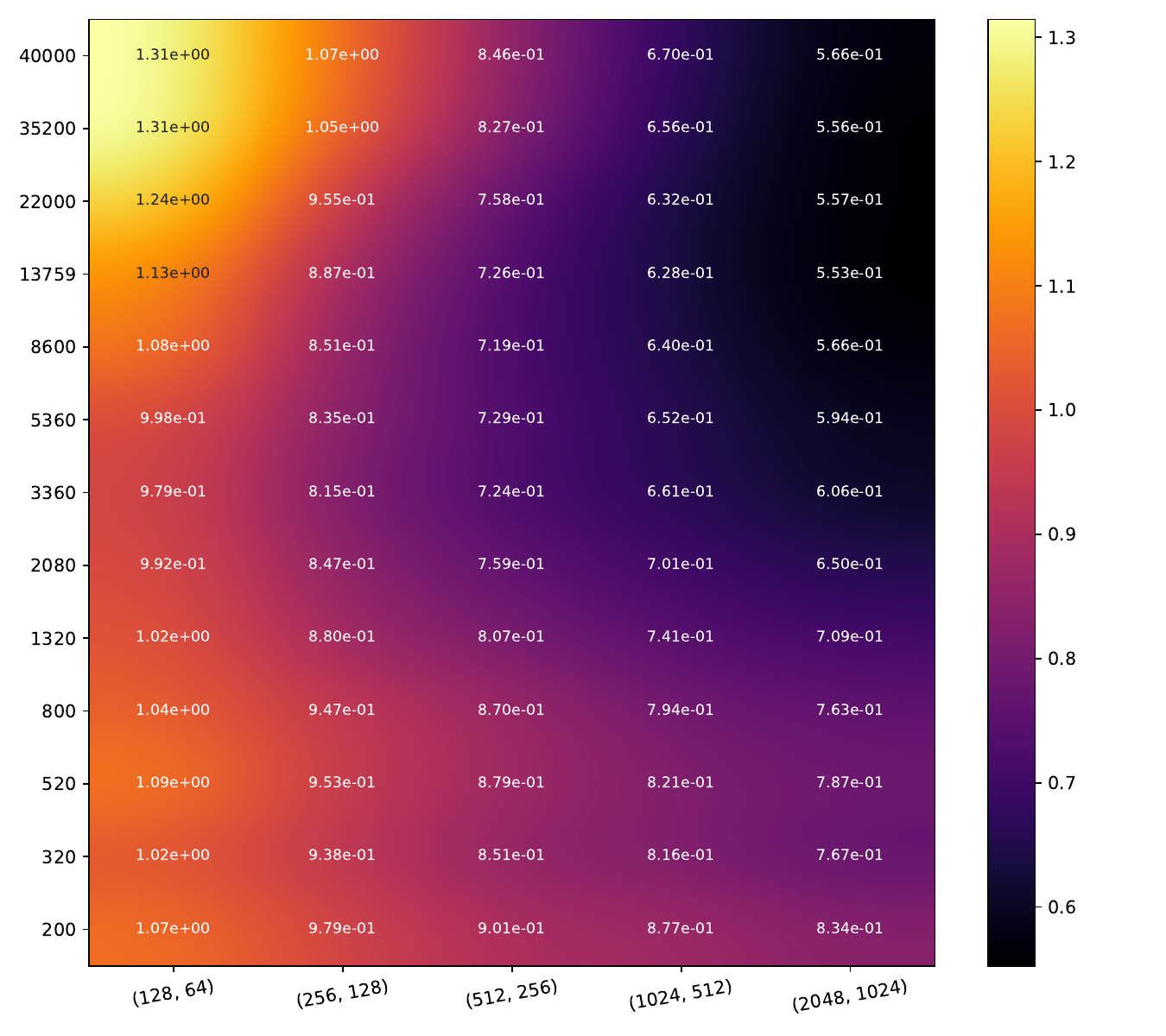}
    \caption*{\qquad MC-Dropout}
  \end{subfigure}\hfill
  \begin{subfigure}[t]{0.185\textwidth}
    \includegraphics[width=\textwidth]{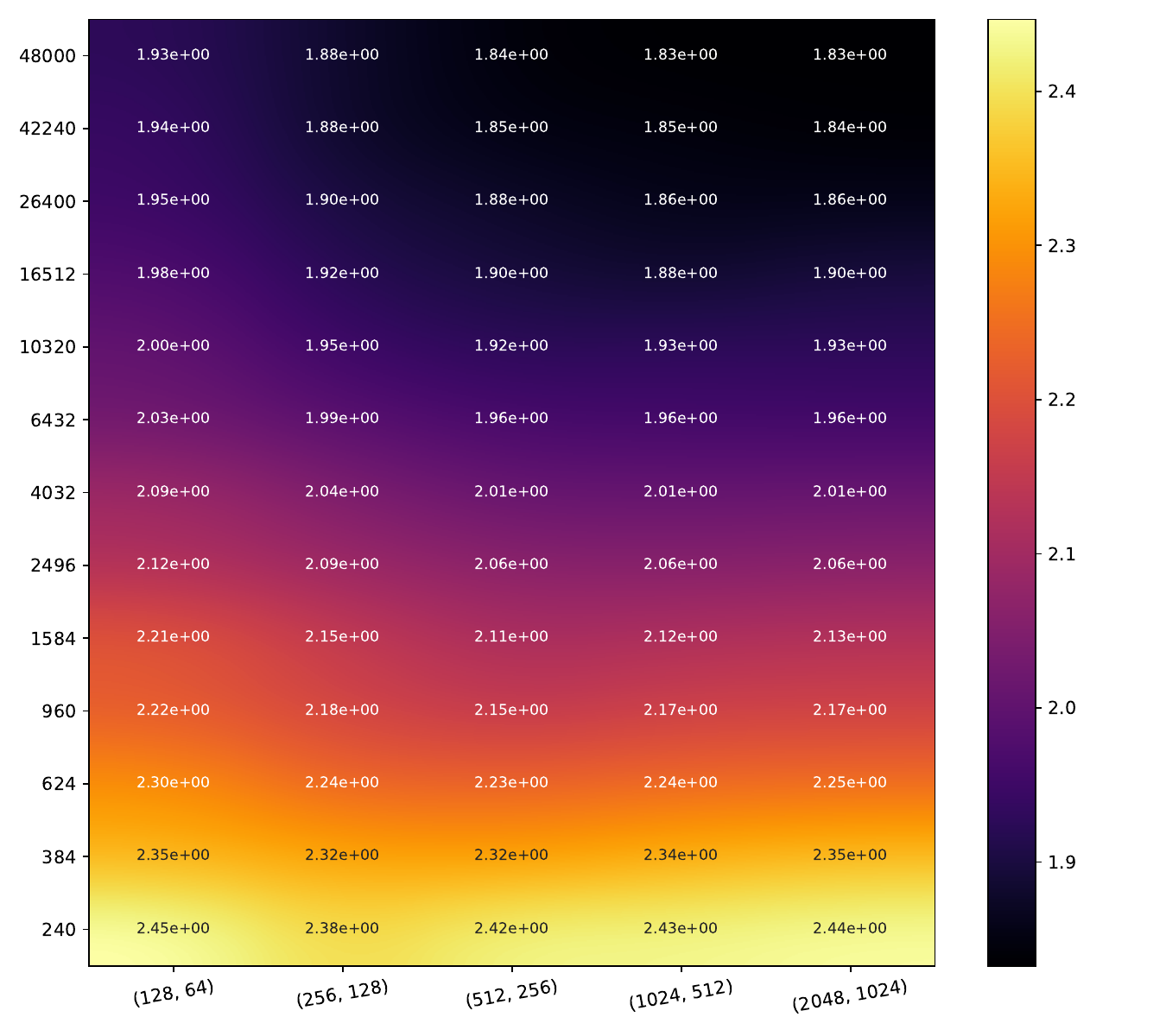}
    \includegraphics[width=\textwidth]{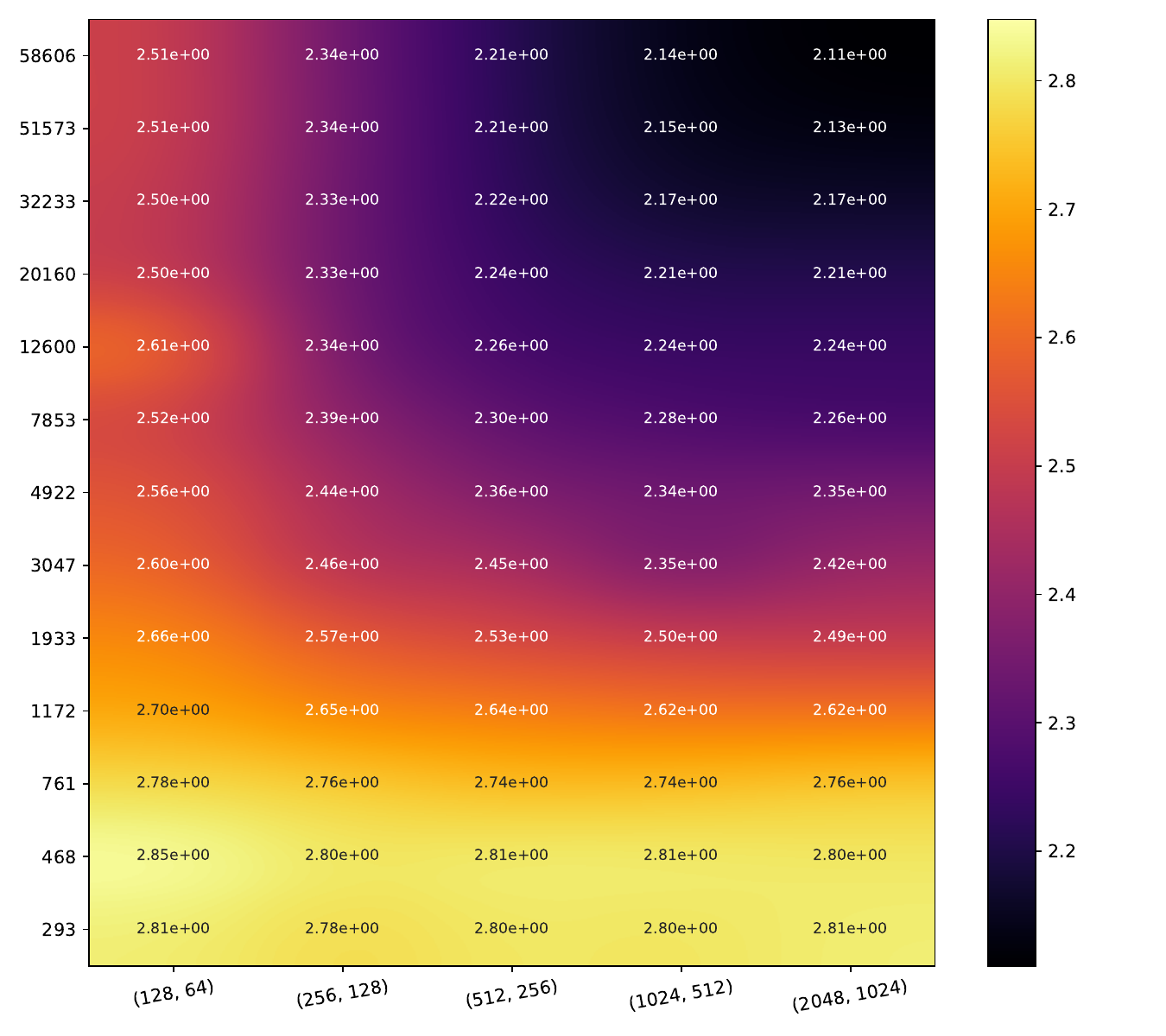}
    \includegraphics[width=\textwidth]{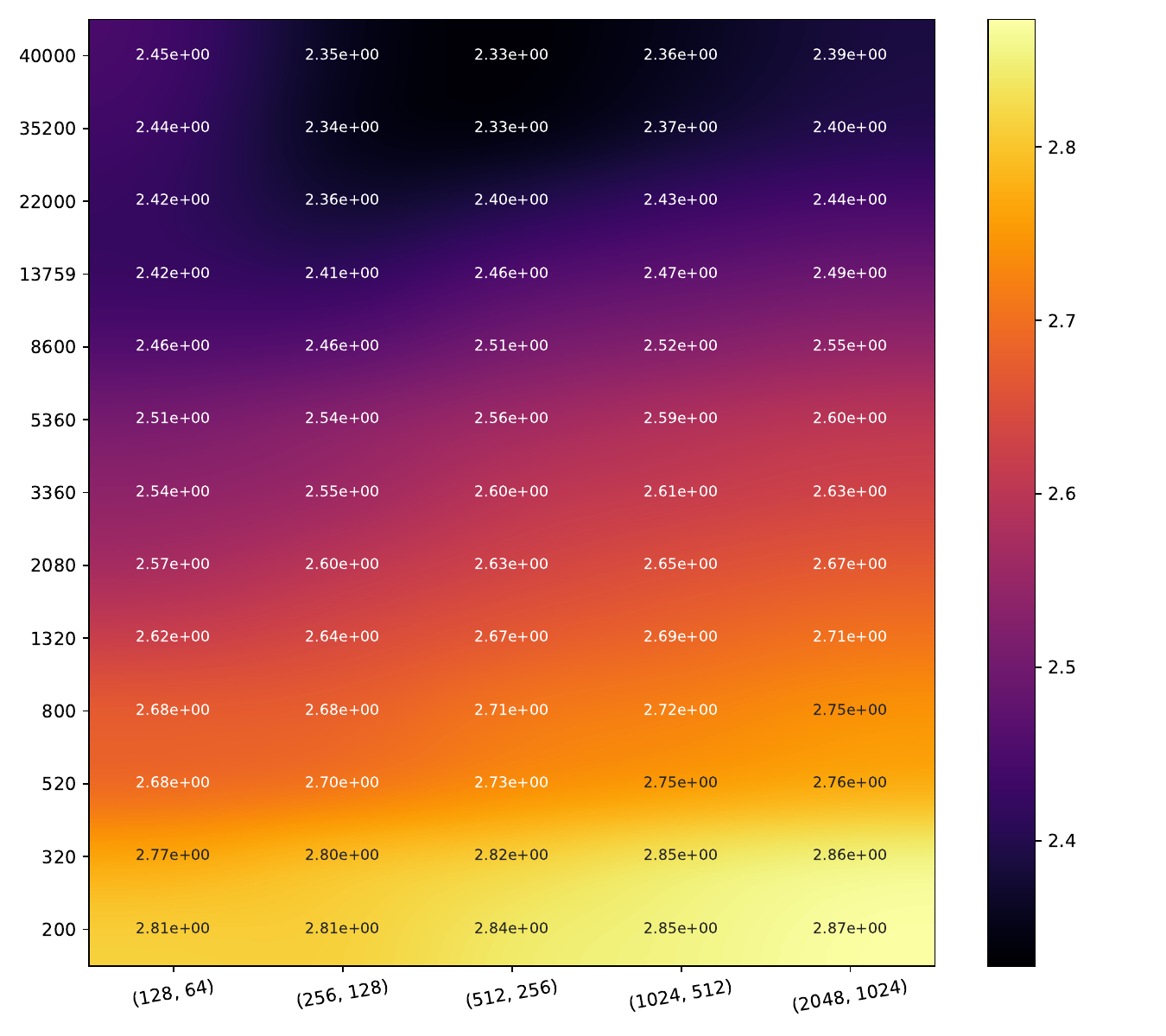}
    \caption*{MC-Dropout LS}
  \end{subfigure}\hfill
  \begin{subfigure}[t]{0.185\textwidth}
    \includegraphics[width=\textwidth]{figures/latest_cifar10_mcdropout_ls_default_mean_entropy_id.pdf}
    \includegraphics[width=\textwidth]{figures/latest_svhn_mcdropout_ls_default_mean_entropy_id.pdf}
    \includegraphics[width=\textwidth]{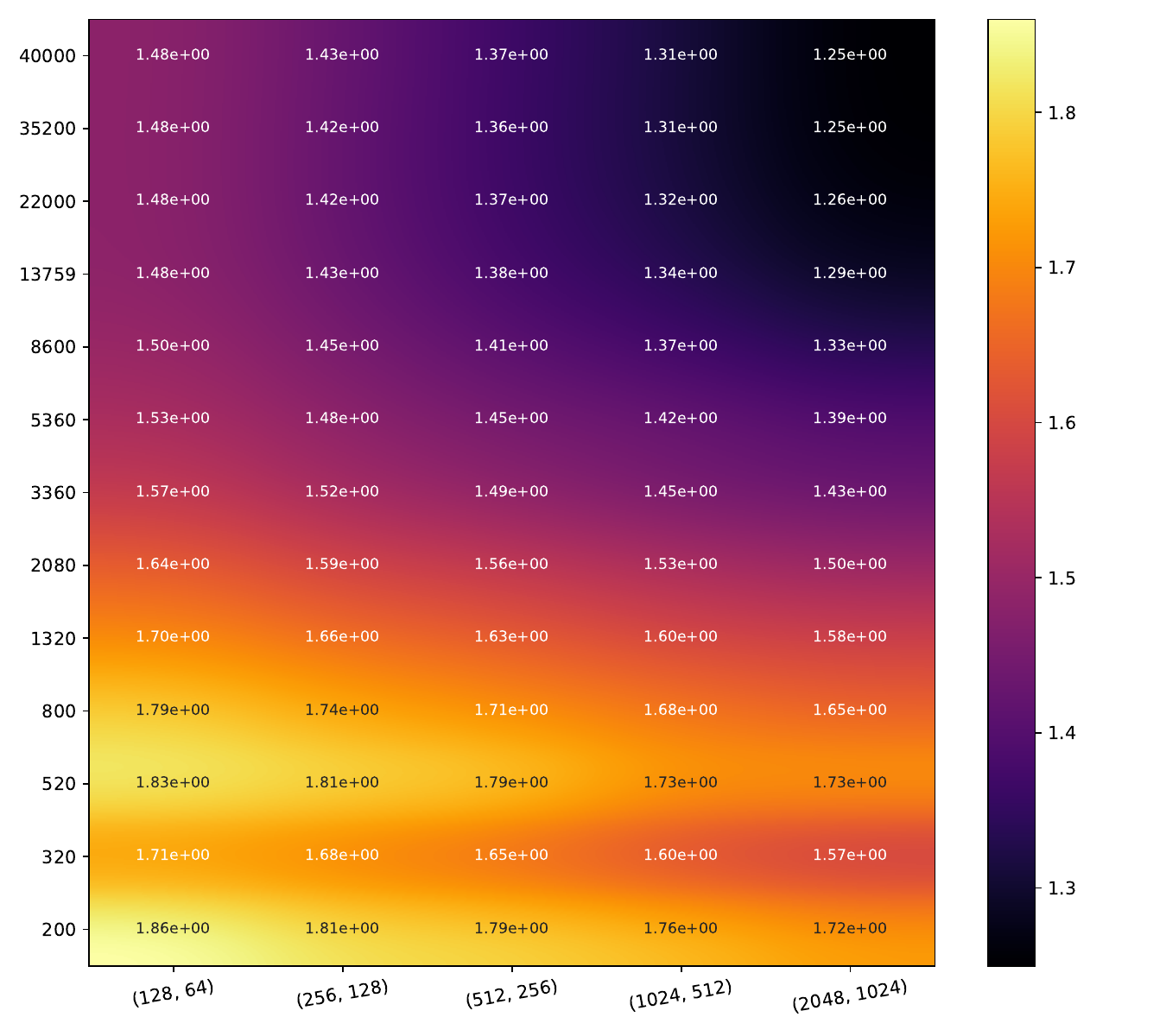}
    \caption*{EDL}
  \end{subfigure}\hfill
  \begin{subfigure}[t]{0.185\textwidth}
    \includegraphics[width=\textwidth]{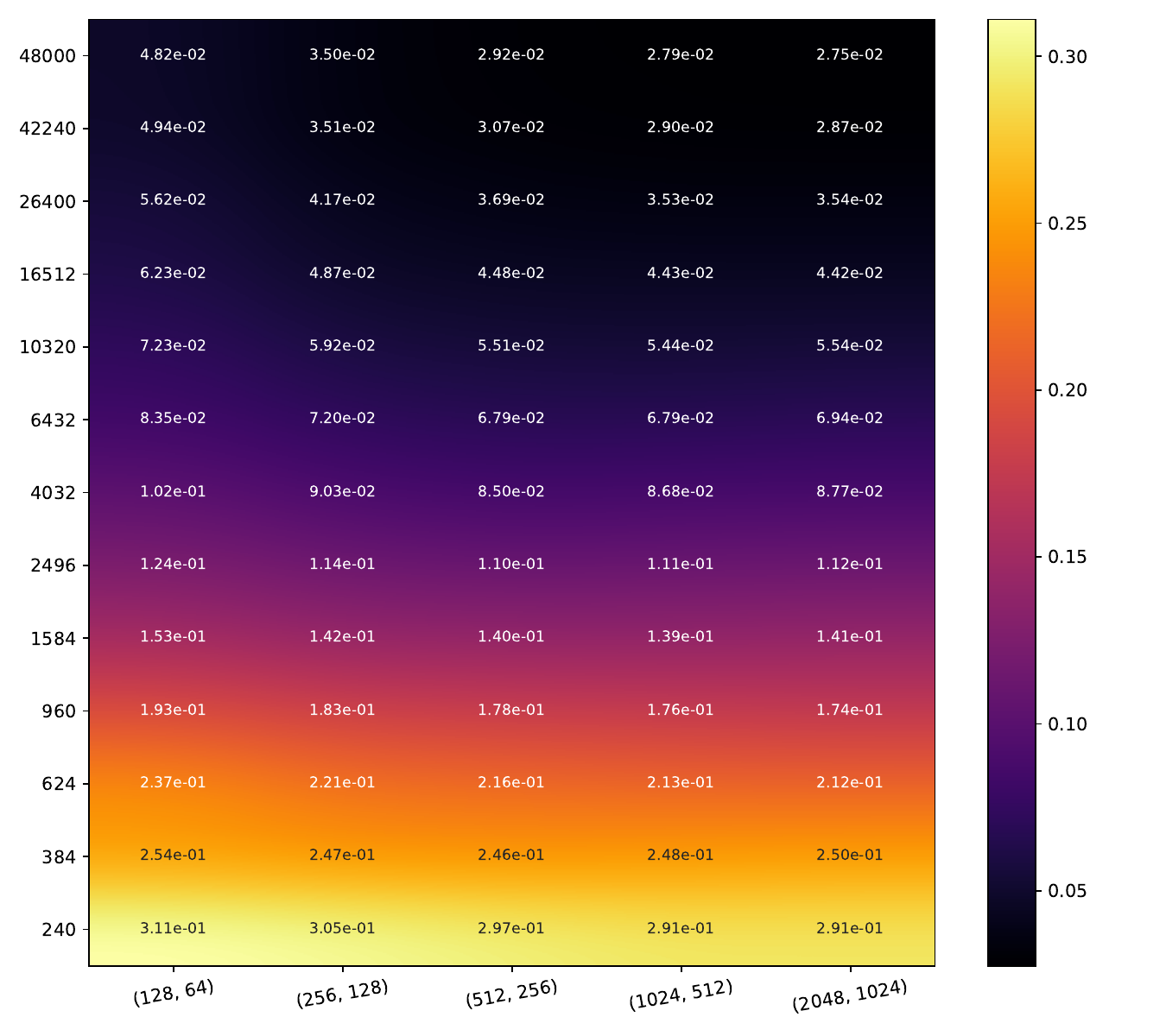}
    \includegraphics[width=\textwidth]{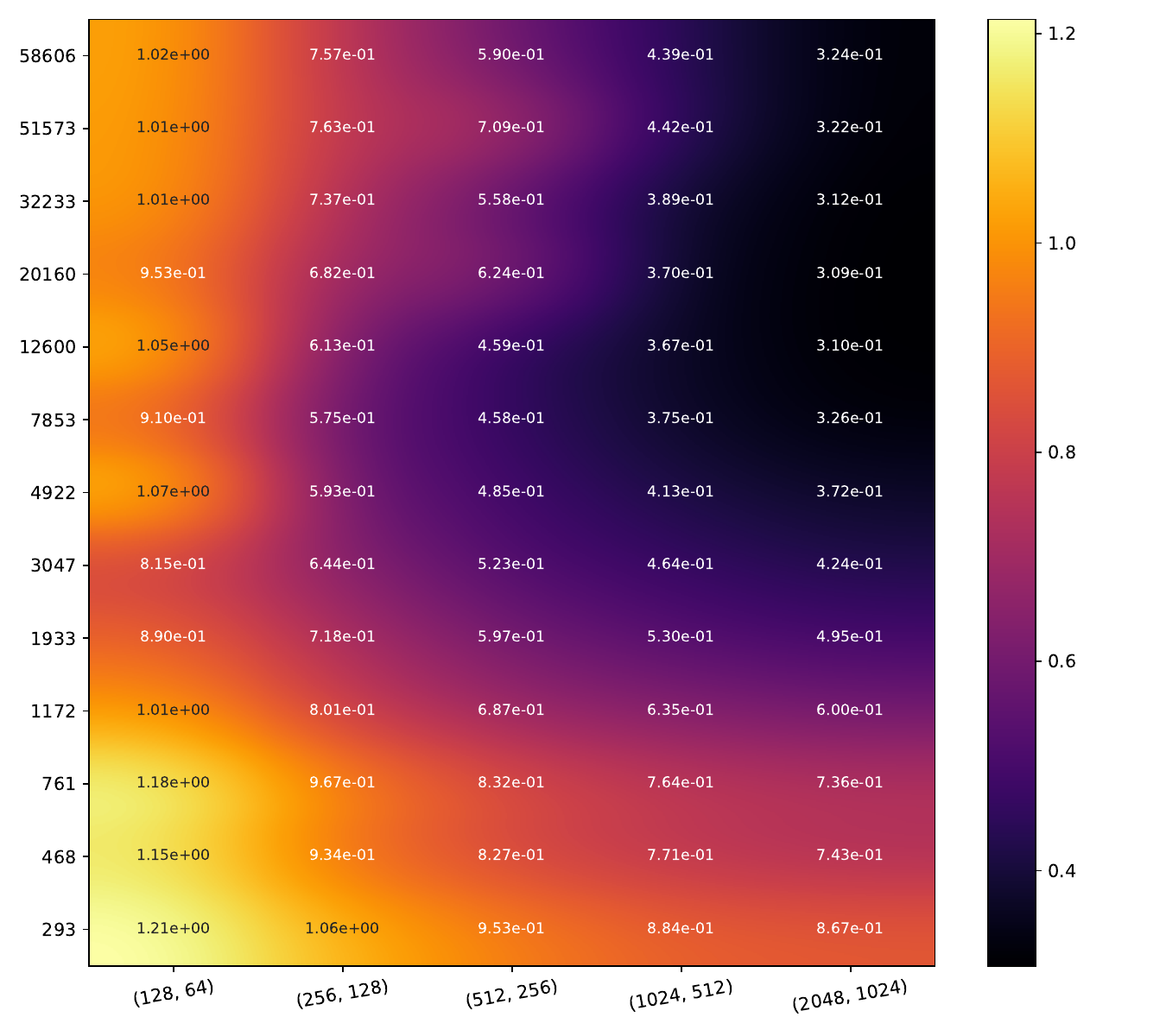}
    \includegraphics[width=\textwidth]{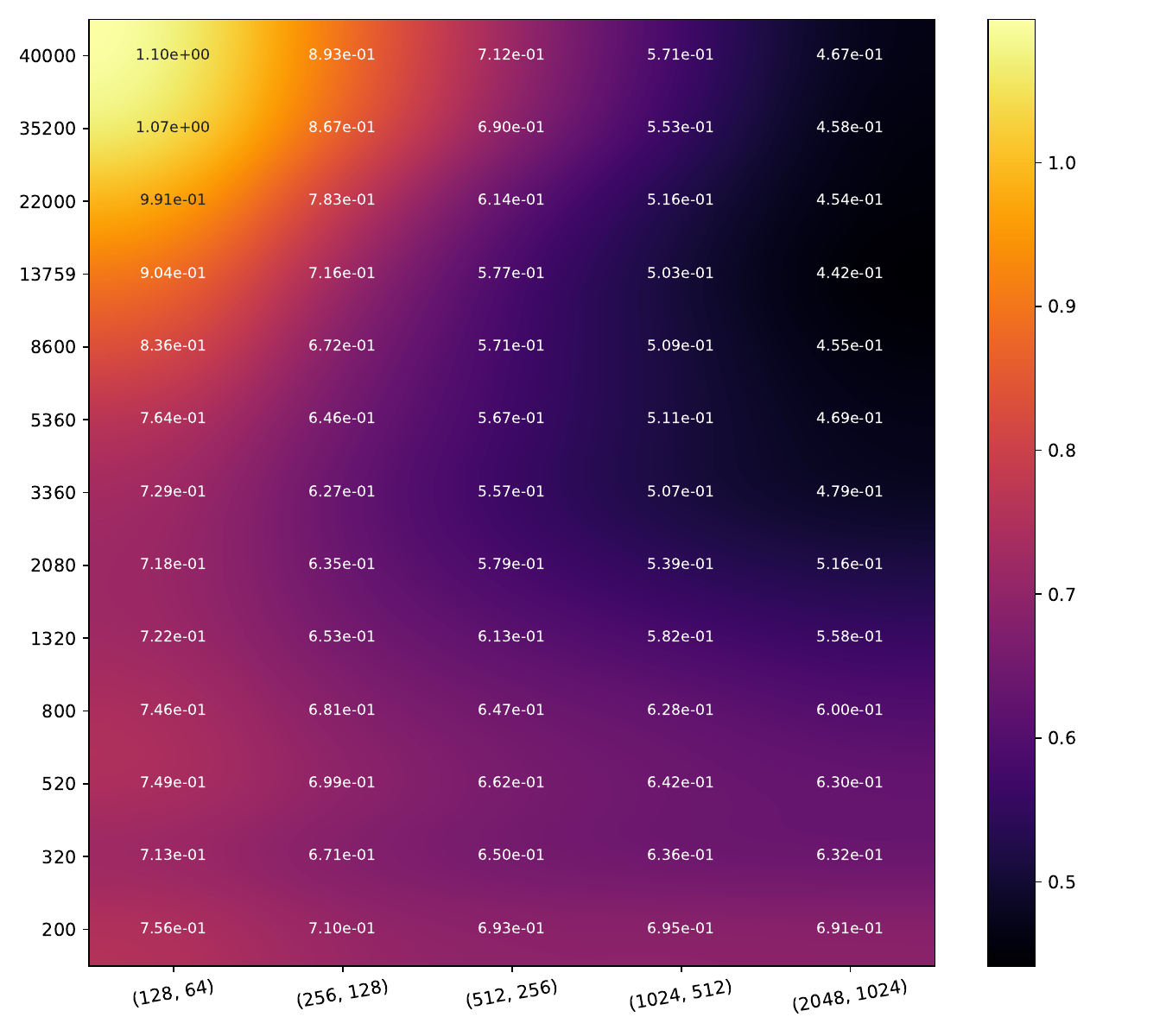}
    \caption*{DE}
  \end{subfigure}\hfill
  \begin{subfigure}[t]{0.185\textwidth}
    \includegraphics[width=\textwidth]{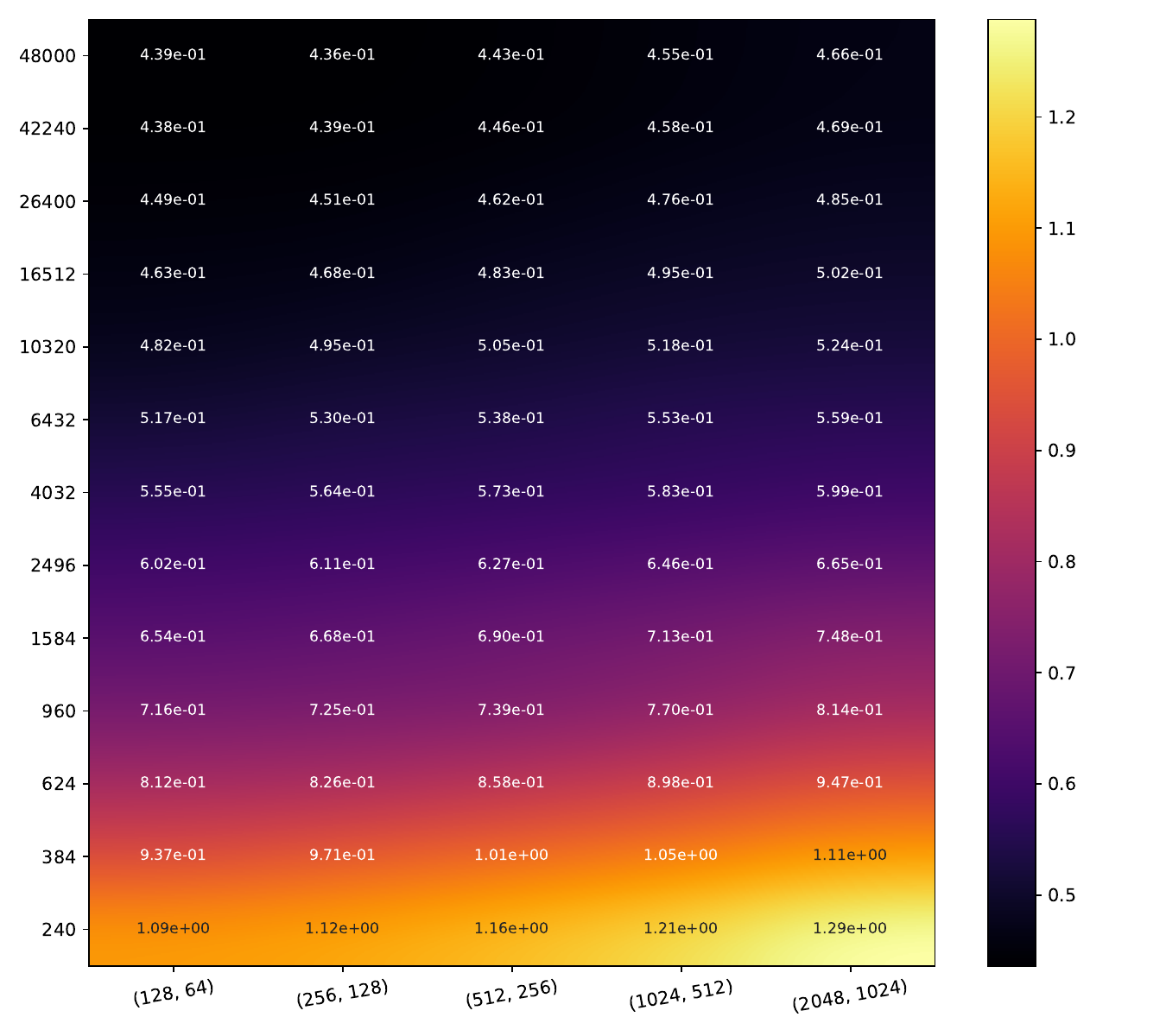}
    \includegraphics[width=\textwidth]{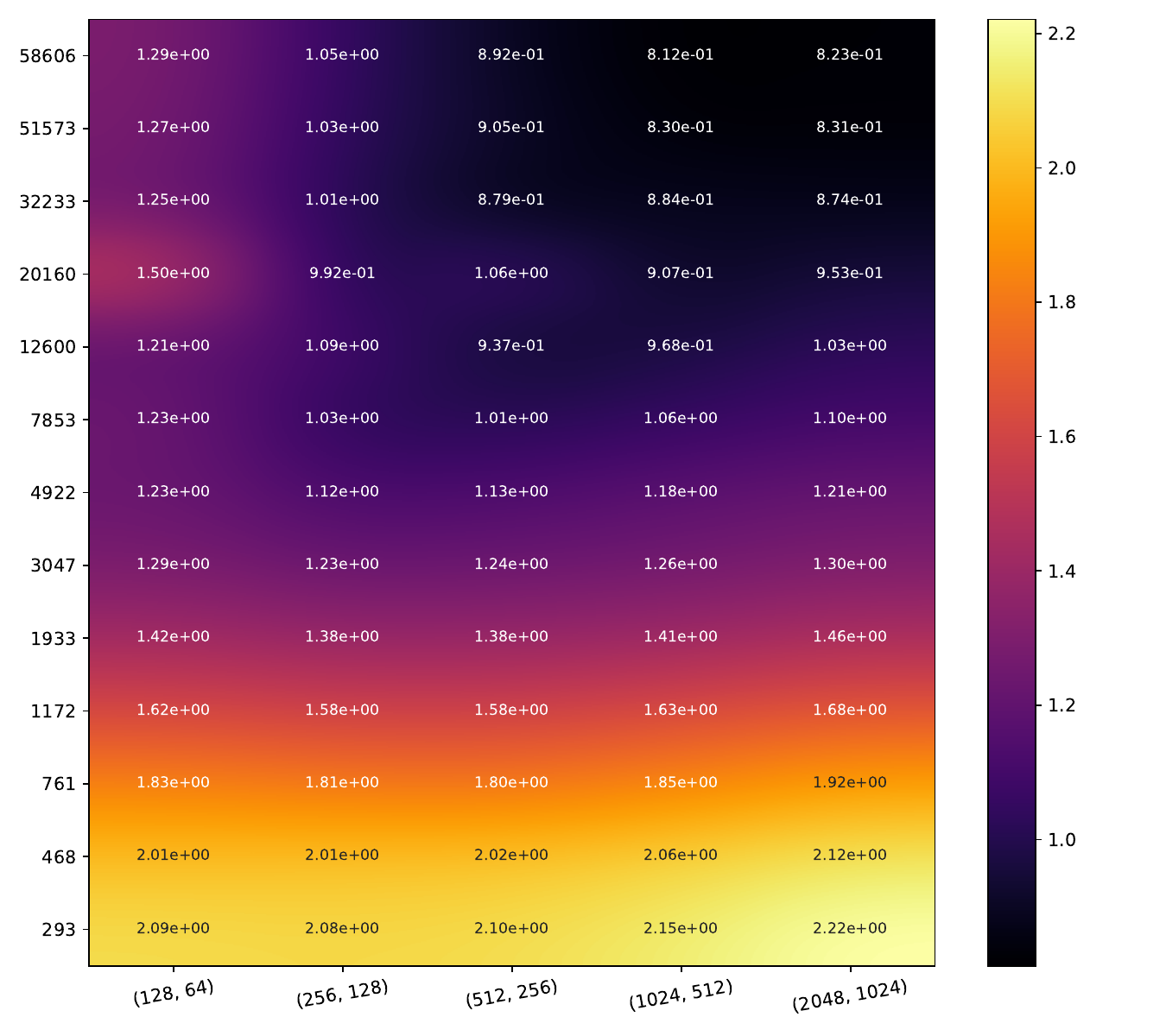}
    \includegraphics[width=\textwidth]{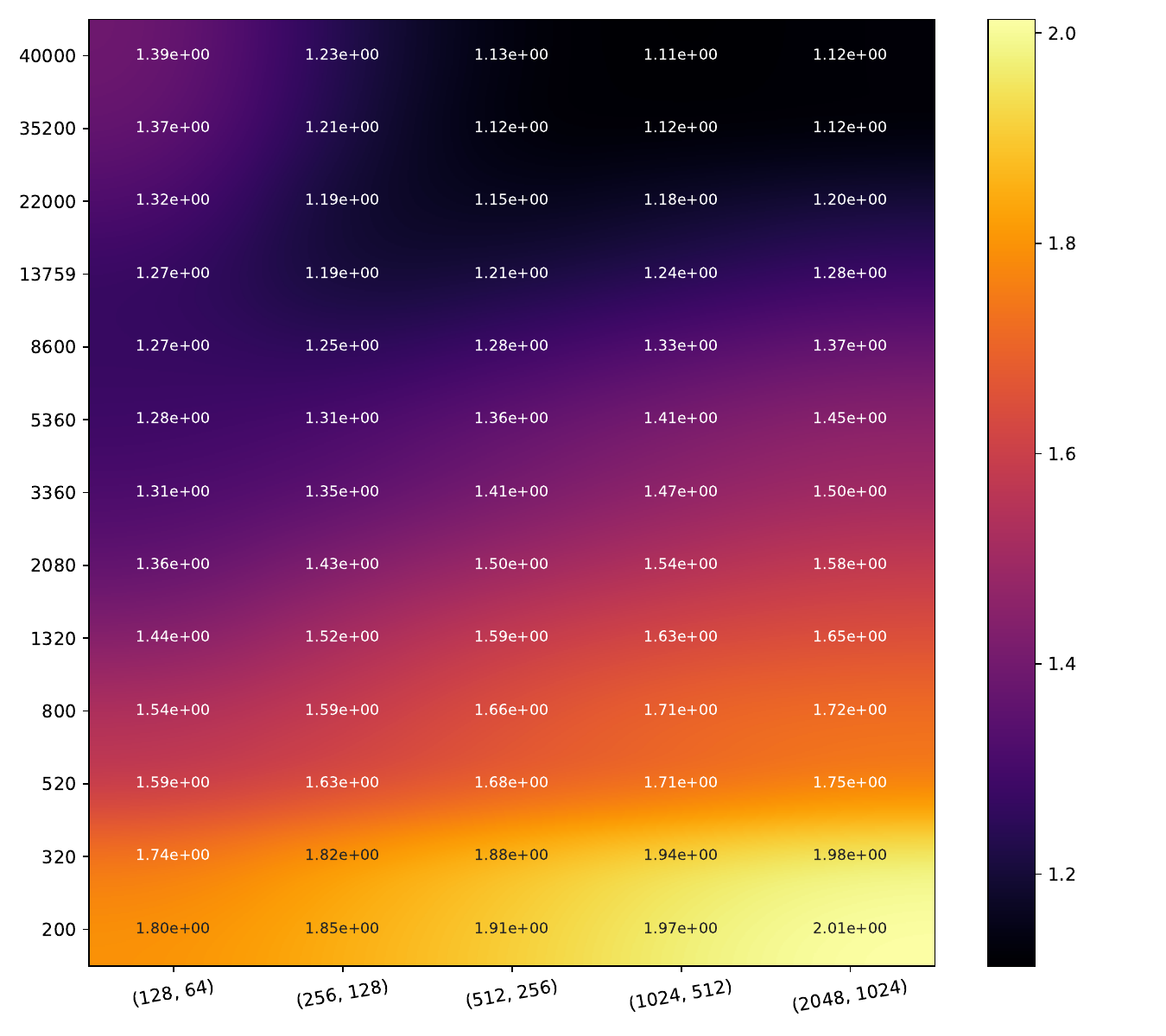}
    \caption*{Conflictual DE}
  \end{subfigure}\hfill
  \caption{Heatmaps of total uncertainty. Color scales are different for each heatmap.}
\end{figure}

\begin{figure}[ht]
  \centering
  \begin{subfigure}[t]{\dimexpr0.185\textwidth+20pt\relax}
    \makebox[20pt]{\raisebox{25pt}{\rotatebox[origin=c]{90}{\scriptsize MNIST}}}%
    \includegraphics[width=\dimexpr\linewidth-20pt\relax]{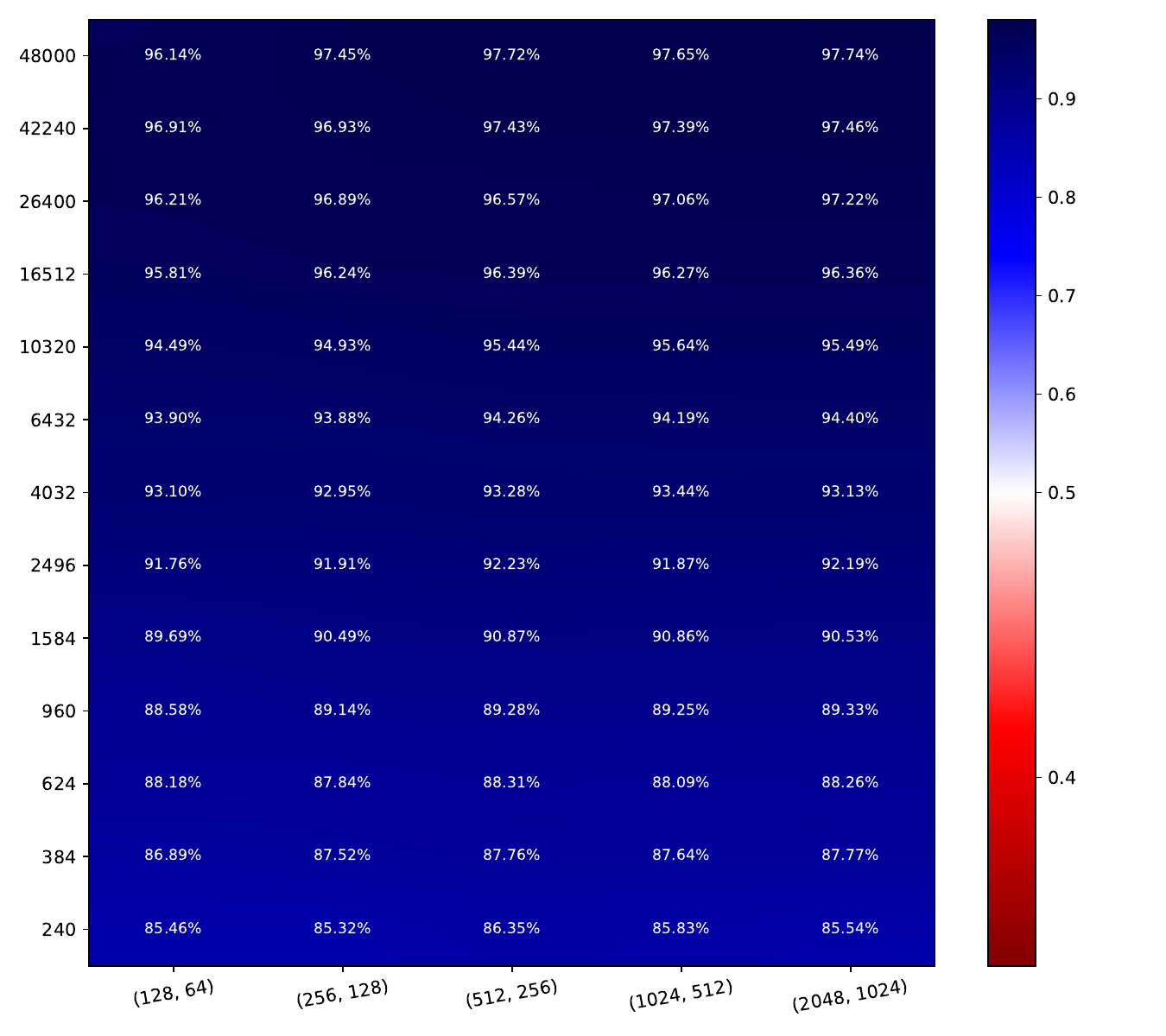}
    \makebox[20pt]{\raisebox{25pt}{\rotatebox[origin=c]{90}{\scriptsize SVHN}}}%
    \includegraphics[width=\dimexpr\linewidth-20pt\relax]{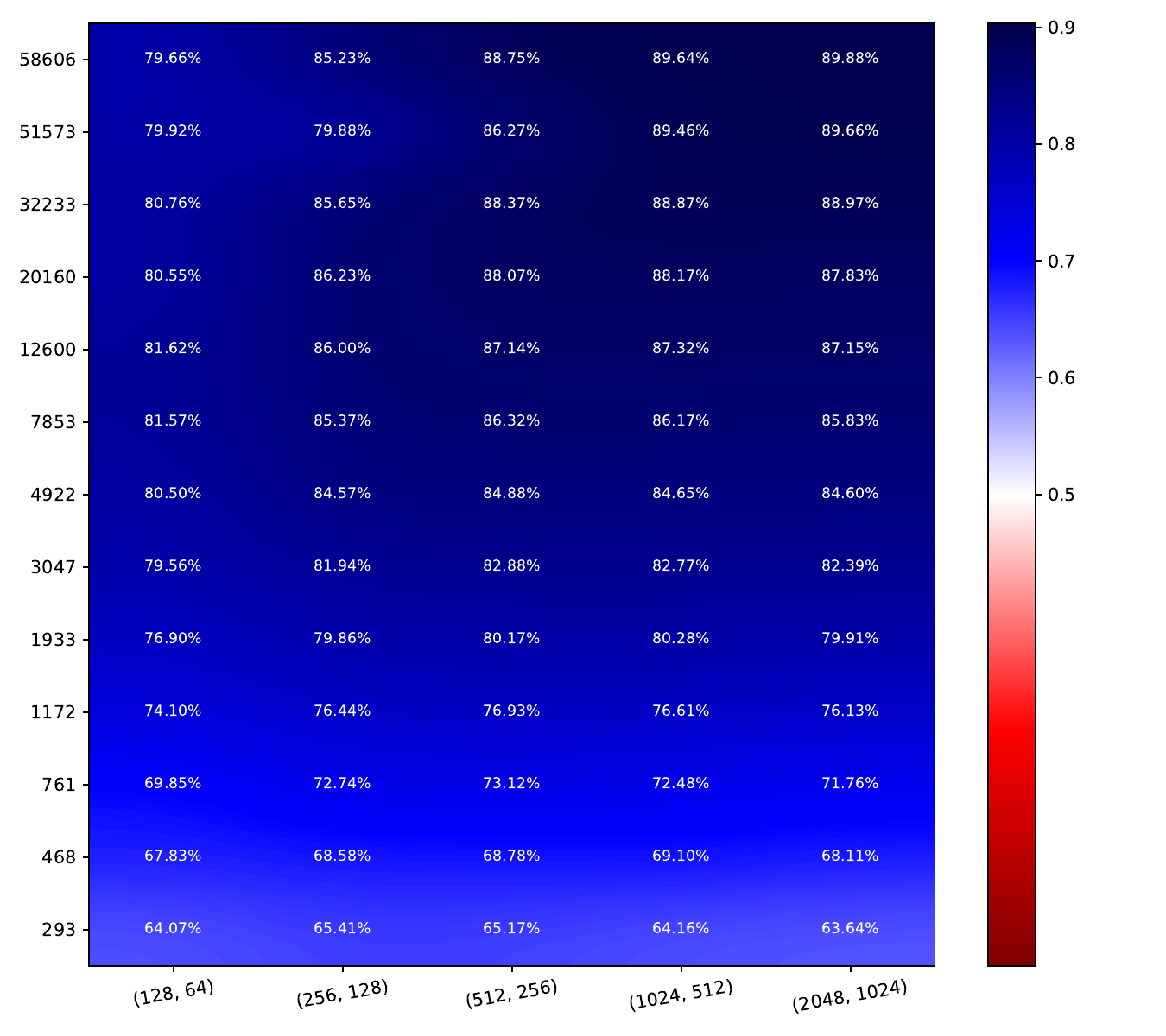}
    \makebox[20pt]{\raisebox{25pt}{\rotatebox[origin=c]{90}{\scriptsize CIFAR10}}}%
    \includegraphics[width=\dimexpr\linewidth-20pt\relax]{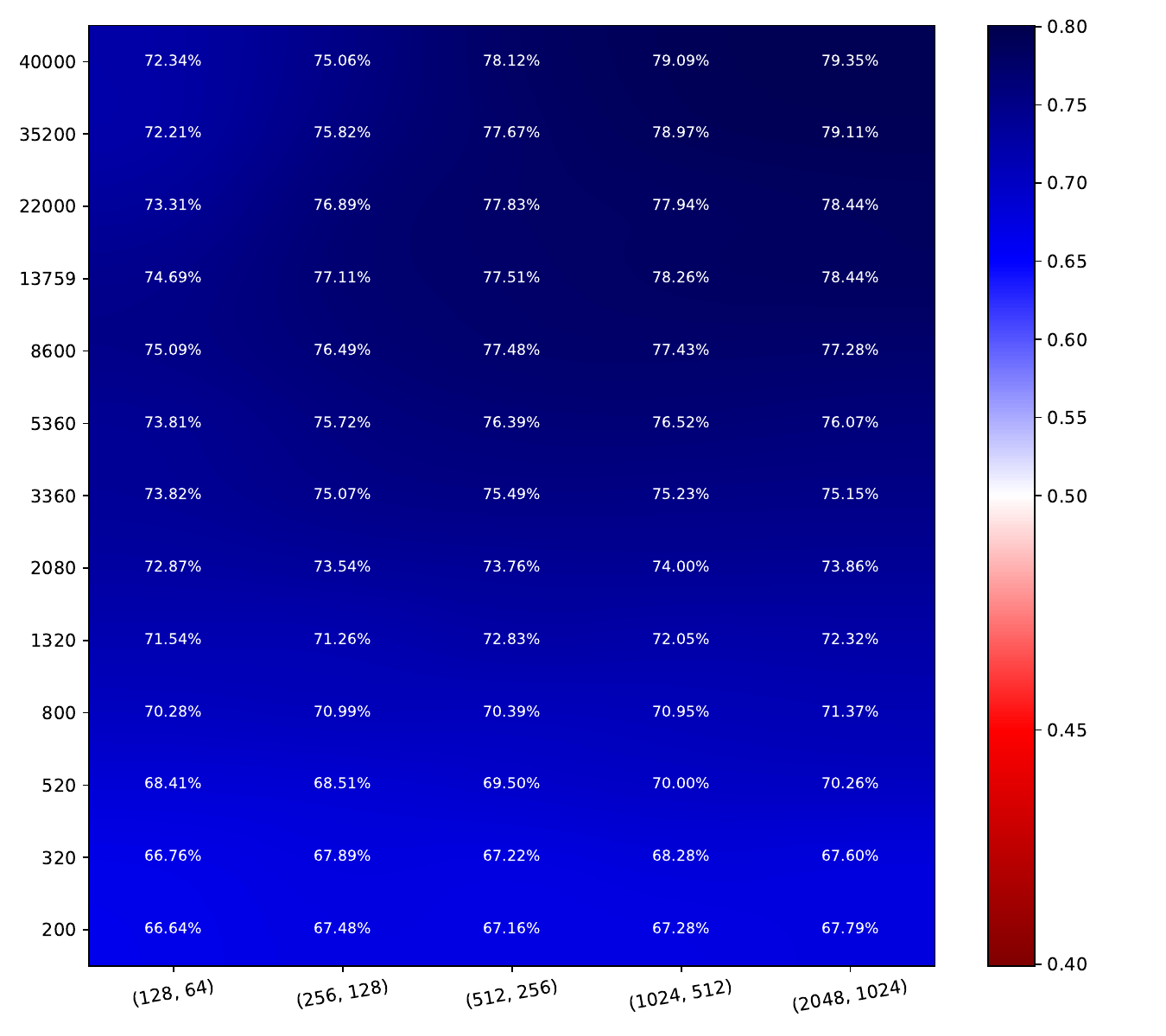}
    \caption*{\qquad MC-Dropout}
  \end{subfigure}\hfill
  \begin{subfigure}[t]{0.185\textwidth}
    \includegraphics[width=\textwidth]{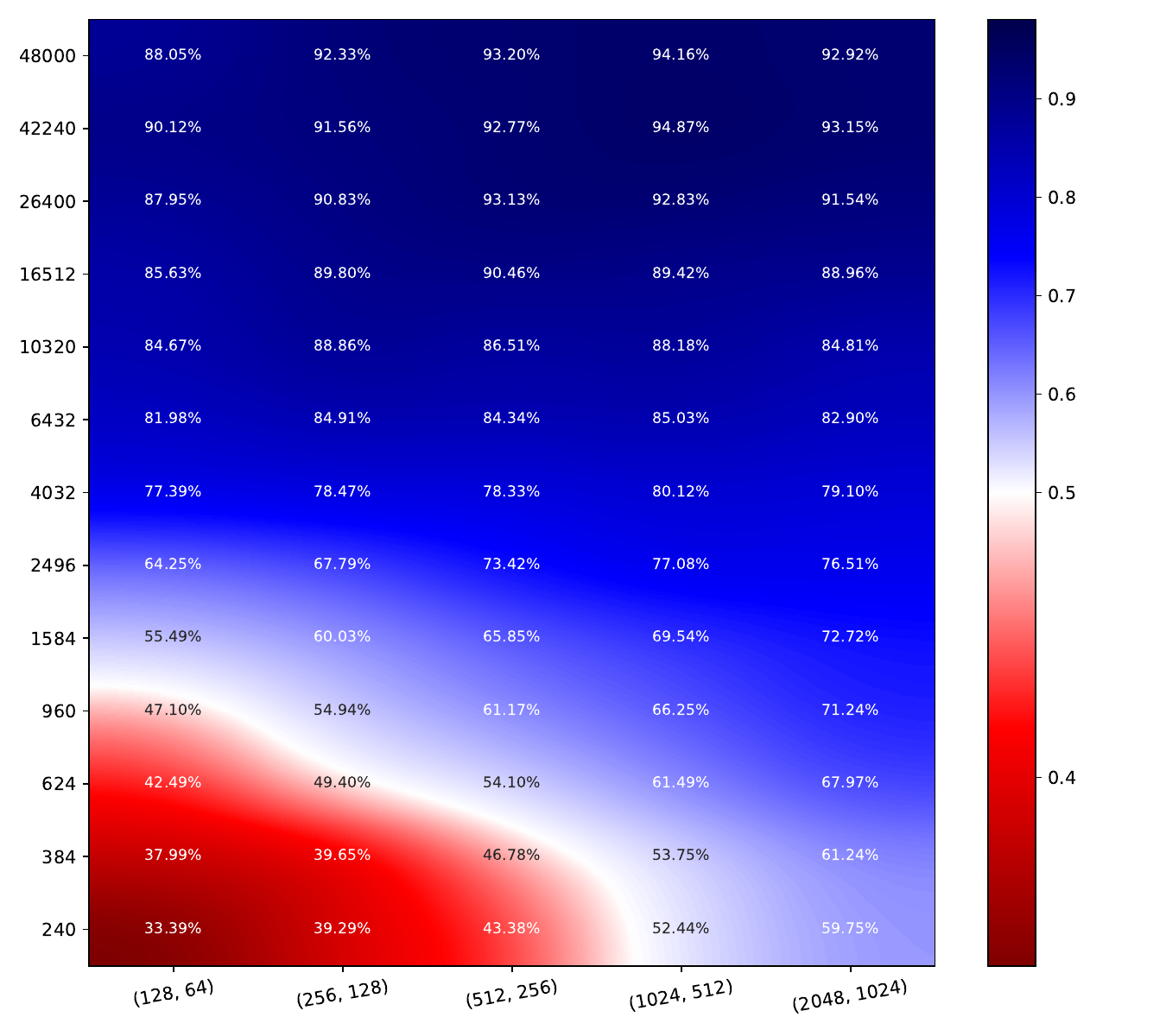}
    \includegraphics[width=\textwidth]{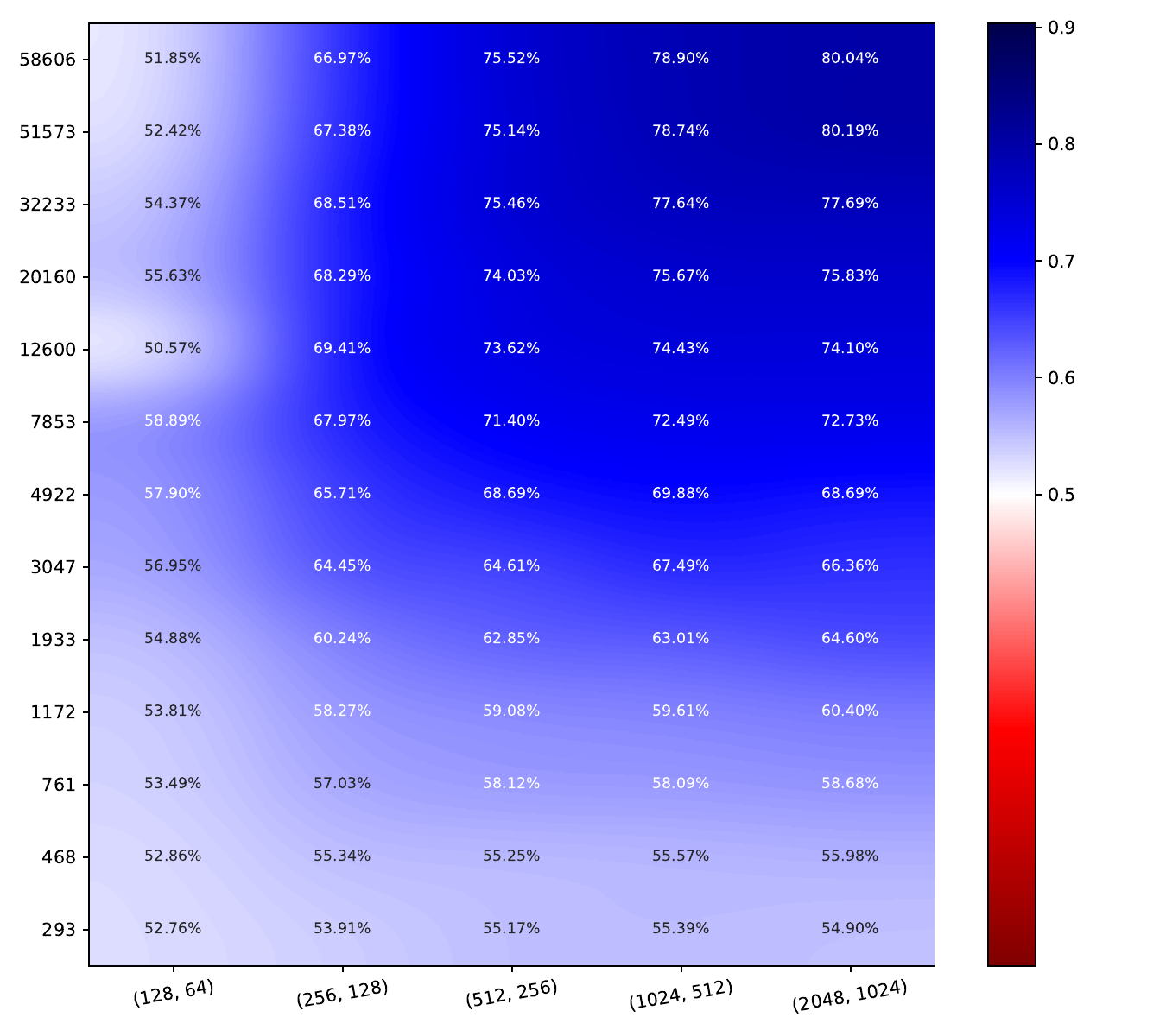}
    \includegraphics[width=\textwidth]{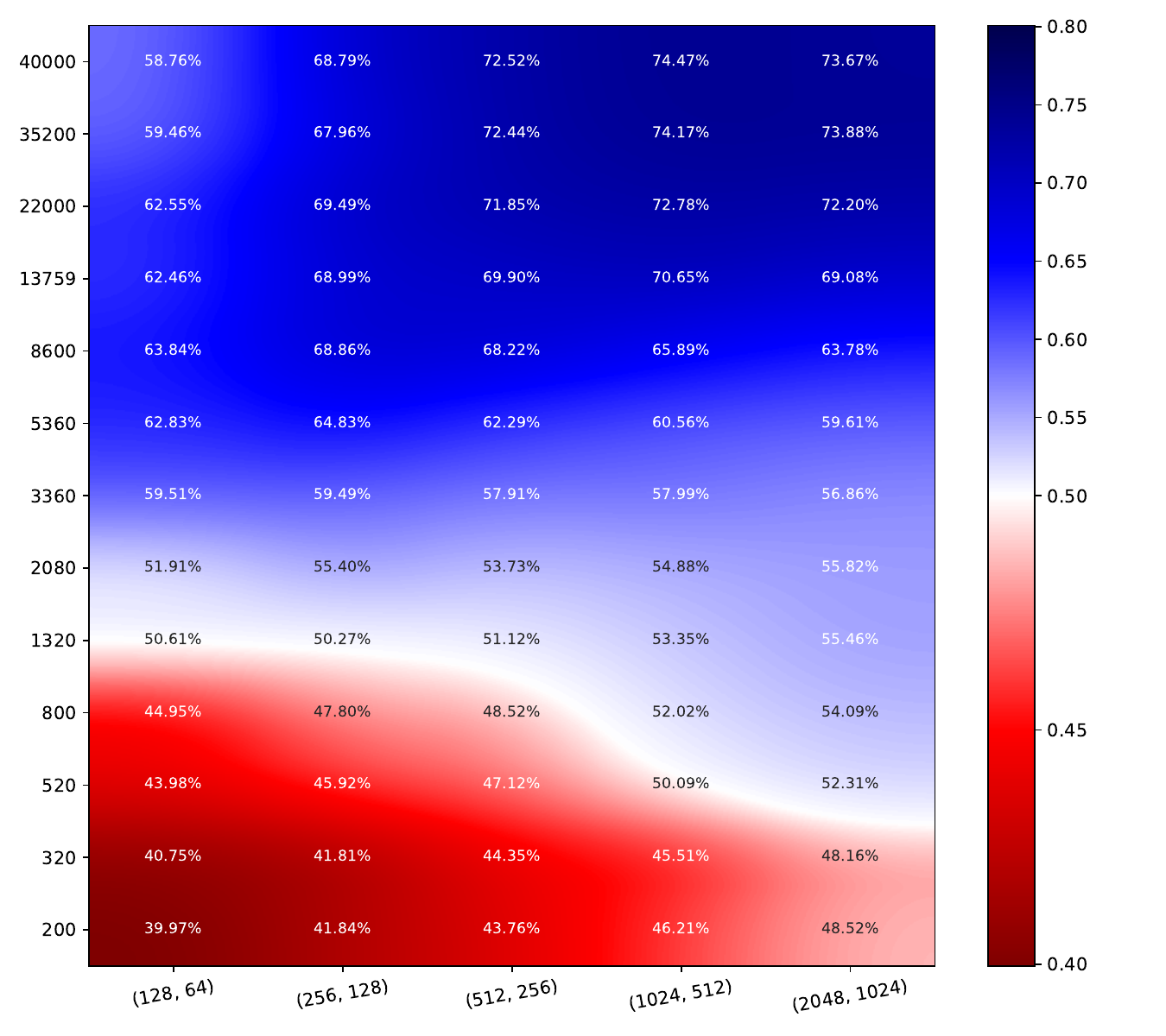}
    \caption*{MC-Dropout LS}
  \end{subfigure}\hfill
  \begin{subfigure}[t]{0.185\textwidth}
    \includegraphics[width=\textwidth]{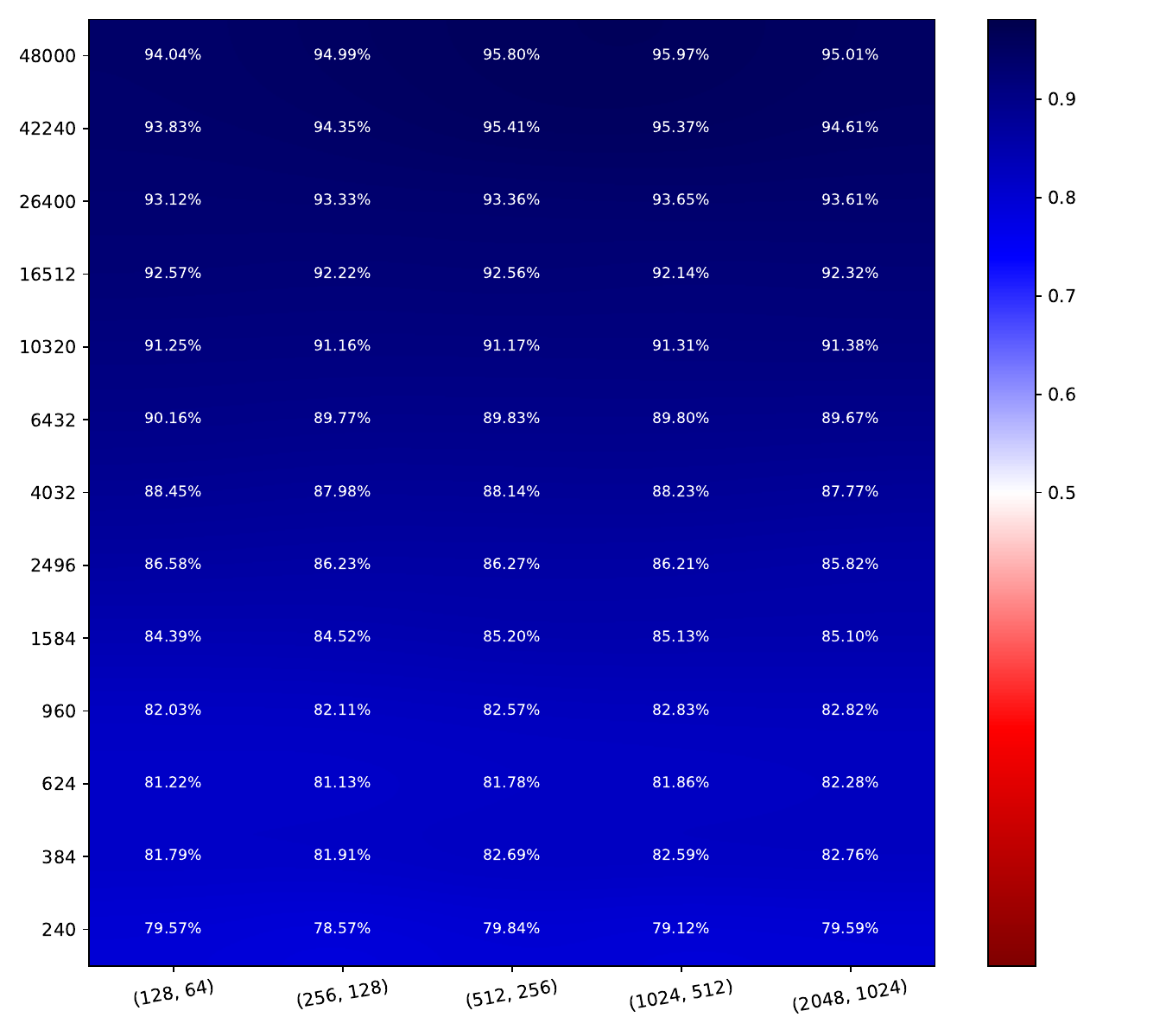}
    \includegraphics[width=\textwidth]{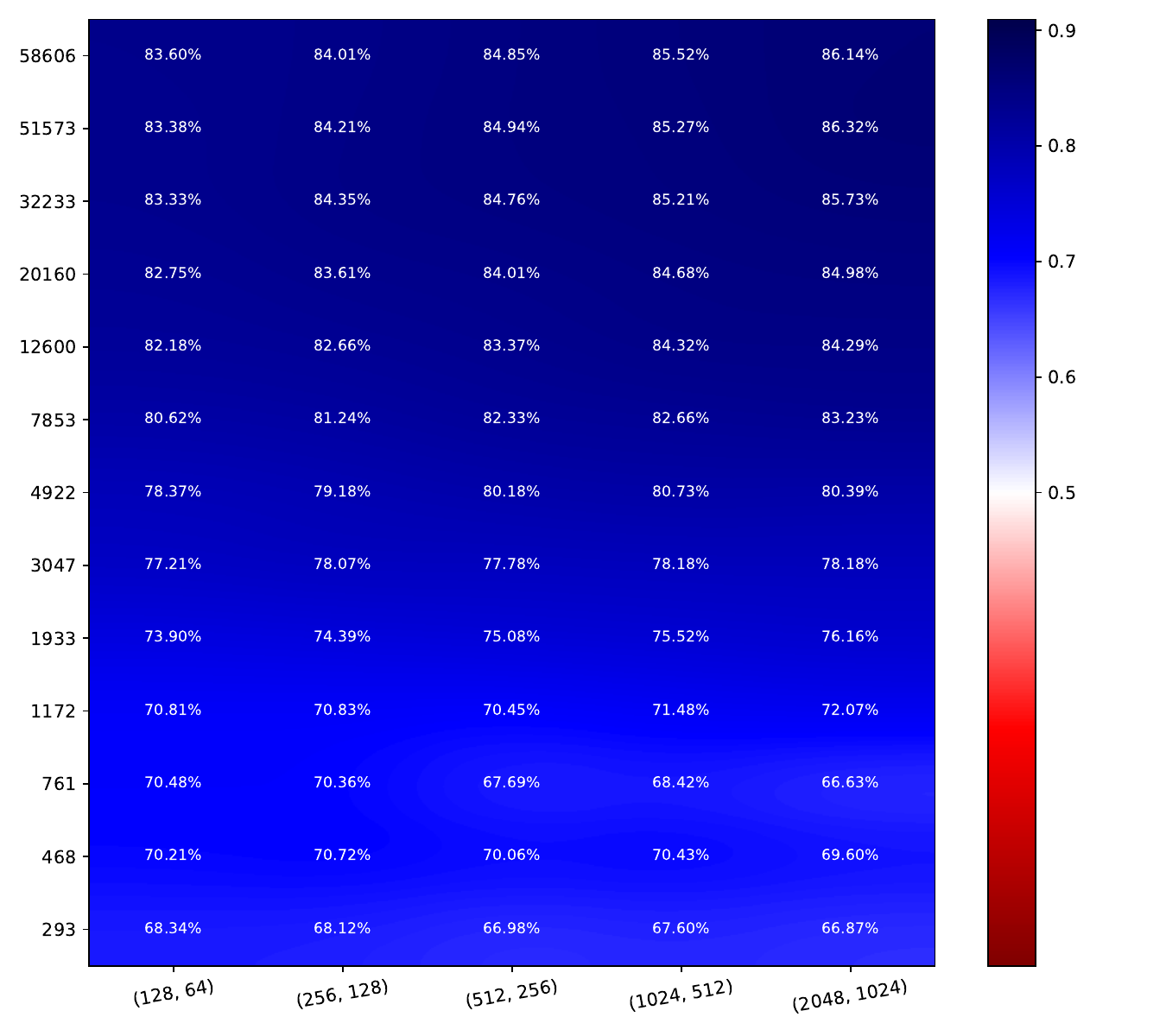}
    \includegraphics[width=\textwidth]{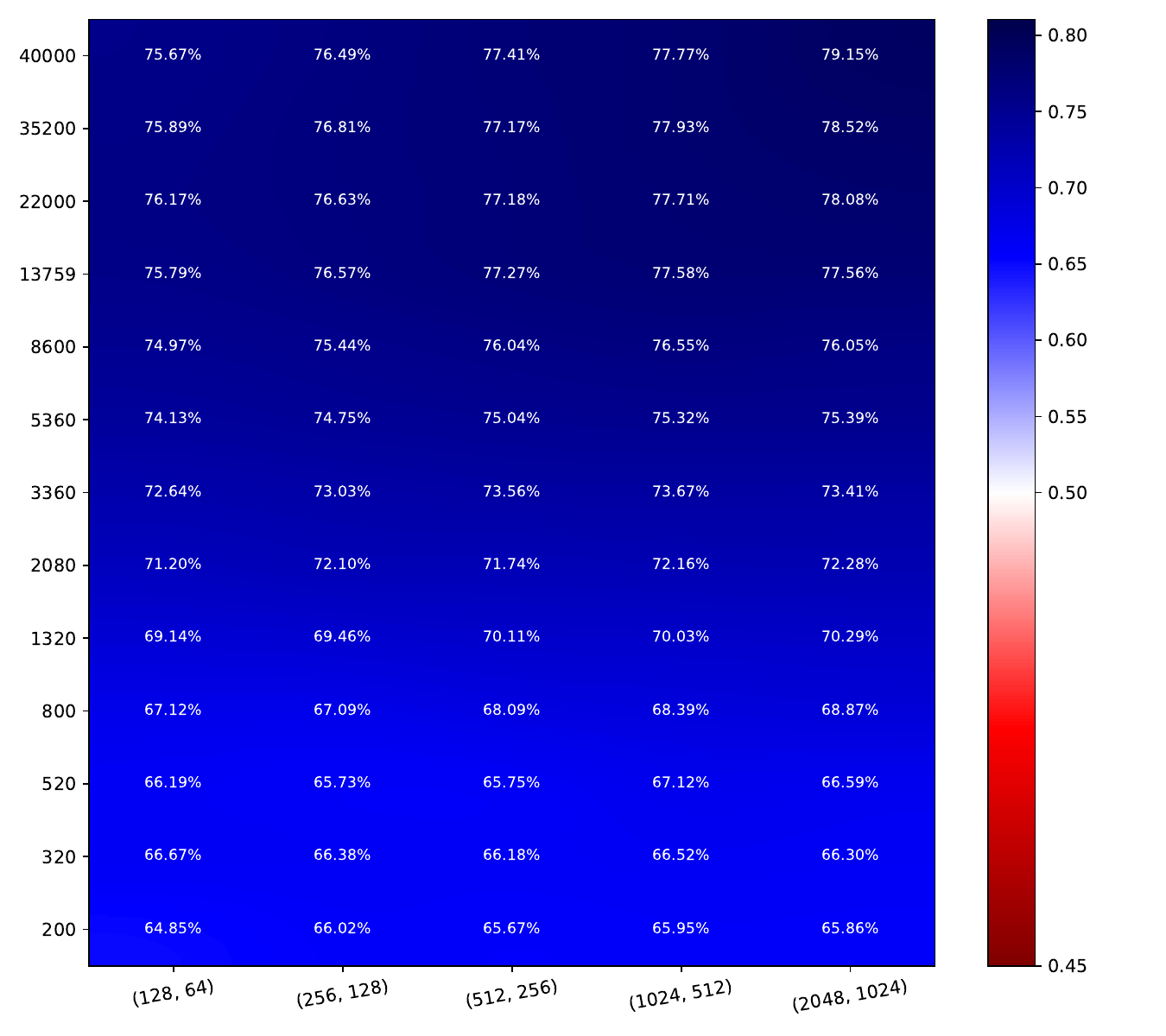}
    \caption*{EDL}
  \end{subfigure}\hfill
  \begin{subfigure}[t]{0.185\textwidth}
    \includegraphics[width=\textwidth]{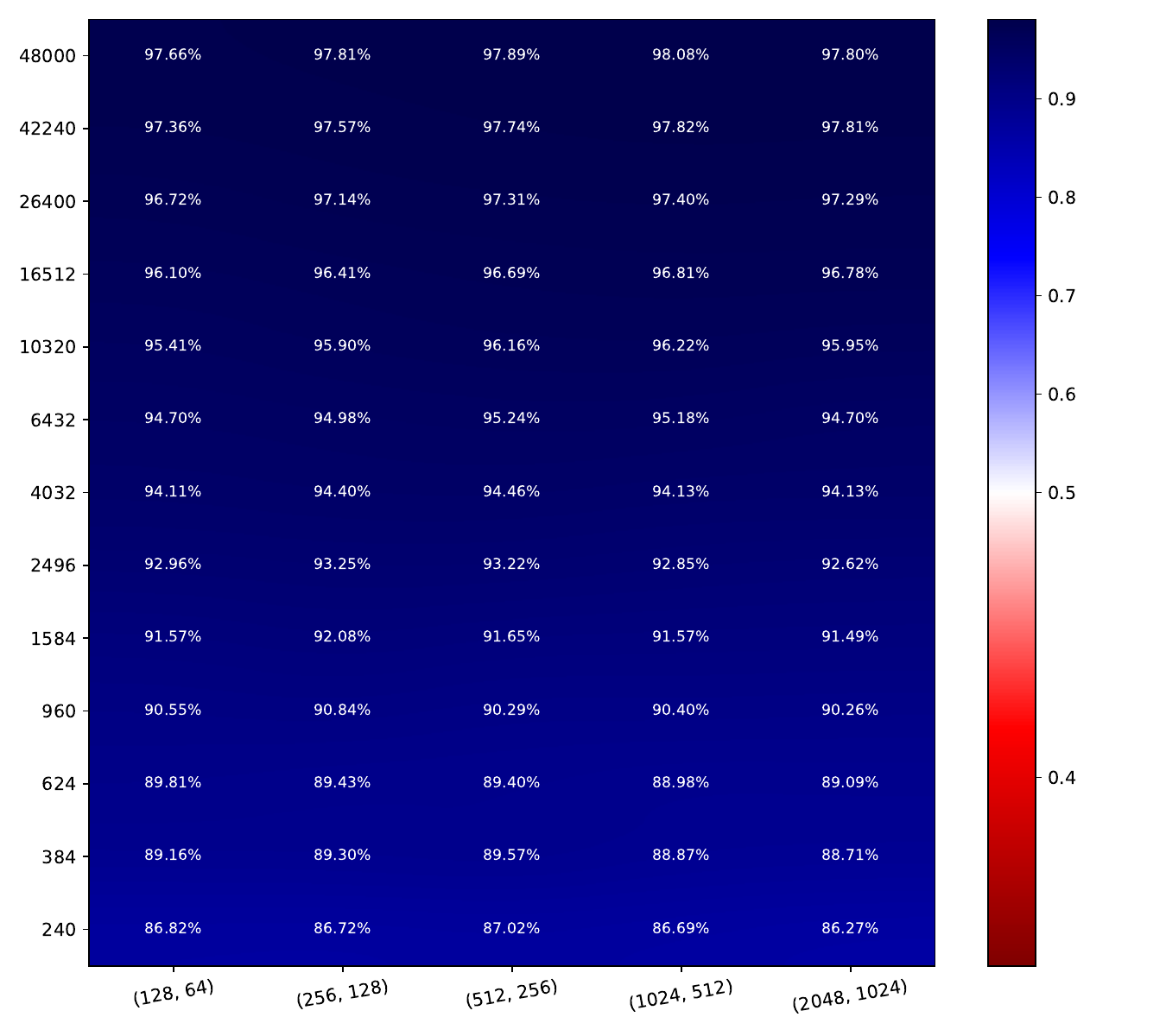}
    \includegraphics[width=\textwidth]{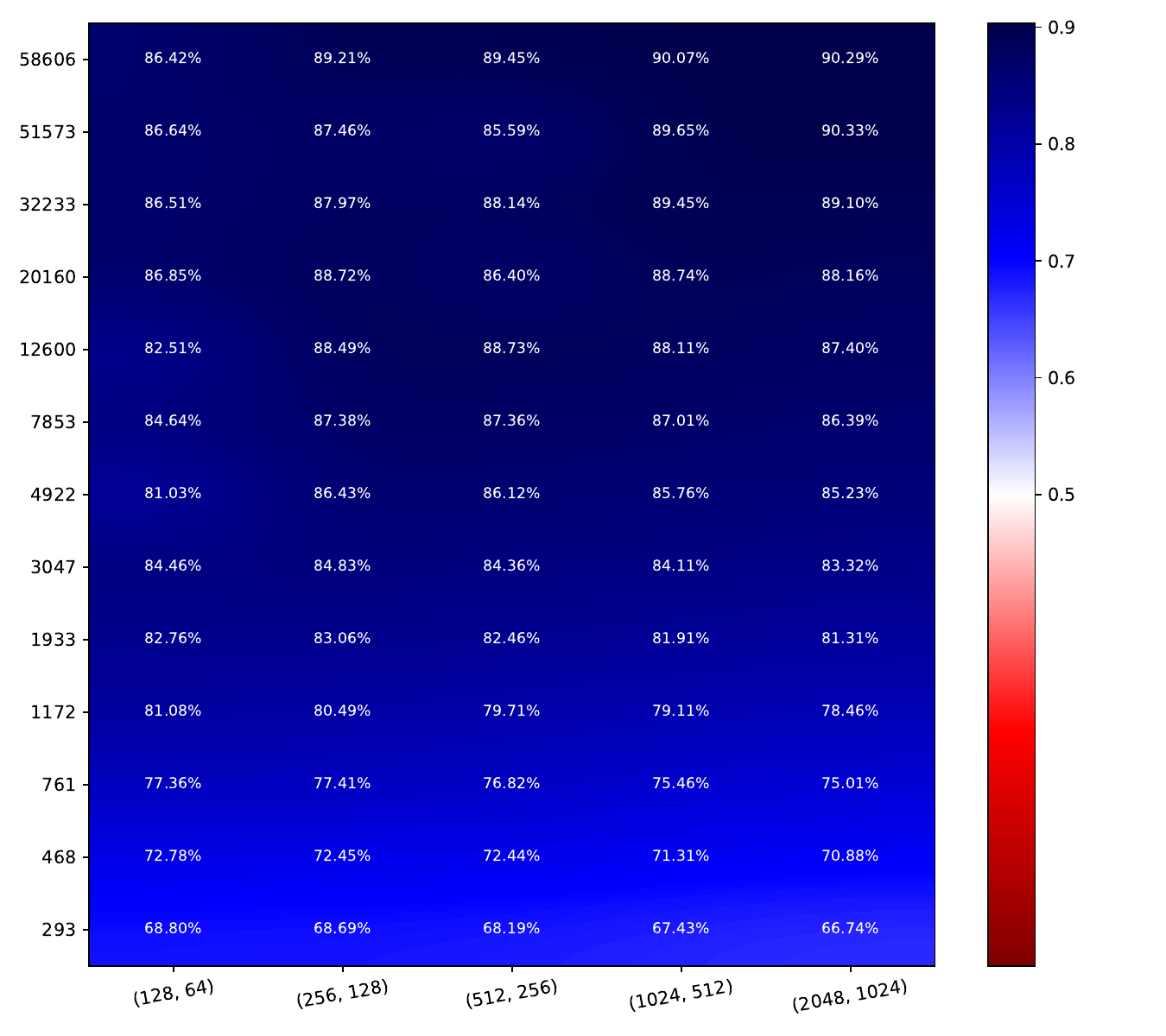}
    \includegraphics[width=\textwidth]{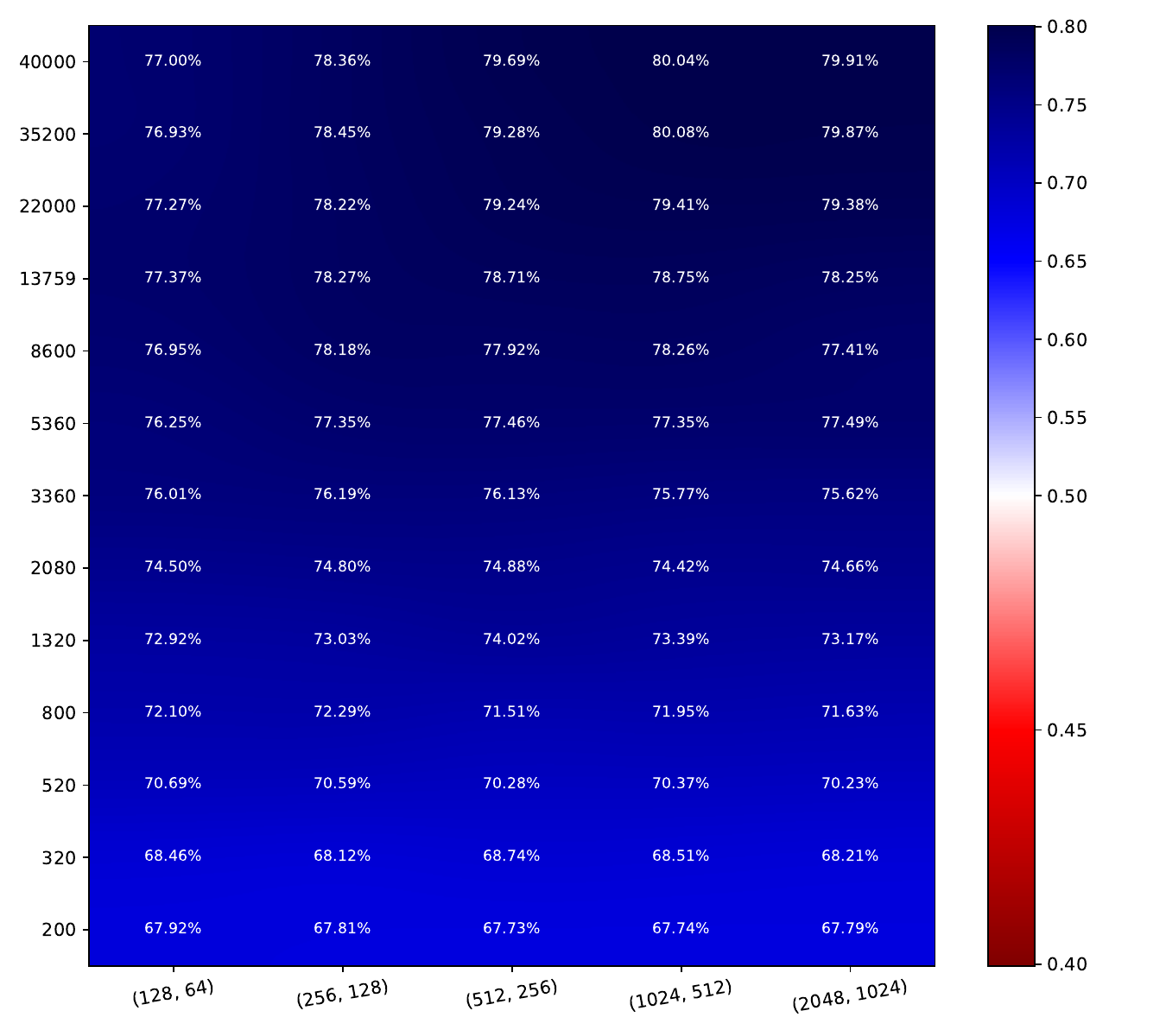}
    \caption*{DE}
  \end{subfigure}\hfill
  \begin{subfigure}[t]{0.185\textwidth}
    \includegraphics[width=\textwidth]{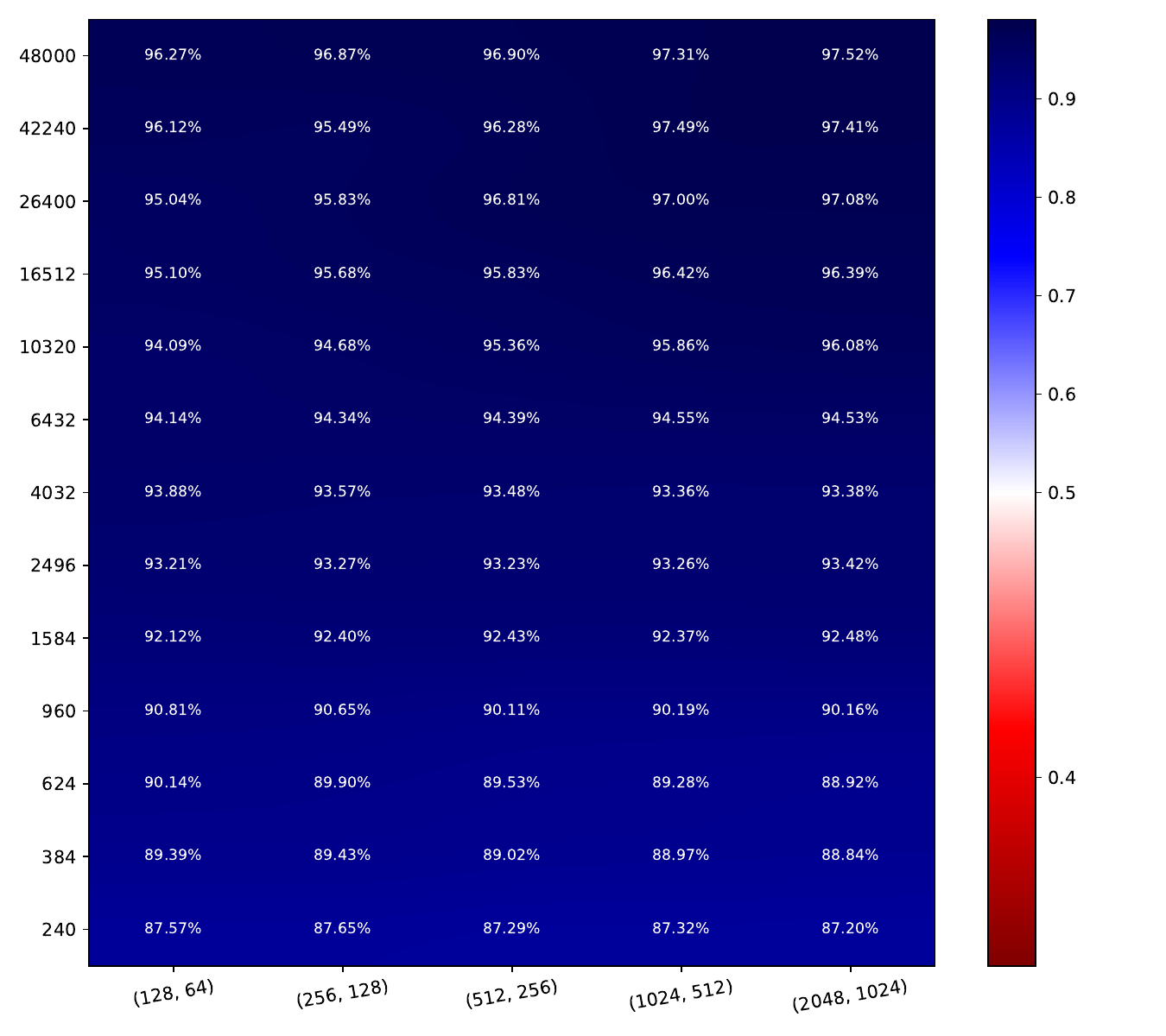}
    \includegraphics[width=\textwidth]{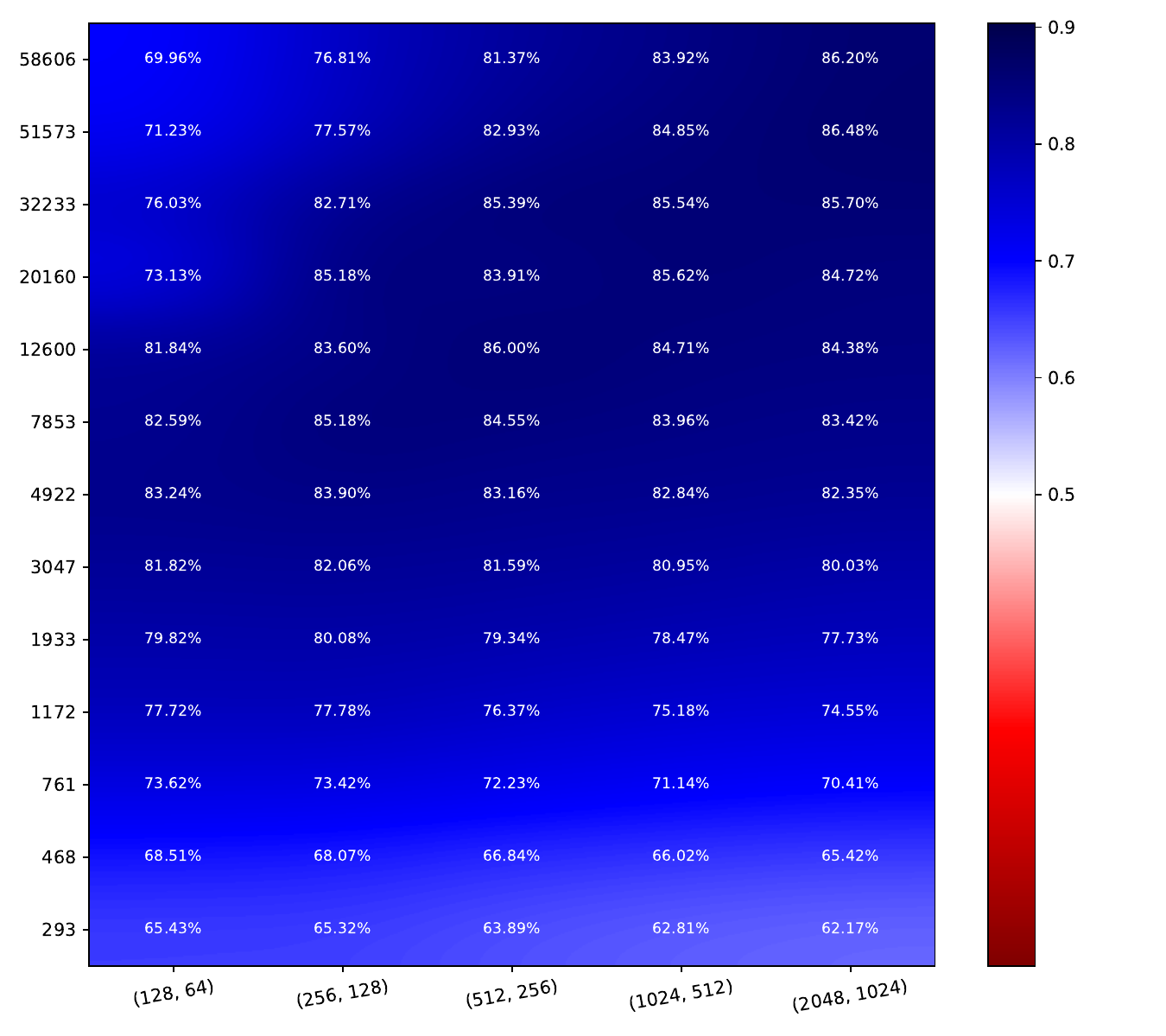}
    \includegraphics[width=\textwidth]{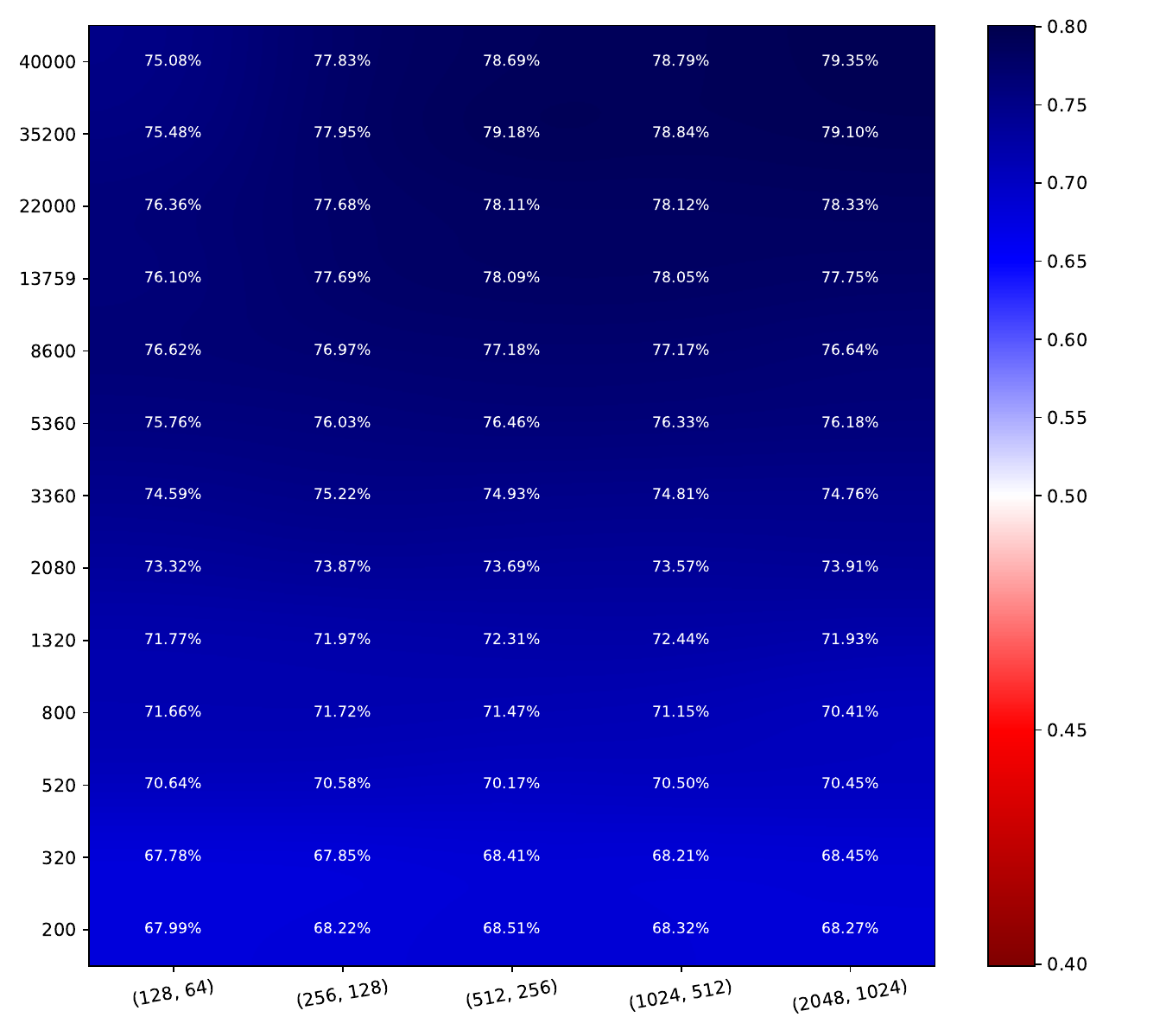}
    \caption*{Conflictual DE}
  \end{subfigure}\hfill
  \caption{Heatmaps of AUROC based on epistemic uncertainty for misclassification detection. Color scales are the same per dataset.}
  \label{fig:auc-mis}
\end{figure}

\clearpage
\newpage

\section{Results for Confidence Penalty\label{app:confidence-penalty}}

In this section, we report the results of Sect.~\ref{sec:experiments} in the case of MC-Dropout with confidence penalty (MC-Dropout CP). We notice that, for both MNIST and CIFAR10, the results are comparable to those on MC-Dropout trained only with cross-entropy loss. The color scales are set per heatmap.

\begin{table}
    \tiny
    \setlength{\tabcolsep}{0.09em}
    \begin{tabular}{l@{\hspace{0.5em}}cccccc}
        \rotatebox[origin=c]{90}{MNIST} %
	& \includegraphics[width=0.159\textwidth,valign=m]{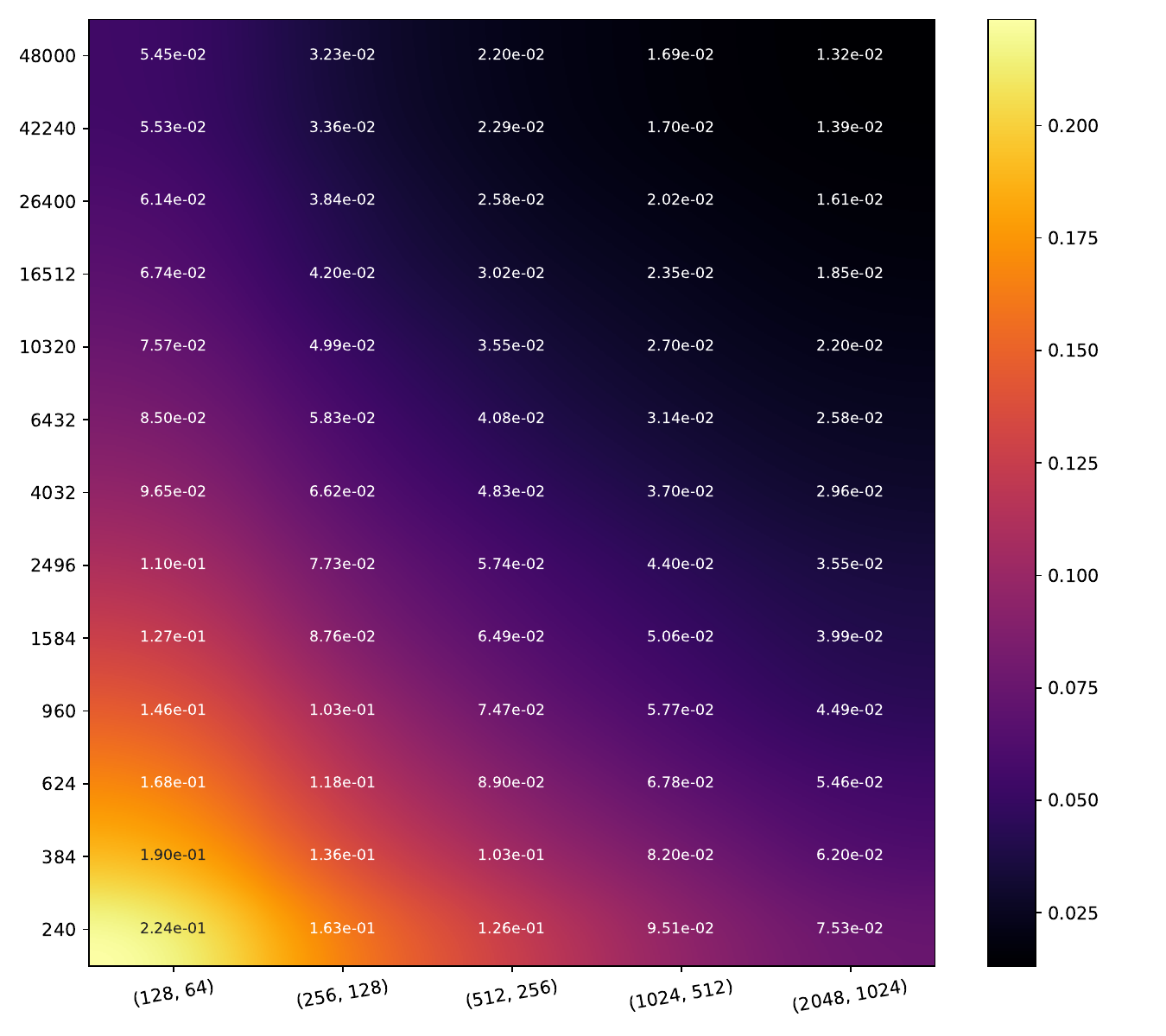} %
        & \includegraphics[width=0.159\textwidth,valign=m]{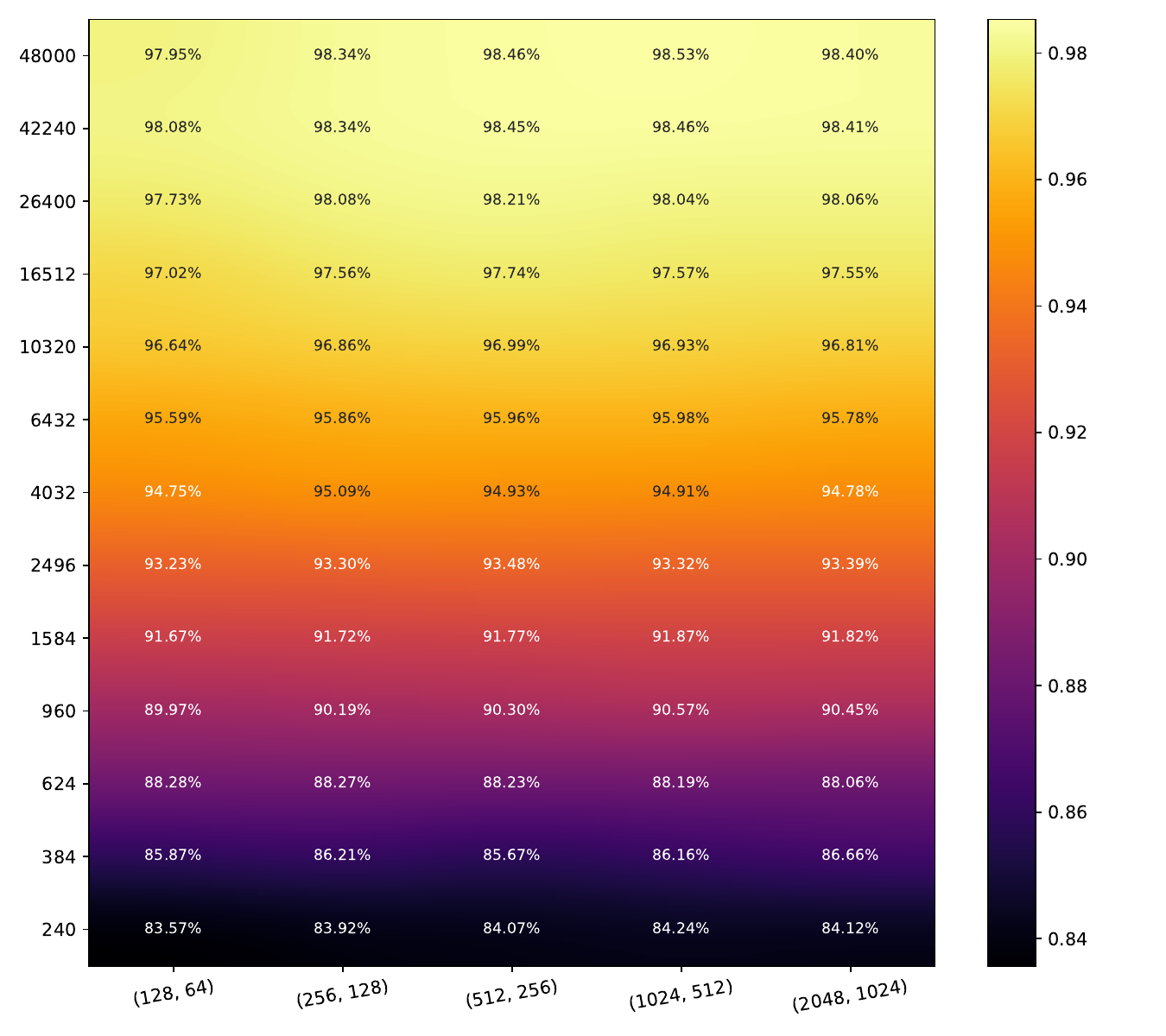} %
        & \includegraphics[width=0.159\textwidth,valign=m]{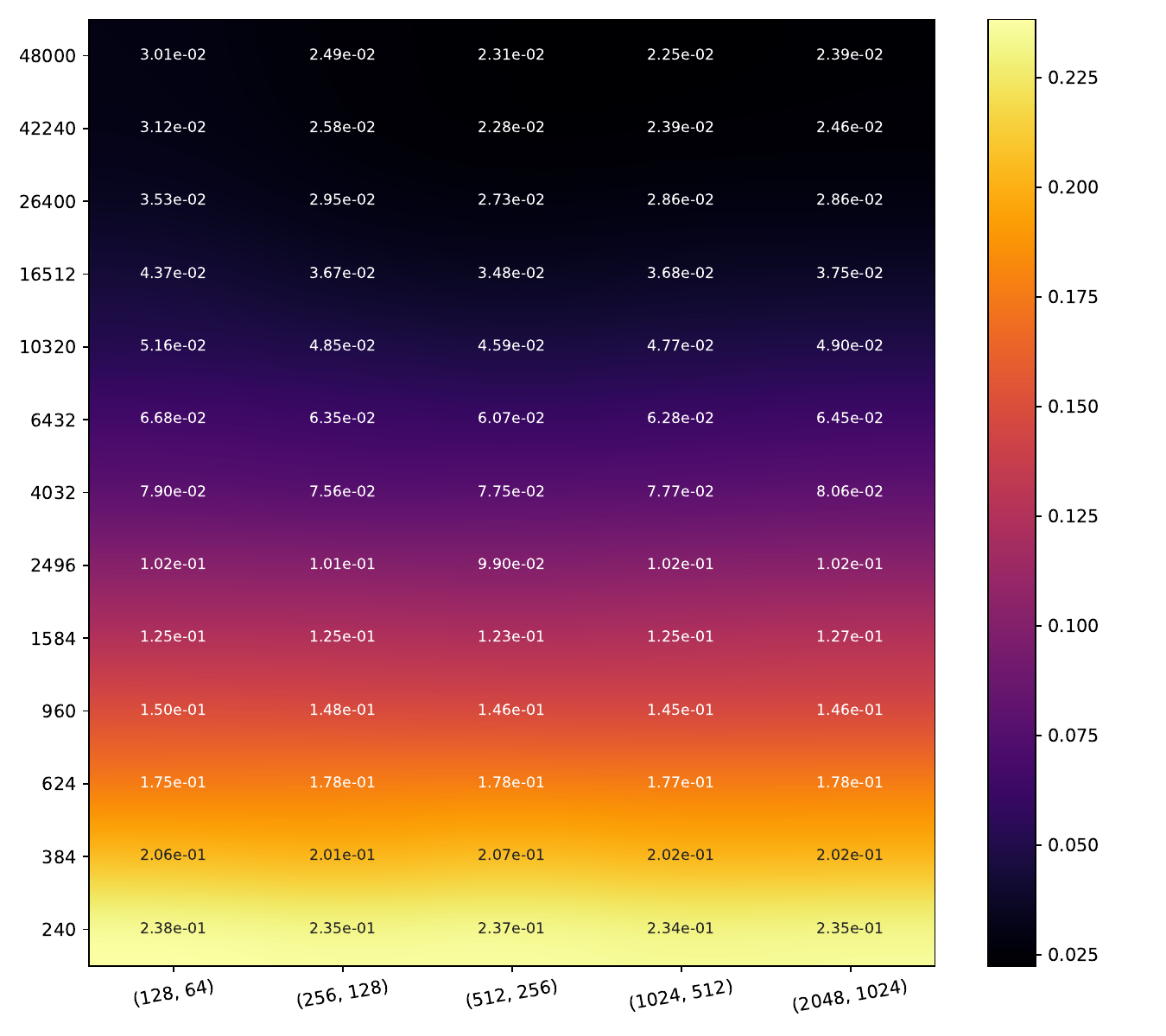} %
        & \includegraphics[width=0.159\textwidth,valign=m]{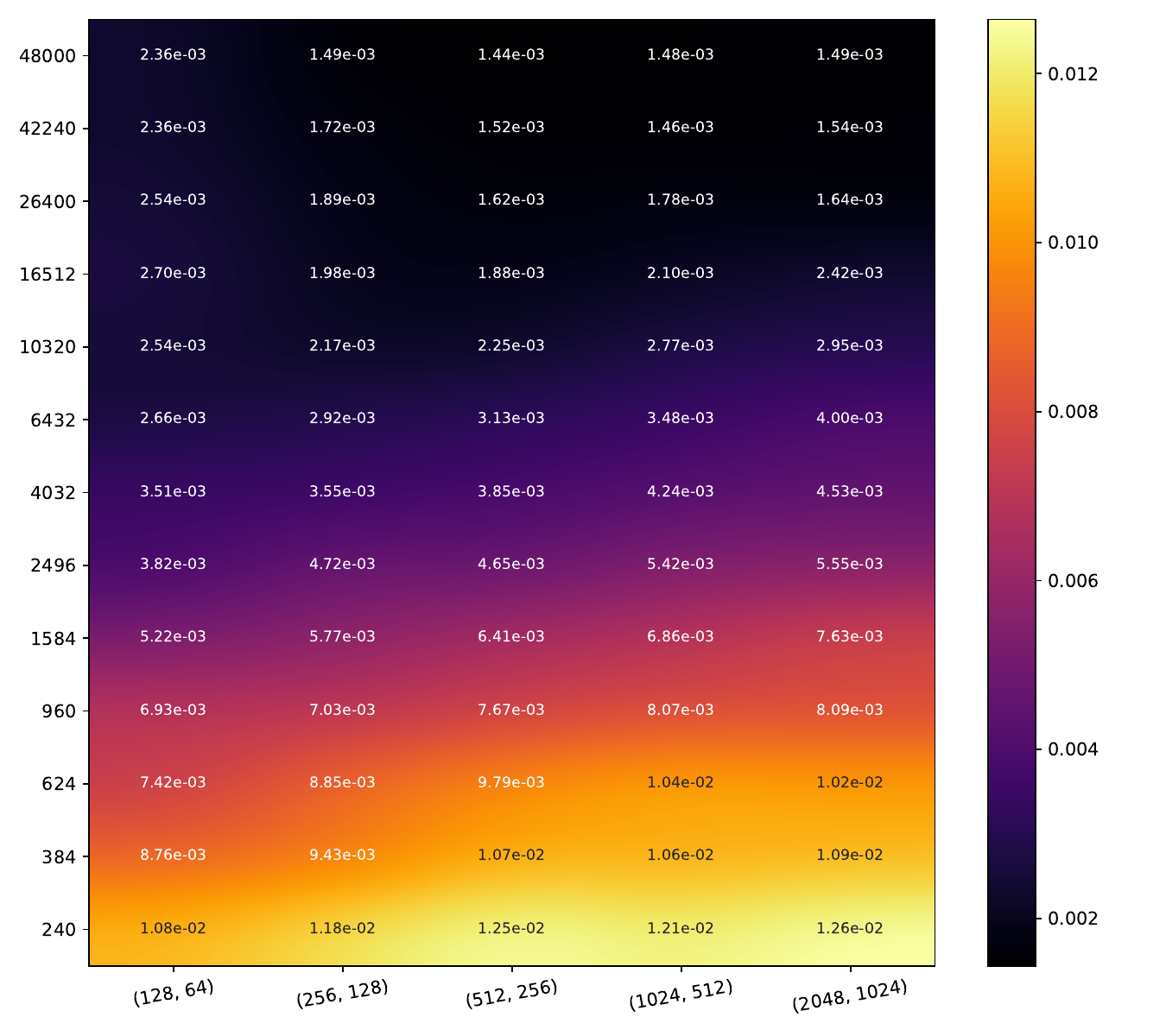} %
	& \includegraphics[width=0.159\textwidth,valign=m]{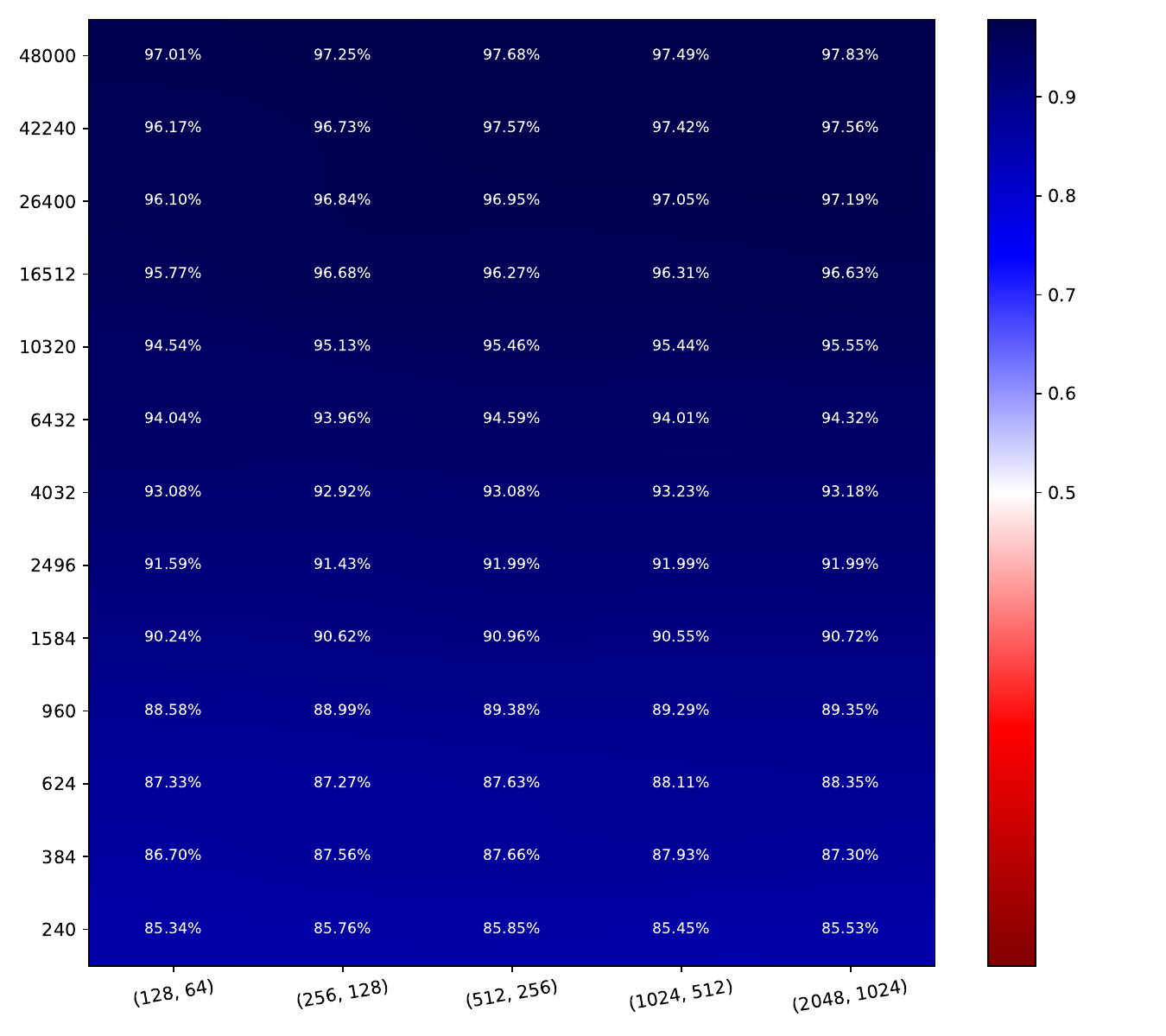} %
        & \includegraphics[width=0.159\textwidth,valign=m]{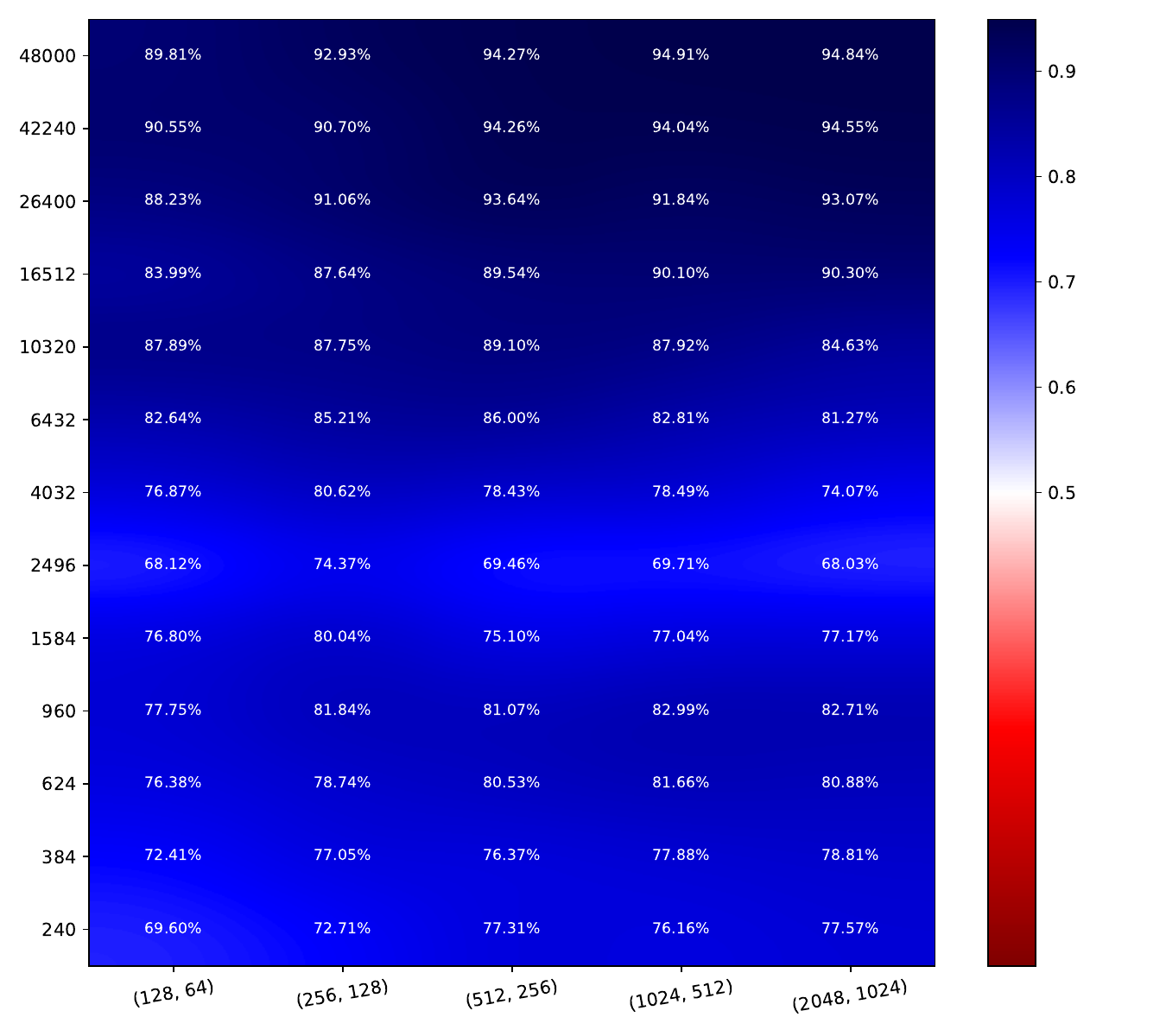} %
        \\
        \rotatebox[origin=c]{90}{SVHN} %
	& \includegraphics[width=0.159\textwidth,valign=m]{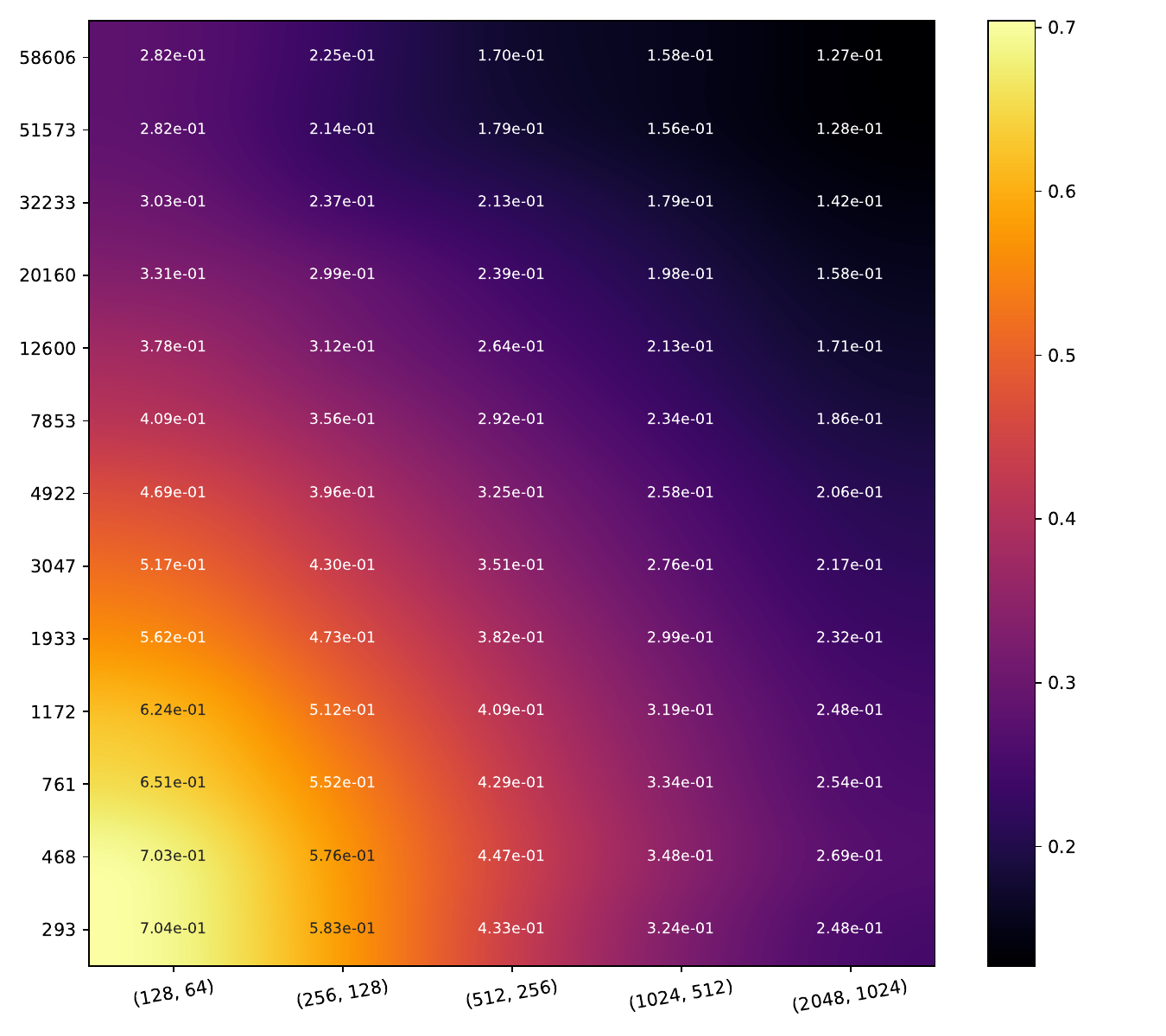} %
        & \includegraphics[width=0.159\textwidth,valign=m]{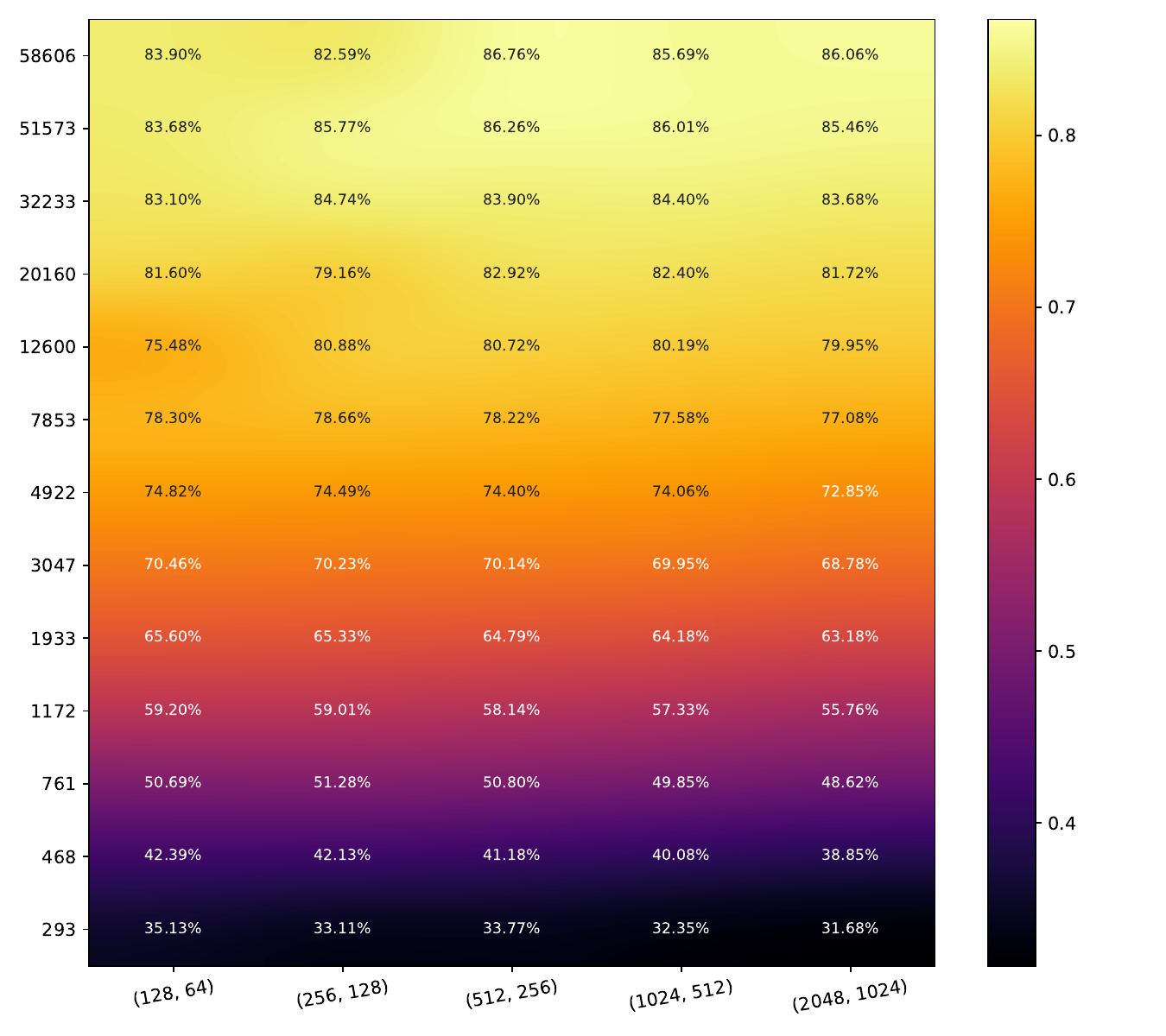} %
        & \includegraphics[width=0.159\textwidth,valign=m]{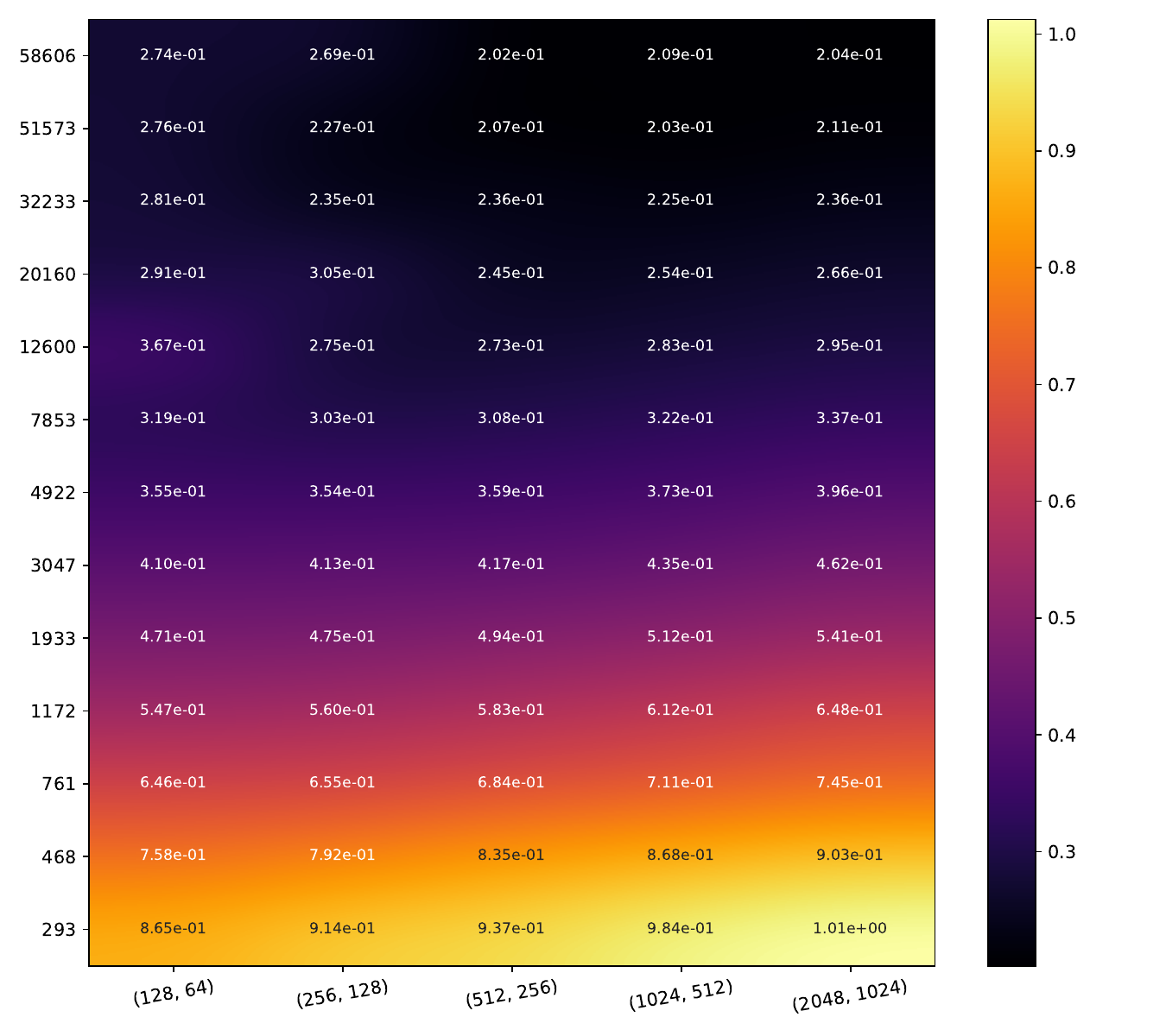} %
        & \includegraphics[width=0.159\textwidth,valign=m]{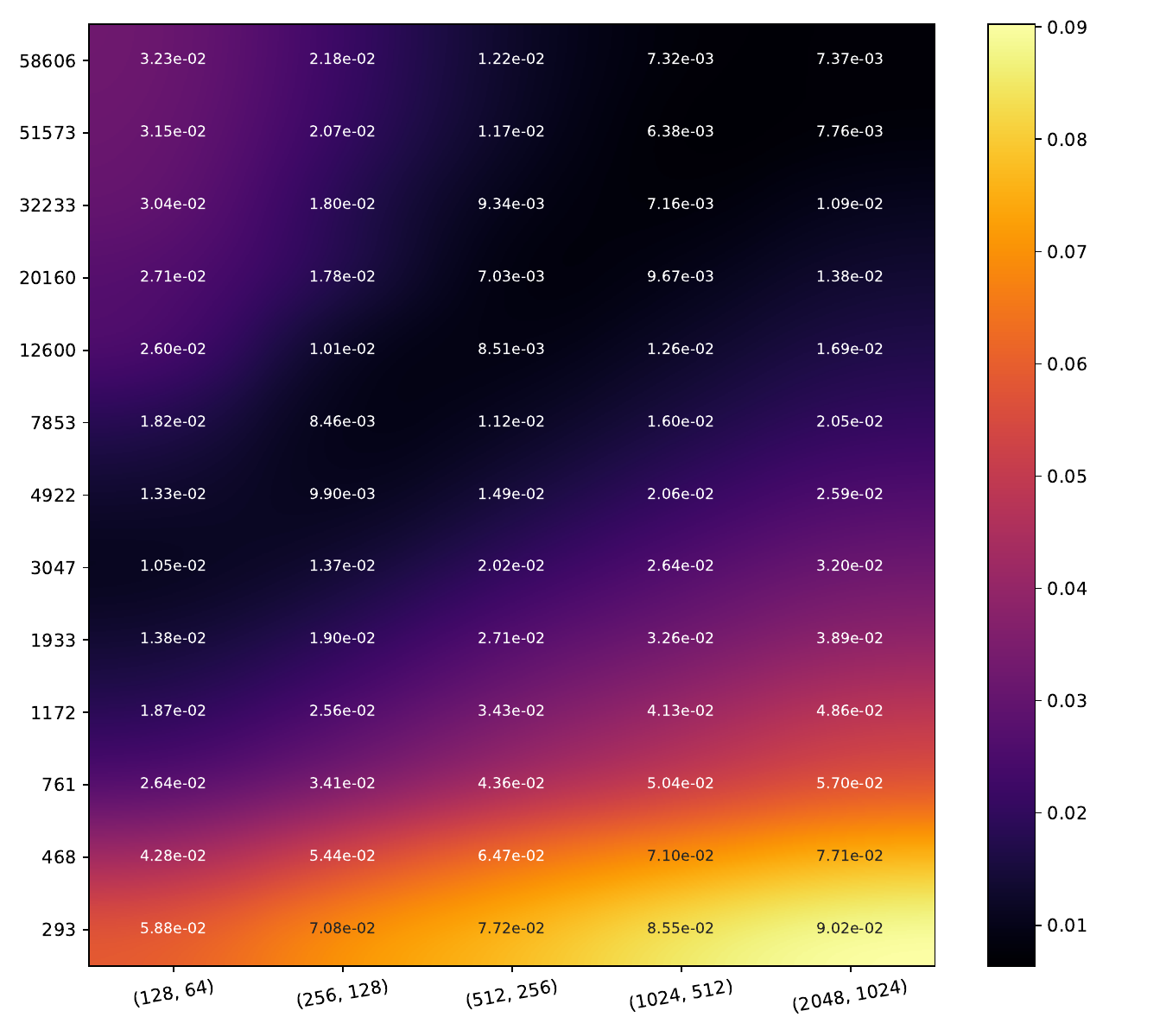} %
	& \includegraphics[width=0.159\textwidth,valign=m]{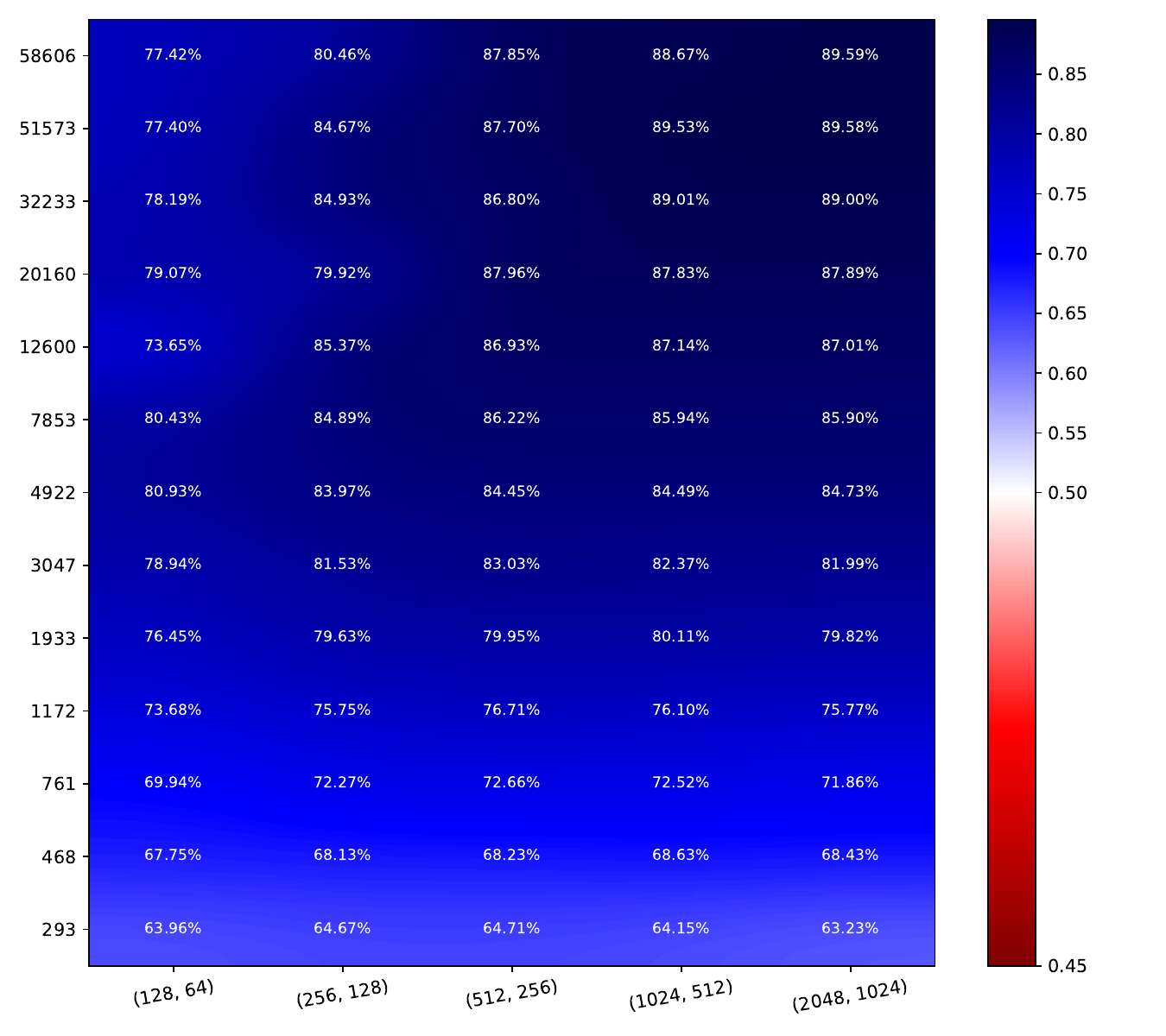} %
        & \includegraphics[width=0.159\textwidth,valign=m]{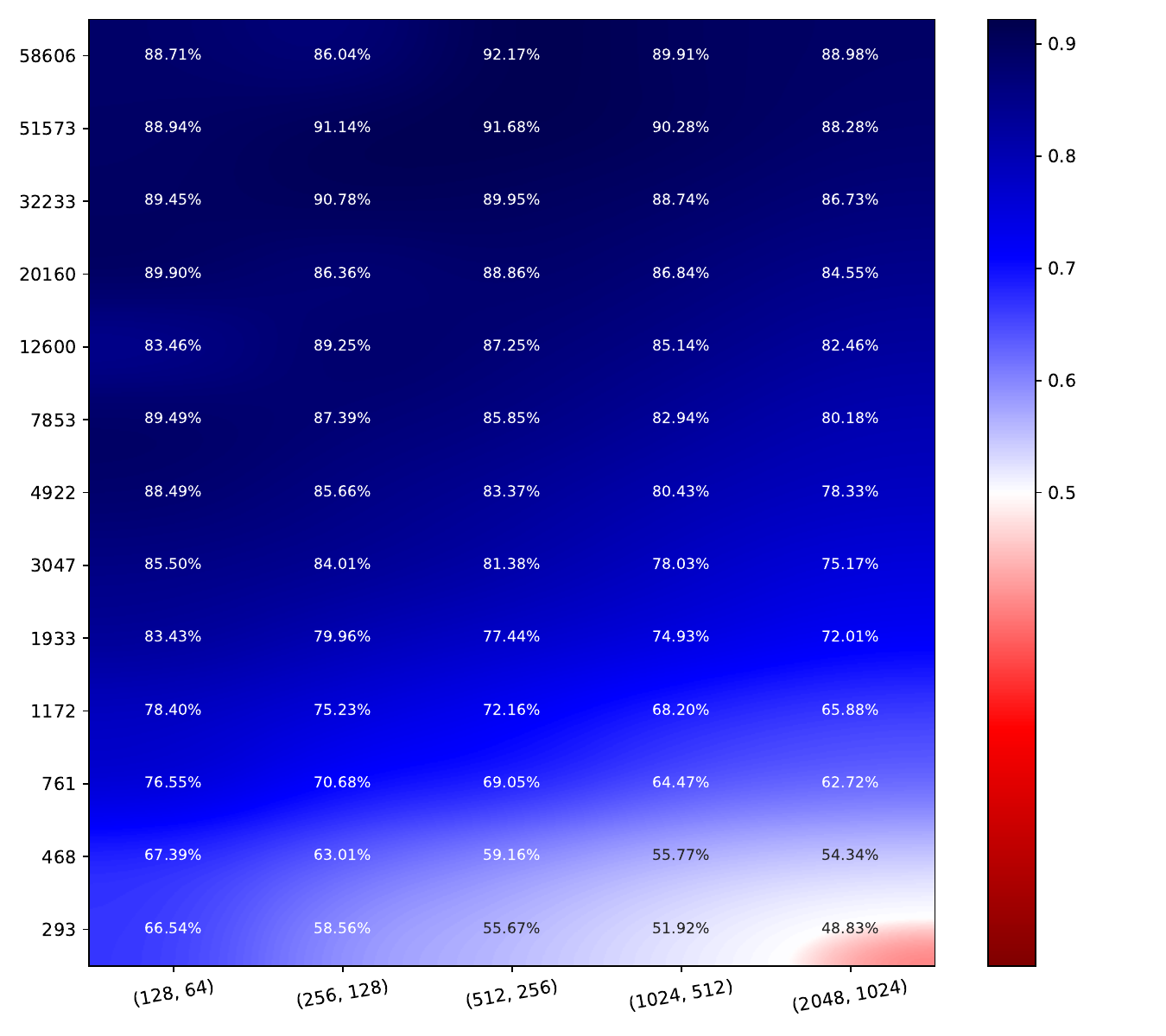} %
        \\
        \rotatebox[origin=c]{90}{CIFAR10} %
	& \includegraphics[width=0.159\textwidth,valign=m]{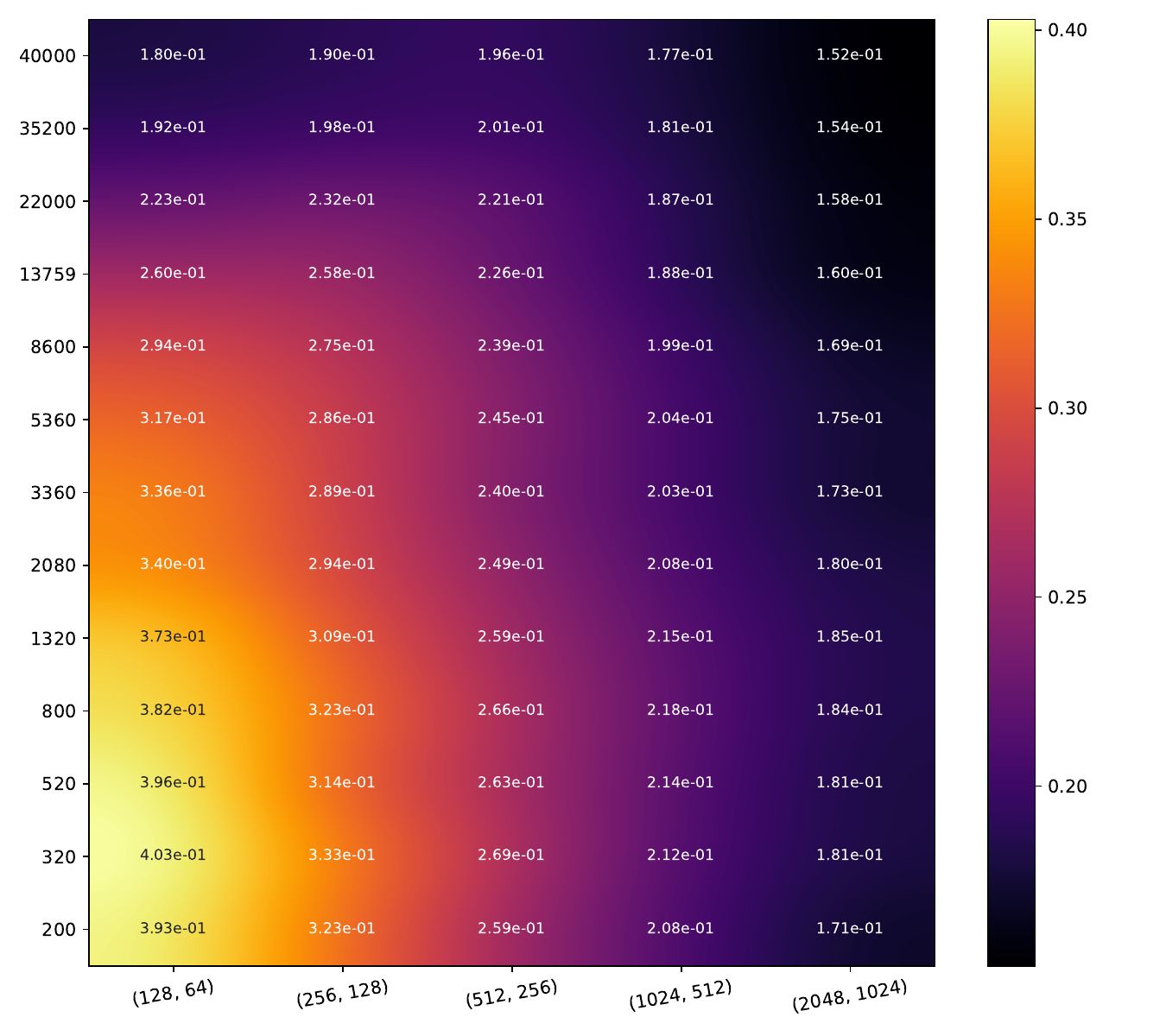} %
        & \includegraphics[width=0.159\textwidth,valign=m]{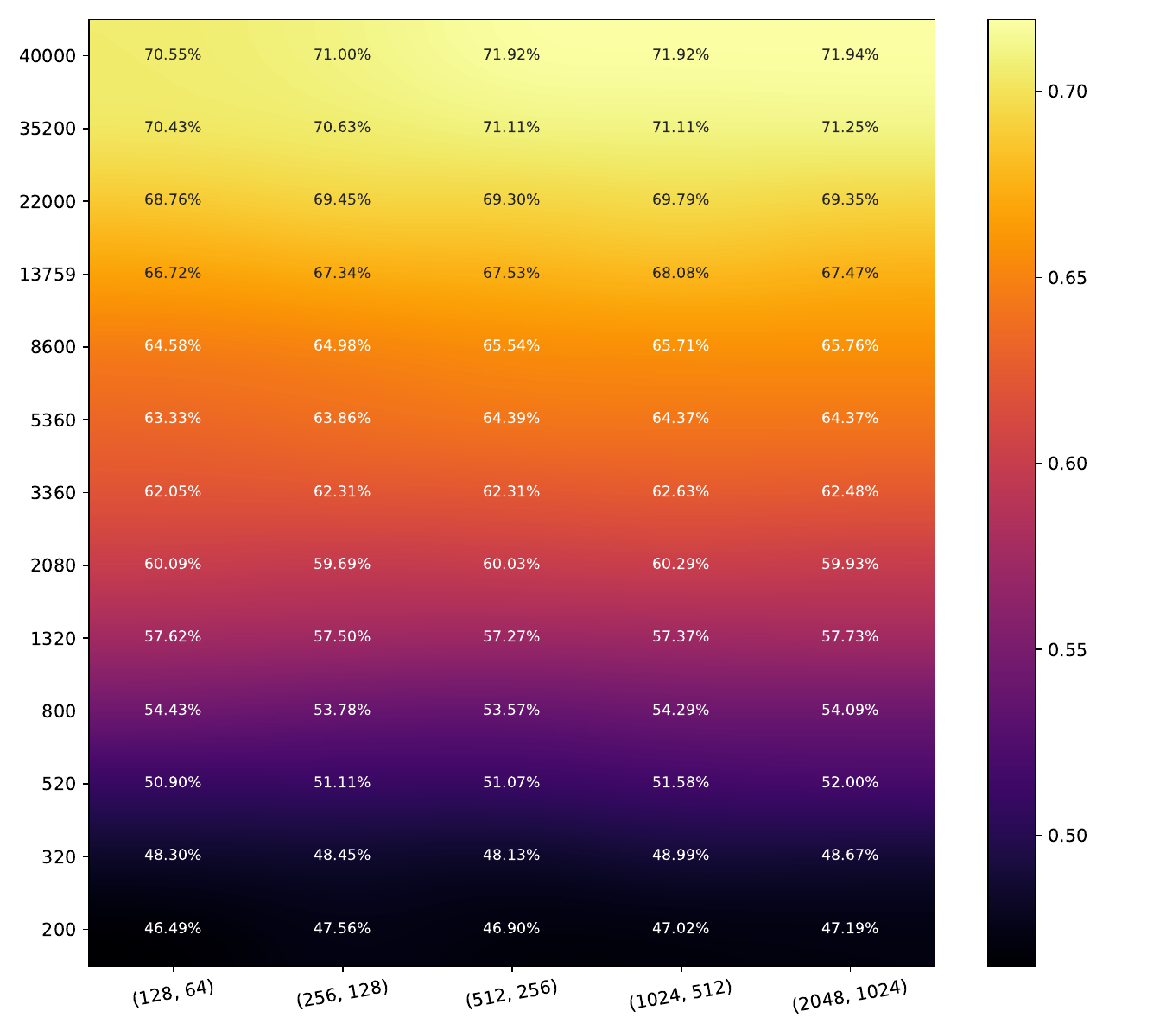} %
        & \includegraphics[width=0.159\textwidth,valign=m]{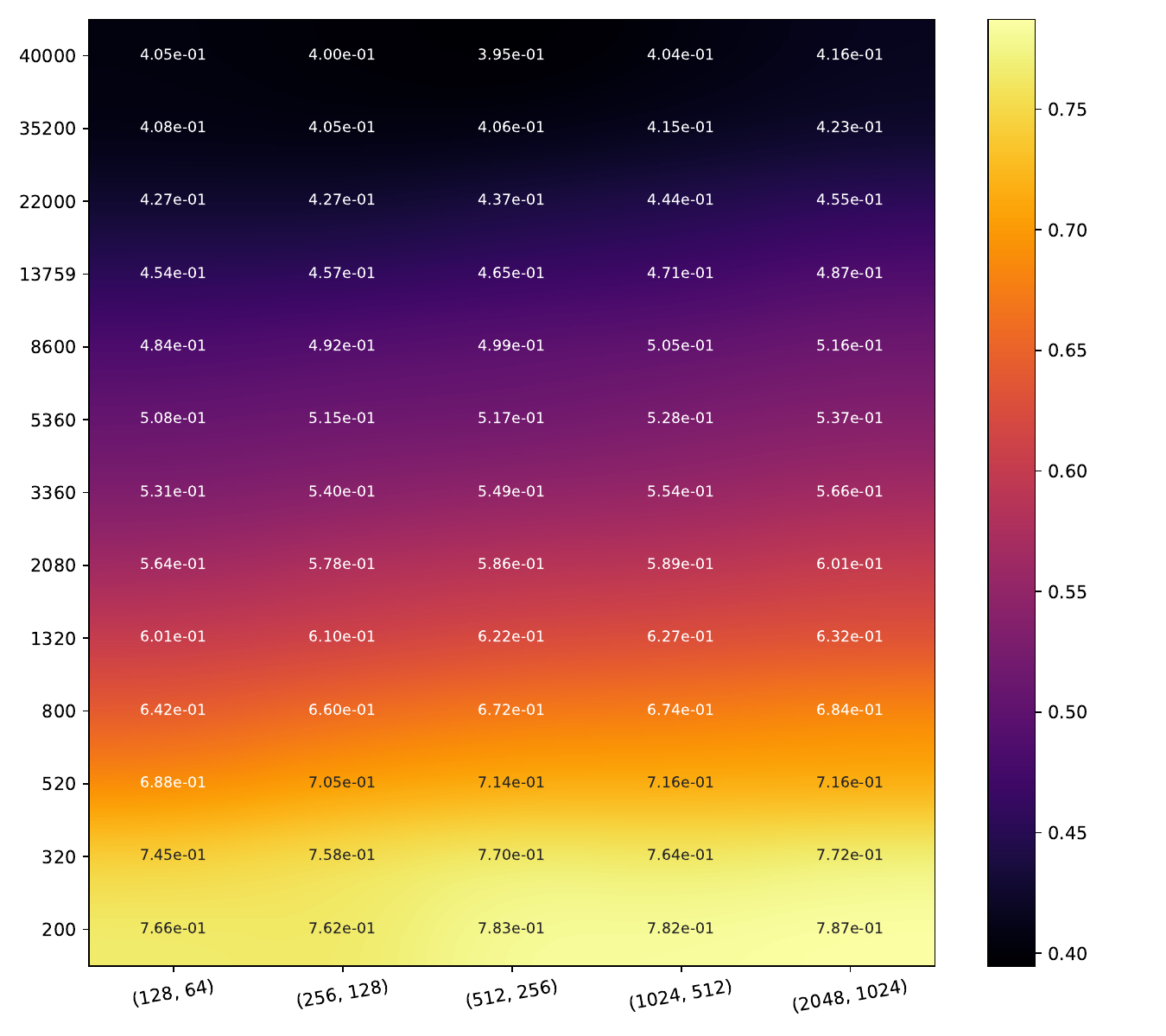} %
        & \includegraphics[width=0.159\textwidth,valign=m]{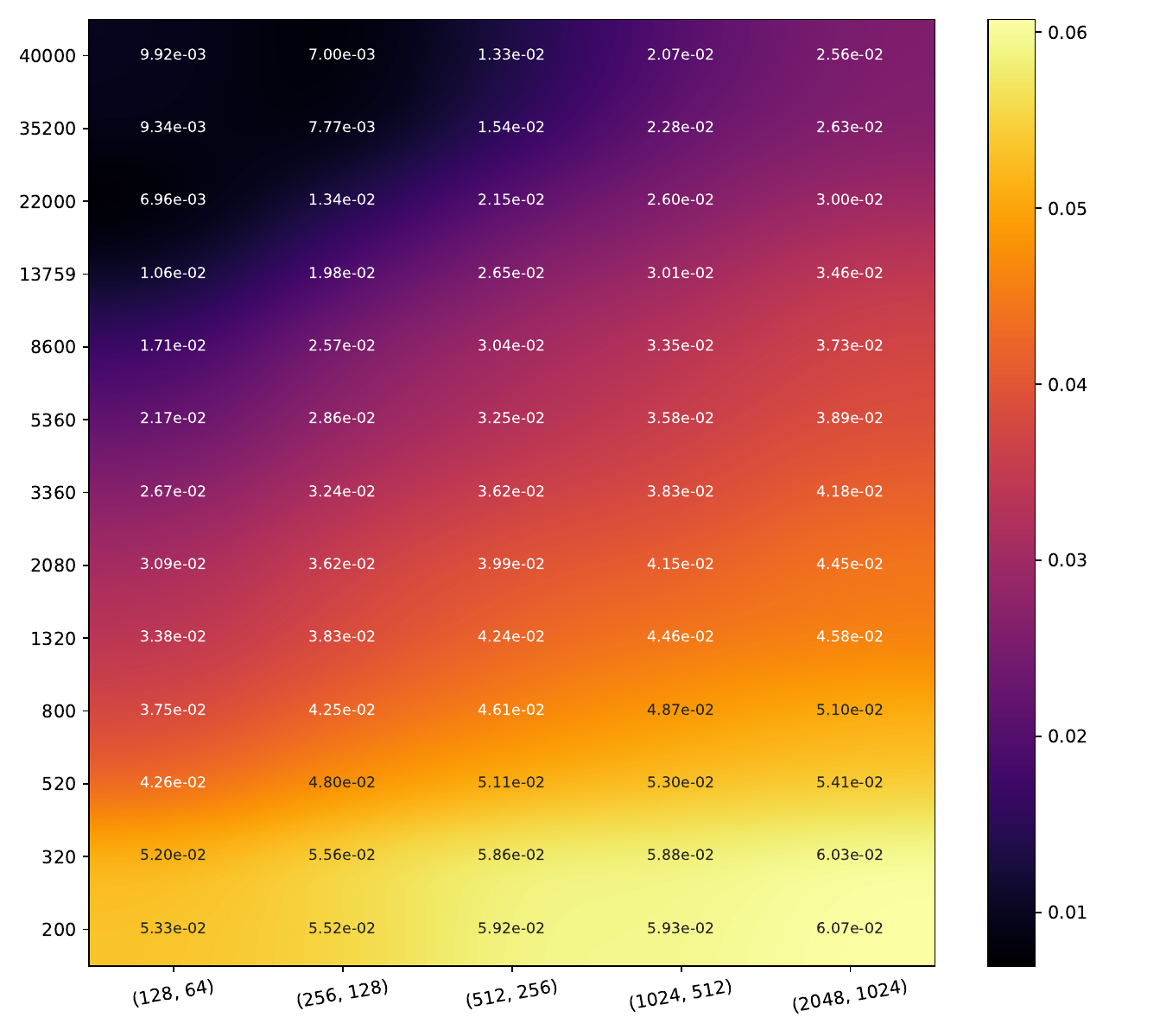} %
	& \includegraphics[width=0.159\textwidth,valign=m]{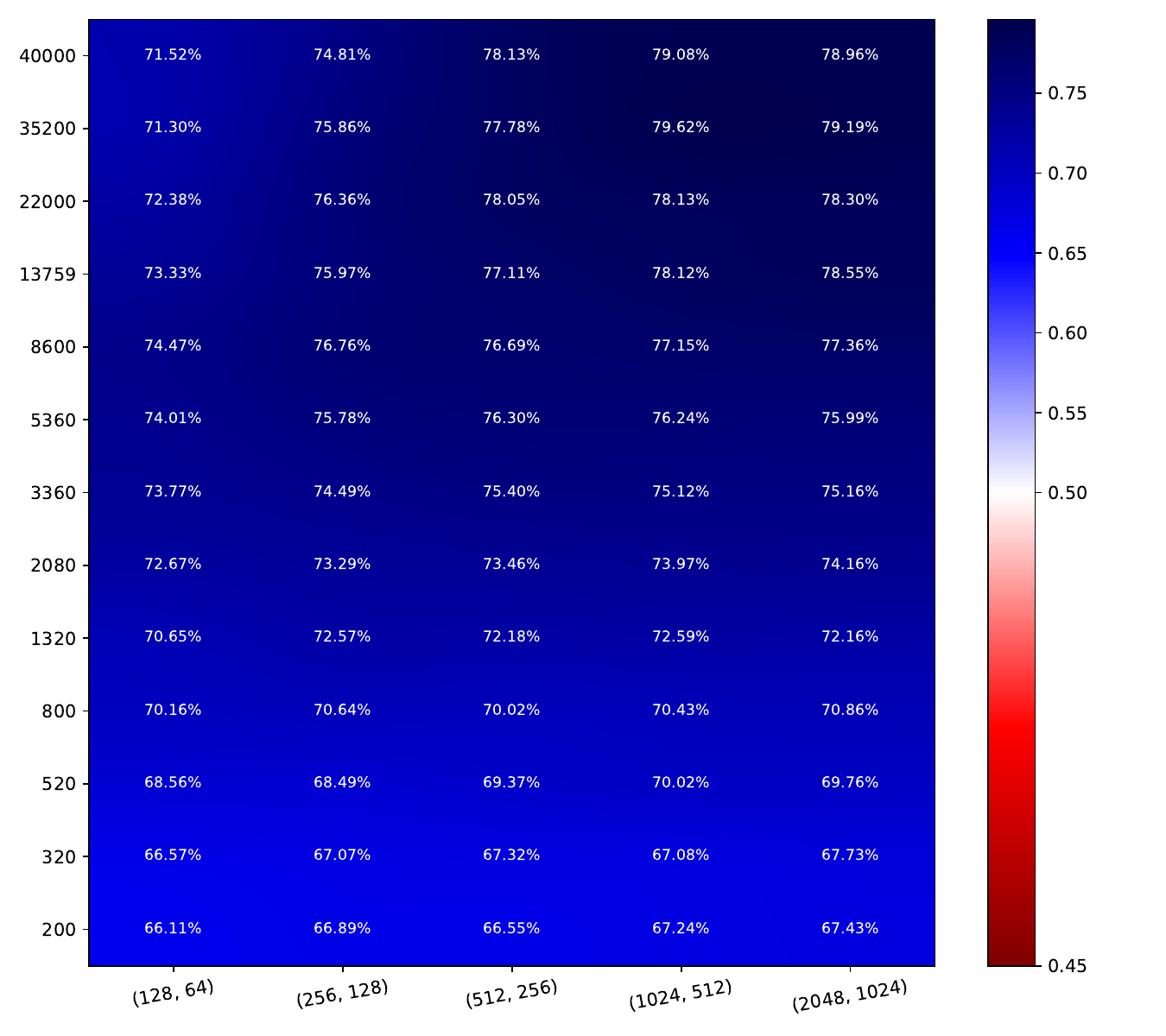} %
        & \includegraphics[width=0.159\textwidth,valign=m]{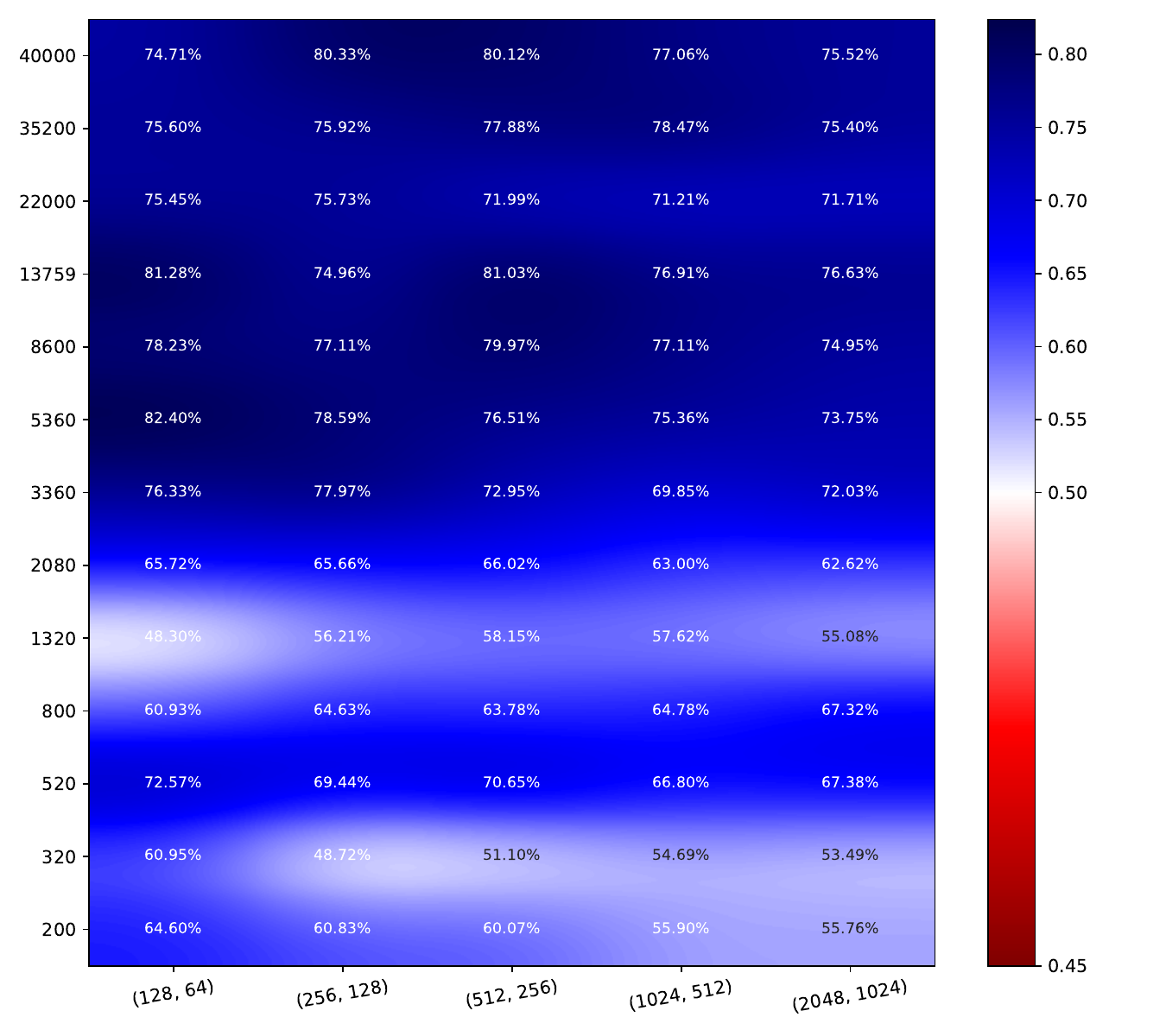} %
        \\
        \addlinespace
        & E.U. & Accuracy & Brier score & SCE & Misclassification & OOD
    \end{tabular}
    \label{fig-app:MC-Dropout-CP}
\end{table}

\begin{table}[ht]
  \centering
  \begin{tabular}{l P{0.1\textwidth} c c P{0.1\textwidth} c c}
    \toprule
     & E.U. & Accuracy & Brier score & SCE & Misclassification & OOD \\
    \midrule
    MNIST & $0.065$ & $93.34\%$ & $0.099$ & $0.0052$ & $92.50\%$ & $82.62\%$ \\
    CIFAR10 & $0.241$ & $60.71\%$ & $0.569$ & $0.0354$ & $73.28\%$ & $69.29\%$ \\
    \bottomrule
  \end{tabular}
  \caption{Aggregation of the results with MC-Dropout CP. Same metrics as in Table~\ref{tab:mean-heatmaps}.}
  \label{tab-app:agg-MC-CE}
\end{table}

\clearpage
\newpage

\section{Models without weight decay\label{app:no-wd}}
We also tested the same setup presented in Sect.~\ref{sec:experiments} but without weight decay. We use the same representation as in Sect.~\ref{sec:experiments} and Appendix~\ref{app:additional-results}. We notice that Conflictual DE stays consistent in all metrics: accuracy, calibration, Brier score, OOD detection, and misclassification detection. In addition, and most importantly, the $2$ principles are verified whereas, for example, both are uncorroborated in the case of DE trained on CIFAR10. Furthermore, the model-related principle is not verified in any other method.

Overall, the results with weight decay are relatively better.

\begin{table}[ht]
  \centering
  \begin{tabular}{p{0.15\textwidth} c P{0.155\textwidth} P{0.1\textwidth} P{0.1\textwidth} P{0.175\textwidth}}
    \toprule
     & & MC-Dropout & EDL & DE & Conflictual DE \\
    \midrule
    \multirow{2}{*}{E.U.} & MNIST & $0.068$ & $0.086$ & $0.014$ & $0.245$ \\
                          & CIFAR10 & $0.299$ & $0.137$ & $0.156$ & $0.544$ \\
    \midrule
    \multirow{2}{*}{Acc. $\uparrow$} & MNIST & $93.31\%$ & $91.45\%$ & $93.52\%$ & $\mathbf{93.86\%}$ \\
                              & CIFAR10 & $60.34\%$ & $60.89\%$ & $61.61\%$ & $\mathbf{61.82\%}$ \\
    \midrule
    \multirow{2}{*}{\shortstack[l]{Brier \\ score} $\downarrow$} & MNIST & $0.101$ & $0.164$ & $0.102$ & $\mathbf{0.094}$ \\
                                                                & CIFAR10 & $0.589$ & $0.539$ & $0.598$ & $\mathbf{0.532}$ \\
    \midrule
    \multirow{2}{*}{SCE $\downarrow$} & MNIST & $\mathbf{0.0057}$ & $0.0329$ & $0.0076$ & $0.0119$ \\
                        & CIFAR10 & $0.0418$ & $0.0187$ & $0.0474$ & $\mathbf{0.0207}$ \\
    \midrule
    \multirow{2}{*}{Mis. $\uparrow$} & MNIST & $92.55\%$ & $87.86\%$ & $\mathbf{93.43\%}$ & $93.09\%$ \\
                                     & CIFAR10 & $73.75\%$ & $72.98\%$ & $\mathbf{74.84\%}$ & $74.13\%$ \\    
    \midrule
    \multirow{2}{*}{OOD $\uparrow$} & MNIST   & $82.85\%$ & $81.85\%$ & $\mathbf{88.19\%}$ & $86.50\%$ \\
                                    & CIFAR10 & $67.49\%$ & $69.00\%$ & $67.68\%$ & $\mathbf{73.06\%}$ \\
    \bottomrule
  \end{tabular}
  \caption{Average value per heatmap for different metrics without weight decay. Same metrics as in Table~\ref{tab:mean-heatmaps}.}
  \label{tab-app:mean-heatmaps-nowd}
\end{table}

\begin{figure}[ht]
  \centering
  \begin{subfigure}[t]{\dimexpr0.23\textwidth+20pt\relax}
    \makebox[20pt]{\raisebox{35pt}{\rotatebox[origin=c]{90}{\scriptsize MNIST}}}%
    \includegraphics[width=\dimexpr\linewidth-20pt\relax]{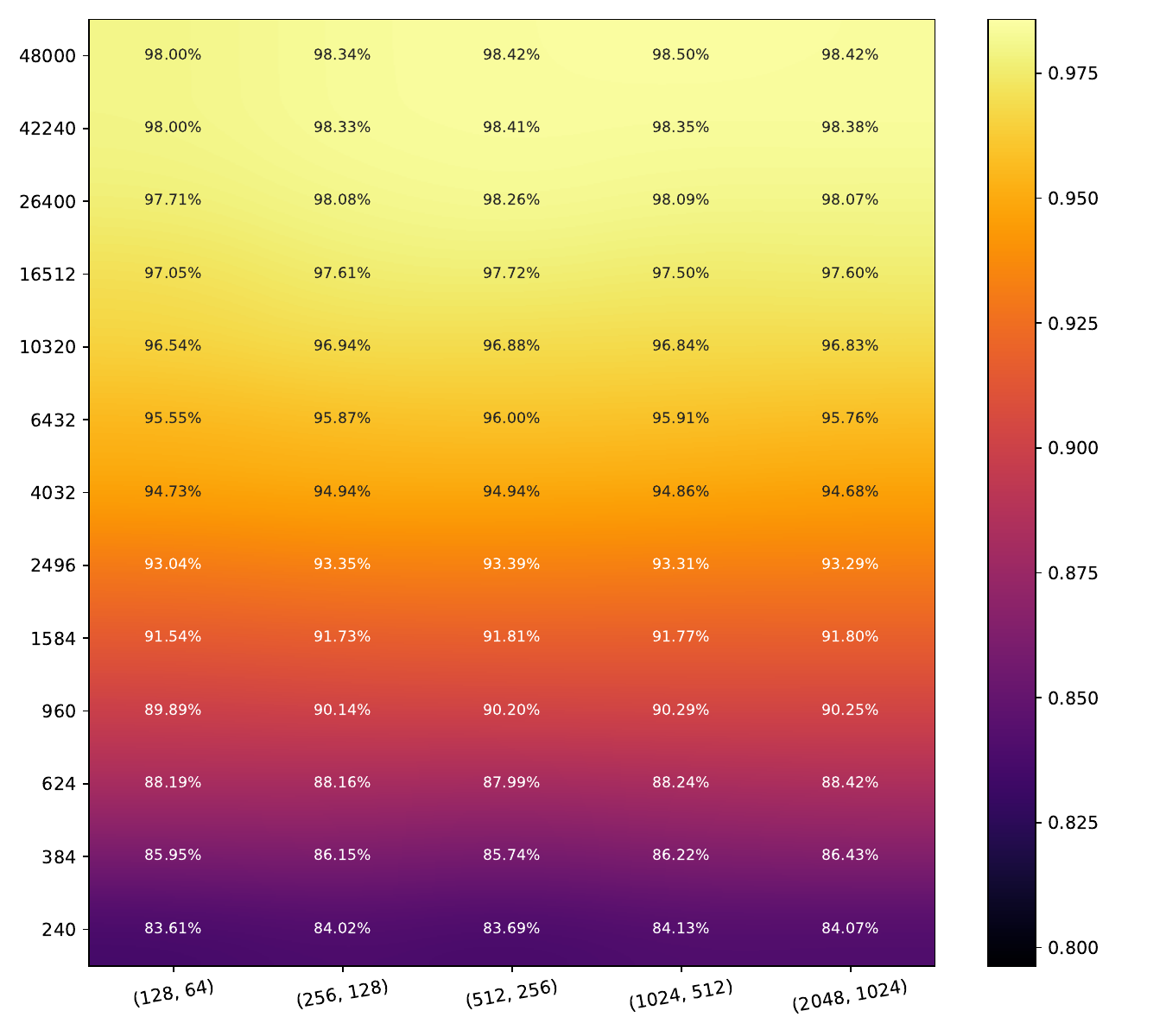}
    \makebox[20pt]{\raisebox{35pt}{\rotatebox[origin=c]{90}{\scriptsize CIFAR10}}}%
    \includegraphics[width=\dimexpr\linewidth-20pt\relax]{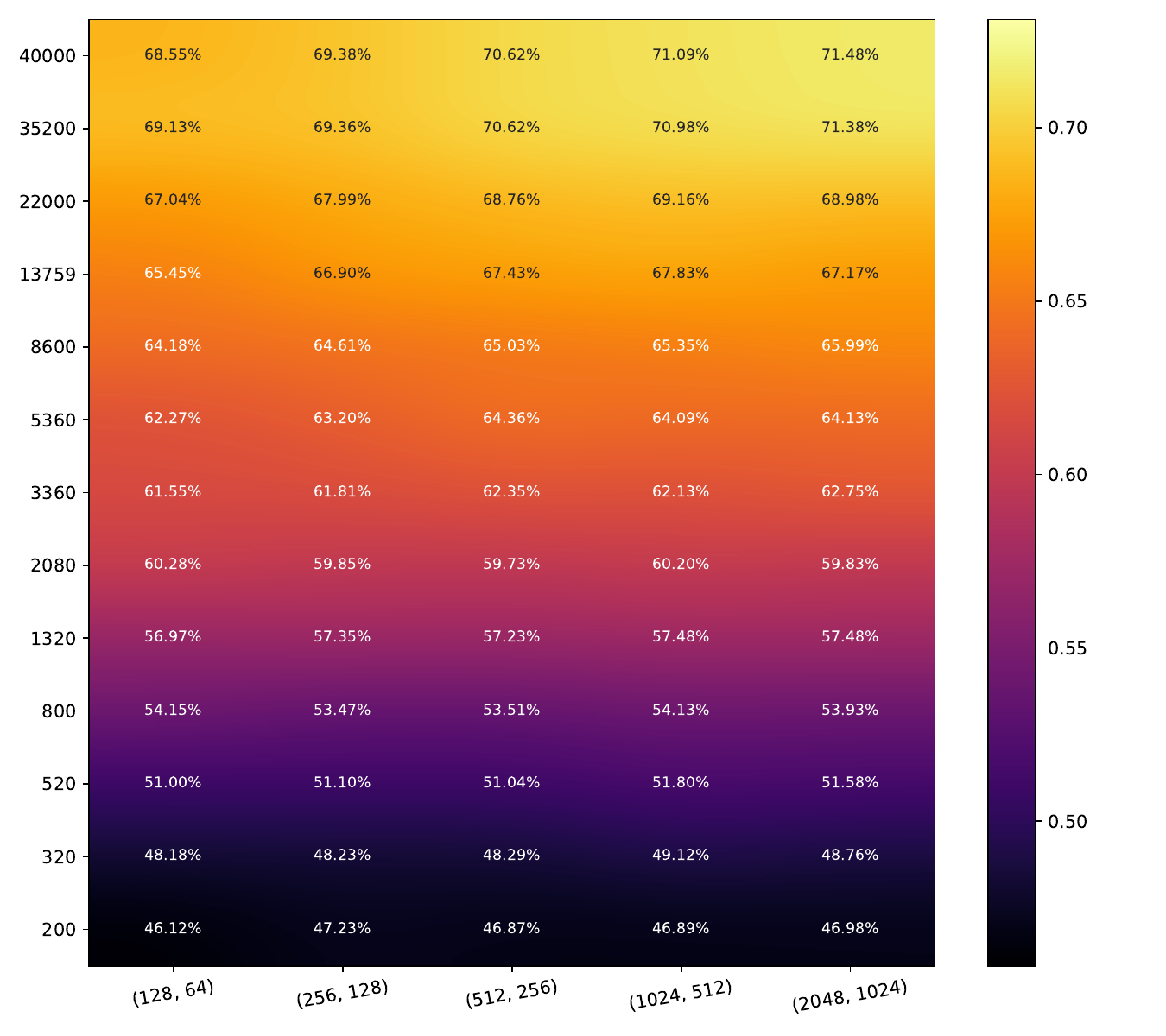}
    \caption*{\qquad MC-Dropout}
  \end{subfigure}\hfill
  \begin{subfigure}[t]{0.23\textwidth}
    \includegraphics[width=\textwidth]{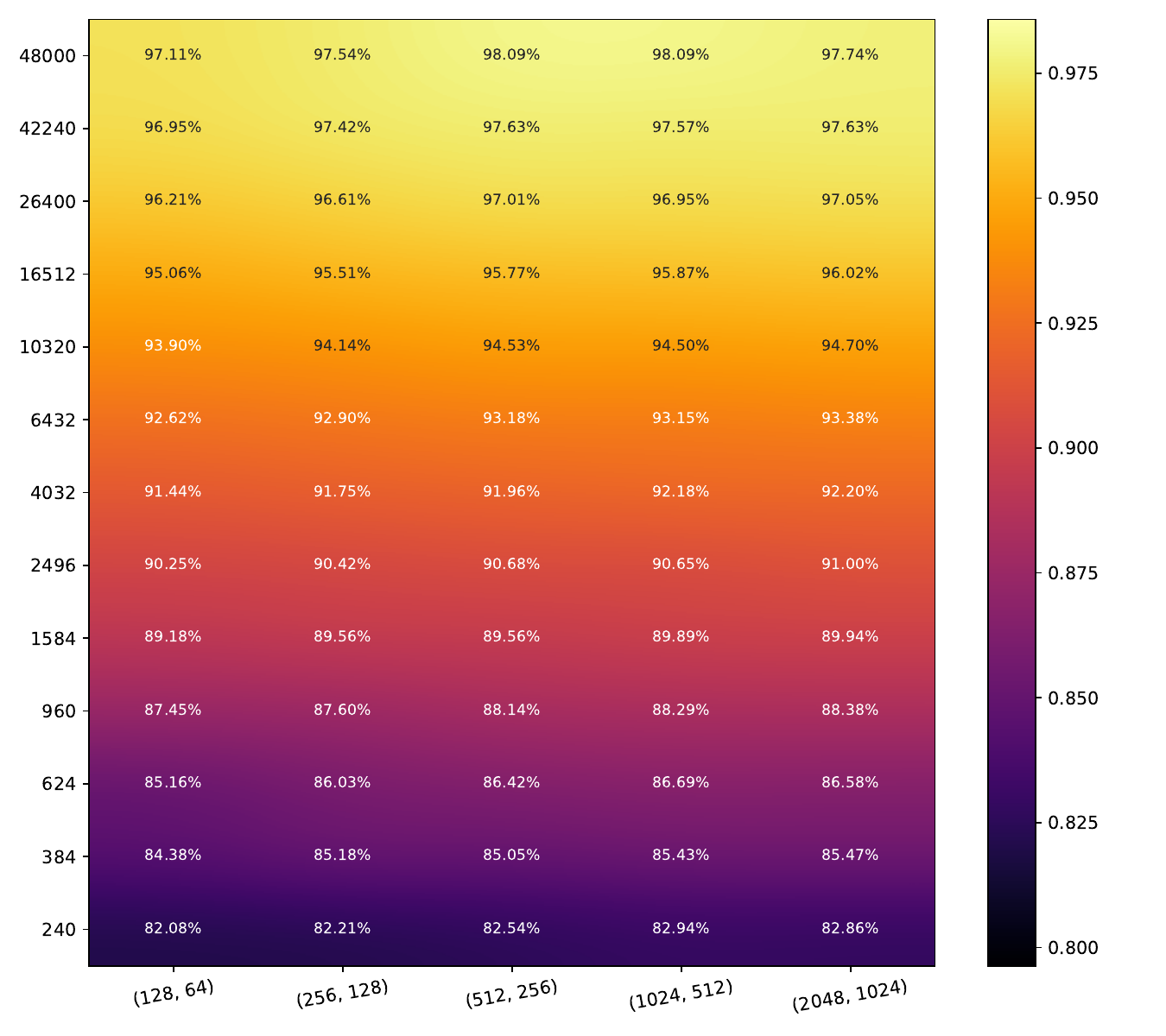}
    \includegraphics[width=\textwidth]{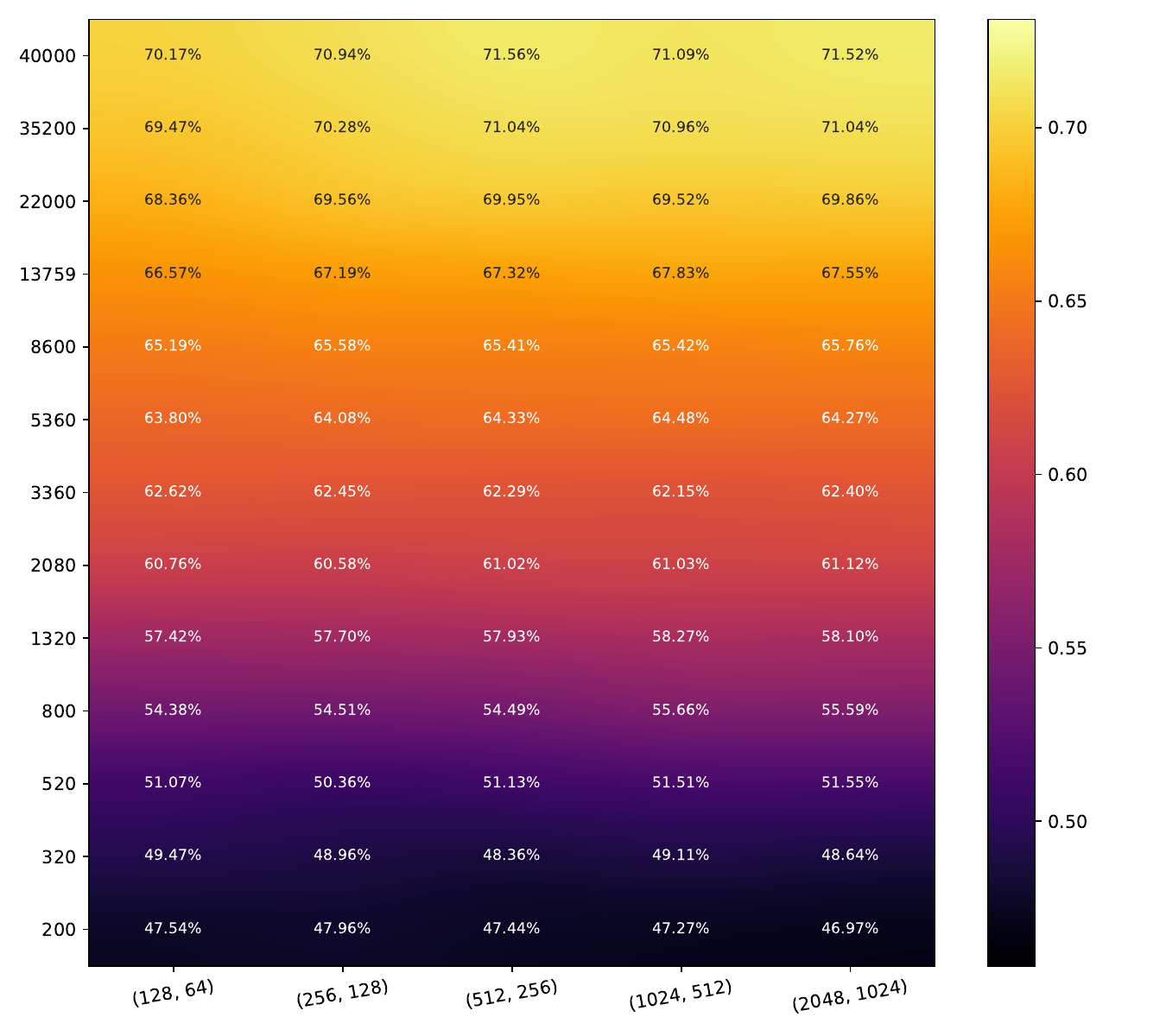}
    \caption*{EDL}
  \end{subfigure}\hfill
  \begin{subfigure}[t]{0.23\textwidth}
    \includegraphics[width=\textwidth]{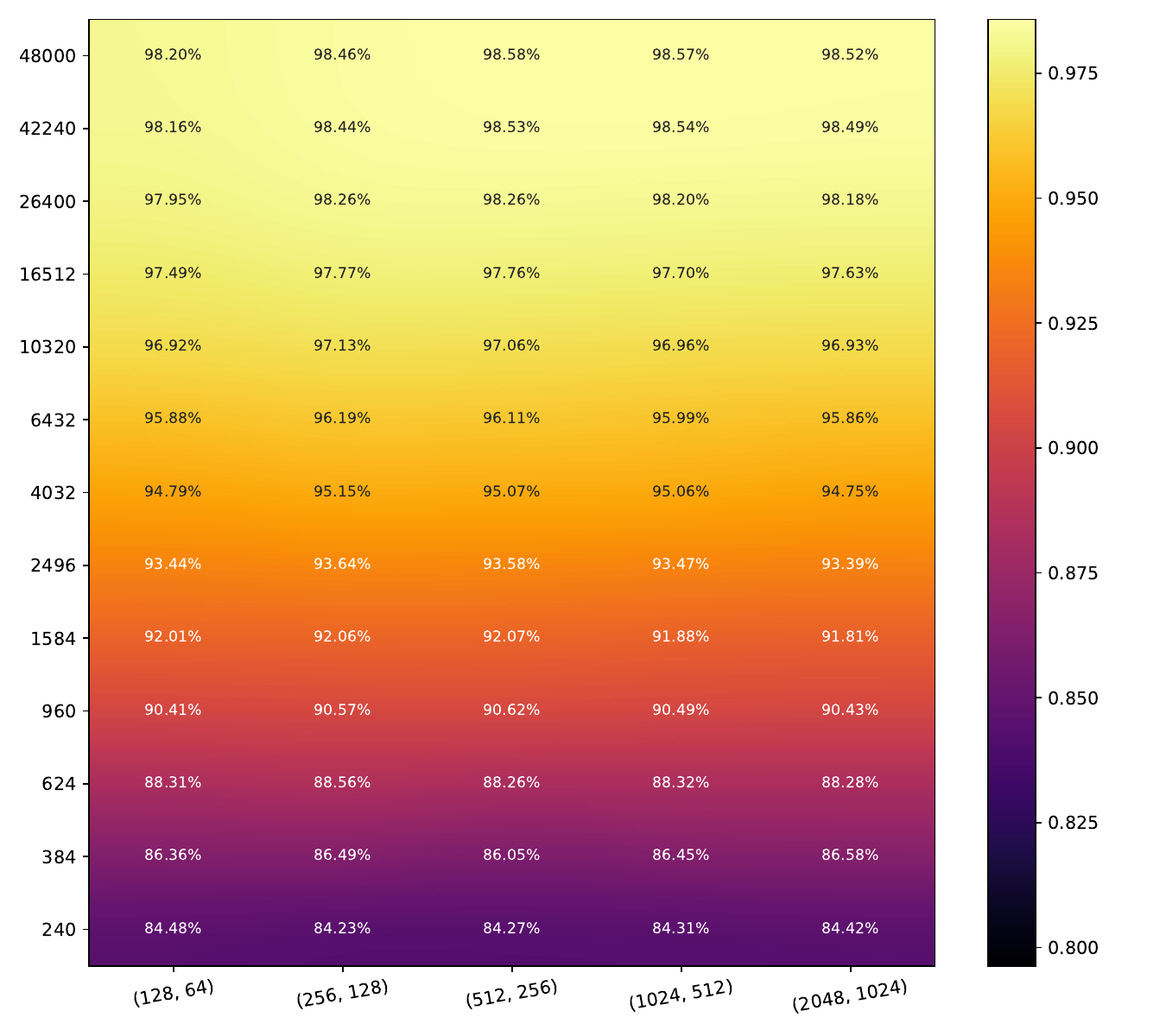}
    \includegraphics[width=\textwidth]{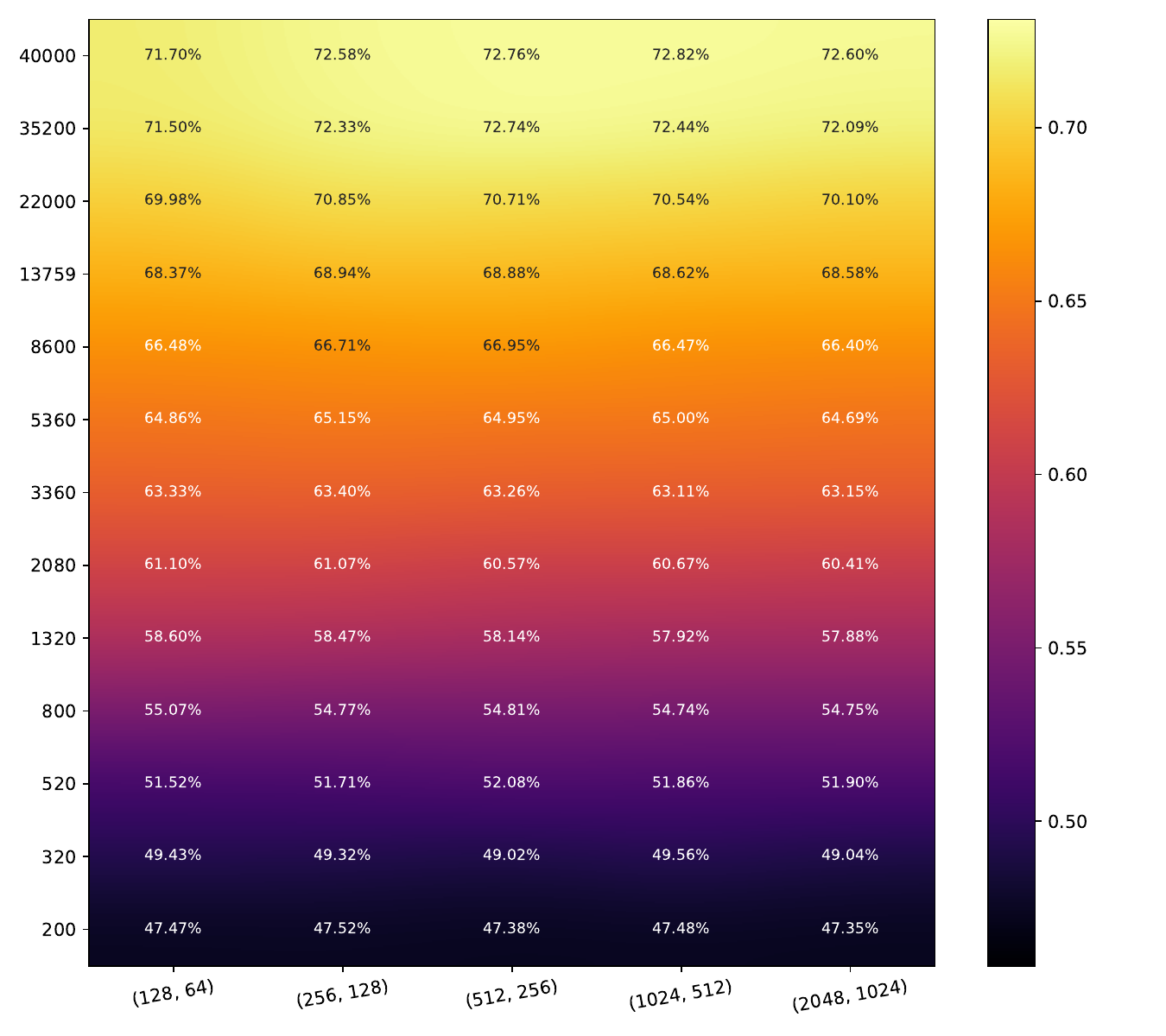}
    \caption*{DE}
  \end{subfigure}\hfill
  \begin{subfigure}[t]{0.23\textwidth}
    \includegraphics[width=\textwidth]{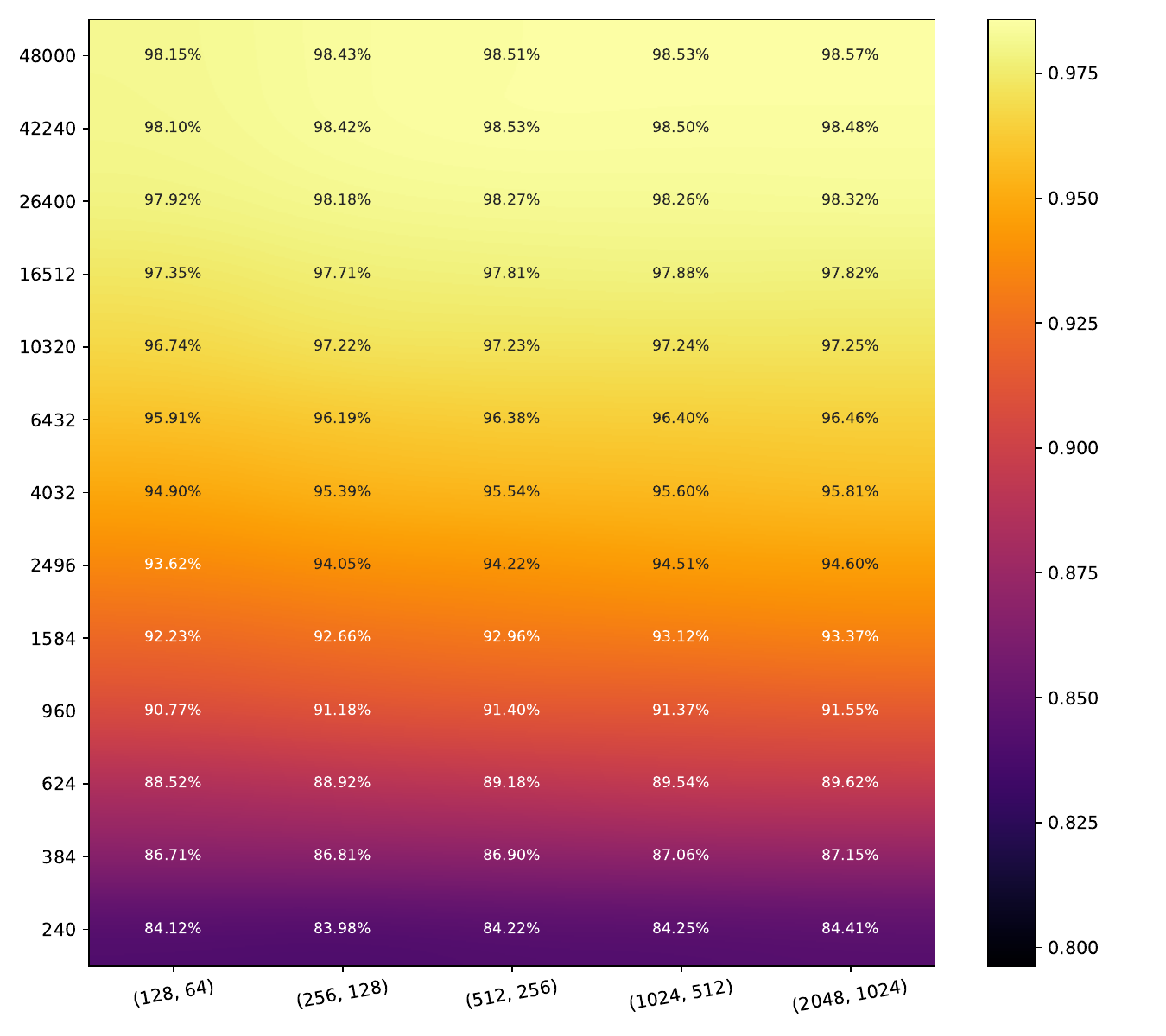}
    \includegraphics[width=\textwidth]{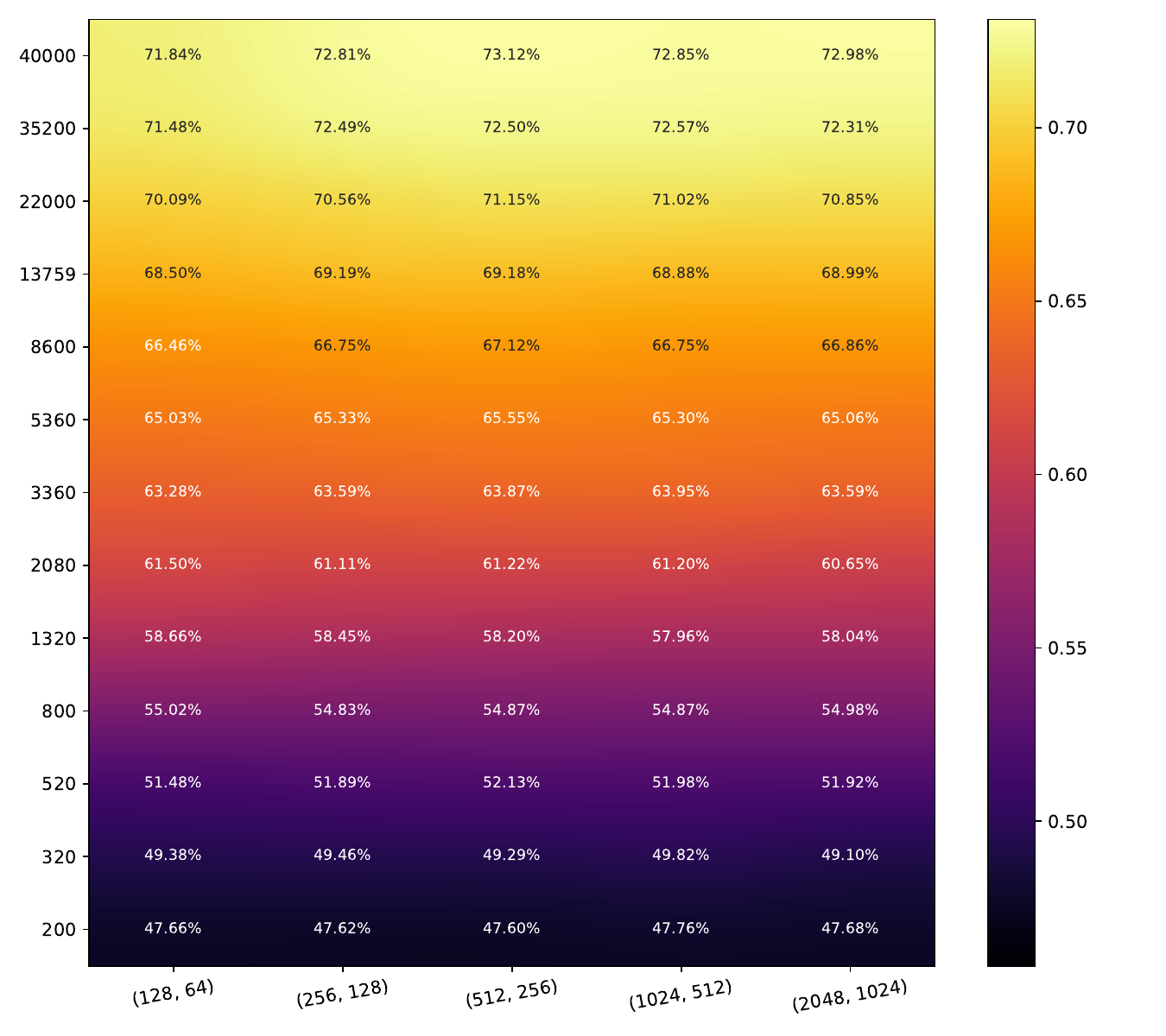}
    \caption*{Conflictual DE}
  \end{subfigure}\hfill
  \caption{Heatmaps of Accuracy. Color scales are the same per dataset.}
\end{figure}

\begin{figure}[ht]
  \begin{subfigure}[t]{\dimexpr0.23\textwidth+20pt\relax}
    \makebox[20pt]{\raisebox{35pt}{\rotatebox[origin=c]{90}{\scriptsize MNIST}}}%
    \includegraphics[width=\dimexpr\linewidth-20pt\relax]{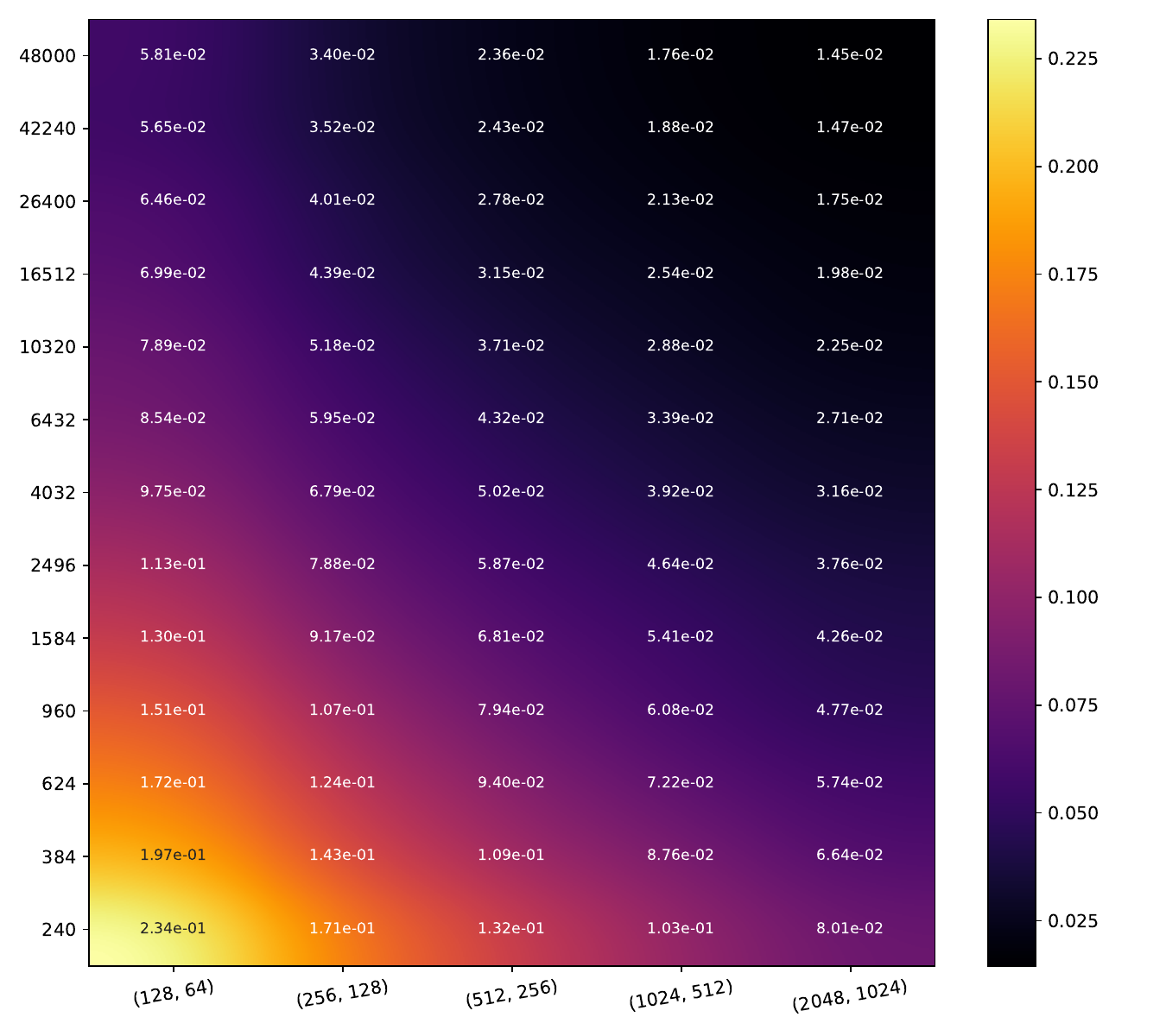}
    \makebox[20pt]{\raisebox{35pt}{\rotatebox[origin=c]{90}{\scriptsize CIFAR10}}}%
    \includegraphics[width=\dimexpr\linewidth-20pt\relax]{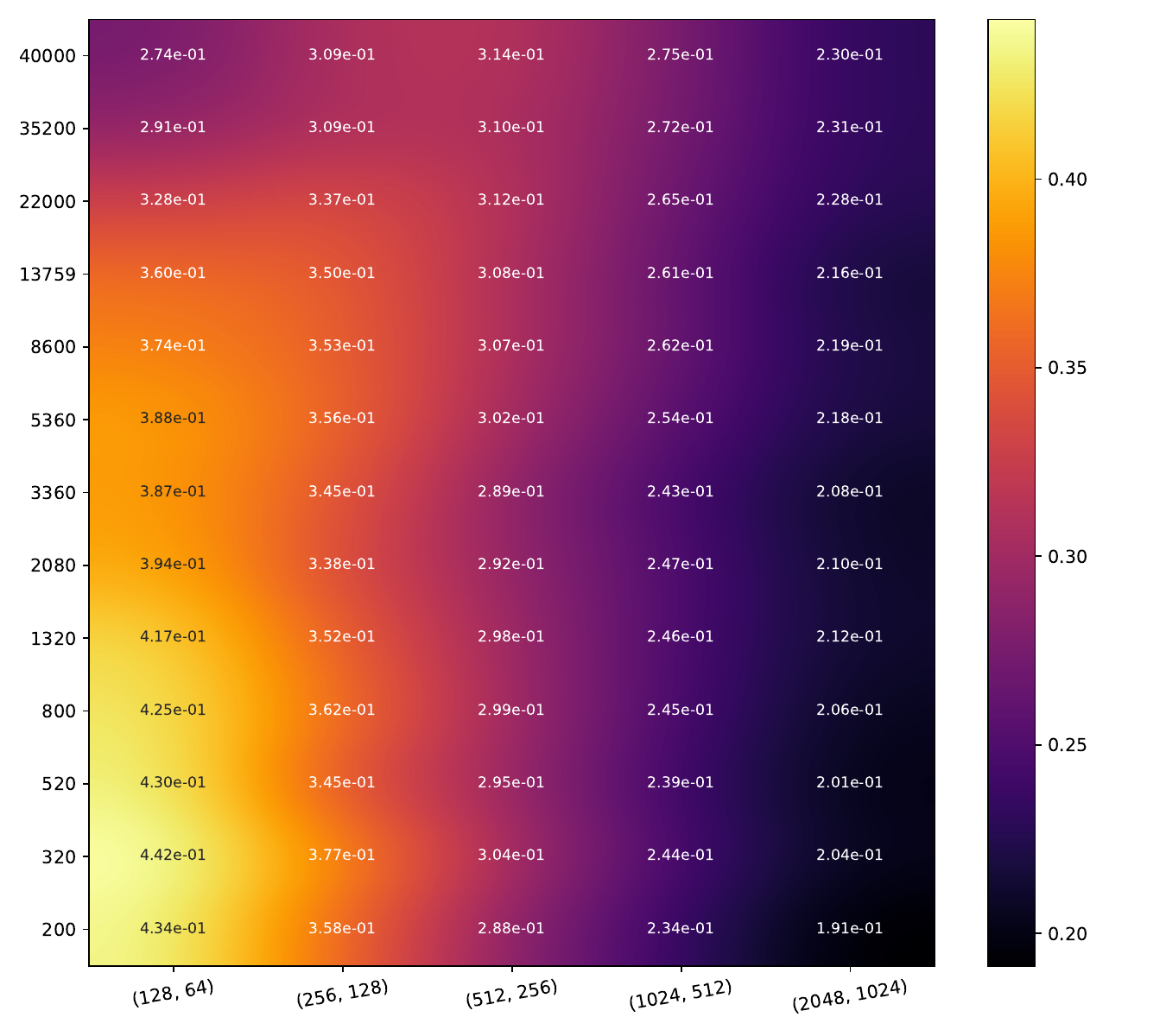}
    \caption*{\qquad MC-Dropout}
  \end{subfigure}\hfill
  \begin{subfigure}[t]{0.23\textwidth}
    \includegraphics[width=\textwidth]{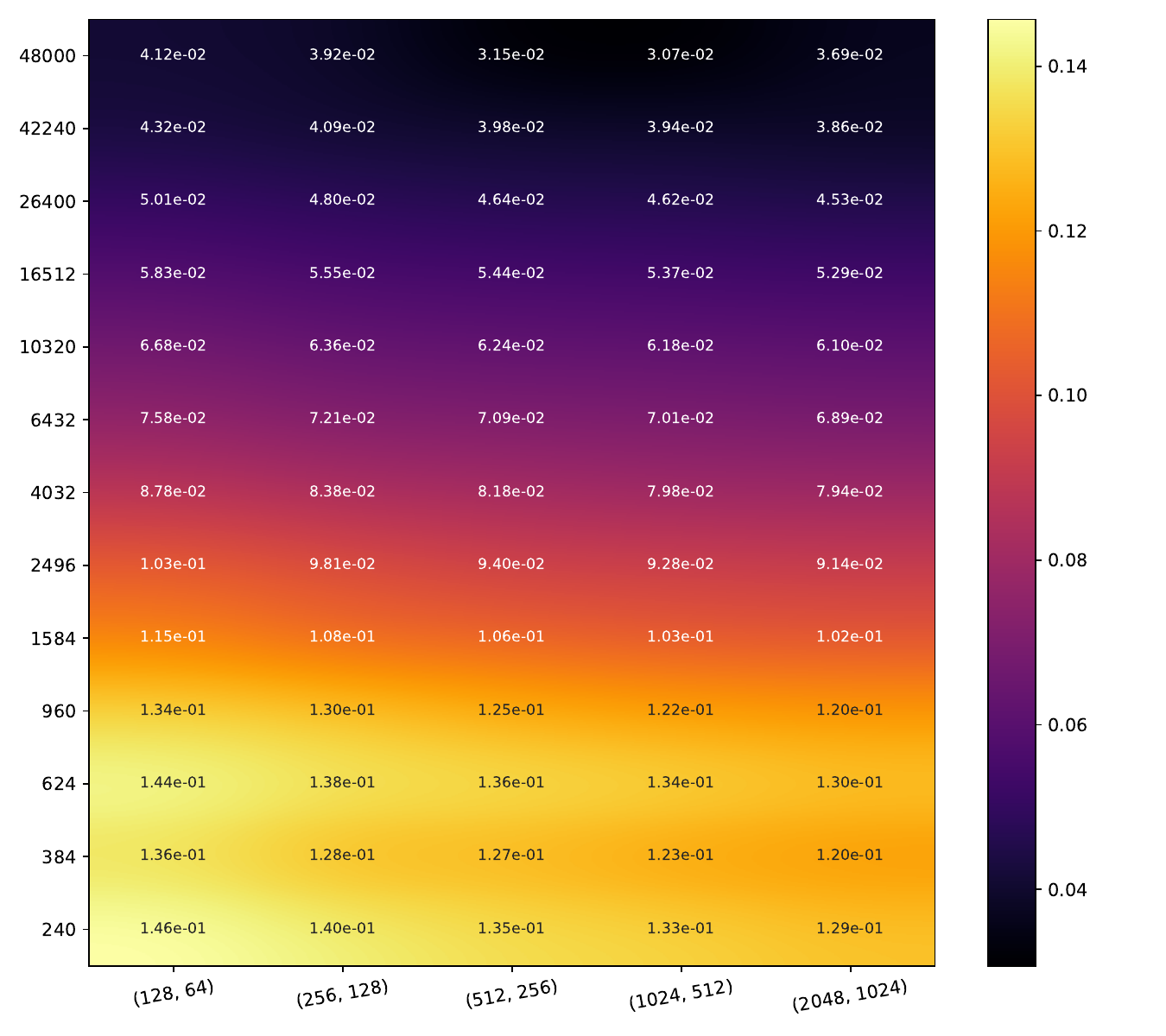}
    \includegraphics[width=\textwidth]{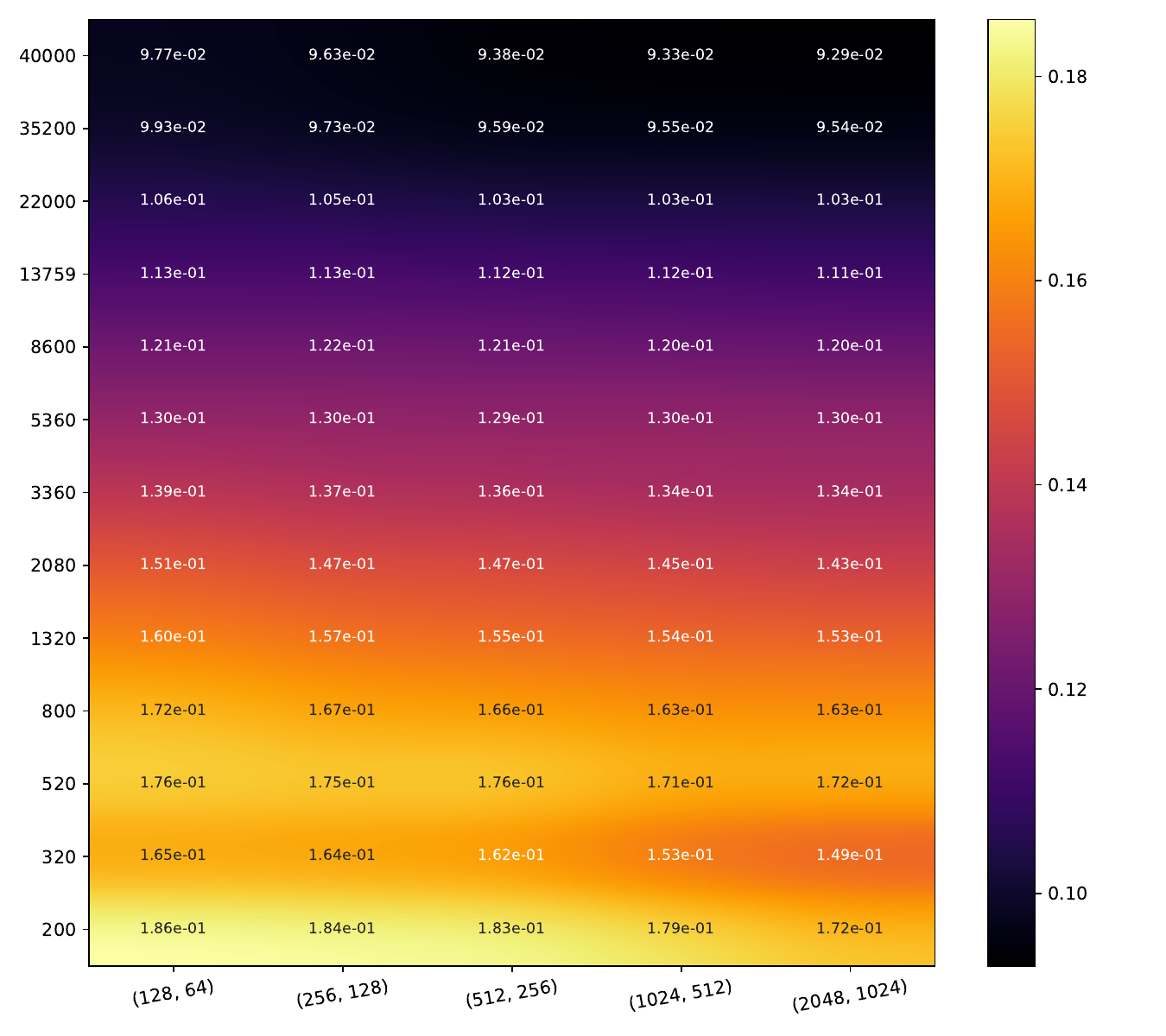}
    \caption*{EDL}
  \end{subfigure}\hfill
  \begin{subfigure}[t]{0.23\textwidth}
    \includegraphics[width=\textwidth]{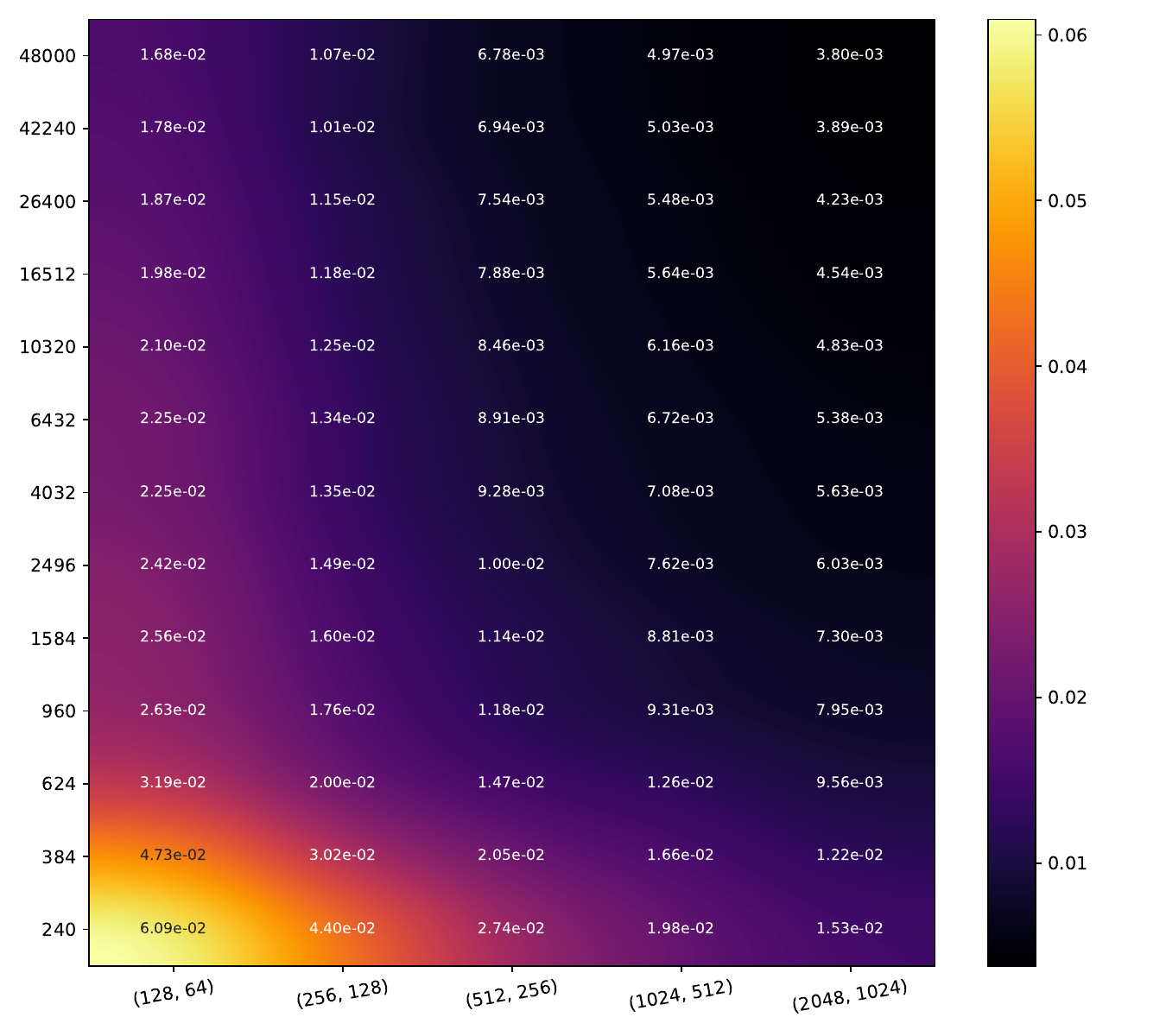}
    \includegraphics[width=\textwidth]{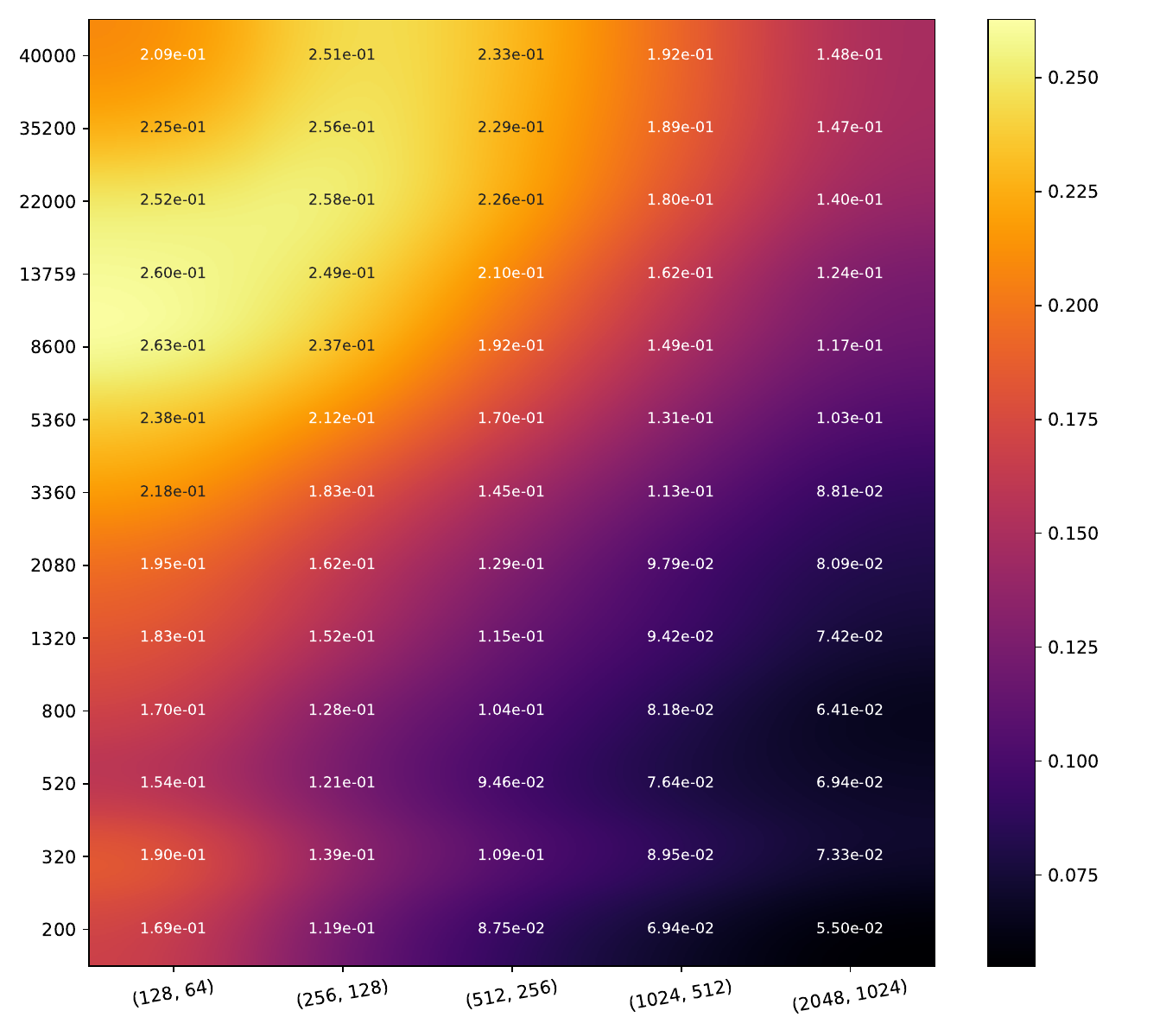}
    \caption*{DE}
  \end{subfigure}\hfill
  \begin{subfigure}[t]{0.23\textwidth}
    \includegraphics[width=\textwidth]{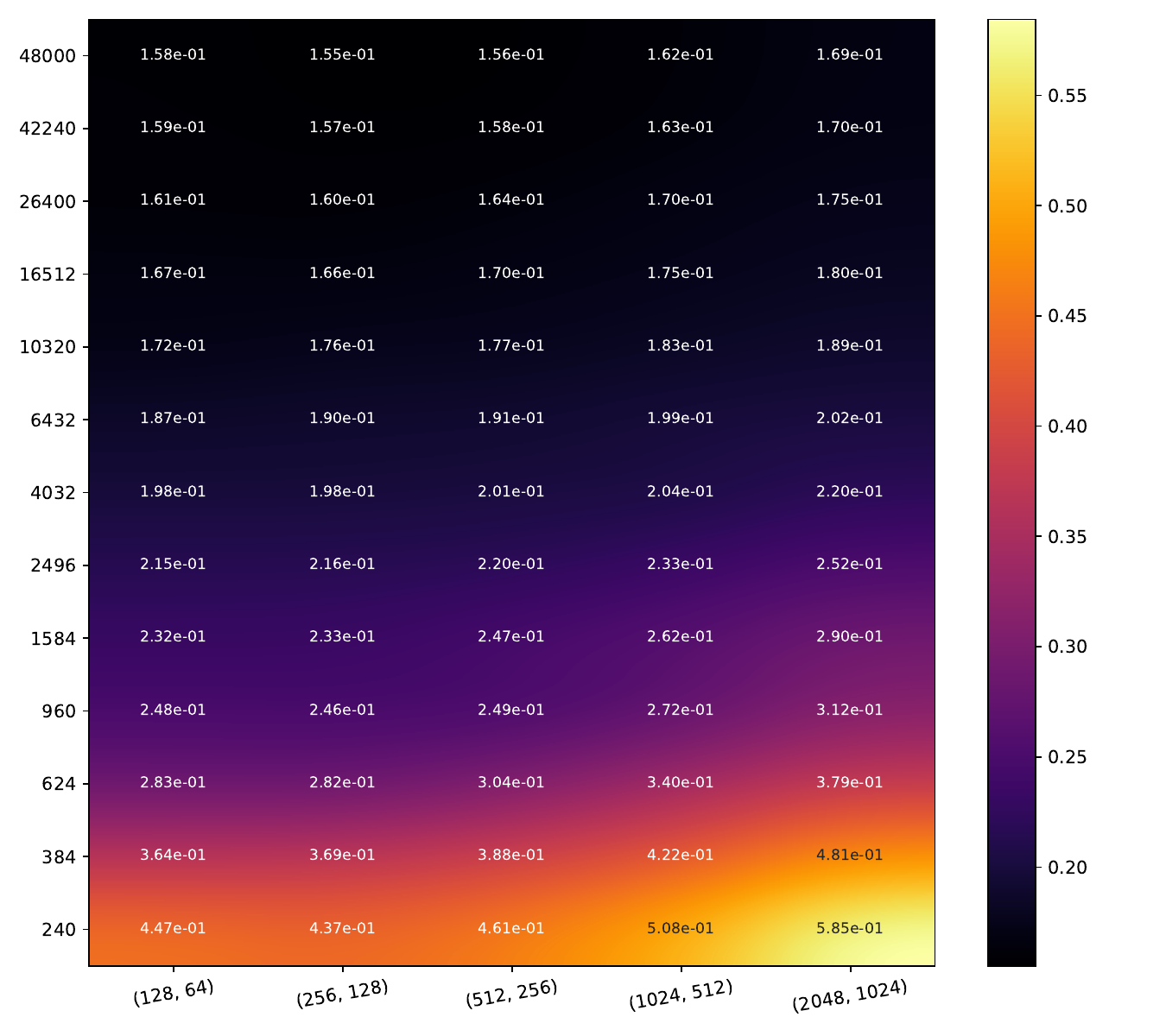}
    \includegraphics[width=\textwidth]{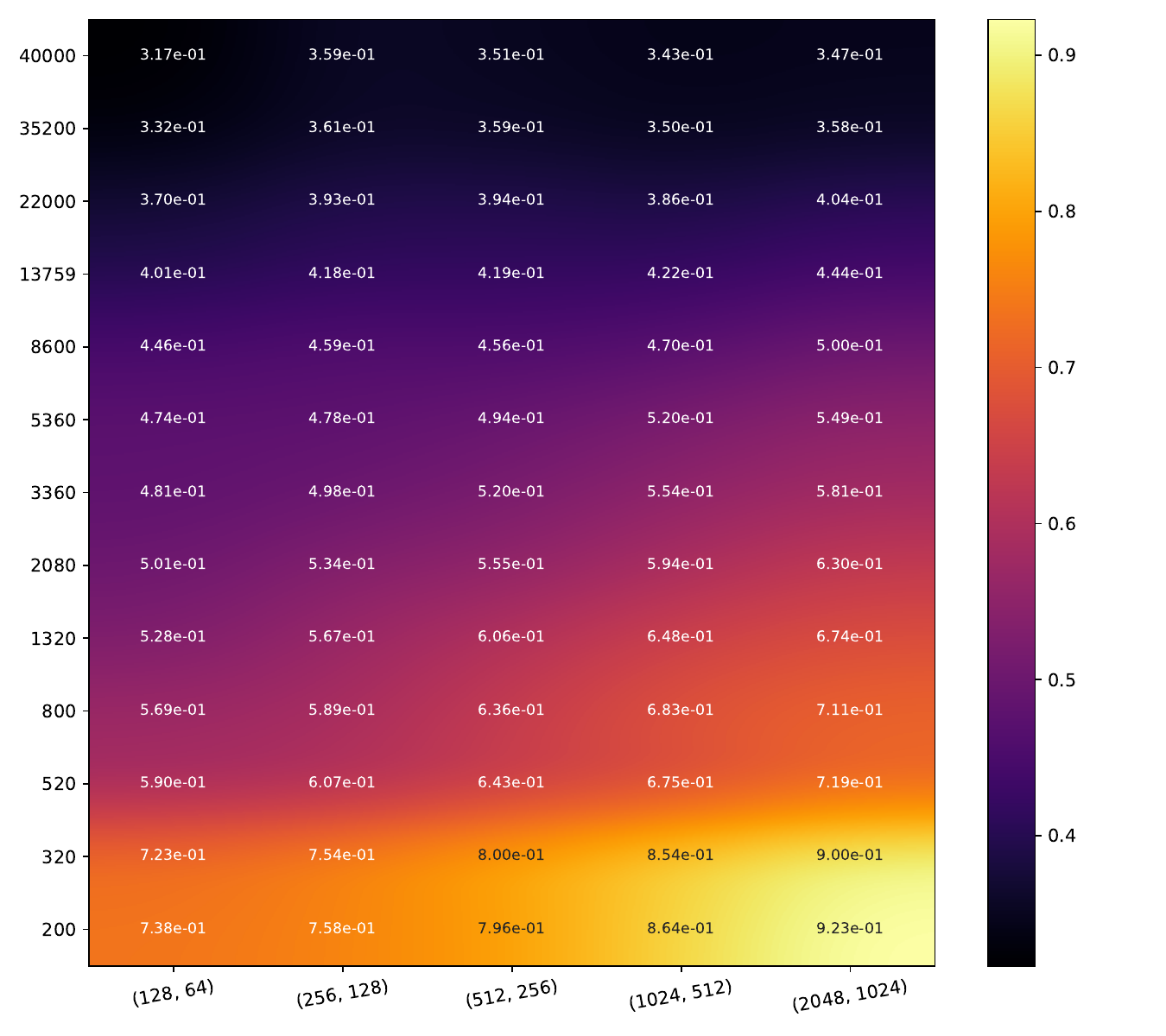}
    \caption*{Conflictual DE}
  \end{subfigure}\hfill
  \caption{Heatmaps of epistemic uncertainty (mutual information). Color scales are different for each heatmap.}
\end{figure}

\begin{figure}[ht]
  \centering
  \begin{subfigure}[t]{\dimexpr0.23\textwidth+20pt\relax}
    \makebox[20pt]{\raisebox{35pt}{\rotatebox[origin=c]{90}{\scriptsize MNIST}}}%
    \includegraphics[width=\dimexpr\linewidth-20pt\relax]{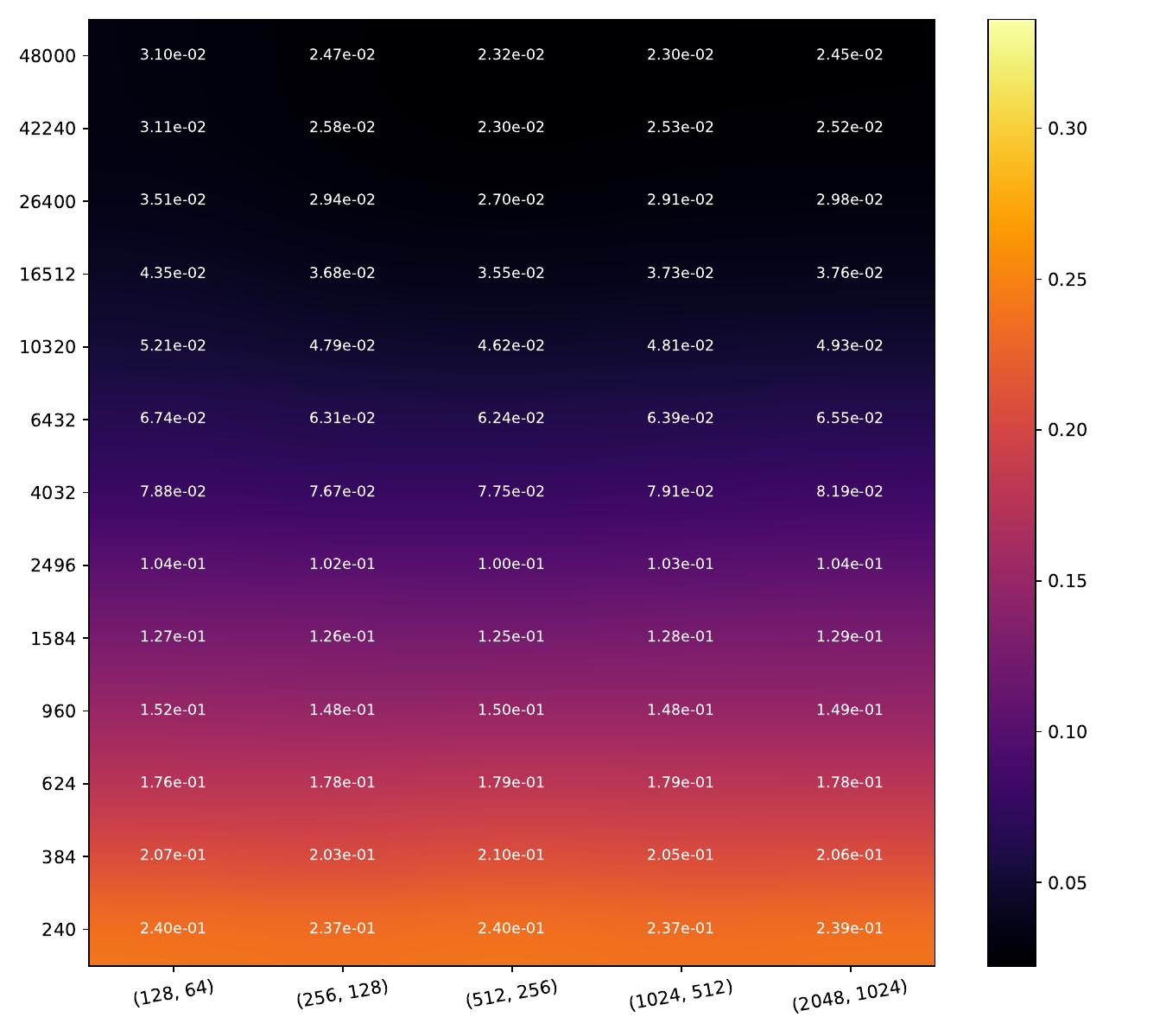}
    \makebox[20pt]{\raisebox{35pt}{\rotatebox[origin=c]{90}{\scriptsize CIFAR10}}}%
    \includegraphics[width=\dimexpr\linewidth-20pt\relax]{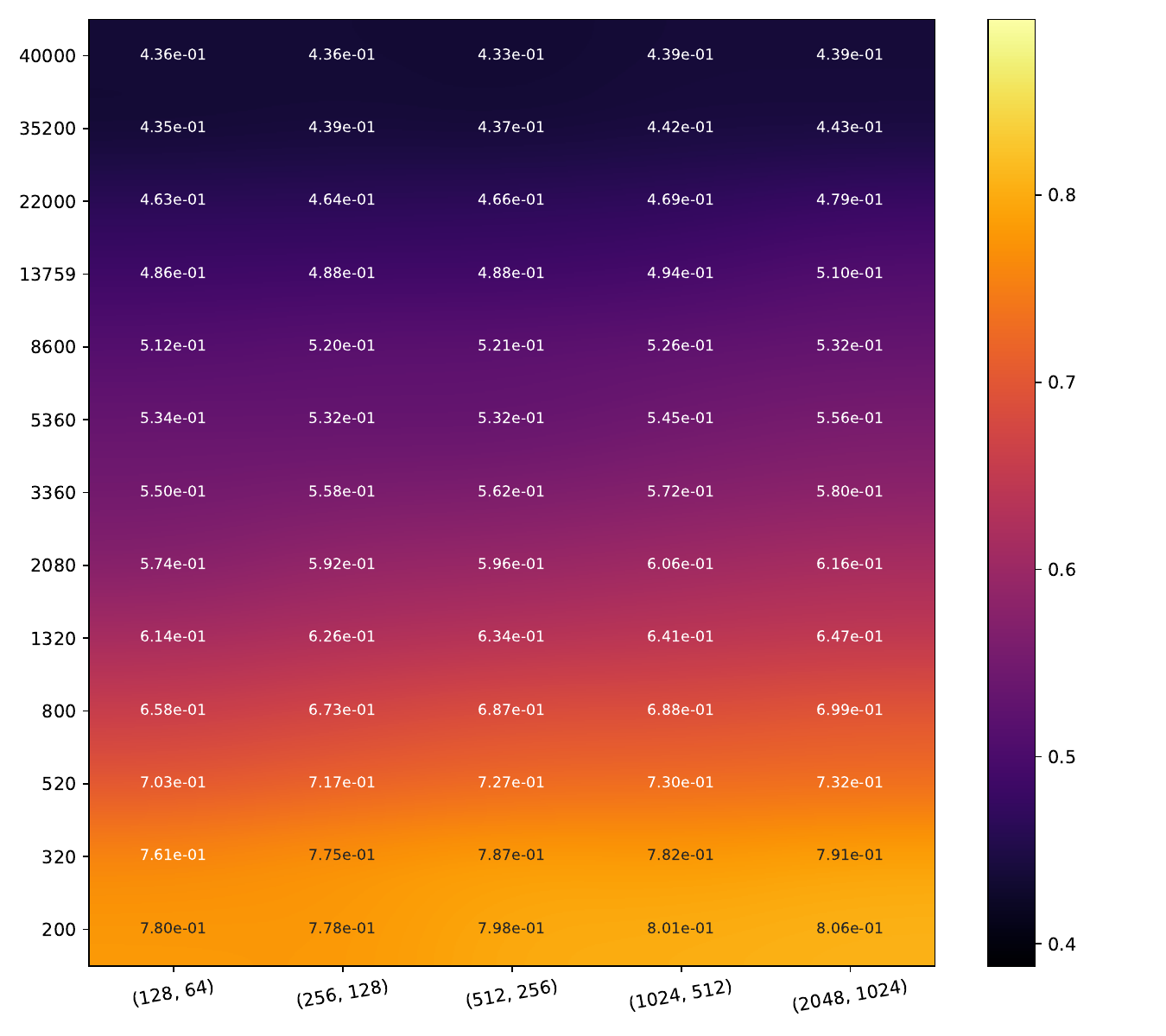}
    \caption*{\qquad MC-Dropout}
  \end{subfigure}\hfill
  \begin{subfigure}[t]{0.23\textwidth}
    \includegraphics[width=\textwidth]{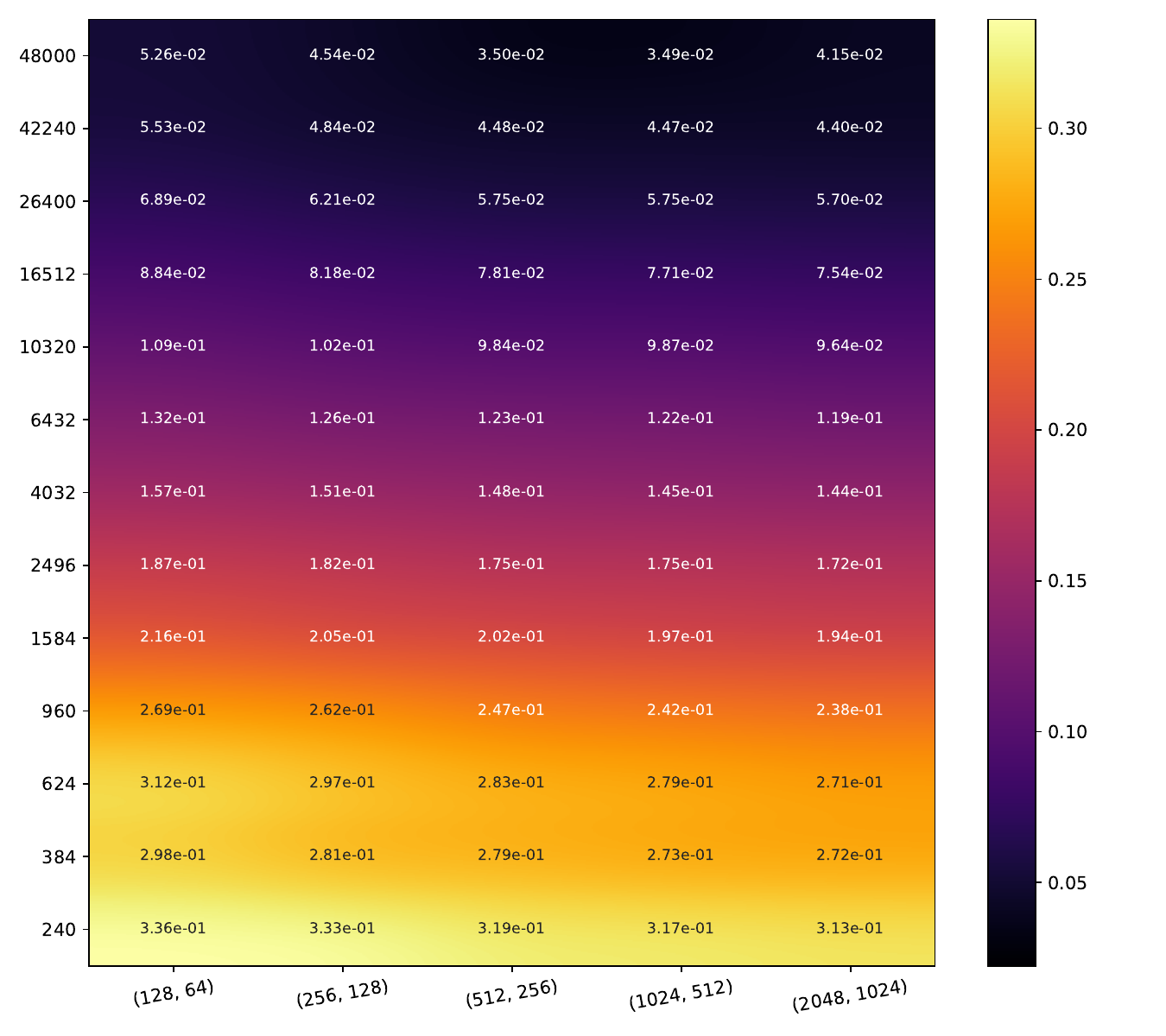}
    \includegraphics[width=\textwidth]{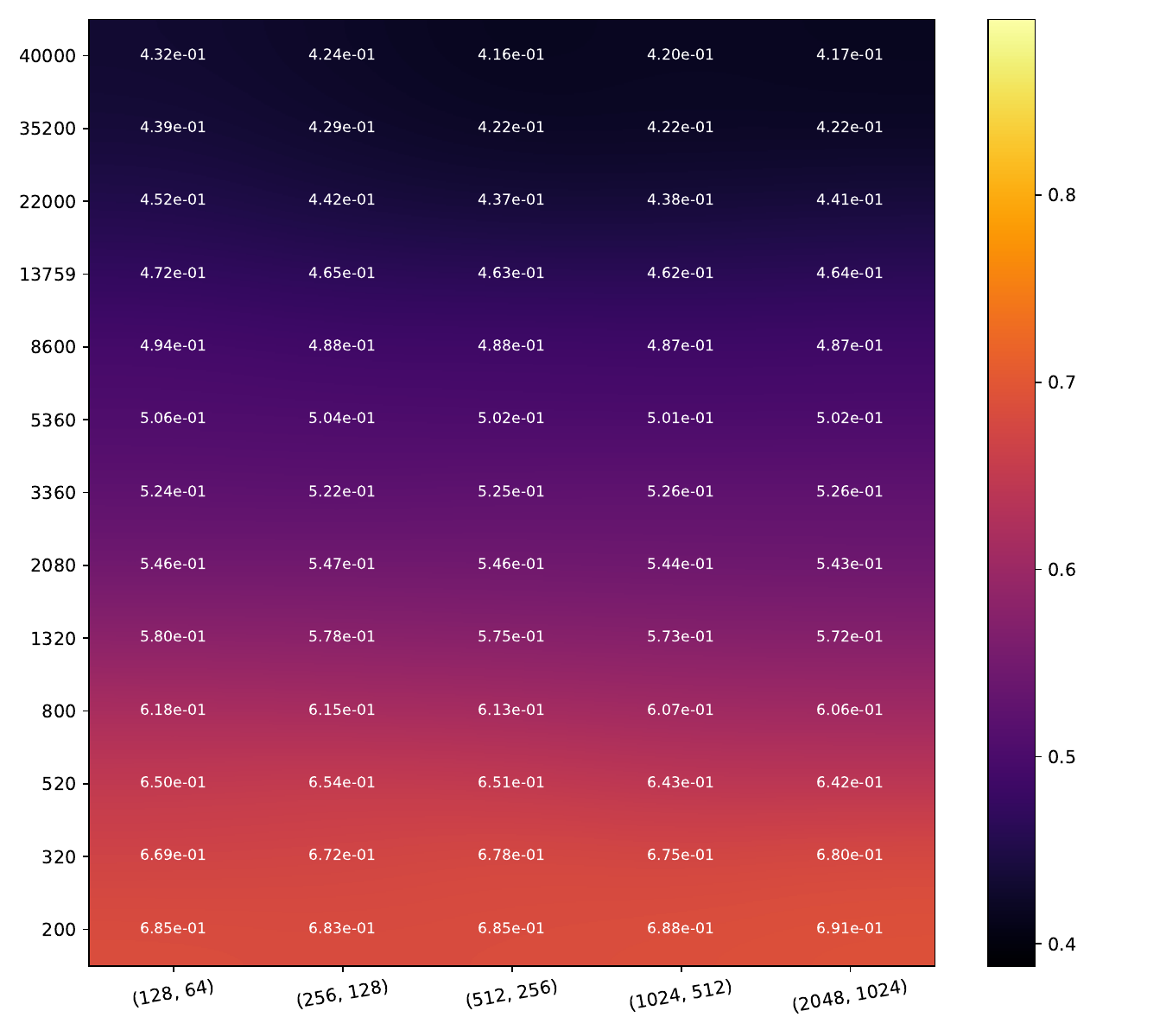}
    \caption*{EDL}
  \end{subfigure}\hfill
  \begin{subfigure}[t]{0.23\textwidth}
    \includegraphics[width=\textwidth]{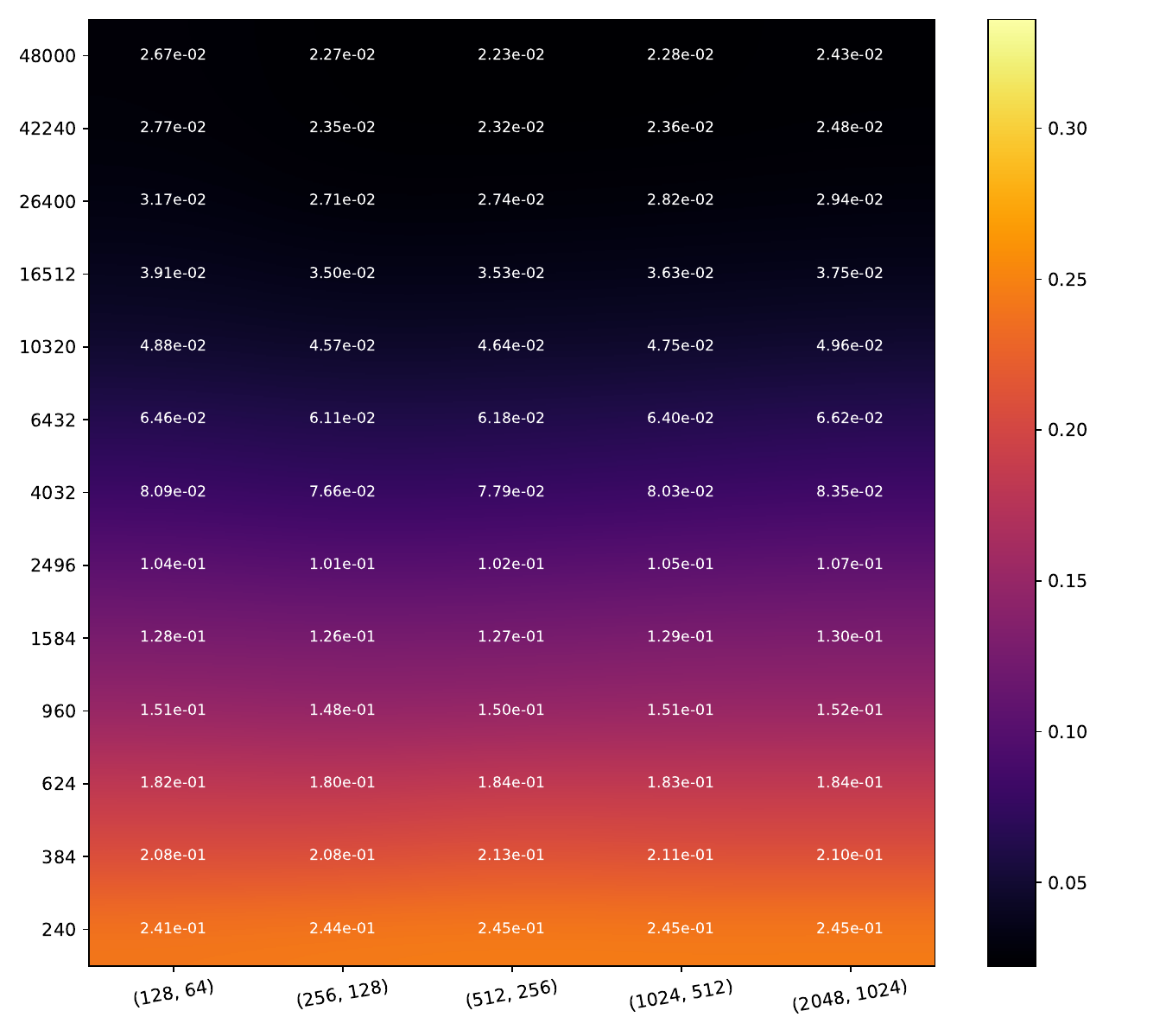}
    \includegraphics[width=\textwidth]{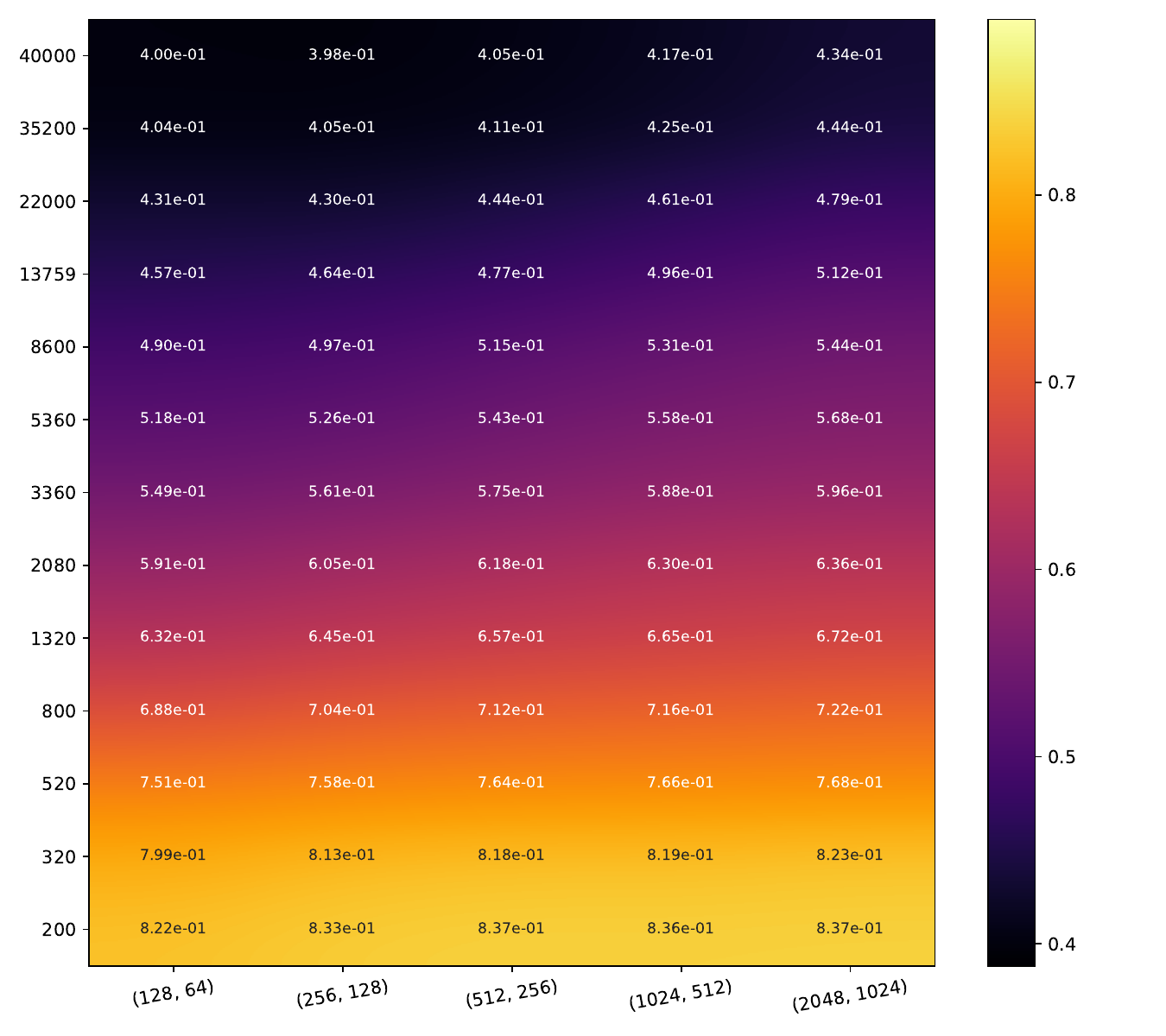}
    \caption*{DE}
  \end{subfigure}\hfill
  \begin{subfigure}[t]{0.23\textwidth}
    \includegraphics[width=\textwidth]{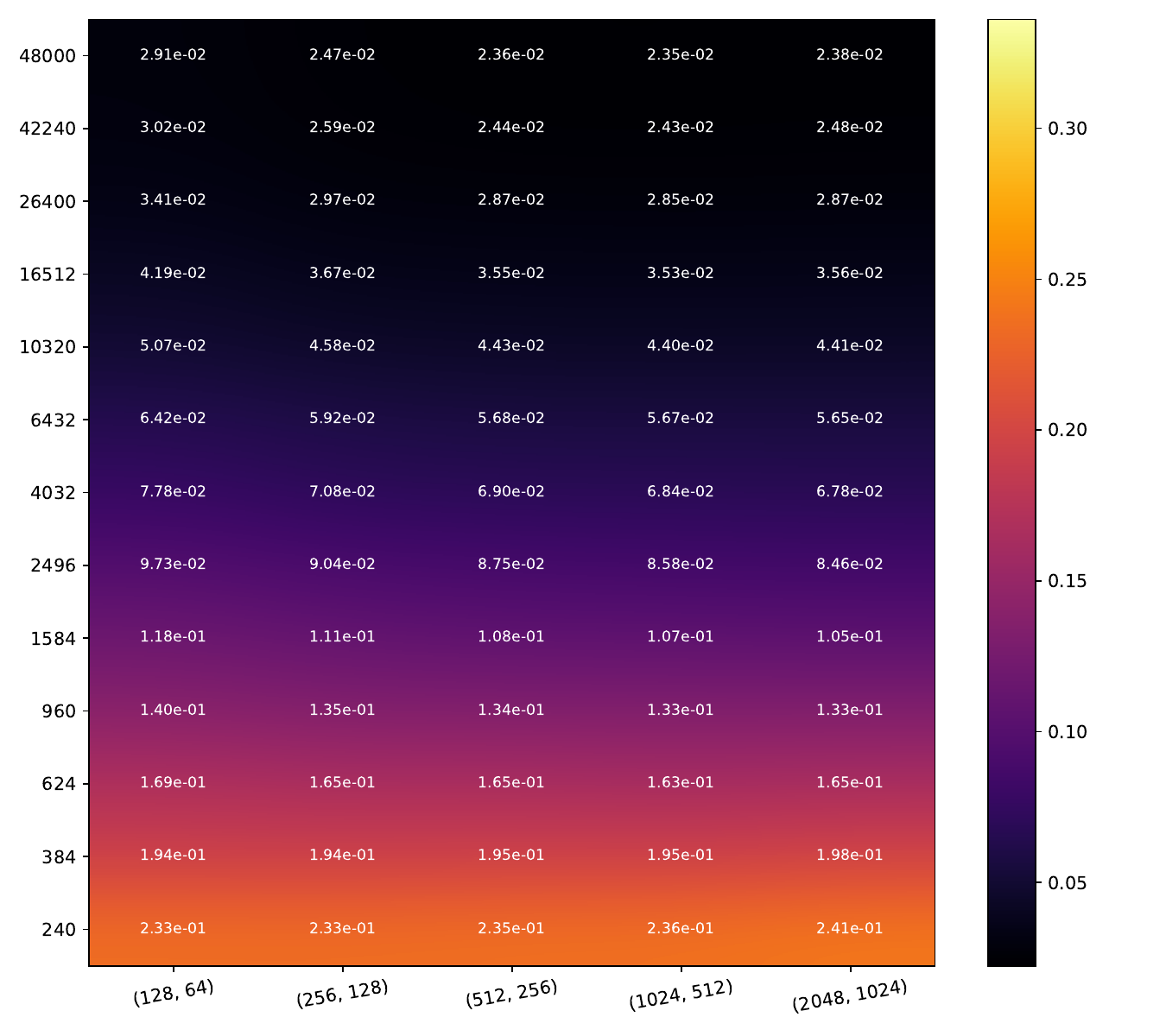}
    \includegraphics[width=\textwidth]{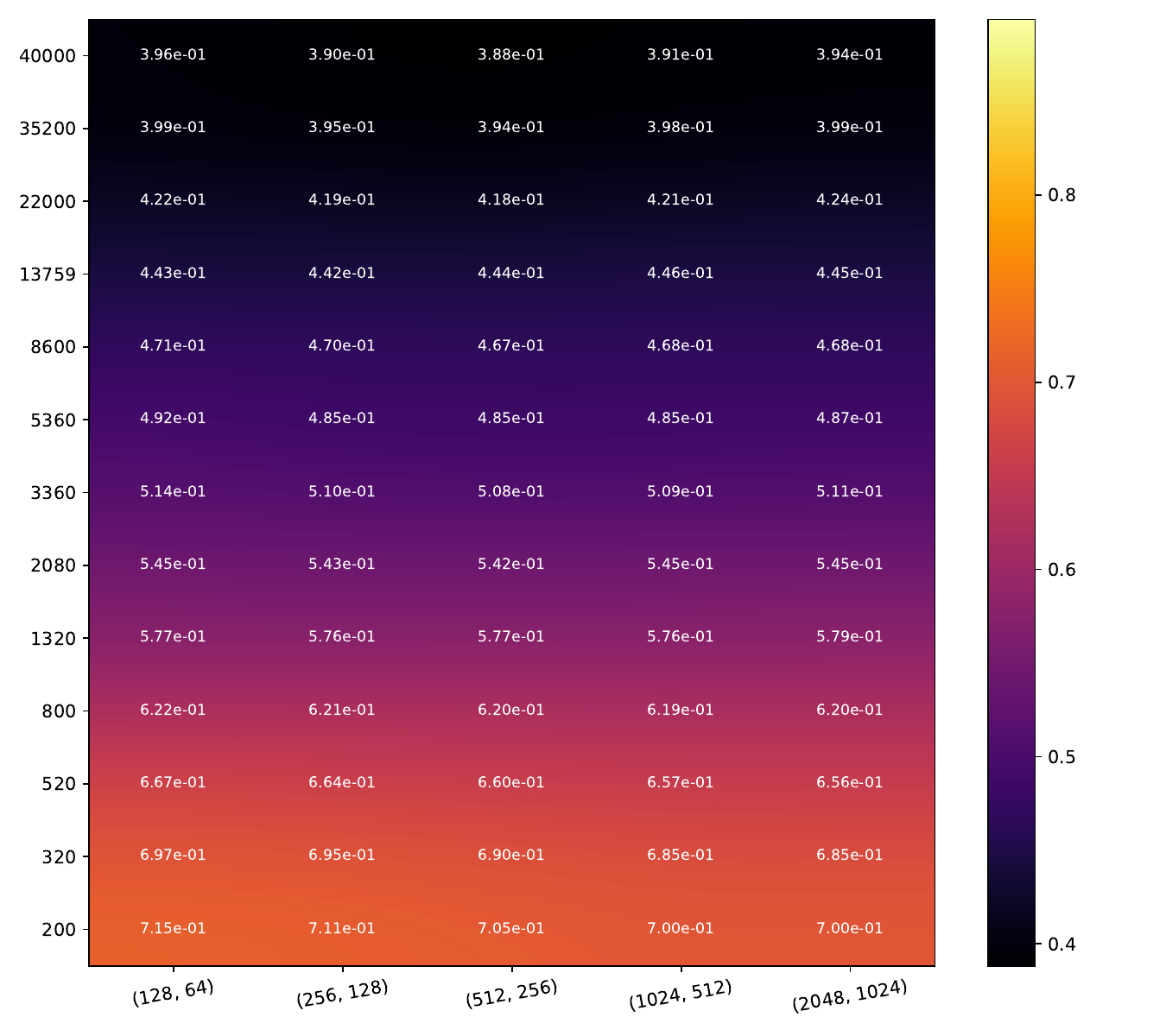}
    \caption*{Conflictual DE}
  \end{subfigure}\hfill
  \caption{Heatmaps of Brier score Color scales are the same per dataset.}
\end{figure}

\begin{figure}[ht]
  \centering
  \begin{subfigure}[t]{\dimexpr0.23\textwidth+20pt\relax}
    \makebox[20pt]{\raisebox{35pt}{\rotatebox[origin=c]{90}{\scriptsize MNIST}}}%
    \includegraphics[width=\dimexpr\linewidth-20pt\relax]{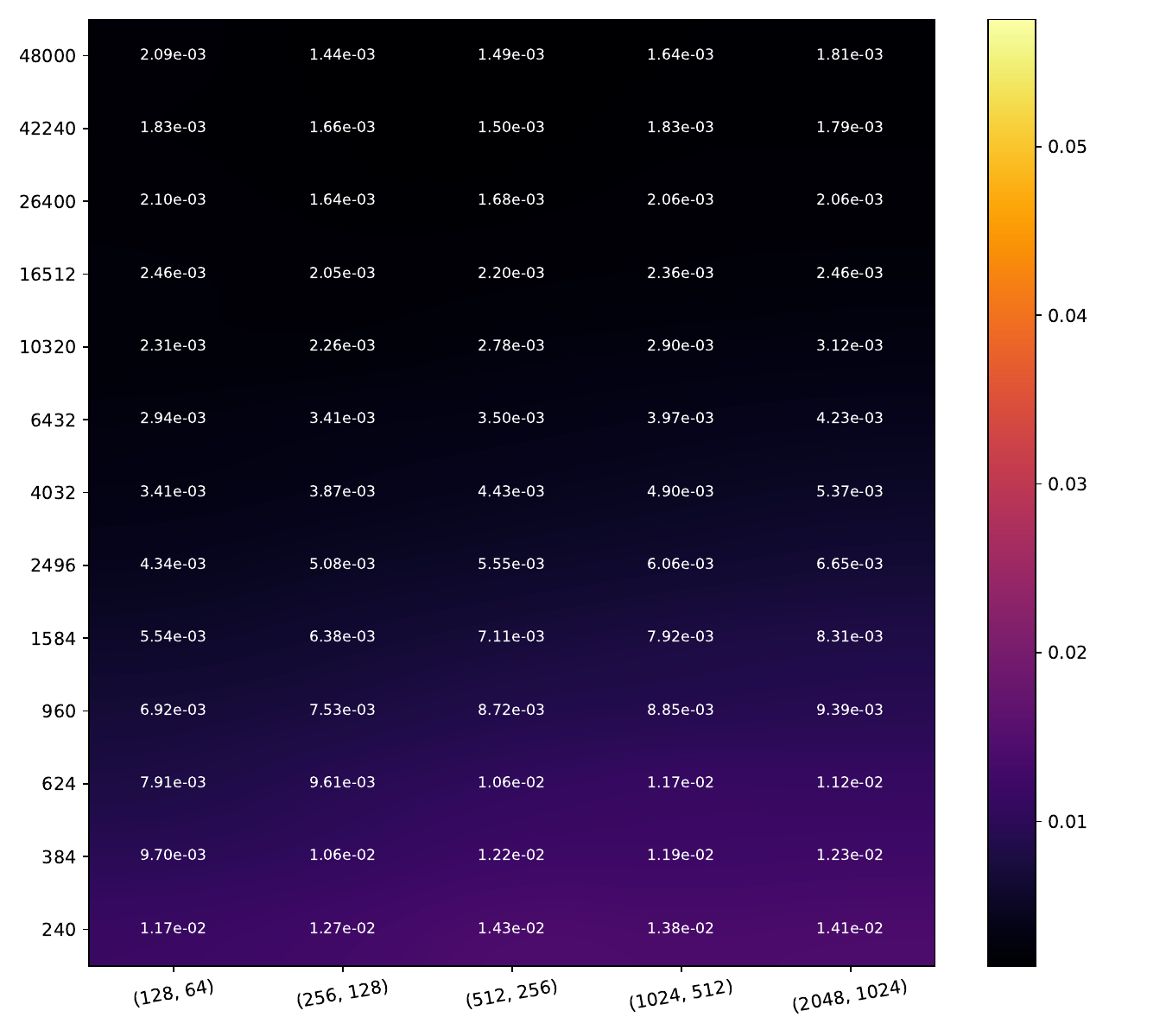}
    \makebox[20pt]{\raisebox{35pt}{\rotatebox[origin=c]{90}{\scriptsize CIFAR10}}}%
    \includegraphics[width=\dimexpr\linewidth-20pt\relax]{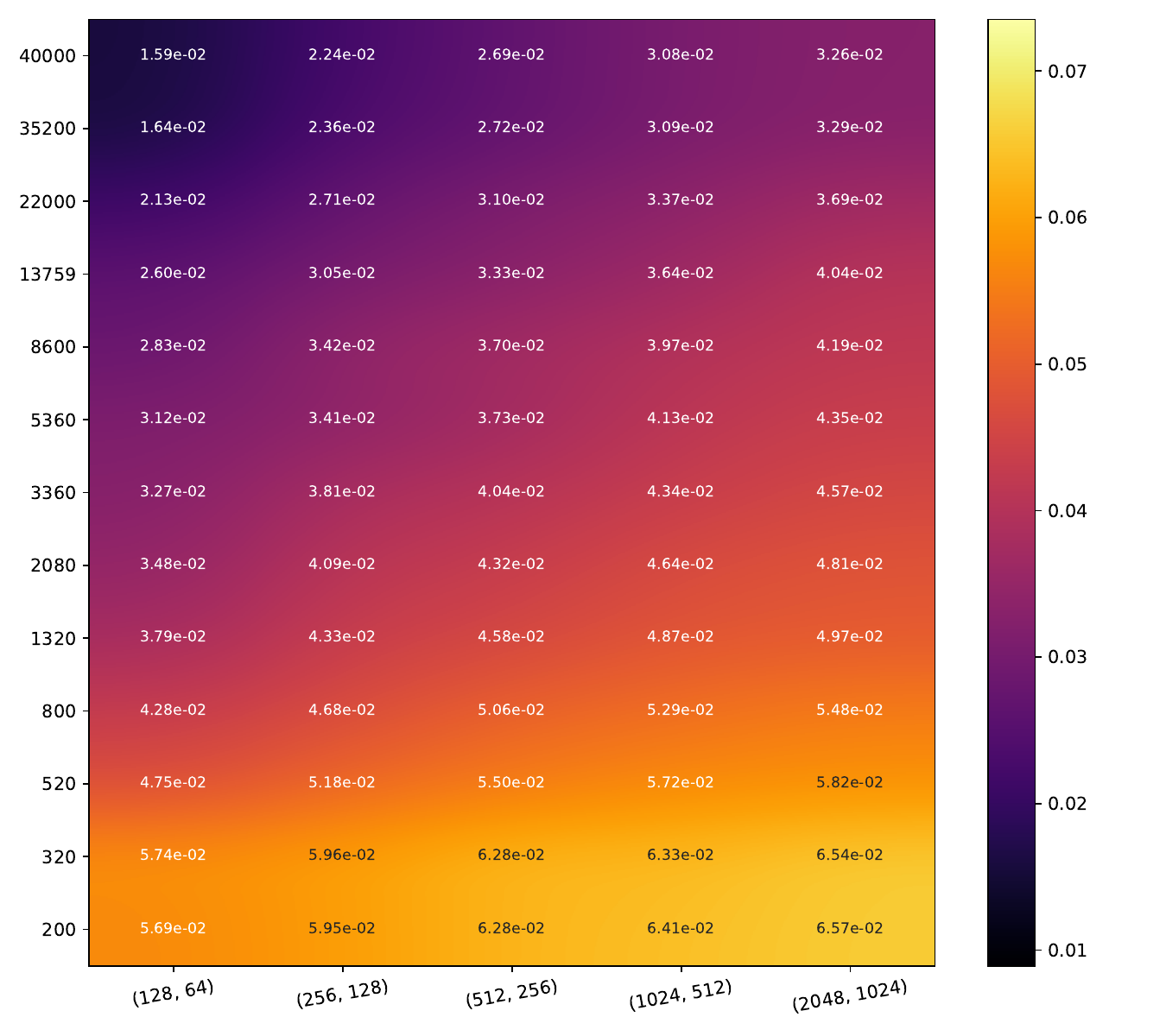}
    \caption*{\qquad MC-Dropout}
  \end{subfigure}\hfill
  \begin{subfigure}[t]{0.23\textwidth}
    \includegraphics[width=\textwidth]{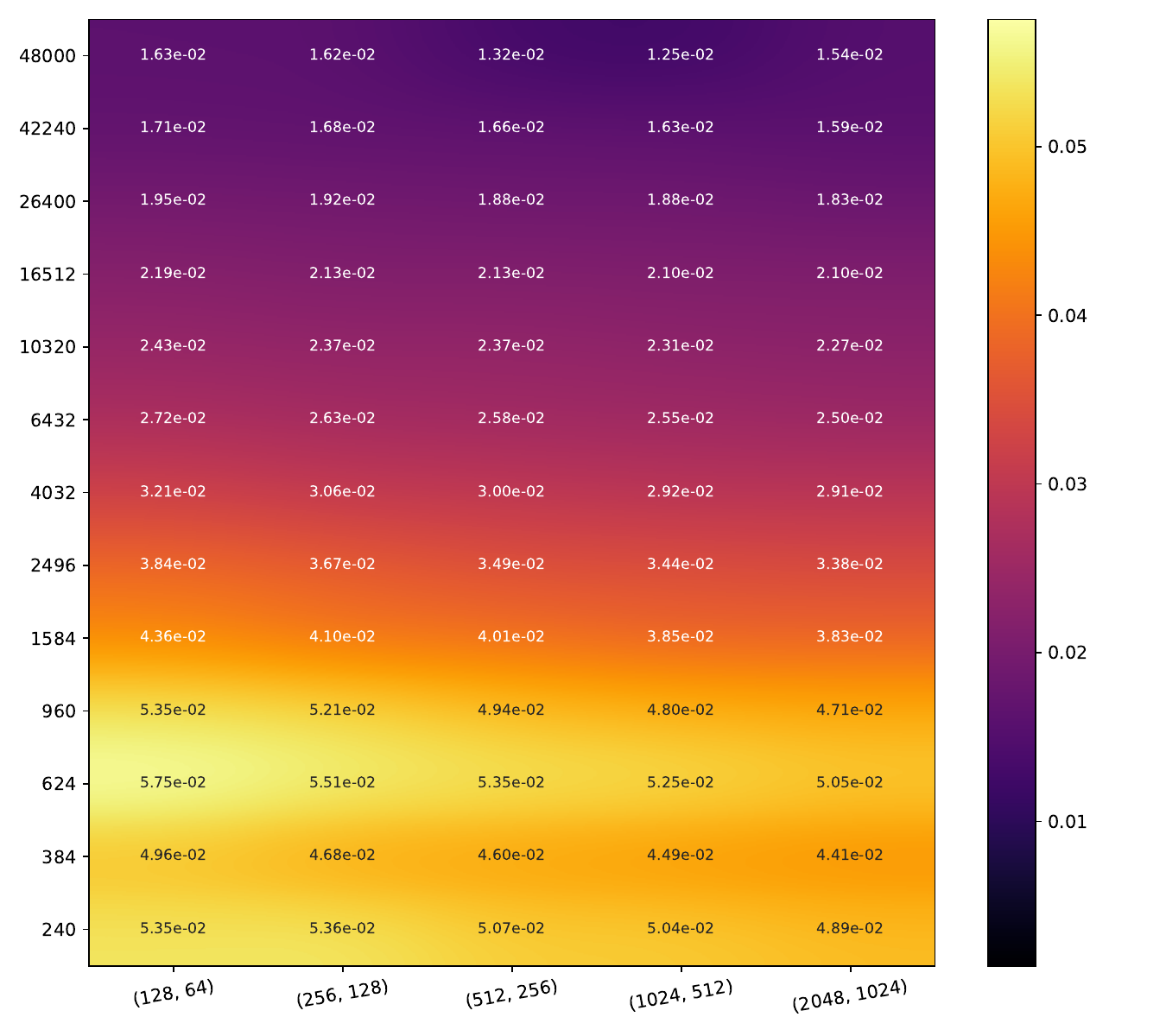}
    \includegraphics[width=\textwidth]{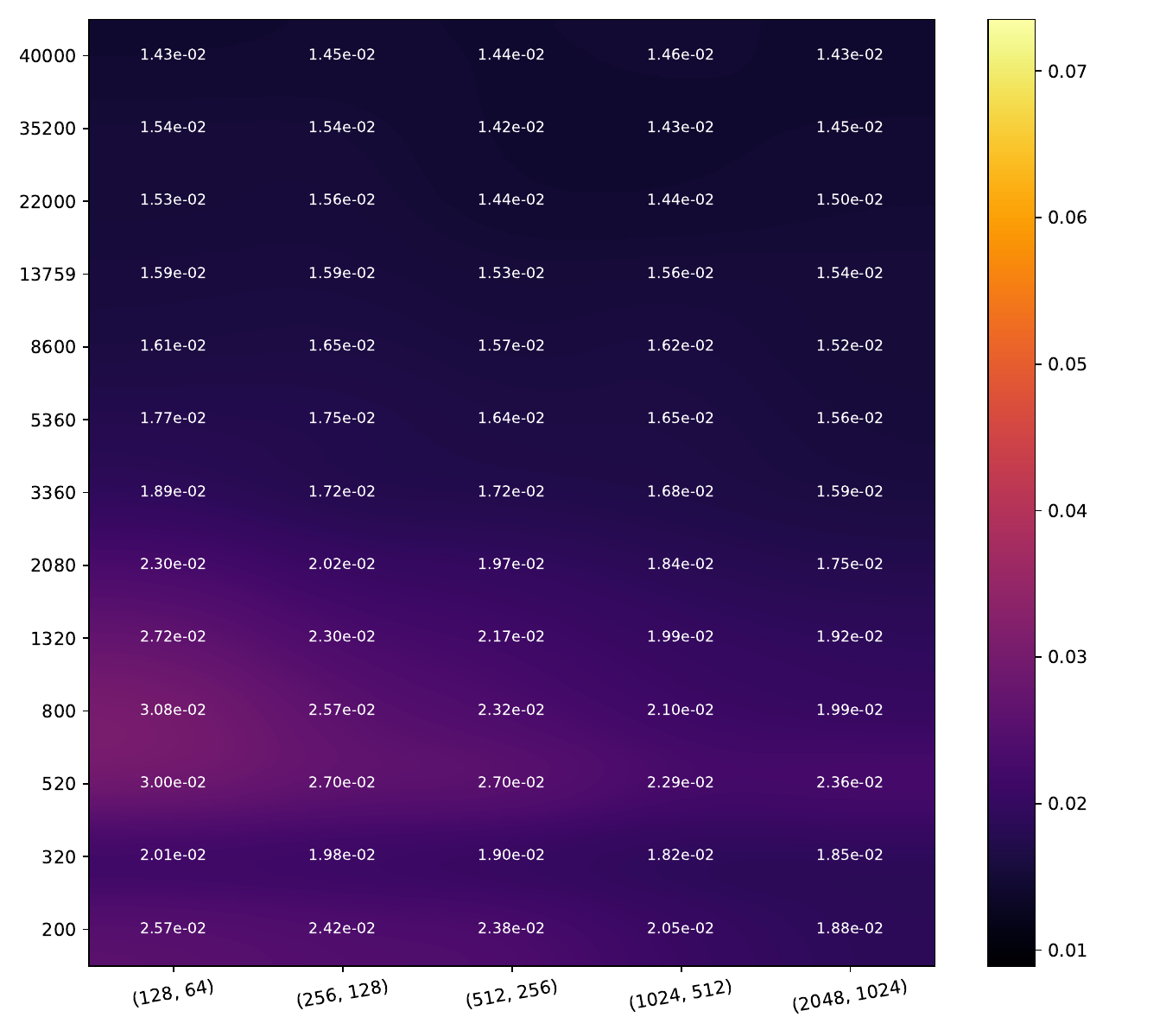}
    \caption*{EDL}
  \end{subfigure}\hfill
  \begin{subfigure}[t]{0.23\textwidth}
    \includegraphics[width=\textwidth]{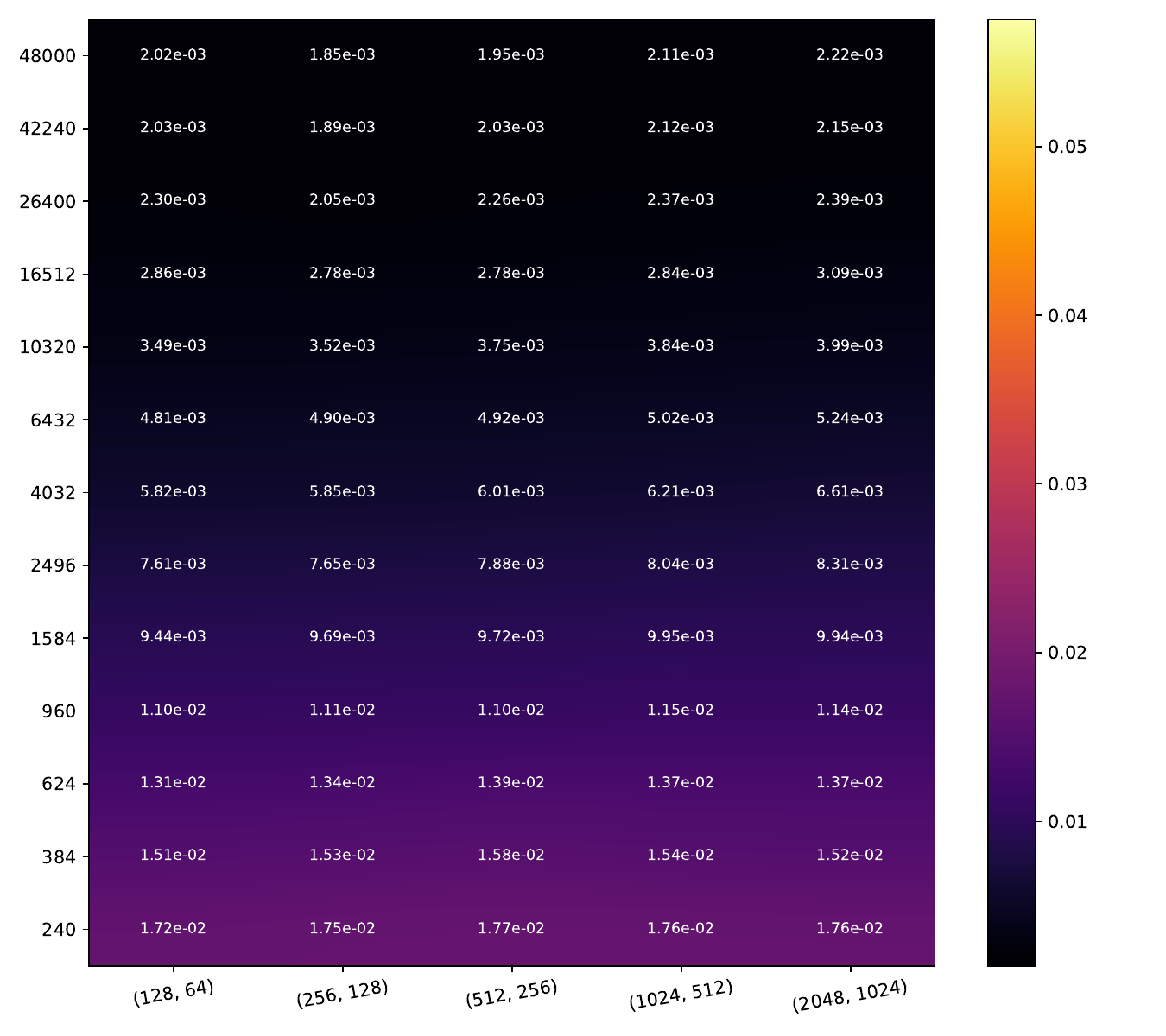}
    \includegraphics[width=\textwidth]{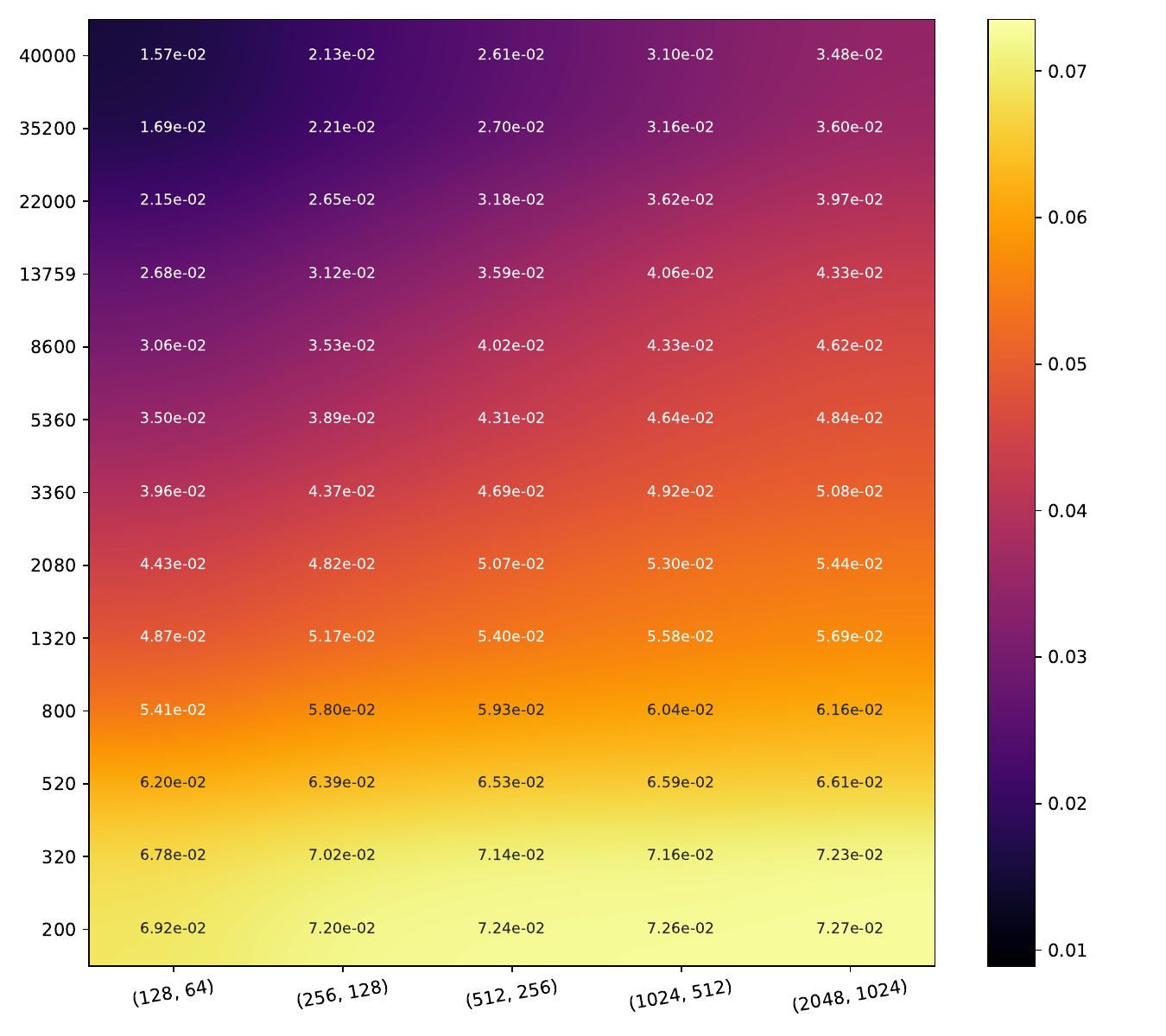}
    \caption*{DE}
  \end{subfigure}\hfill
  \begin{subfigure}[t]{0.23\textwidth}
    \includegraphics[width=\textwidth]{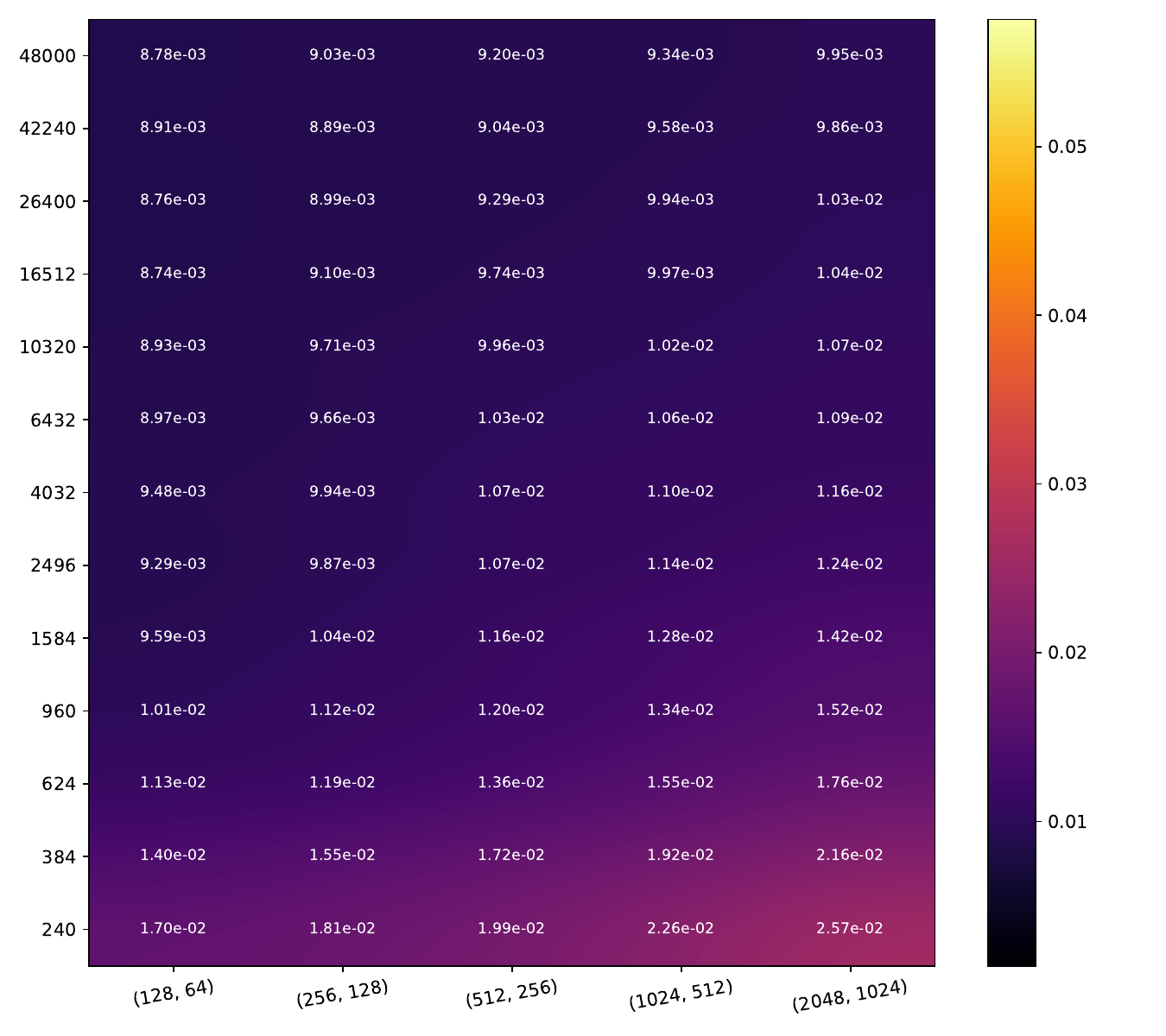}
    \includegraphics[width=\textwidth]{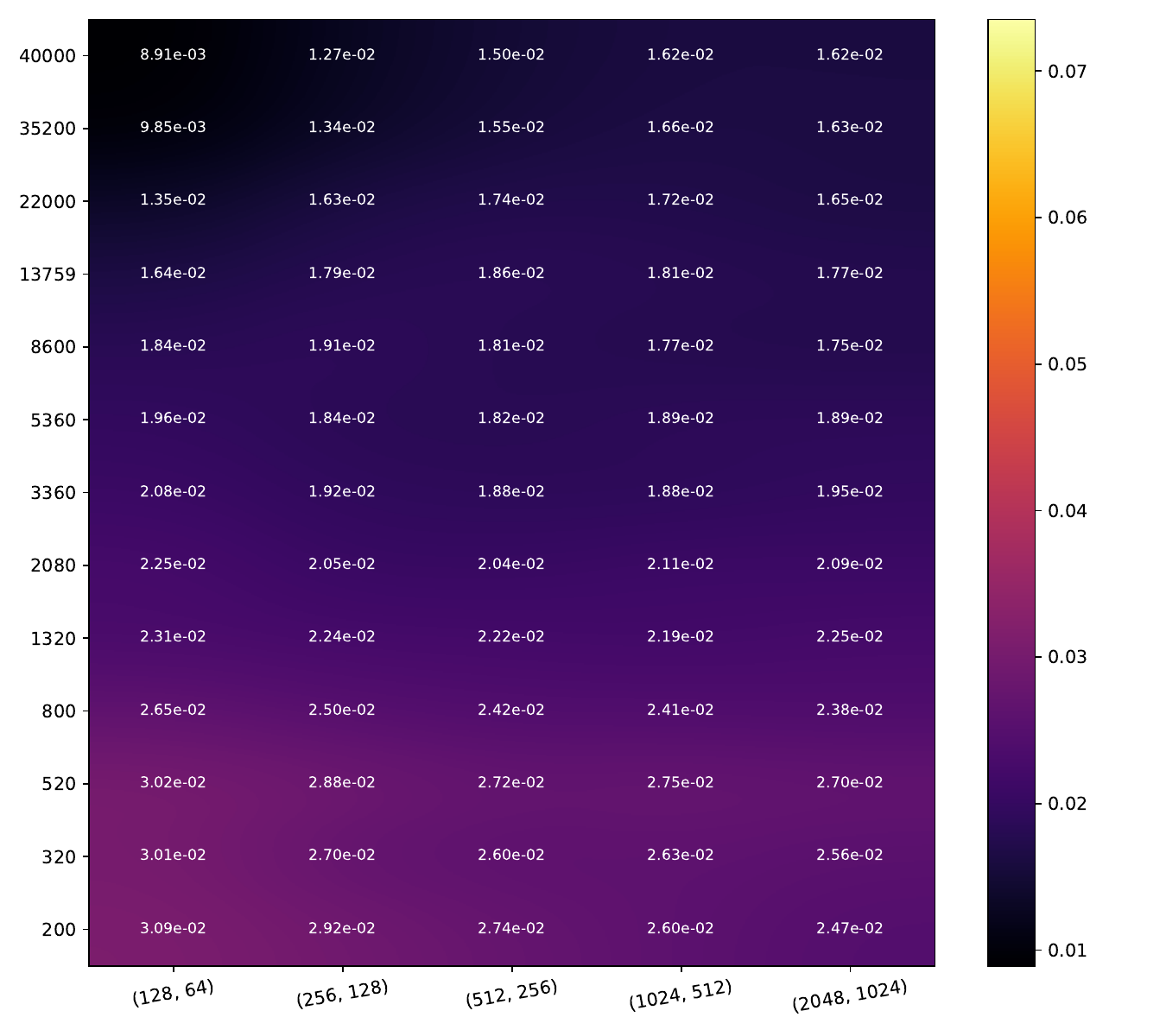}
    \caption*{Conflictual DE}
  \end{subfigure}\hfill
  \caption{Heatmaps of SCE. Color scales are the same per dataset.}
\end{figure}

\begin{figure}[ht]
  \centering
  \begin{subfigure}[t]{\dimexpr0.23\textwidth+20pt\relax}
    \makebox[20pt]{\raisebox{35pt}{\rotatebox[origin=c]{90}{\scriptsize MNIST}}}%
    \includegraphics[width=\dimexpr\linewidth-20pt\relax]{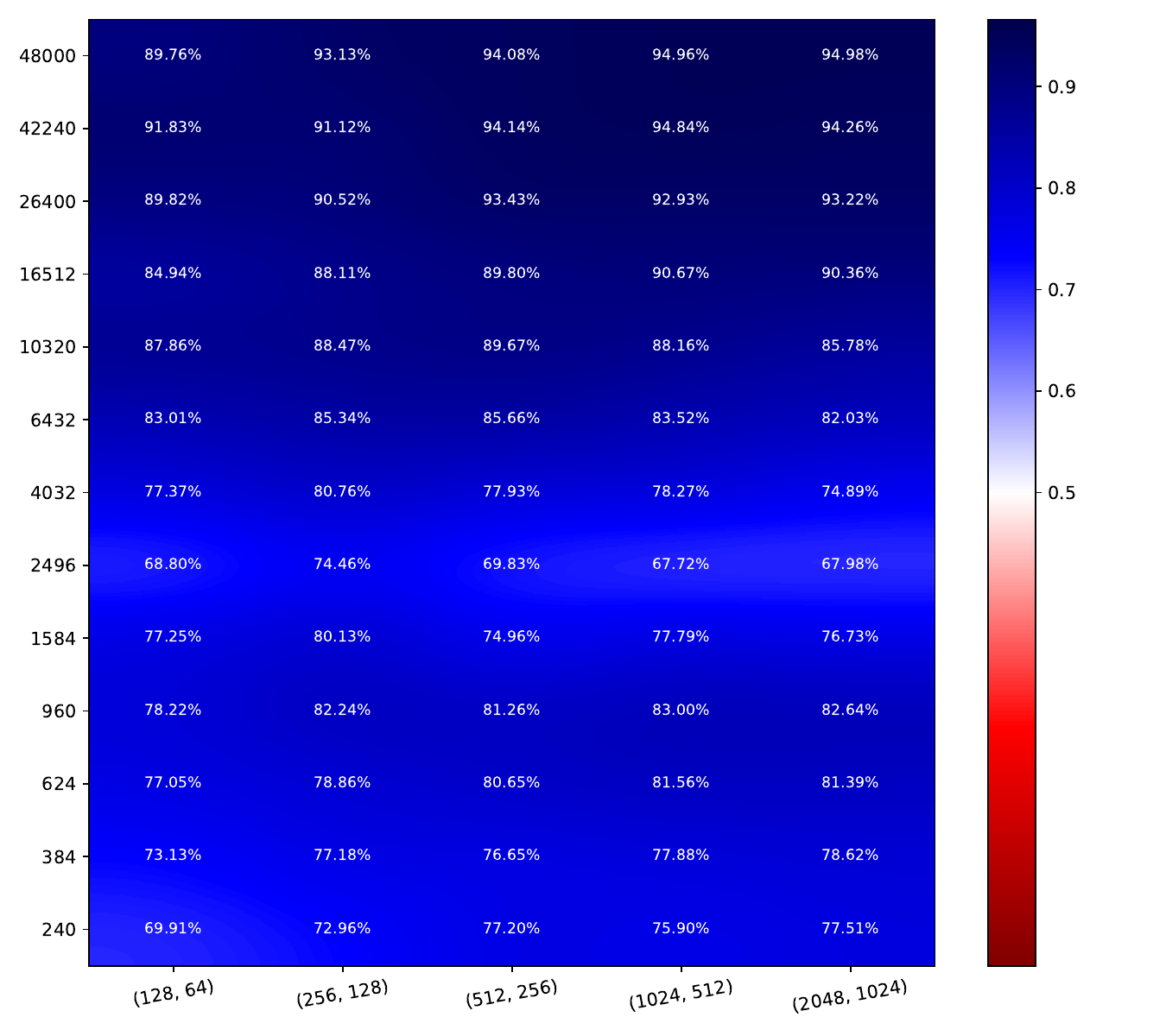}
    \makebox[20pt]{\raisebox{35pt}{\rotatebox[origin=c]{90}{\scriptsize CIFAR10}}}%
    \includegraphics[width=\dimexpr\linewidth-20pt\relax]{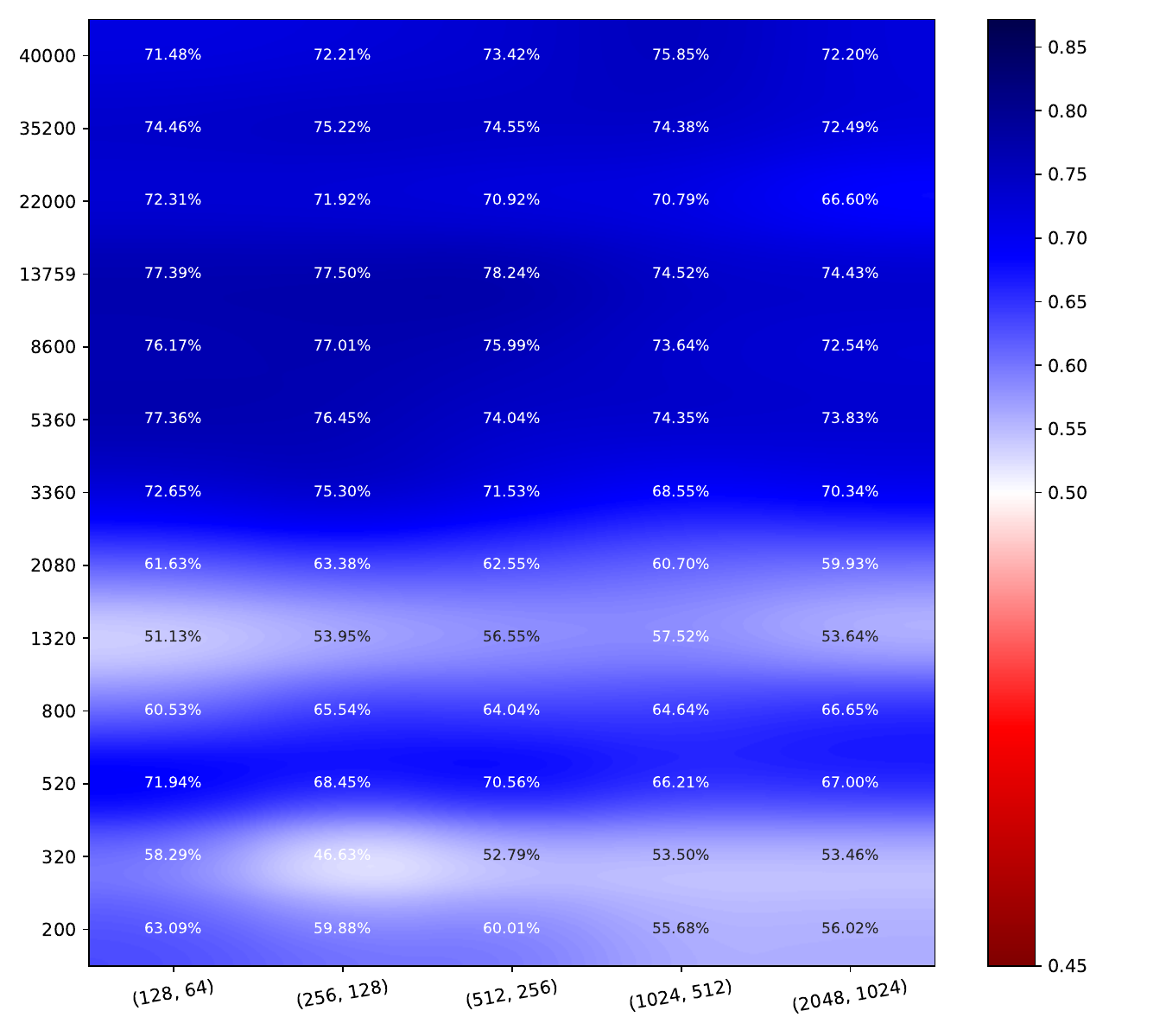}
    \caption*{\qquad MC-Dropout}
  \end{subfigure}\hfill
  \begin{subfigure}[t]{0.23\textwidth}
    \includegraphics[width=\textwidth]{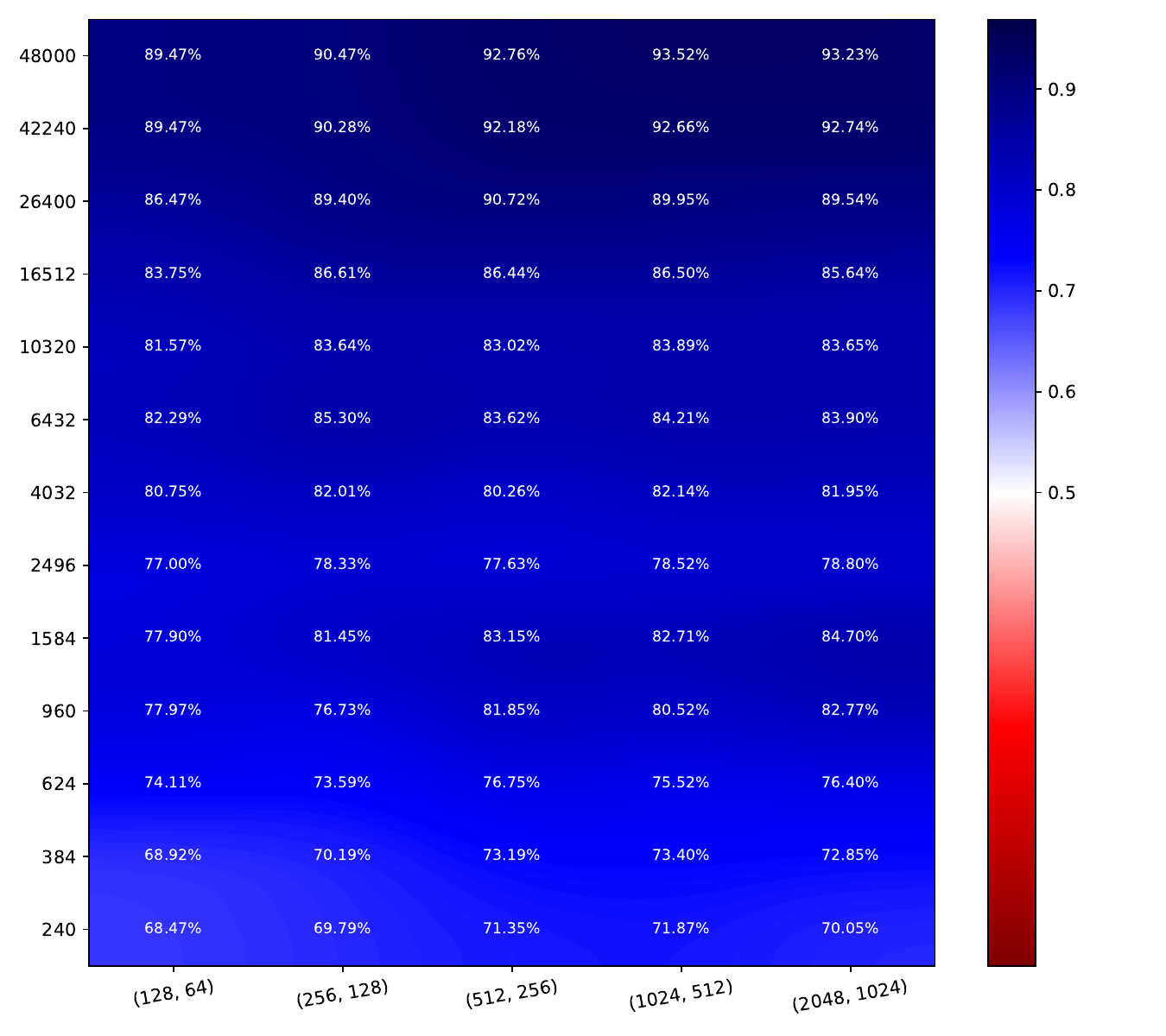}
    \includegraphics[width=\textwidth]{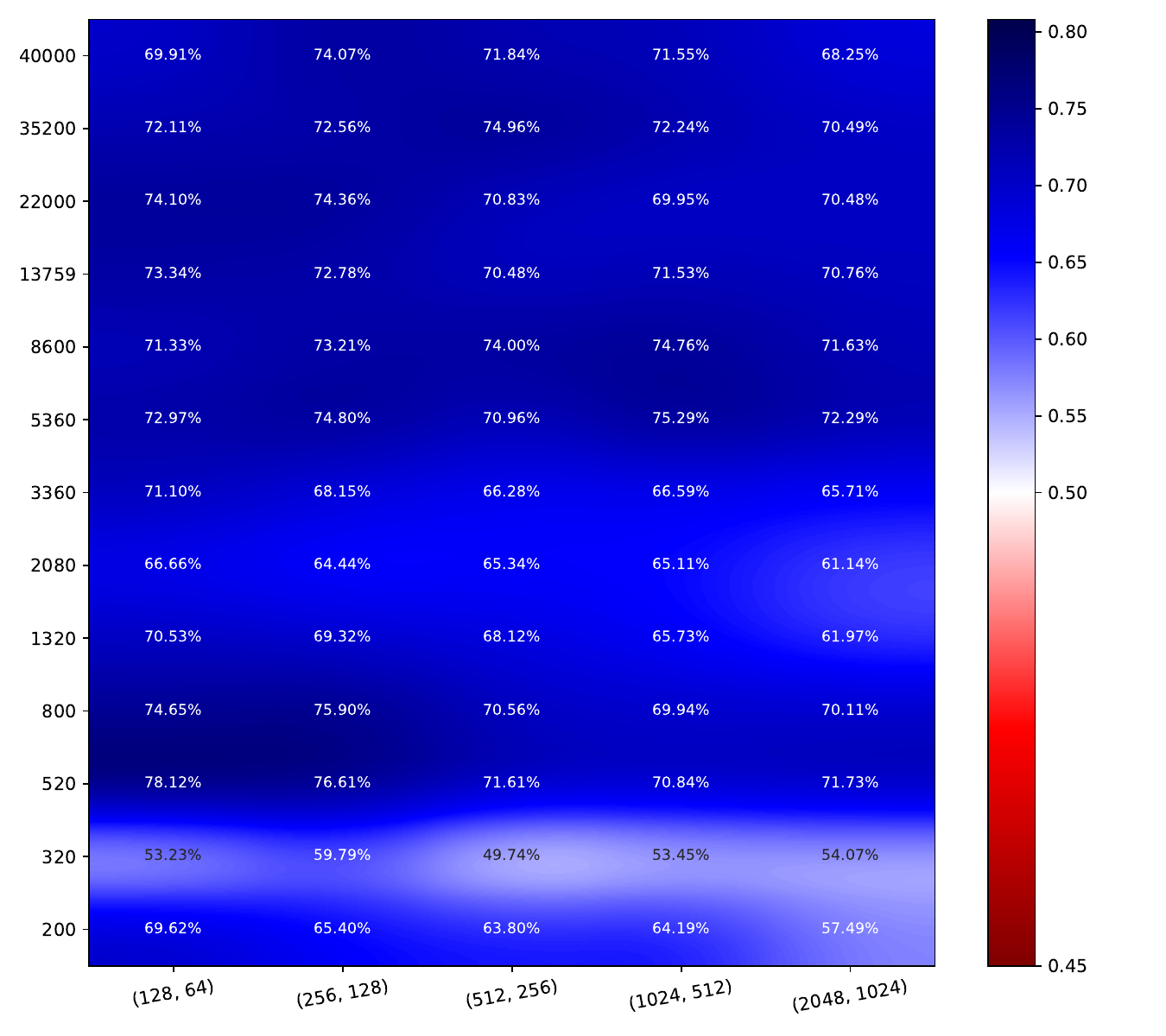}
    \caption*{EDL}
  \end{subfigure}\hfill
  \begin{subfigure}[t]{0.23\textwidth}
    \includegraphics[width=\textwidth]{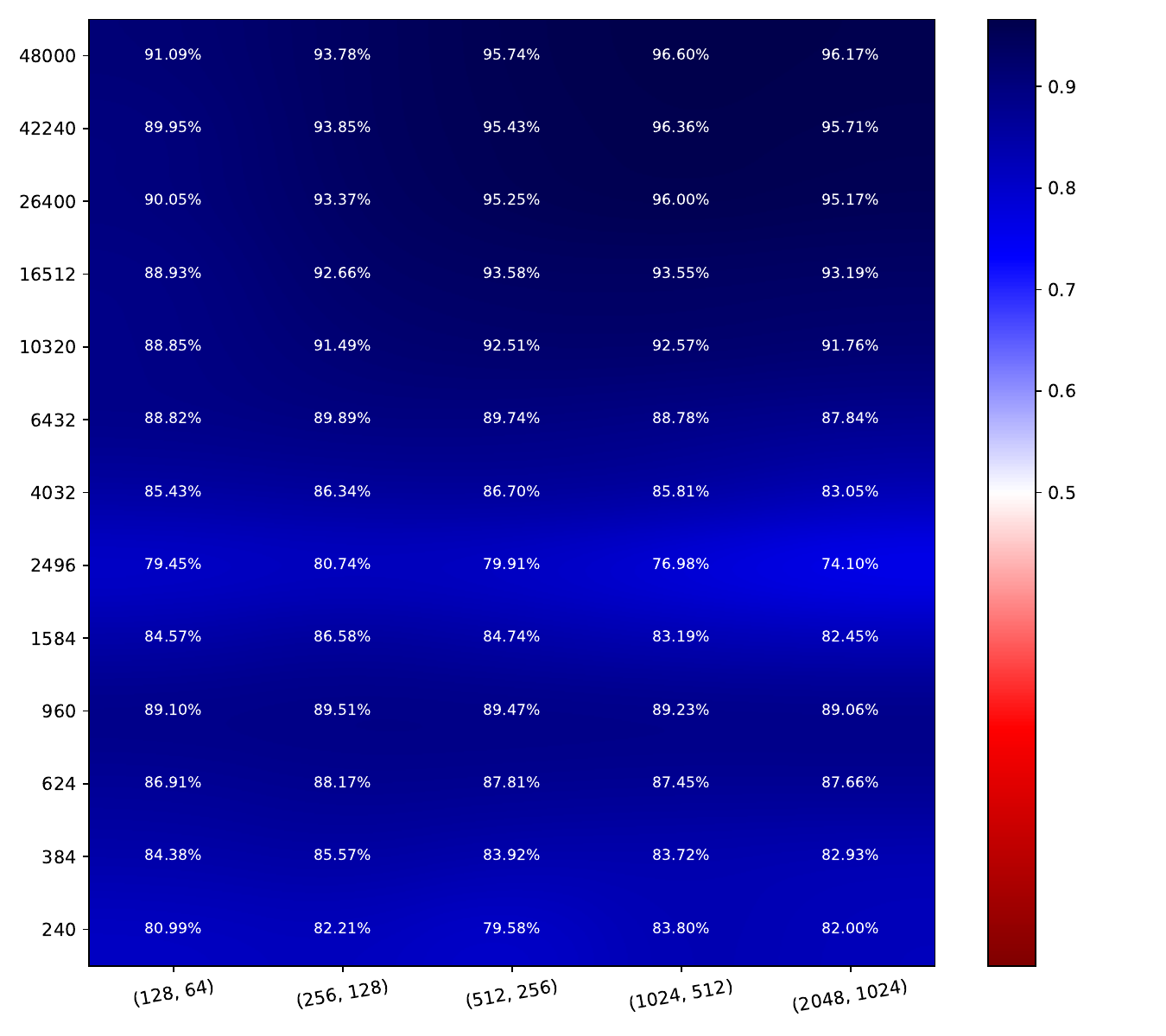}
    \includegraphics[width=\textwidth]{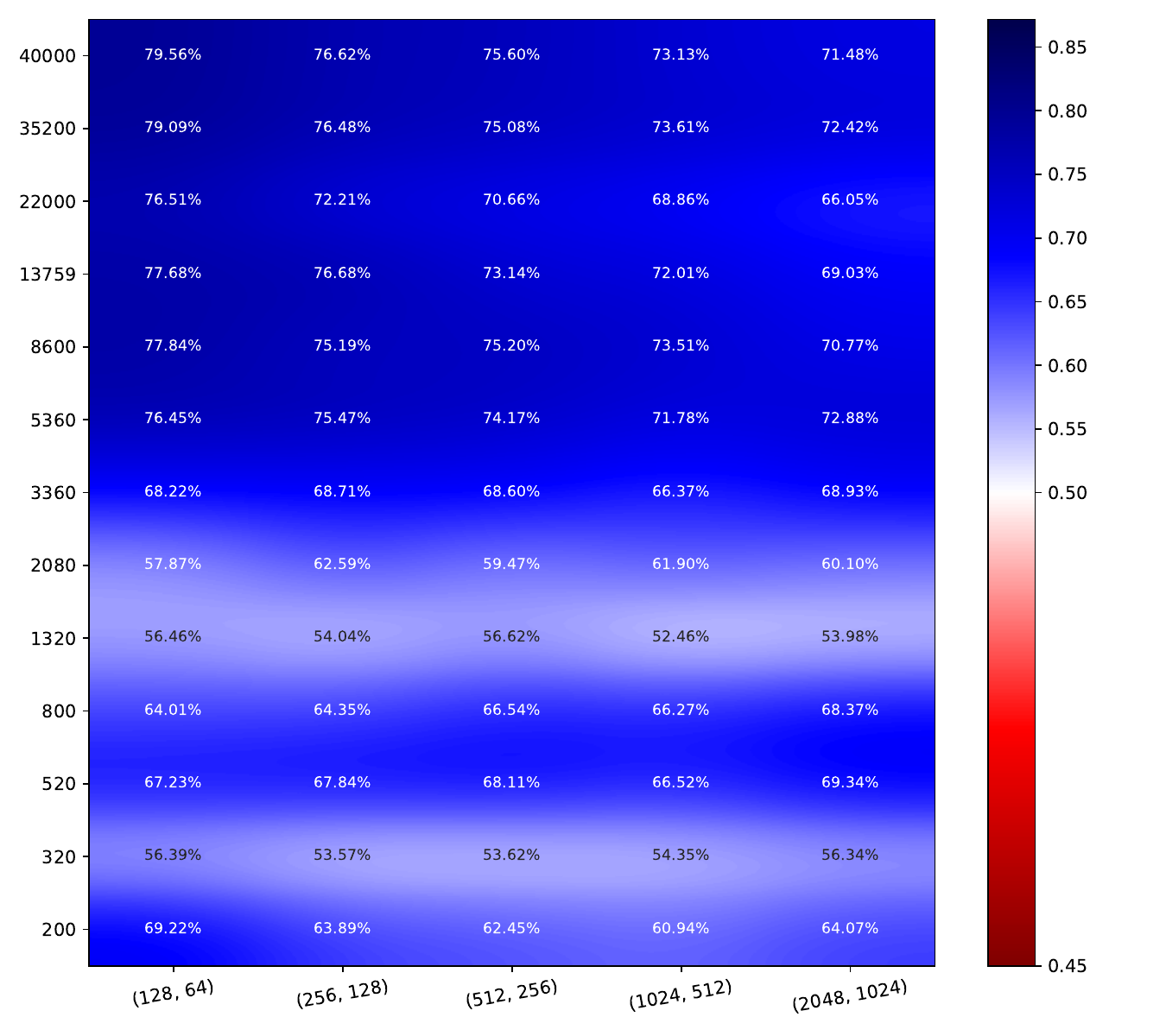}
    \caption*{DE}
  \end{subfigure}\hfill
  \begin{subfigure}[t]{0.23\textwidth}
    \includegraphics[width=\textwidth]{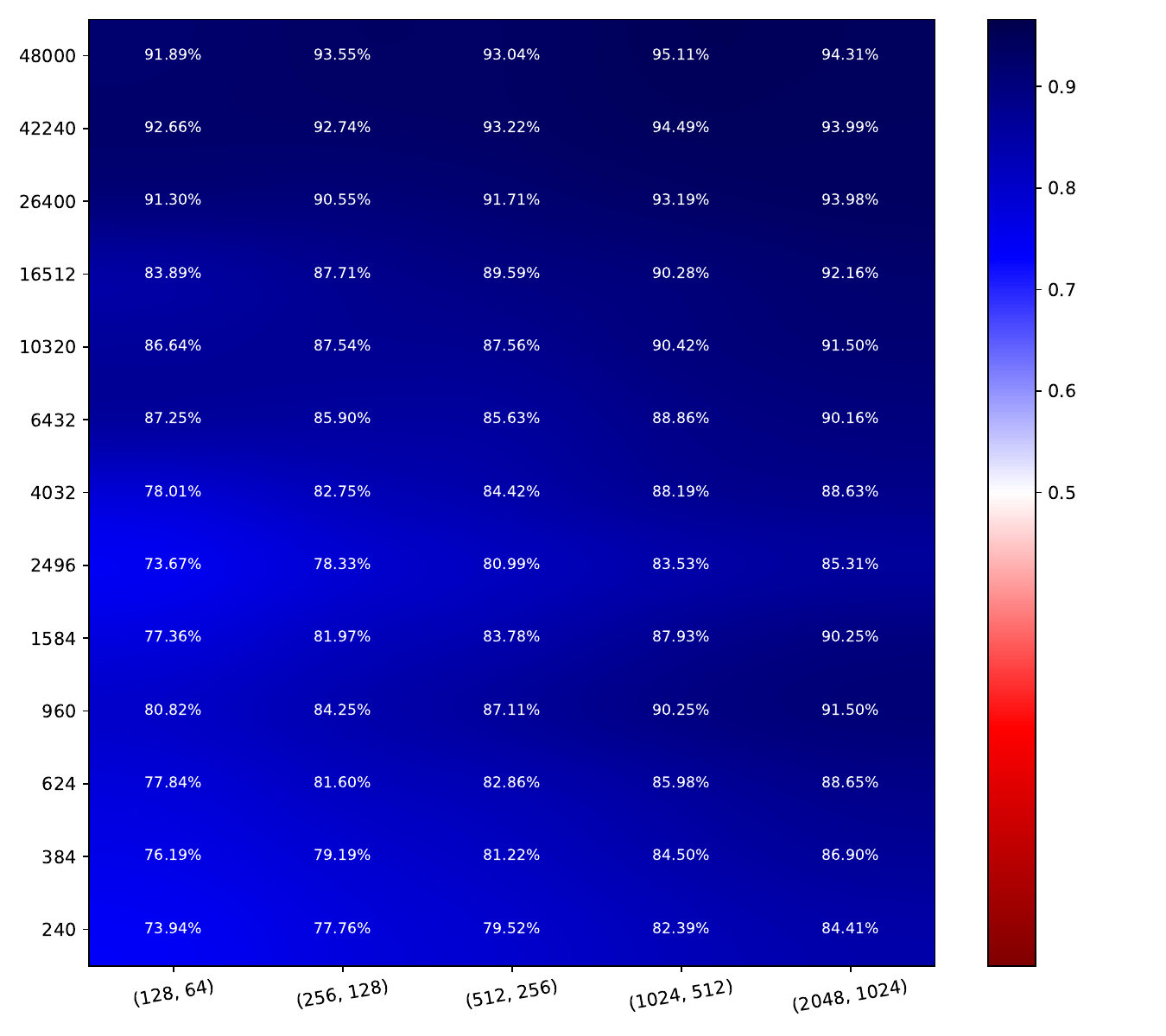}
    \includegraphics[width=\textwidth]{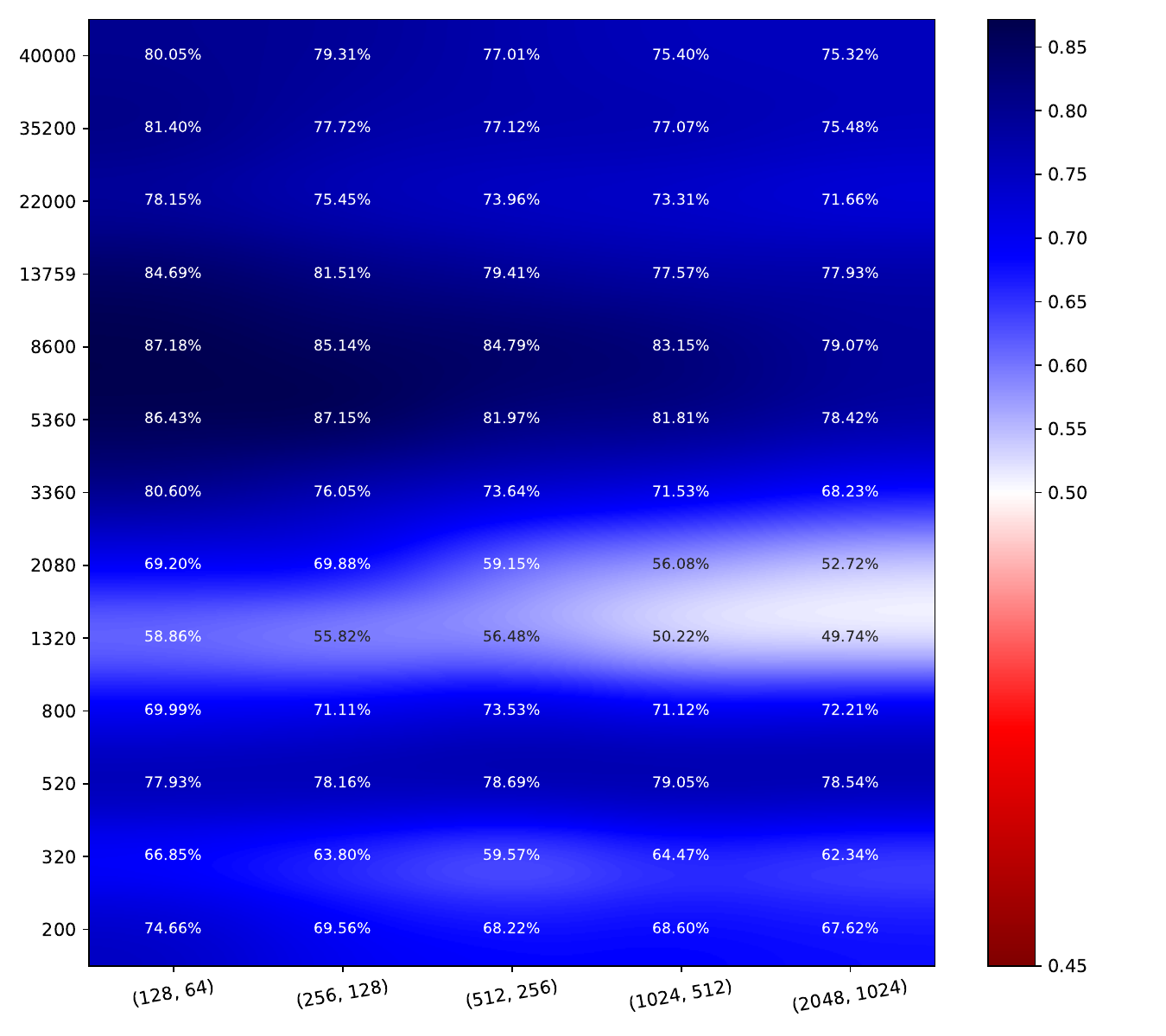}
    \caption*{Conflictual DE}
  \end{subfigure}\hfill
  \caption{Heatmaps of AUROC based on epistemic uncertainty for OOD detection. Color scales are the same per dataset.}
\end{figure}

\begin{figure}[ht]
  \centering
  \begin{subfigure}[t]{\dimexpr0.23\textwidth+20pt\relax}
    \makebox[20pt]{\raisebox{35pt}{\rotatebox[origin=c]{90}{\scriptsize MNIST}}}%
    \includegraphics[width=\dimexpr\linewidth-20pt\relax]{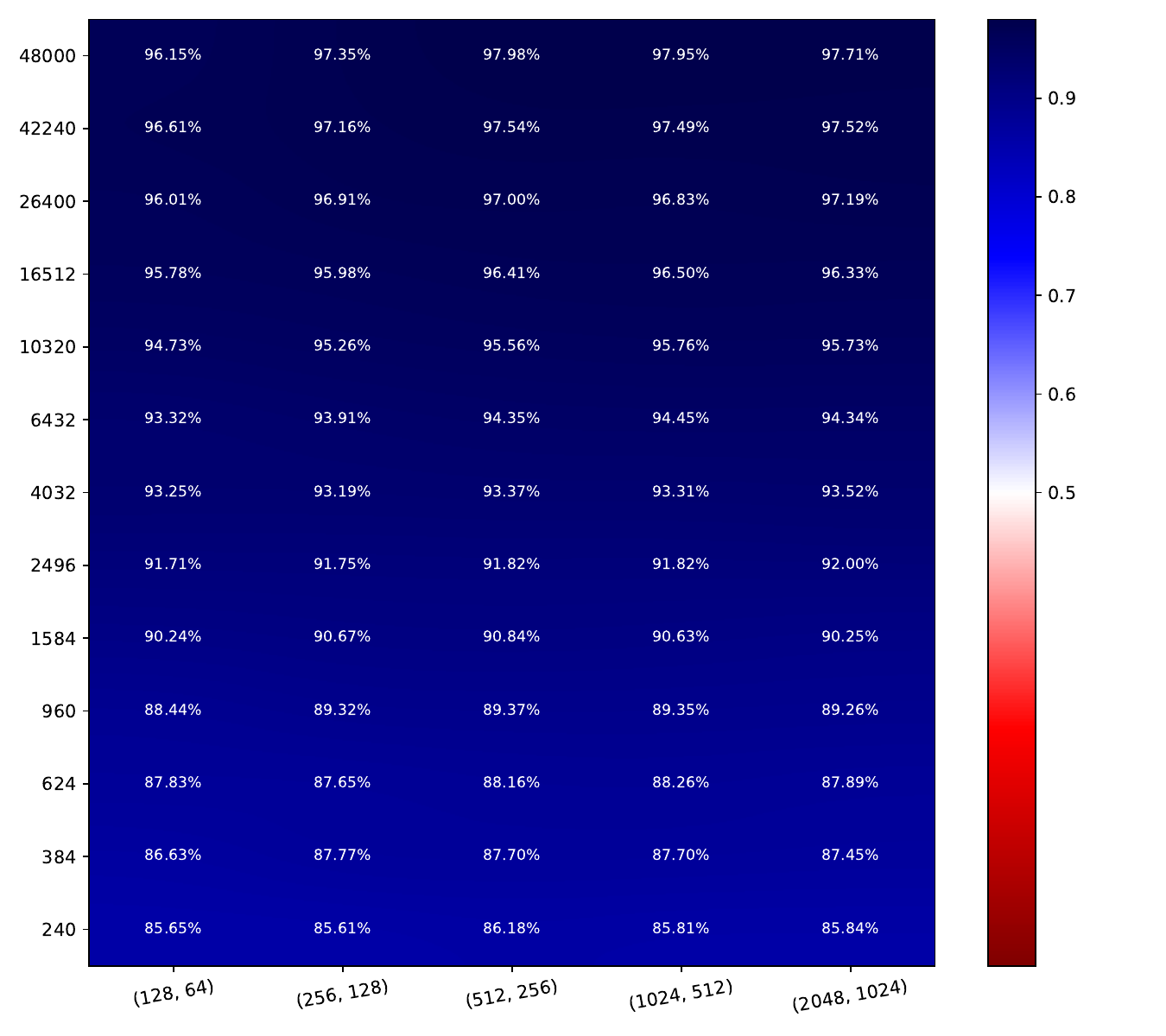}
    \makebox[20pt]{\raisebox{35pt}{\rotatebox[origin=c]{90}{\scriptsize CIFAR10}}}%
    \includegraphics[width=\dimexpr\linewidth-20pt\relax]{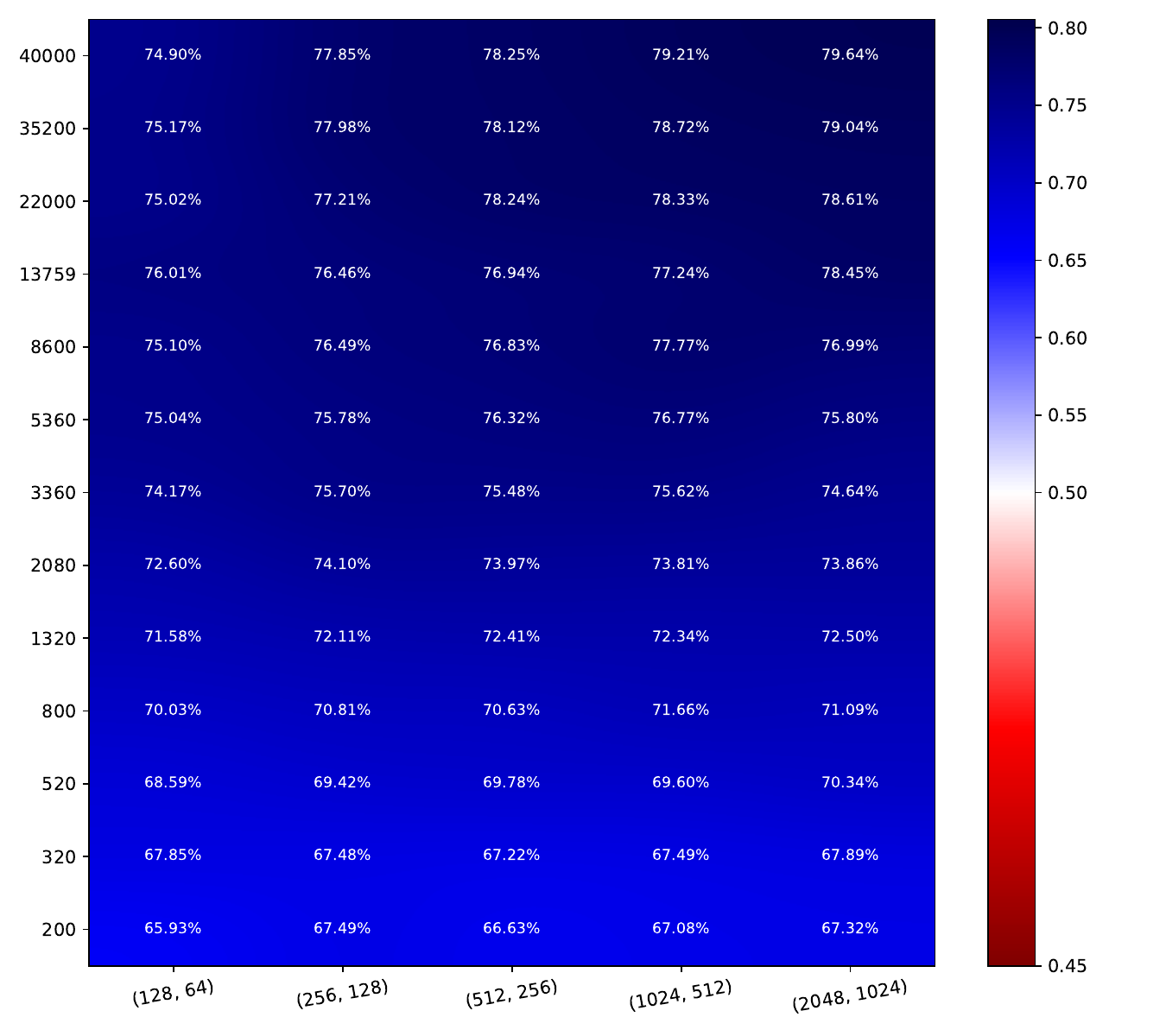}
    \caption*{\qquad MC-Dropout}
  \end{subfigure}\hfill
  \begin{subfigure}[t]{0.23\textwidth}
    \includegraphics[width=\textwidth]{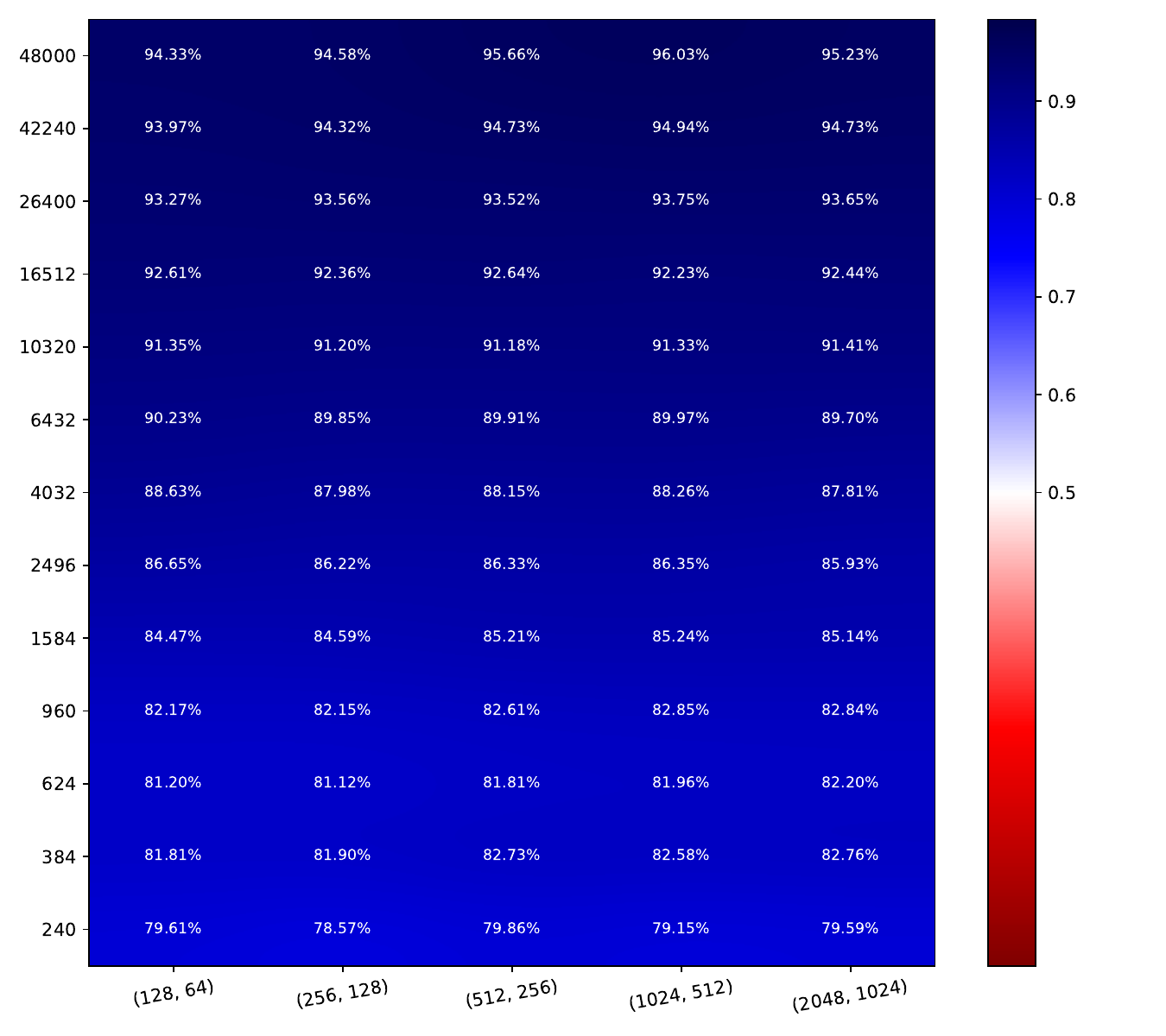}
    \includegraphics[width=\textwidth]{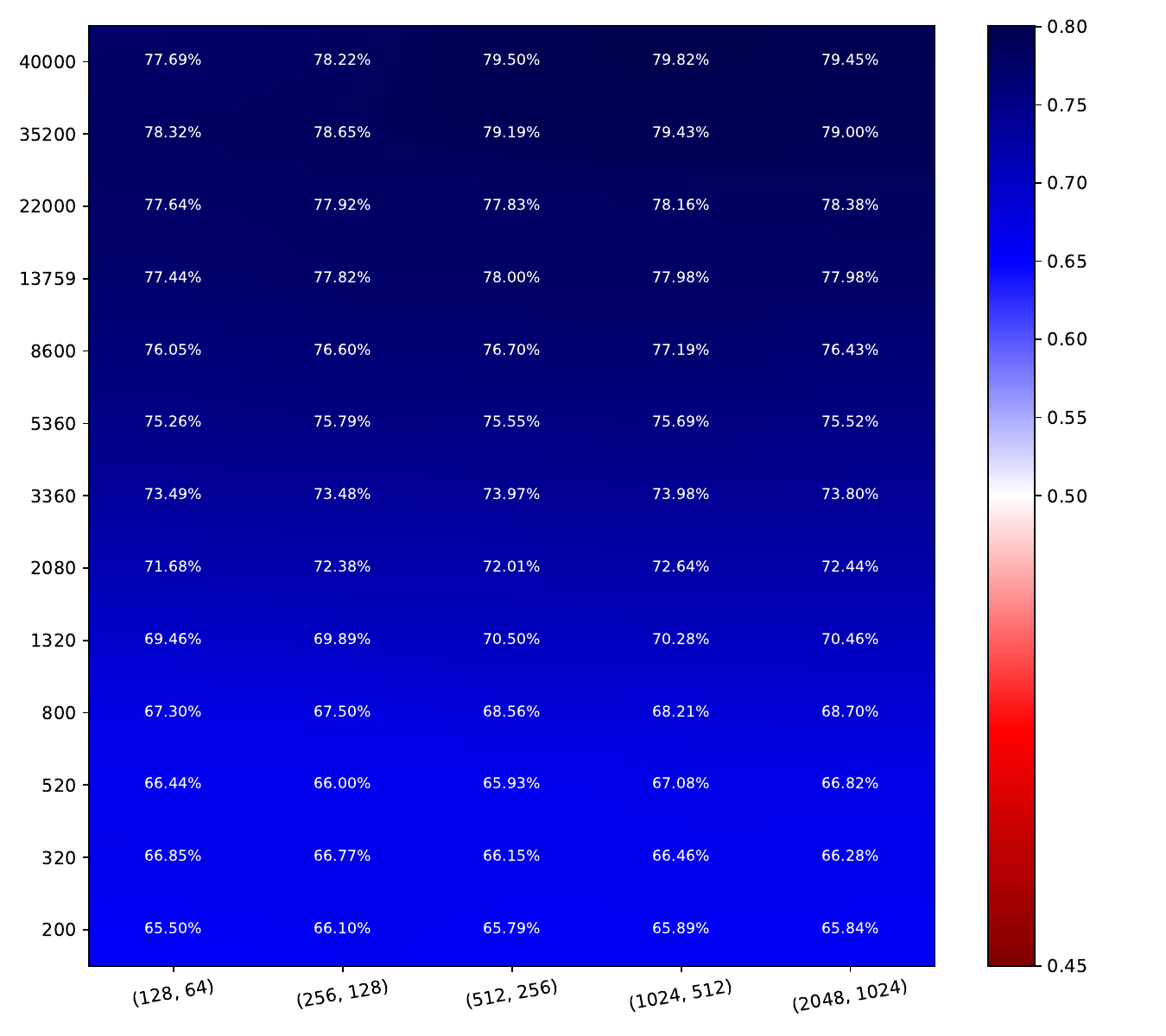}
    \caption*{EDL}
  \end{subfigure}\hfill
  \begin{subfigure}[t]{0.23\textwidth}
    \includegraphics[width=\textwidth]{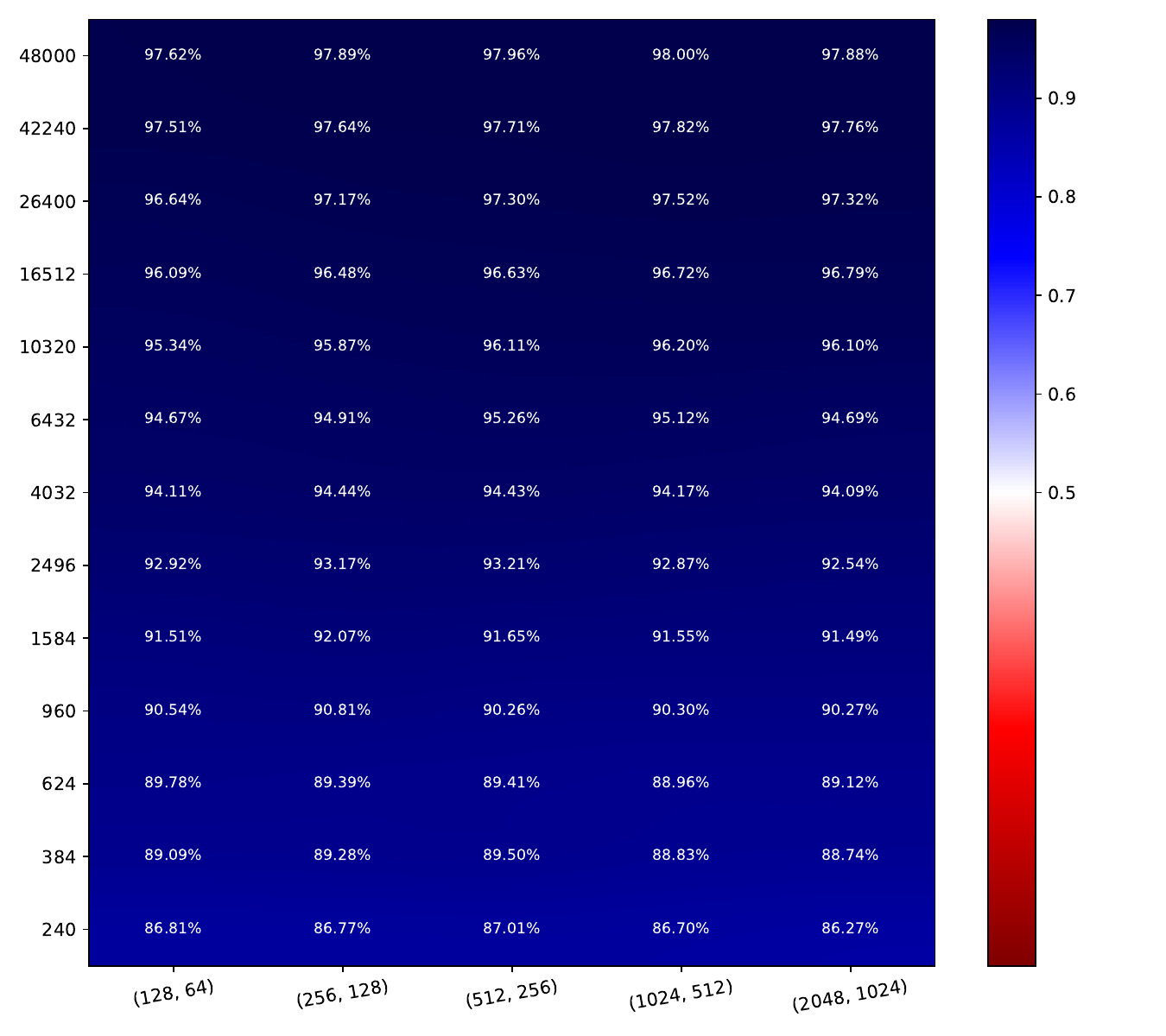}
    \includegraphics[width=\textwidth]{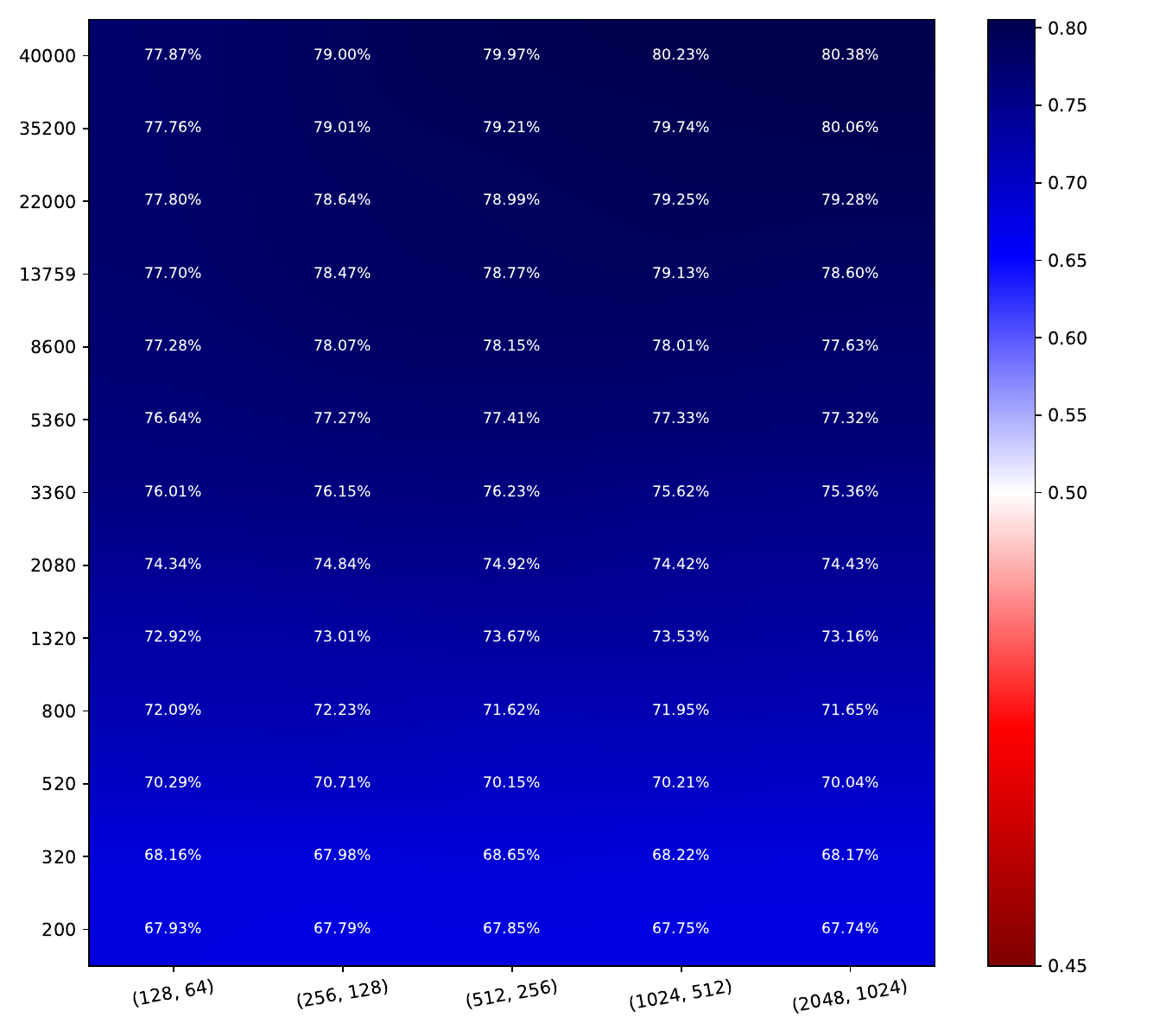}
    \caption*{DE}
  \end{subfigure}\hfill
  \begin{subfigure}[t]{0.23\textwidth}
    \includegraphics[width=\textwidth]{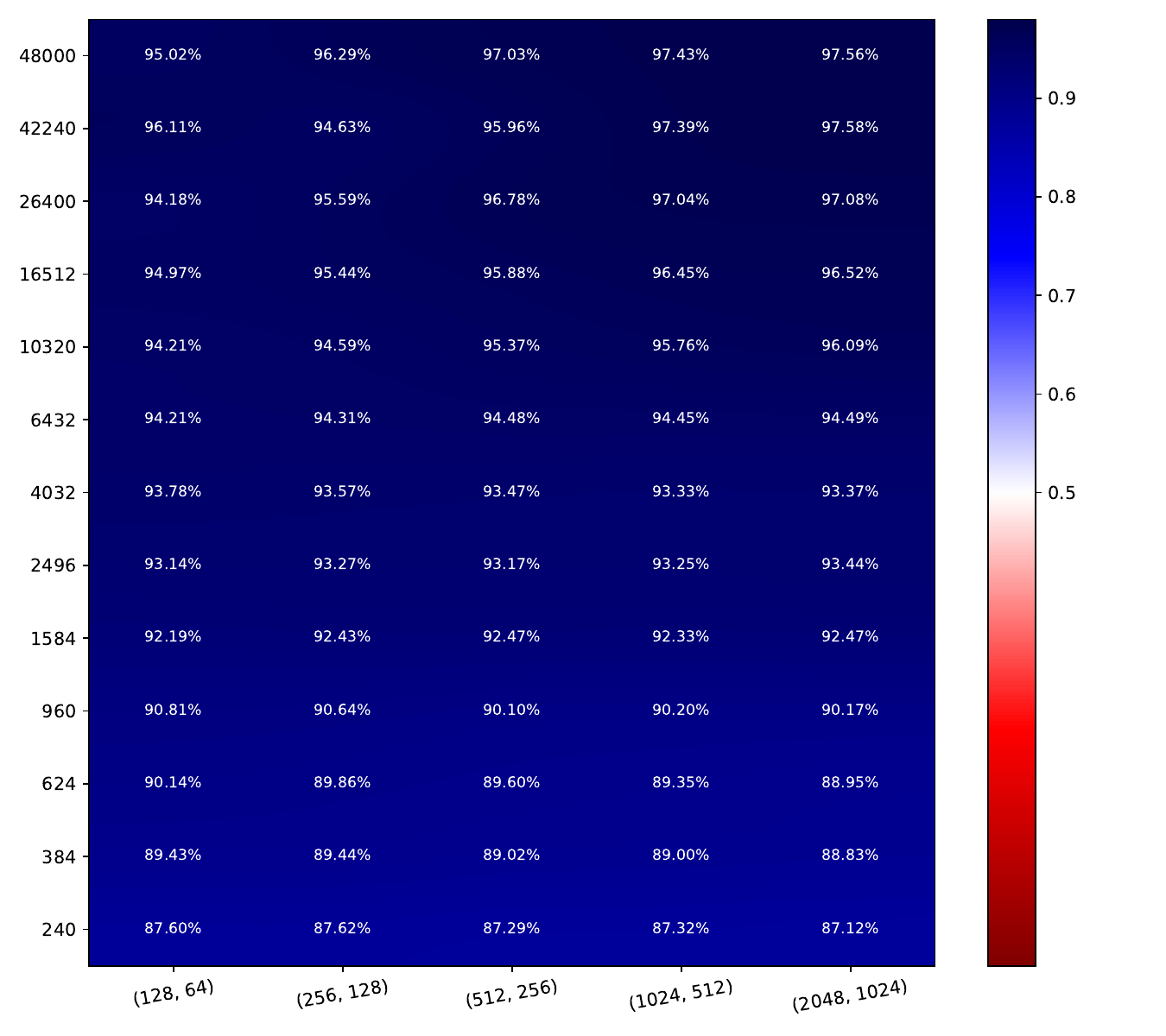}
    \includegraphics[width=\textwidth]{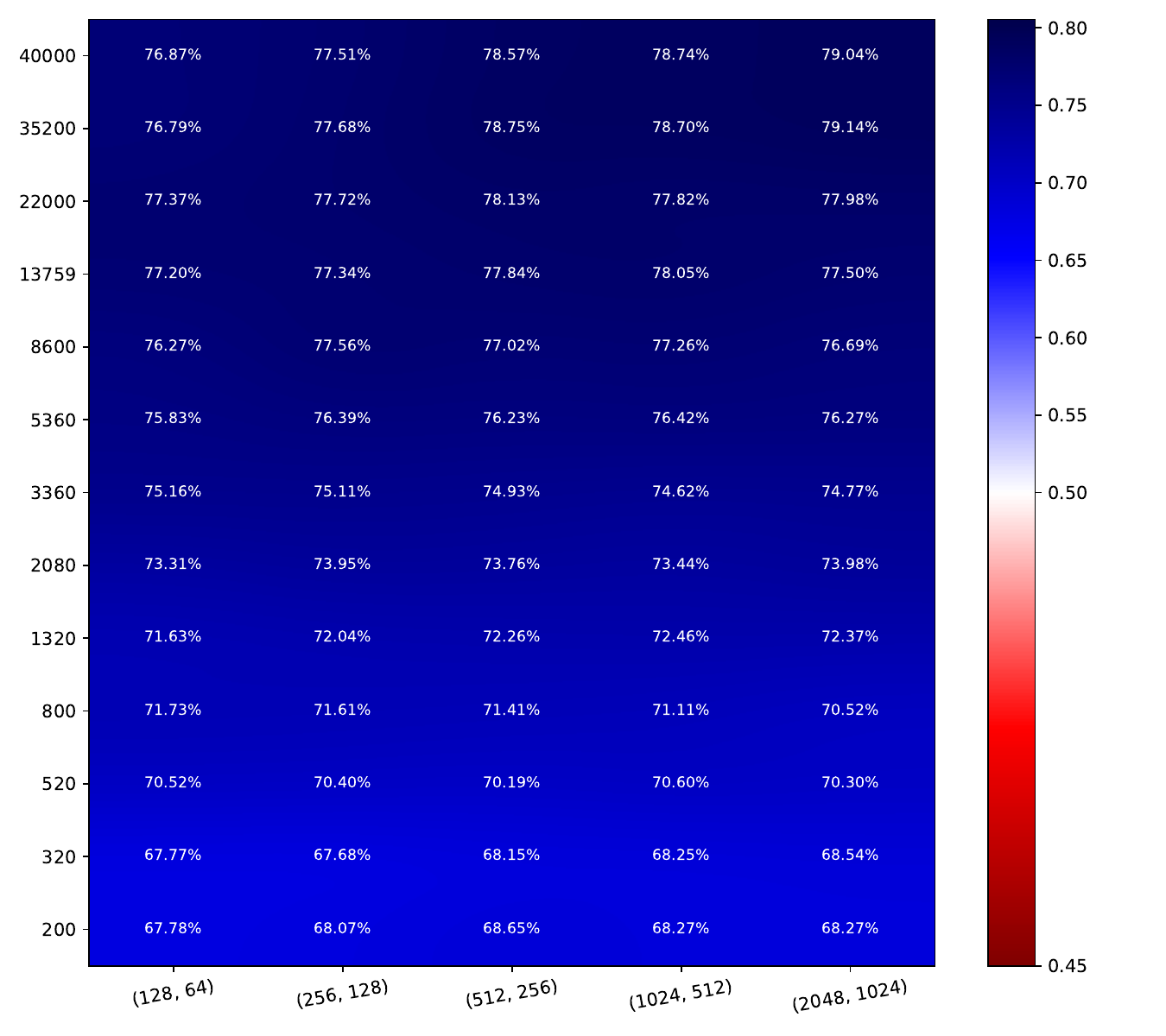}
    \caption*{Conflictual DE}
  \end{subfigure}\hfill
  \caption{Heatmaps of AUROC based on epistemic uncertainty for misclassification detection. Color scales are the same per dataset.}
\end{figure}

\end{subappendices}

\end{document}